\documentclass{article}

\usepackage{iclr2025_conference,times}

\usepackage{hyperref}
\usepackage{url}

\title{TVineSynth: A Truncated C-Vine Copula Generator of Synthetic Tabular Data to Balance Privacy and Utility}

\author{Elisabeth Griesbauer \\
Norwegian Computing Center \\
University of Oslo \\
Integreat - Norwegian Centre for Knowledge-driven Machine Learning \\
\texttt{elismg@uio.no} \\
\And
Claudia Czado \\
Technical University of Munich \\
Munich Data Science Institute
\AND
Arnoldo Frigessi \\
University of Oslo \\
Integreat - Norwegian Centre for Knowledge-driven Machine Learning \\
\And
Ingrid Hobæk Haff \\
University of Oslo \\
Integreat - Norwegian Centre for Knowledge-driven Machine Learning
}
\usepackage[T1]{fontenc}    %
\usepackage{hyperref}       %
\usepackage{url}            %
\usepackage{booktabs}       %
\usepackage{amsfonts}       %
\usepackage{nicefrac}       %
\usepackage{microtype}      %
\usepackage{xcolor}         %

\usepackage{tikz}
\usetikzlibrary{matrix}
\usetikzlibrary{shapes, arrows, calc, arrows.meta, fit, positioning}

\usepackage{amsfonts}
\usepackage{amssymb}
\usepackage{amsmath}
\usepackage{algorithm}
\usepackage{algorithmic}
\usepackage{graphicx}
\usepackage{placeins}
\usepackage{amsthm}
\usepackage[toc,page,header]{appendix}
\usepackage{minitoc}
\usepackage{graphicx}
\usepackage{subcaption}
\usepackage{bm}
\usepackage[inkscapelatex=false]{svg}

\defcitealias{Stadlercode}{github}
\defcitealias{sdv}{SDV library}
\defcitealias{ctgan}{CTGAN library}

\newtheorem{theorem}{Theorem}[section]
\newtheorem{proposition}[theorem]{Proposition}

\newtheorem{definition}[theorem]{Definition}

\newtheorem{remark}[theorem]{Remark}

\DeclareMathOperator*{\argmax}{arg\,max}
 
\newcommand\smallO{
  \mathchoice
    {{\scriptstyle\mathcal{O}}}%
    {{\scriptstyle\mathcal{O}}}%
    {{\scriptscriptstyle\mathcal{O}}}%
    {\scalebox{.7}{$\scriptscriptstyle\mathcal{O}$}}%
}

\iclrfinalcopy

\begin{document}

\maketitle

\doparttoc %
\faketableofcontents %

\begin{abstract}
    We propose TVineSynth, a vine copula based synthetic tabular data generator, which is designed to balance privacy and utility, using the vine tree structure and its truncation to do the trade-off. Contrary to synthetic data generators that achieve DP by globally adding noise, TVineSynth performs a controlled approximation of the estimated data generating distribution, so that it does not suffer from poor utility of the resulting synthetic data for downstream prediction tasks. TVineSynth introduces a targeted bias into the vine copula model that, combined with the specific tree structure of the vine, causes the model to zero out privacy-leaking dependencies while relying on those that are beneficial for utility. Privacy is here measured with membership (MIA) and attribute inference attacks (AIA). Further, we theoretically justify how the construction of TVineSynth ensures AIA privacy under a natural privacy measure for continuous sensitive attributes. When compared to competitor models, with and without DP, on simulated and on real-world data, TVineSynth achieves a superior privacy-utility balance.
\end{abstract}

\section{INTRODUCTION}\label{sec:introduction}

The availability of diverse, high-quality data has led to tremendous advances in science, technology and society at large, when analysed by means of statistical and machine learning (ML) methods. However, real-world data are in many cases %
limited, imbalanced or cannot be made public to the research community due to privacy restrictions, obstructing progress especially in bio-medical research. 
Synthetic data can augment the real data, and as long as they do not disclose private aspects, they can also substitute sensitive real data. Both substituting and augmenting real with synthetic data have proven to be successful in training downstream ML applications
\citep{gao2023synthetic, 
morales2023evaluation, %
shetty2023data, %
wang2023data, %
jain2023fair, %
pezoulasboosting, %
ye2023exploiting, %
goldschmidt2023improving, %
wang2023learning, %
saisho2023sandbox, %
schaufelberger2023impact}. %

We focus on scenarios where the main concern about the real tabular data is privacy and the downstream ML application is classification or regression. Our work is motivated by two objectives: (i) The synthetic data should retain joint dependence and marginal behavior of the real data, so that a regression method trained on the synthetic data performs comparably well on unseen data as it would have done if trained on the real data (\textbf{utility}); (ii) the generative model should not leak sensitive information on an instance of the real data into the synthetic data (\textbf{privacy}). 
While differential privacy (DP) \citep{dwork2014algorithmic} provides a sound approach for privacy preserving generative modelling, the resulting synthetic data have shown to score poorly in terms of utility \citep{jayaraman2019evaluating, %
bagdasaryan2019differential, %
cheng2021can}. %
On the other hand, popular generative models without privacy guarantees, such as generative adversarial networks (GANs) \citep{goodfellow2014generative} or variational autoencoders (VAEs) \citep{kingma2013auto}, tend to generate realistic, but privacy violating synthetic data \citep{chen2020gan, %
van2023membership, %
andrei2023overview}. %
Further, their training is data intensive, making them inappropriate as synthetic data generators for small and moderately sized real data sets.

\paragraph{Contributions} 
We propose TVineSynth, a vine copula based synthetic data generator, designed to balance privacy against utility. 
In contrast to globally adding noise, as is done to obtain DP guarantees, TVineSynth approximates the data generating distribution by (1) setting the focus of the model on dependencies that are relevant for the prediction task and (2) introducing a targeted bias into dependencies that would otherwise leak sensitive information into the synthetic data. 

This is achieved as follows: 
We propose an algorithm to re-order the features\footnote{The terms feature and covariate are used interchangeably in the following.} in the real data, to obtain a block structure dependence. Then, we set the vine tree structure such that we achieve (1) and truncate away tree by tree from the vine copula, cutting off privacy leaking dependencies to achieve (2). 
Re-ordering the real data to obtain a block structured dependence is central in TVineSynth as it amplifies the effect of truncation on privacy, thus making it easier to find a suitable vine tree structure that also keeps high utility.

We conduct an in depth analysis of the privacy of the TVineSynth generated synthetic data under membership and attribute inference attacks \citep{shokri2017membership, yeom2018privacy} and of their utility w.r.t. prediction performance. 
We asses AIA privacy with the mean absolute $\beta$-coefficient (MAB), a measure for AIA privacy, that naturally builds on the implementation of AIA attacks by \citet{stadler2022synthetic} and addresses the weaknesses of previously used measures.
We theoretically justify the construction of TVineSynth by showing how the truncation of the vine copula and the order of the covariates in the vine tree structure ensure AIA privacy under the MAB.
TVineSynth's privacy and utility are compared with those of other generative models with and without privacy guarantees.
We show that if privacy \textit{and} utility matter, TVineSynth is preferable over private and non-private competitors. 

\paragraph{Why do we not use DP?}
TVineSynth does not use the concept of DP in its model design. We argue that it is common that the data holder wants to protect specific sensitive features, while regarding the protection of the remaining features as less important. Contrary to how we design TVineSynth, this knowledge about the real data is not exploited in favor of either utility or of privacy when noise is added uniformly on (statistics of) all features to obtain DP guarantees. 
Real-world medical data is highly complex and inherently noisy and preserving its joint distribution is critical for decision making, which makes the application of DP less suitable in the medical domain. There the most relevant risk to analyse is the risk of identifying a finite set of real patients from a finite synthetic data set generated, which DP does neither address nor provide theoretical bounds for. Finally, DP 'fails to address ethical concerns pertaining to the risk benefit ratio, where minimal risk may be deemed allowable if the societal benefits' from more rapid development of medical treatments are high, \citep{yoon2020anonymization}. 
DP offers theoretical bounds on the effect of substituting a single training data point on the probability of observing an outcome of an algorithm. However, these bounds become weak to meaningless when a privacy budget is chosen, that is non-prohibitive to utility \citep{stock2022defending}. While \citet{ziller2024reconciling} claim that for image data a meaninglessly high $\epsilon$ provides sufficient protection against relaxed but realistic privacy attacks, we argue that their results cannot directly be transferred to tabular data and a worst-case privacy assessment through MIAs is indispensable in the medical domain.
On top of that, the bounds provided by DP are hard to interpret for real-world applications and risks. Data protection laws, such as the GDPR \citep{gdpr2016general}, do not build on DP, which further highlights the problem of translating DP into practice, \citep{yoon2020anonymization}.
While DP translates into a theoretical lower bound on MIA privacy gain (PG) \citep{yeom2018privacy}, empirically TVineSynth achieves a PG comparable to the DP competitors due to the MLE's robustness in TVineSynth. No theoretical bounds on AIA success have been developed w.r.t. DP yet, and we show that through its model design, TVineSynth can handle this attribute specific risk on par with its DP competitors.

\paragraph{Related Work}
\citet{xu2019modeling} extend GANs \citep{goodfellow2014generative} and VAEs \citep{kingma2013auto} with a conditional generator and specific preprocessing to obtain their counterparts for tabular data, namely CTGAN and TVAE.  \citet{kotelnikov2023tabddpm} adapt denoising diffusion probabilistic models to model tabular data. 
These approaches model the real data closely, but do not exploit model structure to achieve privacy like TVineSynth. 
Taking privacy into account, \citet{jordon2018pate} (PATE-GAN), \citet{xie2018differentially} (DP-GAN) and \citet{zhang2017privbayes} (PrivBayes) modify non-private GANs and Bayesian networks to fulfill DP, while \citet{donhauser2024privacy} utilize a particle based approach on privatized marginals (PrivPGD). 
These models add noise in a global fashion in order to guarantee DP, contrasting the precise model approximation through truncation of a vine copula in TVineSynth.
Another line of work generates synthetic data with copulas, such as Gaussian copulas \citep{patki2016sdv, kumi2023sleepsynth}, Student's $t$-copulas \citep{benali2021mtcopula}, an empirical beta copula as latent space distribution in a pre-trained autoencoder \citep{coblenz2023learning}, a DP Gaussian copula by applying DP marginal histograms and correlation matrix \citep{li2014differentially} or a copula estimated with normalizing flows \citep{kamthe2021copula}. 
These copulas lack the flexibility of the vine copula and are not tailored towards the privacy needs of the data holder. 
Generative modeling with vine copulas naturally builds on this line of work. While \citet{chu2022vine} limit themselves to C- and D-vines, \citet{meyer2021copula} utilize an R-vine copula for data generation. Moreover, \citet{sun2019learning} re-formulate the structure selection in an R-vine as a reinforcement learning problem to increase modelling flexibility, \citet{tagasovska2019copulas} model high dimensional data using a vine copula as latent space distribution in an autoencoder and \citet{gambs2021growing} obtain a DP vine copula by applying DP marginal histograms. 
While these generative models benefit from the flexibility of vine copulas, they do not balance utility with privacy, as we do by exploiting the vine structure and truncation.
 \citet{patki2016sdv} and \citet{qian2023synthcity} offer implementations of several SOTA generative models and evaluation metrics, while \citet{meyer2021synthia} focus specifically on vine copulas, with no attention to privacy.

\section{METHODS}\label{sec:methods}

\subsection{Vine Copula Based Synthetic Data Generation}\label{sec:vinecopula_based_synthdata_generation}
A vine copula\footnote{For an extended introduction to vine copulas please consult Appendix \ref{app:intro_to_vines}.} is a probabilistic model that builds on copulas: A $d$-dimensional copula $C:[0,1]^d \rightarrow [0,1]$ is a $d$-dimensional distribution on the unit cube with uniform marginals and corresponding copula density $c$. Sklar's theorem \citep{sklar1959fonctions} states that any multivariate distribution $F$ can be expressed in terms of a copula $C$; and if all densities exist, a multivariate density $f$ can be expressed as a product of the corresponding copula density $c$ and marginal densities.
Vine copulas \citep{joe1997multivariate, bedford2001probability, bedford2002vines, aas2009pair, joe2014dependence, czado2019analyzing} are hierarchical probabilistic graphical models constructed from univariate distributions and bivariate (conditional) copulas. The vine tree structure $\mathcal{V} = (T_1, \dots, T_{d-1})$, which is a nested sequence of $d-1$ trees $T_k = (V_k, E_k), \; k \in [d-1]$, serves as a construction plan of the vine copula.  
Here an edge in $T_1$ represents a bivariate copula $c_{a_e, b_e}$ of the unconditional pair of random variables $(X_{a_e}, X_{b_e}), \; a_e, b_e \in [d]$, and an edge $e$ in $T_k, \; k \in \{2, \dots, d-1\}$ represents a bivariate copula $c_{a_e, b_e ; D_e}$ of a pair $(X_{a_e}, X_{b_e})$, conditioned on $k-1$ random variables $X_j, \; j \in D_e \subset [d]$. 
Taking $\mathcal{V}$ and the pair copulas together, the $d$-dimensional copula density $c$ can be expressed as a product of (conditional) pair copulas over the edges of the trees in $\mathcal{V}$ giving the \textit{vine copula} $c = \prod_{k \in [d-1]} \prod_{e \in E_k} c_{a_e, b_e ; D_e}$.

With Sklar's theorem \citep{sklar1959fonctions} the full joint density is then $f = c \cdot f_1 \cdots f_d$.
The structure of a vine copula and the corresponding conditioning sets are by construction such that computing the vine copula is iterative along the trees and thus efficient.
The univariate margins and the pair copulas can be chosen freely. 
This makes vine copulas a highly flexible, yet tractable model class, that allows to capture complex dependence structures.
A way to reduce a vine copula's capacity to approximate a multivariate distribution, is to truncate the vine copula at a specific tree level $t \in [d-1] := \{1, 2, ..., d-2, d-1\}$. This is equivalent to setting all pair copulas of trees $T_{t+1}, \dots, T_{d-1}$ to independence. 

\begin{definition}[Truncation of the Vine Copula at Level $t$]
    Let $c$ be a vine copula as given above. We define the vine copula truncated at truncation level $t \in [d-1]$ as: $\prod_{k \in [t]} \prod_{e \in E_k} c_{a_e, b_e ; D_e}$.
\end{definition}

Thus, in the resulting vine copula, only trees $T_1, ... T_t$ are left in the model.  For $t=d-1$, we obtain the un-truncated vine copula, while for $t=1$, only the first tree is retained.
Special shapes of trees in the vine tree structure lead to certain sub-classes of vines. In particular, in a C-vine, each tree is star-shaped, i.e. contains a fully connected node called root node.

We use vine copulas to generate synthetic data to substitute the private real data in a general regression setting with response variable $Y$. The synthetic data generated for this case should not leak sensitive information about any real observation (privacy) and at the same time allow training a regression method equally well as would happen on the real data (utility). Our idea is to strike a balance between privacy and utility of the vine copula generated synthetic data, by exploiting weak stochastic dependencies of sensitive covariates with the covariates that are important for the prediction task. 
Thus, it might not be necessary to protect all covariates equally well by adding noise in a global fashion (which decreases utility), or to capture all dependencies present in the real data (which might impair privacy). 
Based on these considerations we propose TVineSynth, a framework building on a star-shaped C-vine copula with $Y$ as root node of $T_1$, to focus early specifically on those dependencies that matter for the prediction task.
Further, TVineSynth finds an order $\cal{O}^*$ of the $d$ covariates in which they are arranged in the remaining trees of the C-vine, that yields a block structured dependence. Truncation of the resulting C-vine then cuts off privacy leaking dependencies, while maintaining high utility.
By truncating the vine copula at a moderate tree level $t$, we cut away dependencies that might not add to utility, but challenge privacy. 
Combining the C-vine structure with the appropriate order of the covariates and truncation of the vine copula model in TVineSynth, we obtain a generative model that can be tailored towards the desired privacy and utility requirements for the real data.

\subsection{TVineSynth Construction}\label{sec:TVineSynth_construction}

The construction of TVineSynth consists of three steps:
\begin{itemize}
    \item[(1)] Execute Algorithm \ref{alg:find_order}\footnote{The implementation of Algorithm \ref{alg:find_order} as well as experiments can be found at: \url{https://github.com/ElisabethGriesbauer/T-Vine-Synth}.} to determine the order $\cal{O}^*$ in which the covariates of the real data enter the C-vine copula.
    \item[(2)] For the specific order $\cal{O}^*$ of the covariates in step (1), generate synthetic data from the C-vine at all candidate truncation levels, and for each truncation level assess their privacy and utility.
    \item[(3)] Find the truncation that offers optimal privacy-utility balance to the user by consulting the privacy-utility plot.
\end{itemize}

Steps (1) to (3) are executed by the data holder and are not made public; only the resulting synthetic data fulfilling the data holder's privacy and utility demands is published. 
See Figure \ref{fig:workflow_tvinesynth} in Appendix \ref{sec:TVineSynth_workflow} for a graphical illustration of TVineSynth.

In Algorithm \ref{alg:find_order} in step (1), user knowledge about sensitive covariates is considered together with the empirical dependence properties of the real data in order to find an order $\cal{O}^*$ of the $d$ covariates in which they will enter the C-vine copula model, such that truncation of the C-vine cuts off privacy leaking dependencies, while maintaining high utility. A theoretical justification of Algorithm \ref{alg:find_order} is given in Section \ref{sec:theory_TVineSynth} and further details on the algorithm are provided in Appendix \ref{sec:orderRVs}.
The order $\mathcal{O}^*$, together with the vine tree structure $\mathcal{V}$ of the star-shaped C-vine, then determines the cascade of pair copulas of conditional distributions across the hierarchy of vine trees, see Proposition \ref{prop:order_V}. In each order $\mathcal{O}^*$ we require the response $Y$ to be the center of the first tree in the C-vine. More specifically, let $X_{(1)},\ldots,X_{(d)}$ be the covariates in the chosen order. Then $T_{1}$ of the C-vine models the pairwise dependence between $Y$ and each of $X_{(1)},\ldots,X_{(d)}$, $T_{2}$ the pairwise dependence between $X_{(d)}$ and each of $X_{(1)},\ldots,X_{(d-1)}$, conditioning on $Y$, $T_{3}$ the pairwise dependence between $X_{(d-1)}$ and each of $X_{(1)},\ldots,X_{(d-2)}$, conditioning on $Y$ and $X_{(d)}$, and so on until the last tree $T_{d}$, which captures the pairwise dependence between $X_{(2)}$ and $X_{(1)}$, conditioning on $Y,X_{(d)},\ldots,X_{(3)}$. In Appendix \ref{sec:orderRVs}, we explore how orders other than $\mathcal{O}^*$ affect privacy.

For a selected ordering $\mathcal{O}^*$ of the covariates, the C-vine is estimated at user-defined maximal truncation level $t_{max} \leq d$.
We advise to set $t_{max} := d + 1 - j$ or lower,  where $j$ is the position of the first sensitive feature to appear in the center node of a tree of the C-vine 
according to $\mathcal{O}^*$, as tree levels that model pairwise (conditional) dependencies with a sensitive feature and all other features should not be considered, see Section \ref{sec:theory_TVineSynth} for further theoretical grounding, and for large $d$ (e.g. $d = 500$) we recommend to set $t_{max} << d$ because of uncertainty in the parameter estimation.
Then for user-defined candidate truncation levels $t \in T \subset [t_{max}]$\footnote{This can for example be every 5th truncation level, i.e. $T := \{1, 5, 10, 15, 20, 26\}$ for $d = 26$.} 
the C-vine truncated at level $t \in T \setminus \{t_{max} \}$ is obtained by setting pair copulas of tree levels $t+1$ and above to independence, i.e. removing tree after tree from the model. This means that for obtaining the C-vines of all candidate truncation levels $t \in T$ the sensitive real data only needs to be accessed once, namely for estimating the un-truncated C-vine.
More precisely, let $(X, \bm{y}) \in \mathbb{R}^{n \times (d+1)}$ denote the real data where $X := (x_{i j}) \in \mathbb{R}^{n \times d}$, $i \in [n], \; j \in [d]$, is the matrix of $n$ realizations of the random vector $(X_1, \dots, X_d)$ and $\bm{y} := (y_1, \dots, y_n)^T$ is the vector of $n$ realizations of the random variable $Y$.
Then the vine copula model $g$ with truncation level $t$ is fit to the real data resulting in $\hat{g} := g\big((X, \bm{y}); \mathcal{V}, t \big)$ and the synthetic data $(Z, \bm{w}) \in \mathbb{R}^{n \times (d+1)}$ are sampled from $\hat{g}$.
We use the estimation and sampling algorithms introduced in \citet{dissmann2013selecting} and implemented by \citet{rvinecopulib}. The vine copula model is estimated in an iterative, hierarchical fashion: Proceeding tree by tree, a greedy maximum spanning tree search with pairwise association measure as edge weights is conducted and parametric pair copulas corresponding to the edges are estimated with MLE and selected with AIC \citep{akaike1998information}.

After synthetic data have been generated from the C-vine truncated at each $t \in T$, their privacy $P_t$ and utility $U_t$ are assessed in step (2) using the methods explained hereafter. This results in points $(U_t,P_t)$ for truncation levels $t \in T$ in the privacy-utility plot of step (3). Here, $U_t$ is a measure of prediction performance over several synthetic data sets generated from the C-vine with specific ordering and truncation level $t$ (see below for details), whereas $P_t$ is either the median MAB of an AIA or the median PG of a MIA over several synthetic data sets generated from the model and several runs of the privacy attack, see definitions below. 
Due to Theorem \ref{thm:utility}, it is necessary to evaluate all truncation levels $t$ in the candidate set $T$.
Figure \ref{fig:2dPrivUt_simreal_X6AIA_MAB_Cvine} shows a plot of $(U_t,P_t)$, where higher values along each axis indicate a better privacy or utility.\footnote{Boxes and whiskers are not displayed in the privacy-utility plots, as they can already be found in Figures \ref{fig:AIA_MAB_Cvine_competitors_simreal_X1_X6_X11} (AIA privacy) and \ref{fig:utility_Cvine_competitors_simreal} (utility) in Appendix \ref{app:sim_real_results_appendix}, and to simplify visual inspection of the figure.} The privacy-utility plot allows to observe a trajectory of how the privacy-utility trade-off of a C-vine with a specific ordering develops with its truncation level. Adding the results of competitor models, the privacy-utility plot allows to take a well-informed decision on the TVineSynth model, offering the desired privacy-utility balance. In Appendix \ref{sec:TVineSynth_trunc_opt} we further elaborate on how finding the best truncation level according to user demands can be formalized as an optimization problem. Considerations on the computational complexity of TVineSynth can be found in Appendix \ref{sec:scaling_vines}.

\begin{algorithm}[tb]
   \caption{Finding Order $\mathcal{O}^*$}
   \label{alg:find_order}
        \begin{algorithmic}
           \STATE {\bfseries Input:} $(X, \bm{y})$, initial order $\mathcal{O}^0 = (X_1, ..., X_d, Y)$ with $X_{d+1} := Y$, pairwise association measure $\rho: \mathbb{R}^{n \times 2} \rightarrow \mathbb{R}$,\footnotemark pairwise association threshold $\rho^* > 0$, sensitive covariates $X_{j^*}$ with $j^* \in S \subset [d]$
           \STATE {\bfseries Output:} order $\mathcal{O}^*$
           \STATE set $\mathcal{O}^*_{d+1} := Y$
           \STATE compute $\rho_{j,k} := \rho(\bm{x}_j, \bm{x}_k)$ for $j \in [d] \; , k > j$
           \STATE set $K := \{k \in [d]: \; |\rho_{j^*,k}| > \rho^* \; \text{for} \; j^* \in S \}$, the set of variables highly associated with sensitive features 
           \FOR{$j \in \{1, ...,|S| \}$}
           \STATE set $\mathcal{O}^*_{j} := X_{j^*}$ with $j^* \in S$
           \ENDFOR
           \STATE order $|\rho_{k,j^*}|$ for $k \in K$ and $j^* \in S$ in descending order $|\rho_{(1)}|, ..., |\rho_{(|K|)}|$
           \FOR{$j \in \{1, ..., |K| \}$}
           \STATE set $\mathcal{O}^*_{|S| + j} := X_k \; $ if $\; |\rho_{(j)}| = |\rho_{k,j^*}|$ with $k \in K$
           \ENDFOR
            \STATE $r := 1$
            \FOR{$j \in [d]\setminus (S\cup K)$}
            \STATE $\mathcal{O}^*_{|S|+|K|+r} := X_{j}$
            \STATE $r := r+1$
            \ENDFOR
        \end{algorithmic}
\end{algorithm}

\footnotetext{We recommend choosing a pairwise association measure that is scale invariant, such as Kendall's $\tau$.}

\begin{figure}[t!]
    \centering
   \begin{subfigure}[t]{0.33\columnwidth}
        \centering
        \includegraphics[width=\columnwidth]{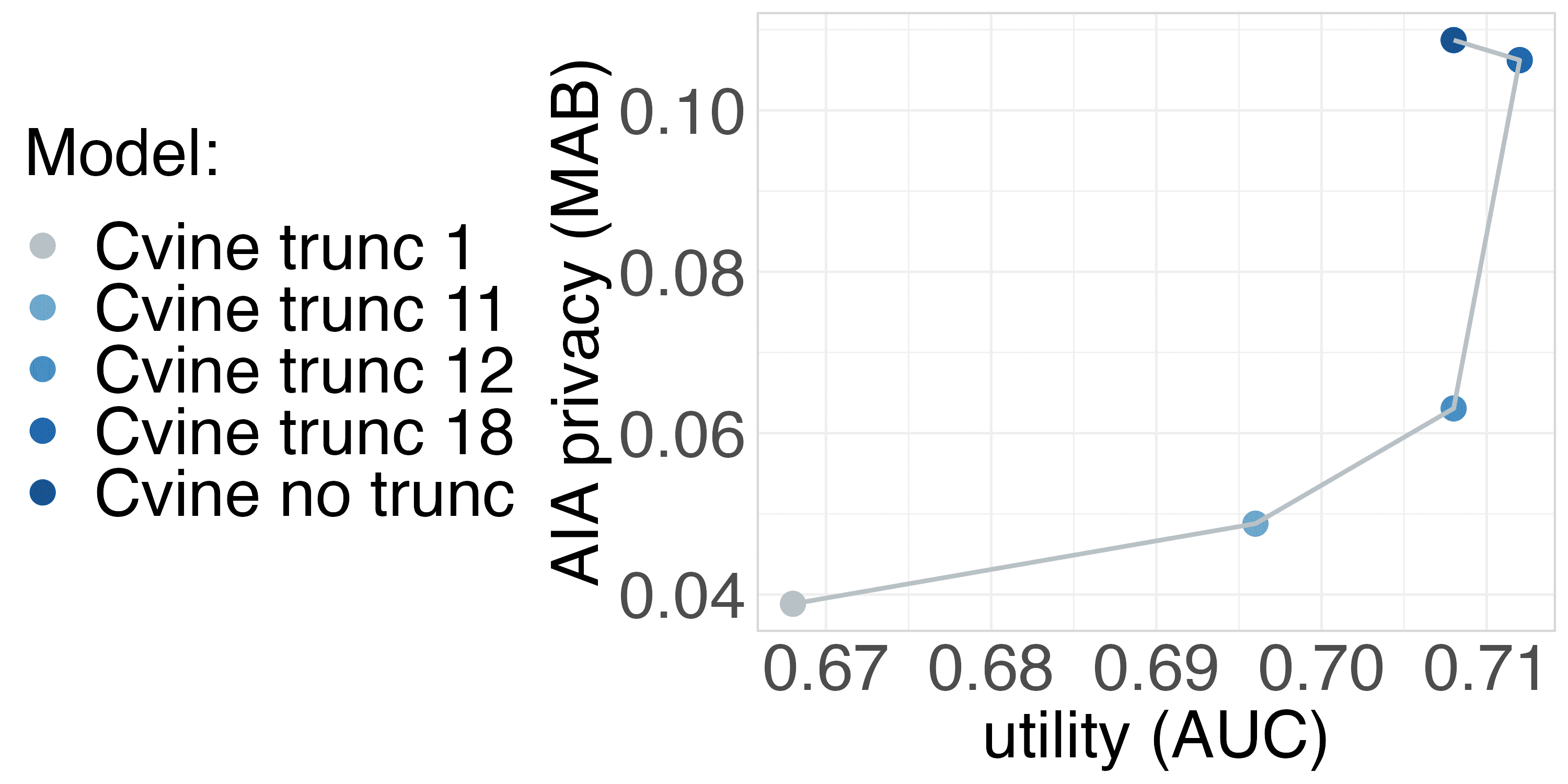}
        \caption{Sensitive feature $X_6$.}\label{fig:2dPrivUt_simreal_X6AIA_MAB_Cvine_1}
    \end{subfigure}
    ~
    \begin{subfigure}[t]{0.37\columnwidth}
        \centering
        \includegraphics[width=\columnwidth]{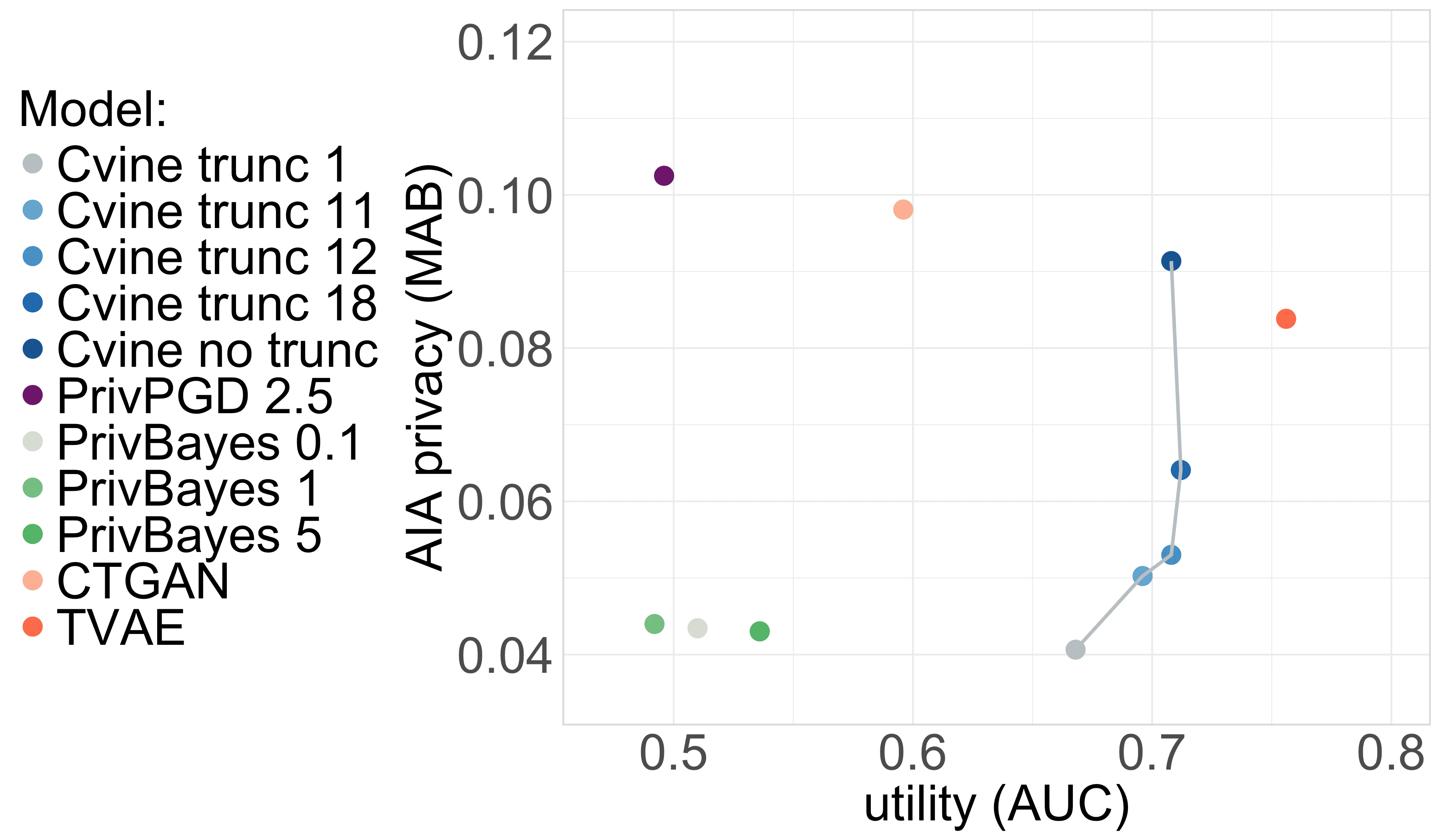}
        \caption{Sensitive feature $X_1$.}\label{fig:2dPrivUt_simreal_X6AIA_MAB_Cvine_2}
    \end{subfigure}
    \caption{Privacy-utility plot of synthetic data generated with a C-vine truncated at $t \in \{1,11,12,18\}$ and no truncation (a) and competitors (b) from simulated real data. For AIA privacy, the MAB and for utility the median over 50 synthetic data sets are reported. Parameters of the generative models and privacy attacks can be found in Appendix \ref{sec:model_and_attack_parameters}.}
    \label{fig:2dPrivUt_simreal_X6AIA_MAB_Cvine}
\end{figure}

\subsection{Competitor Models}
TVineSynth is benchmarked against PrivBayes \citep{zhang2017privbayes} and PrivPGD \citep{donhauser2024privacy}, which offer DP guarantees, and CTGAN and TVAE \citep{xu2019modeling}, which do not provide any DP guarantees, but are designed to resemble the real data as closely as possible. For details on the competitor models and their choice see Appendix \ref{sec:competitors_appendix}.

\subsection{Utility}
We generate synthetic data to substitute private real data in a general regression task with response variable $Y$. This includes classification when $Y$ is binary. Fitting a vine copula is more challenging on discrete than on continuous data\footnote{The copula is uniquely defined only on the Cartesian product of the ranges of the marginal distributions for a discrete $Y$, \citep{panagiotelis2012pair}.}. Therefore we focus on a binary classification task. 
For assessing the utility of the synthetic data we compare Train on Synthetic - Test on Real (TSTR) to Train on Real - Test on Real (TRTR).
Let $(X^*, \bm{y}^*)$ be a hold-out, real test data set of size $n_{test}$ that was not used to learn the generative models.
Let $f: \mathbb{R}^d \rightarrow \{0,1\}$ be a classifier and $\hat{f}$ be its estimate from the real data $(X, \bm{y})$. Let $\hat{\bm{y}}^*$ be the prediction of the classifier $f$ estimated from $(X, \bm{y})$ applied to the test data $(X^*, \bm{y}^*)$ and let $\hat{\bm{w}}^*$ be the prediction of the classifier $f$ estimated from the synthetic data $(Z, \bm{w})$ applied to $(X^*, \bm{y}^*)$.
The utility of the synthetic data is assessed by comparing $\hat{\bm{w}}^*$ to $\hat{\bm{y}}^*$ through comparing $AUC(\bm{y}^*, \hat{\bm{w}}^*)$ and $AUC(\bm{y}^*, \hat{\bm{y}}^*)$, the area under the receiver operating characteristic curve (AUC). This allows us to analyse how the performance of the classifier on real test data changes when it is trained on synthetic instead of real data.

\subsection{Privacy}
The privacy of the synthetic data is assessed through a membership and an attribute inference attack (MIA and AIA) \citep{shokri2017membership, yeom2018privacy}. We follow the framework of \citet{stadler2022synthetic}, who model these attacks as privacy games between an attacker and a challenger, the data holder.

\paragraph{Membership Inference Attack (MIA)}
In a MIA, the attacker aims to infer from $(Z, \bm{w})$ whether a target observation  $(\bm{x}_t^T, y_t)$ is part $(X, \bm{y})$. The attacker has access to a reference data set $(X, \bm{y})_{ref}$ coming from the same distribution as the real data, and knows the size of the real and synthetic data (both $n$) and which generative model class is used. Then the attacker repeatedly samples data sets of fixed size from the reference data, adds the target observation half of the time and trains the generative model on them. After that, the attacker samples several synthetic data sets from each trained model and labels them according to whether the target observation has been added to the training data or not. A classifier is trained on the labeled synthetic data sets to estimate whether the target observation was part of the real data. The MIA game is repeated $N \in \mathbb{N}$ times.

\paragraph{Attribute Inference Attack (AIA)}
In an AIA, the attacker aims to infer the sensitive feature value $x_{t,j^*}$ of a target observation $(\bm{x}_t^T, y_t)$ from $(Z, \bm{w})$ for some sensitive feature $X_{j^*}$, $j^* \in S \subset [d]$.
The attacker has access to $(X, \bm{y})_{ref}$, a reference data set of fixed size coming from the same distribution as the real data, and knows which generative model class is used. Then the attacker trains the generative model on the reference data and samples $n_{synth} \in \mathbb{N}$ synthetic data sets from the estimated model. Subsequently, the attacker standardizes\footnote{By standardizing $\bm{x} \in \mathbb{R}^n$ to obtain $\tilde{\bm{x}} = (\tilde{x}_1, \dots, \tilde{x}_n)^T$ we refer to $\tilde{x}_i := \frac{x_i - \bar{x} }{\big(\frac{1}{n-1} \sum_{i = 1}^n (x_i - \bar{x})^2 \big)^{\frac{1}{2}}}$ with $\bar{x} := \frac{1}{n} \sum_{i=1}^n x_i$.} the synthetic data to obtain $(\tilde{Z}, \tilde{\bm{w}})$, fits a linear regression model on the non-sensitive features of $(\tilde{Z}, \tilde{\bm{w}})$ with $\tilde{\bm{z}}_{j^*}$ as response and issues a guess $\hat{\tilde{x}}_{t,j^*}$ based on real $(\tilde{\bm{x}}_{t, - j^*}^T, y_t)$ that was standardized by the data holder. The AIA game is repeated $N \in \mathbb{N}$ times.

\paragraph{Choice of Sensitive Features}
The definition of sensitive features is based on domain knowledge and legal considerations such as GDPR \citep{gdpr2016general}. Sensitive features involve personal information about health, demography, financial situation, behaviors, etc. that, if available to adversaries can be used to cause harm to data subjects or related people \citep{ohm2014sensitive}. For the case that domain knowledge is lacking, \citet{yoon2020anonymization} propose a definition of sensitive features: They consider features as sensitive, if they allow identification of an individual with high probability, for example because the feature values are extreme or rare.

\paragraph{Measures of Privacy}
As a measure of privacy protection against MIAs we use the \textit{privacy gain (PG)} w.r.t. a given target observation, as proposed in \citet{stadler2022synthetic}. The PG is defined as the 'reduction in the attacker’s advantage when given access to the synthetic data instead of the real data', where $PG \in [0,2]$ and $PG = 1$ indicates best possible privacy. 

The definition of the PG for AIAs provided by \citet{stadler2022synthetic} does not make sense for continuous sensitive features, see Appendix \ref{sec:privacy_measures}. \citet{olatunji2023does} propose to use the MSE in this case, which measures distance between the attacker's guess and the actual sensitive feature value. However, the MSE may be low just because the actual sensitive feature value is close to the sensitive feature's mean and not because the non-sensitive features inform the sensitive feature in the synthetic data, see Appendix \ref{sec:privacy_measures} for an example.
The influence of a covariate in a regression model (non-sensitive feature) on the dependent variable (sensitive feature) can be assessed by the magnitude of its regression coefficient. The \textit{mean absolute $\beta$-coefficient (MAB)} summarizes how much the non-sensitive features inform the sensitive feature when the target observation was part of the generative model training in one number and naturally builds on how AIAs are commonly implemented, such as by \citet{stadler2022synthetic}.

\begin{definition}[Mean Absolute $\beta$-Coefficient, MAB]\label{def:MAB}    
    Let an attacker perform an AIA according to \citet{stadler2022synthetic}. Then in a given run $m$ of the game, with $X_{j^*}$ as the sensitive feature, a linear Gaussian regression is fitted by ordinary least squares to each of the $l$ standardized synthetic data sets $\bm{V}_{l}=(\tilde{Z}_{l},\tilde{\bm{w}}_{l})$. This results in the coefficients $\hat{\bm{\beta}}^{(j^*)}_{m, l}=(\hat{\beta}^{(j^*)}_{1,m,l},\ldots,\hat{\beta}^{(j^*)}_{d,m,l})^{T}$, with $m \in [N]$ and $l \in [n_{synth}]$. The $MAB_{j^*}$ for sensitive covariate $X_{j^*}, \; j^* \in S \subset [d]$ is defined as:
    \begin{align}\label{def:MAB_equ} 
        MAB_{j^*} := \frac{1}{d N n_{synth}} \sum_{k \in [d]} \sum_{m \in [N]} \sum_{l \in [n_{synth}]}  | \hat{\beta}^{(j^*)}_{k, m, l} | \; .
    \end{align}
\end{definition}

The intercept $\hat{\beta}^{(j^*)}_{0, m, l}$ is not included in the definition of the MAB. The MAB will only be low if the AIA is unsuccessful. 
In Appendix \ref{sec:privacy_measures} we present an extension of the MAB to measure worst-case AIA privacy. 

\paragraph{Choice of Target Observations}\label{sec:choice_of_targets}
The PG is defined w.r.t. a \textit{single target observation} of the real data. Thus, the results of an MIA do not only depend on the synthetic data, but also on the choice of target observation. To provide a realistic privacy evaluation we follow \citet{stadler2022synthetic} and pick two sets of target observations: outlying targets outside the 95\% quantile and randomly sampled targets. Although a set of target observations is needed to conduct an AIA, the MAB is independent of the choice of target observation, see Definition \ref{def:MAB}.

\subsection{Theoretical Justification of TVineSynth}\label{sec:theory_TVineSynth}
In the theorems below, we provide a theoretical justification for the TVineSynth
construction. Let $(\bm{x}^T_{i},y_{i})$, $i \in [n]$ be i.i.d. samples of 
$(\bm{X}^T,Y)$, that follow a C-vine distribution with parameters
$\bm{\theta}=(\bm{\theta}_{1},\ldots,\bm{\theta}_d)$, where
$\bm{\theta}_{t}$ are the parameters of tree number $t$\footnote{The index $t$, that in prior sections was used to denote the truncation level, is in this section used as a running index for trees; the truncation level will instead be denoted by $\tau$.}, $\bm{X}$ is
arranged according to the order of the C-vine and $Y$ is binary with $P(Y=1)=\pi_{Y}$ and
$X_{j} \sim U(0,1)$, $j \in [d]$. Further, let $\psi$ be the log-odds ratio
for a given observation $\bm{x}$ of $\bm{X}$, i.e.
\begin{align}
    \psi &= \psi(\pi_{Y},\bm{\theta}_{1},\ldots,\bm{\theta}_{d};\bm{x}) = \log\frac{P(Y=1|\bm{X}=\bm{x};\pi_{Y},\bm{\theta}_{1},\ldots,\bm{\theta}_{d})}{P(Y=0|\bm{X}=\bm{x};\pi_{Y},\bm{\theta}_{1},\ldots,\bm{\theta}_{d})} \; .
\end{align}
It is easily shown that $\psi$ is of the form:
\begin{align}
    \psi &= \sum_{t=1}^{d}\psi_{t}(\pi_{Y},\bm{\theta}_{1},\ldots,\bm{\theta}_{t};\bm{x}) \; ,
\end{align}
where the $\psi_{t}$'s are given by:
\begin{align}
    \psi_{1} &= \log\frac{\pi_{Y}}{1-\pi_{Y}}+\sum_{j=1}^{d}\log\frac{f_{j|y}(x_{j}|1)}{f_{j|y}(x_{j}|0)} \; ,
\end{align}
and 
\begin{align}
    \psi_{t} &= \sum_{j=1}^{d+1-t}\log\frac{c_{j,d+2-t; d+3-t, \ldots, d,y}^{1}}{c_{j,d+2-t; d+3-t, \ldots,  d,y}^{0}} \; ,
\end{align}
with $\ t \in \{2,\ldots,d \}$ where $c_{j,d+2-t; d+3-t\ldots d,y}^{k}$ is evaluated at $(\bm{x},y)=(\bm{x},k)$. Moreover, let:
\begin{align}
    \hat{\psi} &= \psi(\hat{\pi}_{Y},\hat{\bm{\theta}}_{1},\ldots,\hat{\bm{\theta}}_{d};\bm{x}) \; ,
\end{align}
where $(\hat{\pi}_{Y},\hat{\bm{\theta}}_{1},\ldots,\hat{\bm{\theta}}_{d})$ are
the maximum likelihood estimators of $(\pi_{Y},\bm{\theta}_{1},\ldots,\bm{\theta}_{d})$,
and
\begin{align}
    \tilde{\psi}^{\tau} &= \sum_{t=1}^{\tau}\psi_{t}(\hat{\pi}_{Y},\hat{\bm{\theta}}_{1},\ldots,\hat{\bm{\theta}}_{t};\bm{x})   
\end{align}
be the estimator of $\psi$ from the C-vine truncated at level $\tau$.

\begin{theorem}\label{thm:utility}
Under these assumptions, it holds for large enough $n$ that:
\begin{align}
    MSE(\hat{\psi}) &= E\Big[ (\hat{\psi}-\psi)^{2} \Big] = \frac{1}{n}\cdot \bm{v}^{T}\bm{J}^{-1}\bm{v}+\smallO\left(\frac{1}{n}\right) \; , \\
    MSE(\tilde{\psi}^{\tau}) &= \left( \sum_{t=\tau + 1}^{d} \psi_{t} (\pi_{Y}, \bm{\theta}_{1}, \ldots, \bm{\theta}_{t}; \bm{x}) \right)^{2} + \frac{1}{n}\cdot \left(\bm{v}^{1\ldots \tau}\right)^{T}\bm{J}^{1\ldots \tau,1\ldots \tau}\bm{v}^{1\ldots \tau} + \smallO\left(\frac{1}{n}\right) \; ,
\end{align}
with: 
\begin{align}
    \bm{v} &= \frac{\partial \psi}{\partial(\pi_{Y},\bm{\theta}_{1},\ldots,\bm{\theta}_{d})} \; , \\
    \bm{v}^{1\ldots \tau} &= \frac{\partial \sum_{t=1}^{\tau}\psi_{t}}{\partial(\pi_{Y},\bm{\theta}_{1},\ldots,\bm{\theta}_{\tau})} \; , \\
    \bm{J} &= - E\Big[\frac{\partial^{2}}{\partial(\pi_{Y},\bm{\theta}_{1},\ldots,\bm{\theta}_{d})\partial(\pi_{Y},\bm{\theta}_{1},\ldots,\bm{\theta}_{d})^{T}}  \log f(\bm{X},Y)\Big]
\end{align}
and $\bm{J}^{1\ldots \tau, 1 \ldots \tau}$ is the upper left sub-matrix of $\bm{J}^{-1}$ corresponding 
to the parameters $(\pi_{Y},\bm{\theta}_{1},\ldots,\bm{\theta}_{\tau})$.
\end{theorem}

This means that as the size $n$ of the training data increases, the MSE of the estimated log-odds ratio for the full vine vanishes, while the one for the $\tau$-truncated vine is dominated by the squared bias. Hence, for large $n$ the utility of the truncated vine is lower than that of the full one. However, the bias does not necessarily increase monotonically as $\tau$ decreases, i.e. as more trees are truncated away. This is due to the fact that it consists of sums of log-differences of pair copula densities, the sign of which will vary with the copula families, parameters and $\bm{x}$. Further, the variance term of the MSE will typically be smaller for the truncated vine, meaning that for smaller $n$ the utility of the full vine is not necessarily higher than that of a $\tau$-truncated one, as seen e.g. in Figure \ref{fig:2dPrivUt_simreal_X6AIA_MAB_Cvine}. This is one of the reasons why we recommend to go through all, or at least several, truncation levels in order to find the best one.

Assume now that the order of the columns of $\bm{V}$, defined in 
Definition \ref{def:MAB}, is the same as order of the variables in the 
C-vine, where we omit the subscripts $m$ and $l$ for simplicity. Also, 
note that a multivariate normal distribution may be expressed as a C-vine 
with only Gaussian pair copulas, combined with normal margins. 

\begin{theorem}\label{thm:mab1}
    Assume that each row of $\bm{V}$ follows a standard $(d+1)$-variate
    normal distribution with correlation matrix $\bm{\rho}$. Then $\hat{\bm{\beta}}$
    follows a $d$-variate normal distribution with:
    \begin{align}
        E[ \hat{\bm{\beta}}^{(j^*)} ] &= \bm{\beta}^{(j^*)} = \bm{\rho}_{[d+1]\setminus \{j^*\},[d+1]\setminus \{j^*\}}^{-1}\bm{\rho}_{[d+1]\setminus \{j^*\},j^*}   
    \end{align}
    and covariance matrix:
    \begin{align}
        Var(\hat{\bm{\beta}}^{(j^*)}) &= (\sigma^{(j^*)})^{2}(\bm{V}_{[d+1]\setminus \{j^*\}}^{T}\bm{V}_{[d+1]\setminus \{j^*\}})^{-1}   
    \end{align}
    with: 
    \begin{align}
        (\sigma^{(j^*)})^{2} &= 1-\bm{\rho}_{[d+1]\setminus \{j^*\},j^*}^{T}\bm{\rho}_{[d+1]\setminus \{j^*\},[d+1]\setminus \{j^*\}}^{-1} \bm{\rho}_{[d+1]\setminus \{j^*\},j^*} \; .
    \end{align}
    Let now the C-vine of $\bm{V}$ be truncated at level $\tau \leq d+1-j^*$ and let $\bm{\beta}_{(\tau)}^{(j^*)}$ be the coefficient corresponding to the C-vine truncated at level $\tau$. Then:
    \begin{align}
        \bm{\beta}_{(\tau) \; 1\ldots d- \tau}^{(j^*)} &= \bm{0} \; , \\
        \bm{\beta}_{(\tau) \; d+1-\tau \ldots d}^{(j^*)} &= \bm{\rho}_{d+2-\tau \ldots d+1,d+2-\tau \ldots d+1}^{-1}\bm{\rho}_{d+2-\tau\ldots d+1,j^*}
    \end{align}
    and: 
    \begin{align}
        (\sigma^{(j^*)}_{(\tau)})^{2} &= 1- \bm{\rho}_{d+2-\tau \ldots d+1,j^*}^{T} \bm{\rho}_{d+2-\tau \ldots d+1,d+2-\tau \ldots d+1}^{-1} \bm{\rho}_{d+2-\tau \ldots d+1,j^*} \; .
    \end{align}
\end{theorem}

\begin{theorem}\label{thm:mab2}
    Under the same assumptions as Theorem \ref{thm:mab1}, if $\bm{\rho}$ has a block
    structure with $\rho_{kl}=0$, $\forall (k,l)$ with $k \in (K \cup S)$ and $l \in [d+1]\setminus (K \cup S)$, where $K$ and $S$ are as defined in Algorithm \ref{alg:find_order}, and
    the C-vine of $\bm{V}$ is truncated at level $\tau \leq d+1-|K|-|S|$, then
    $\bm{\beta}^{(j^*)}_{(\tau)}=\bm{0}$. 
\end{theorem}
Proofs of Theorems \ref{thm:utility}, \ref{thm:mab1} and \ref{thm:mab2}
are given in Appendix \ref{sec:proofs_TVineSynth}.

This means that if the C-vine is truncated somewhere below the tree where the sensitive feature appears in the center node, some of the $\hat{\beta}_{k}$s will have mean $0$, and will thus tend to be small, which reduces the attacker's ability to guess the value of the sensitive variable, and improves the protection of privacy. Further, the number of $\hat{\beta}_{k}$s with mean $0$ increases by $1$ for each tree that is truncated away. If in addition the correlation matrix of the C-vine follows a block structure, where the block containing the sensitive features is approximately uncorrelated with the remaining block(s), then all $\hat{\beta}_{k}$s will have mean (approximately) $0$ already at truncation level $\tau = d+1-|K|-|S|$, where the first variable not in the sensitive block appears in the center of the tree. This gives a high protection of privacy, without truncating away too many trees, thus increasing the potential for high utility. This is exactly the purpose of ordering the C-vine according to Algorithm \ref{alg:find_order}. Without the block structure, one might have to truncate away all trees to obtain the same protection of privacy, which would correspond to removing all dependencies between the variables, and a correspondingly minuscule utility. Note that the choice of $\rho^*$ in Algorithm \ref{alg:find_order} affects the block structure of the correlation matrix $\bm{\rho}$ and thus the $\beta$s. A larger $\rho^*$ leads to a smaller $K$, and potentially more correlated blocks, which reduces the protection of privacy.

\section{RESULTS}\label{sec:results}

\paragraph{Simulated Data}\label{subsec:results_simulated_realId20}

We simulate a real data set to study the effect of truncation and $\mathcal{V}$ on privacy and utility, see Appendix \ref{sec:simulated_real_data_d20_appendix} for details. AIA results of the C-vine confirm Theorem \ref{thm:mab2} as the MAB jumps at the truncation level expected from the block structure of the  real data's correlation matrix. Truncated at the level corresponding to the position of the sensitive covariate in $\mathcal{O}^*$, TVineSynth offers AIA and MIA privacy as good as PrivBayes and superior to CTGAN, TVAE and PrivPGD and a utility superior to CTGAN and especially to PrivBayes and PrivPGD, and comparable to TVAE, see Figure \ref{fig:2dPrivUt_simreal_X6AIA_MAB_Cvine_2}. For more detailed results, please consult Appendix \ref{app:sim_real_results_appendix}.

\subsection{Real-world Data}\label{subsec:results_support2}

We apply TVineSynth to the real-world SUPPORT2 data containing patients suffering from various conditions \citep{support2data}. The binary response $Y$ indicates if a patient died during the study. Covariates \textit{crea} and \textit{totcst} are selected as sensitive features, see Appendix \ref{sec:support2_data} for details.

\paragraph{Privacy: Attribute Inference Attack}\label{sec:AIA_support2_results}

In accordance with the block correlation matrix of the real data, see Figure \ref{fig:Cvine_synth_data_corr_rrealsupport2} in Appendix \ref{sec:support2_data}, after applying Algorithm \ref{alg:find_order} and Theorem \ref{thm:mab2}, TVineSynth provides high AIA privacy when the C-vine is truncated below level 15, outperforming TVAE and PrivPGD and comparable to CTGAN. For truncation at level 10 and lower (\textit{totcst}) and at level 1 (\textit{crea}) the C-vine's MAB is as low as for PrivBayes, see Figure \ref{fig:AIA_Cvine_competitors_support2_small_MAB}. Moving from truncation at level 20 to no truncation the $MAB_{totcst}$ of the C-vine changes its trend and decreases. This is because the un-truncated C-vine starts to model noise in the real data. Figure \ref{fig:utility_Cvine_competitors_support2} confirms this, showing a decrease in utility for the un-truncated C-vine. 
Comparing to the generative models' utility, Figure \ref{fig:utility_Cvine_competitors_support2}, we observe that the privacy protection offered by PrivBayes and CTGAN comes at the cost of utility. 
The PrivPGD exhibits a surprisingly high MAB. For this reason we additionally consulted the AIA's estimated $\beta$-coefficients. These exhibit a mean close to 0, indicating moderate privacy protection, but a high variation, which explains the PrivPGD's high MAB. This is confirmed by the MSE which for outlying targets is moderate to high, see Appendix \ref{sec:realsupport2_small_AIA_MSE}. Hence the PrivPGD seems quite unstable compared to the other synthetic data generators between different runs of the AIA.

\paragraph{Privacy: Membership Inference Attack}\label{sec:MIA_support2_results}

The PG of C-vine generated synthetic data is around 1 with low variation for all truncation levels, indicating optimal MIA privacy for outlying (orange) and randomly sampled (blue) targets, Figure \ref{fig:MIA_Cvine_competitors_support2_small}. 
The C-vine's PG is seemingly independent of truncation level because the estimation of the un-truncated C-vine with Maximum Likelihood (ML) is robust w.r.t. adding/removing a single observation to the real data.\footnote{The robustness of ML estimation depends on the sample size of the (real) data. Thus, for lower sample sizes than the ones used here, we would expect to see a MIA PG that varies more with truncation level.} As a consequence, also a C-vine truncated at level $t < d$ shows the same robustness as the un-truncated C-vine. 
These results compare to PrivBayes for $\epsilon \in \{0.1, 1, 5\}$ and PrivPGD with $\epsilon=2.5$ and $\delta=10^{-5}$. The PG of CTGAN is about 1 at median, but exhibits a high variation over different observations and repetitions of the MIA. The TVAE provides very low MIA privacy with a PG of around 0. This indicates that the TVAE generated synthetic data reproduce the SUPPORT2 data too detailed, harming privacy. 

\paragraph{Utility}\label{sec:utility_support2}

For evaluating utility, 50 synthetic data sets of the same size as the real data ($n=884$) are generated from each model, a random forest classifier is trained on each of them and tested on hold-out test data ($n_{test} = 220$).
The C-vine generated synthetic data consistently outperform synthetic data generated from a CTGAN and PrivPGD and by far PrivBayes for all truncation levels, yielding an $AUC(\bm{y}^*, \hat{\bm{w}}^*)$ almost as high as $AUC(\bm{y}^*, \hat{\bm{y}}^*) \approx 0.71$, Figure \ref{fig:utility_Cvine_competitors_support2}. Only the TVAE performs comparable to the C-vine. Considering its low PG and high MAB in Figures \ref{fig:AIA_Cvine_competitors_support2_small_MAB} and \ref{fig:MIA_Cvine_competitors_support2_small}, the TVAE violates the privacy by modeling the real data too closely. The C-vine, contrarily, captures the dependencies in the real data without compromising privacy.

\paragraph{Privacy-Utility Plots}\label{sec:2dPrivUt_AIAX1_support2}

If we truncate at level 10 or lower for sensitive covariate \textit{totcst} and at level 5 or lower for \textit{crea}, TVineSynth generated synthetic data offer a privacy-utility balance superior to that of the competitors, Figure \ref{fig:2d_priv-ut_support2_crea_totcstAIA_MAB}. 
See Appendix \ref{sec:2d_priv-ut_support2_additional} for further results.

\paragraph{Statistical Fidelity} In terms of the statistical fidelity and discrepancy between real and synthetic joint and marginal distributions TVineSynth outperforms its competitors, see Appendix \ref{app:stat_discrepancy}.

\begin{figure}[t!]
    \centering
    \begin{subfigure}[t]{0.48\textwidth}
        \centering
        \includegraphics[width=\columnwidth]{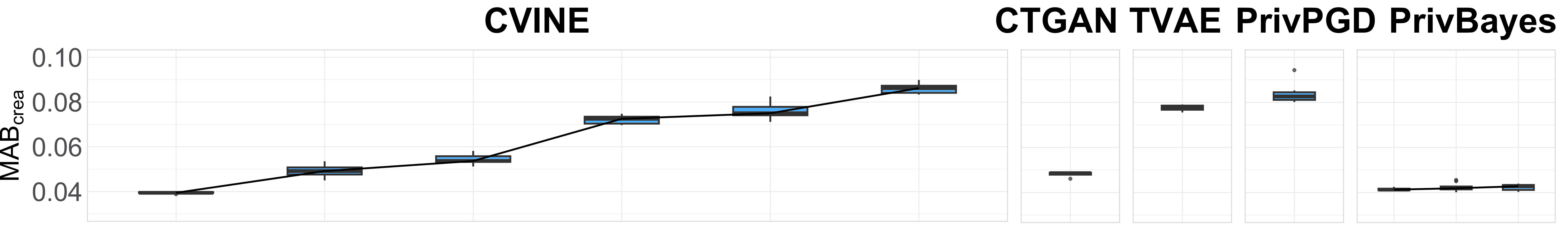}
        \includegraphics[width=\columnwidth]{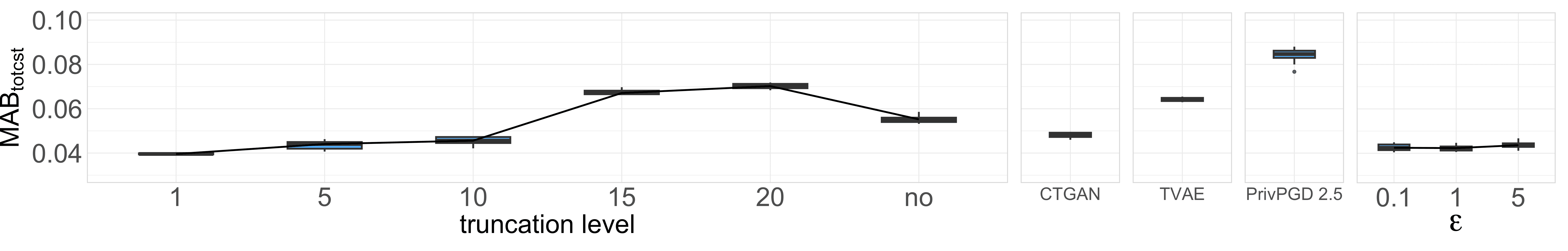}
        \caption{$MAB_j$ under an AIA w.r.t. sensitive covariate \textit{crea} (top row) and \textit{totcst} (bottom row). }\label{fig:AIA_Cvine_competitors_support2_small_MAB}
    \end{subfigure}
    ~
    \begin{subfigure}[t]{0.48\textwidth}
        \centering
        \includegraphics[width=\columnwidth]{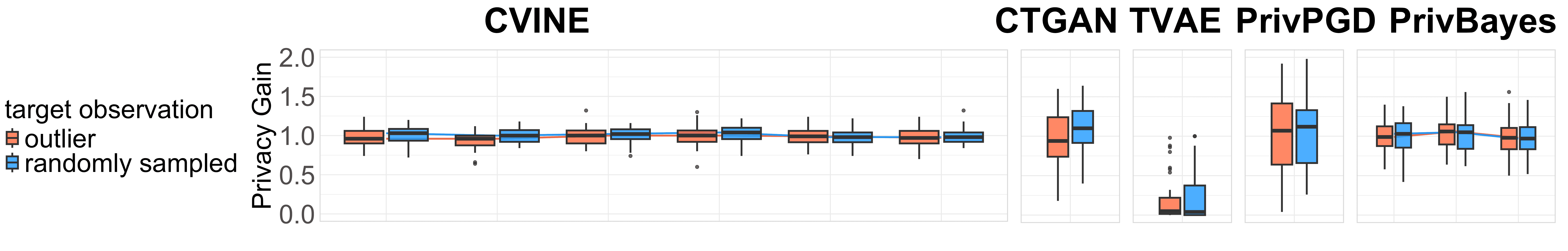}
        \includegraphics[width=\columnwidth]{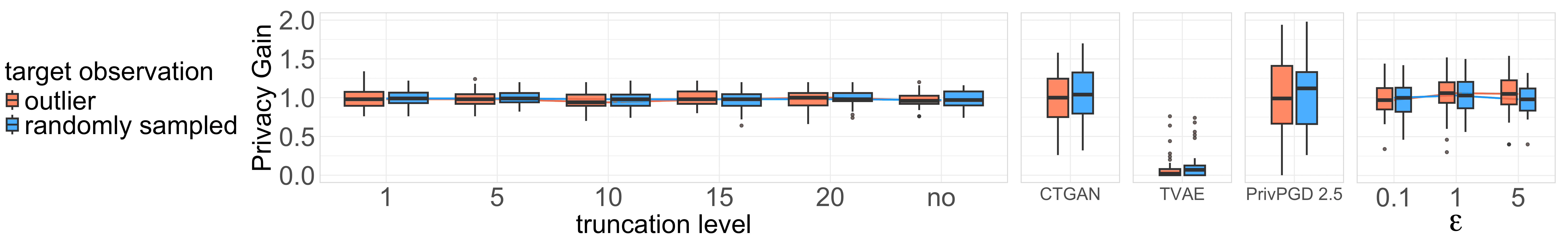}
        \caption{PG under a MIA w.r.t randomly sampled (blue) and outlying targets (orange) w.r.t. sensitive covariate \textit{crea} (top row) and \textit{totcst} (bottom row).}\label{fig:MIA_Cvine_competitors_support2_small}
    \end{subfigure}%
    \\
    \begin{subfigure}[t]{0.48\textwidth}
        \includegraphics[width=\columnwidth]{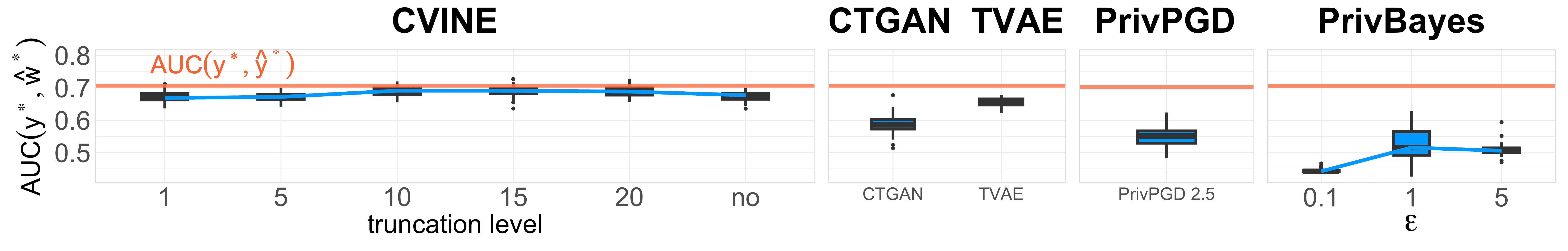}
        \caption{Utility measured with $AUC(\bm{y}^*, \hat{\bm{w}}^*)$ (blue) w.r.t.  a random forest classifier and compared to $AUC(\bm{y}^*, \hat{\bm{y}}^*)$ (orange).}
    \label{fig:utility_Cvine_competitors_support2}
    \end{subfigure}
    ~
    \begin{subfigure}[t]{0.48\textwidth}
        \centering
        \includegraphics[width=0.498\columnwidth]{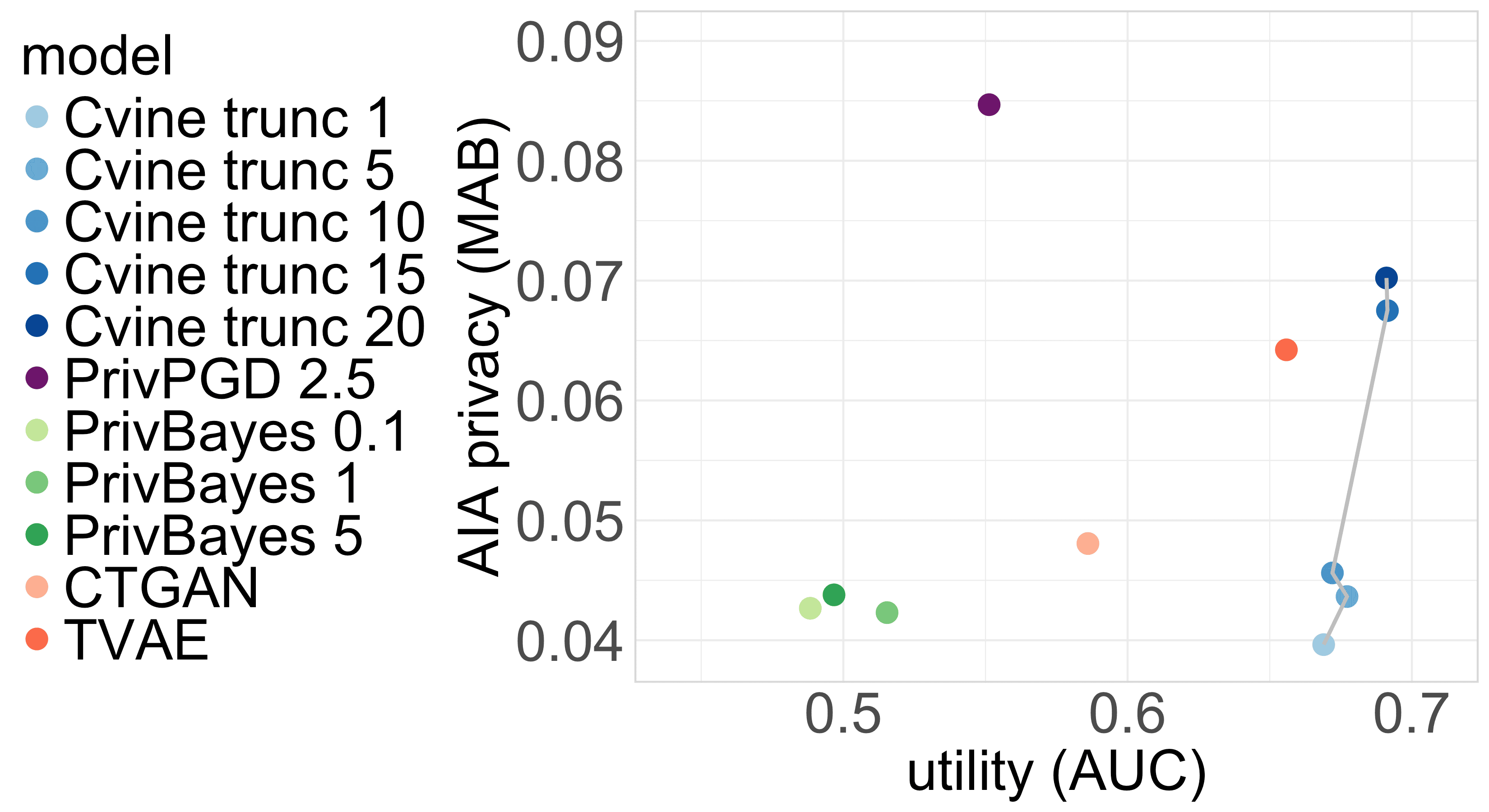}
        \includegraphics[width=0.32868\columnwidth]{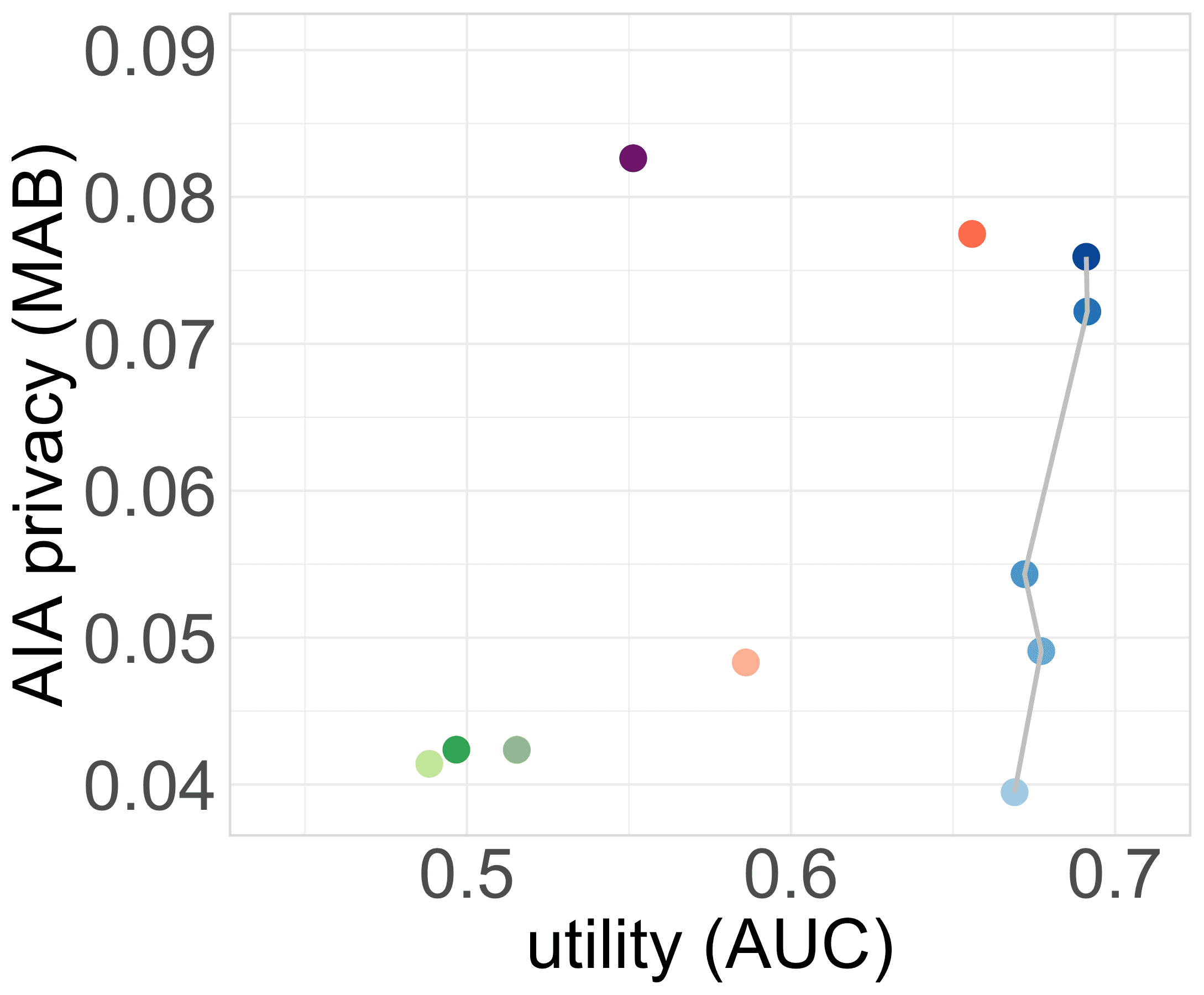}
        \caption{Privacy-utility plot w.r.t. AIA and sensitive features \textit{totcst} (left) and \textit{crea} (right).}
    \label{fig:2d_priv-ut_support2_crea_totcstAIA_MAB}
    \end{subfigure}
    
    \caption{MAB (a), PG (b), utility (c) and privacy-utility plots (d) of synthetic data generated with a C-vine for different truncation levels, CTGAN, TVAE, PrivPGD ($\epsilon=2.5, \, \delta=10^{-5}$) and PrivBayes ($\epsilon \in \{0.1, 1, 5\}$). Boxplots are obtained from 10 game iterations in the AIA and MIA, 50 synthetic data sets in the utility evaluation. Model and privacy attack parameters can be found in Appendix \ref{sec:model_and_attack_parameters}.}
\end{figure}

\section{CONCLUSION}\label{sec:conclusion}
We present TVineSynth, a synthetic tabular data generator based on a truncated C-vine to balance privacy and utility and theoretically justify its construction. Experiments show that TVineSynth offers a privacy-utility trade-off superior to that of competitors. While TVineSynth is not limited to supervised ML tasks, the vine structure might have to be changed for applications such as clustering. Further work could focus on improving scalability, as inference on a vine copula is computationally difficult for more than 500 dimensions, and evaluating the synthetic data also w.r.t. fairness and explainability of predictions.

\subsection*{Acknowledgments}
This work was supported by the Research Council of Norway, Integreat - Norwegian Centre for knowledge-driven machine learning, project number 332645. The work of C. Czado is supported in part by the Deutsche Forschungsgemeinschaft (DFG CZ-86/6-1 CZ-86/10-1).

\bibliography{references}
\bibliographystyle{plainnat}
\section*{Checklist}

 \begin{enumerate}

 \item For all models and algorithms presented, check if you include:
 \begin{enumerate}
   \item A clear description of the mathematical setting, assumptions, algorithm, and/or model. [Yes] See Section \ref{sec:TVineSynth_construction} and Appendix \ref{sec:orderRVs}.
   \item An analysis of the properties and complexity (time, space, sample size) of any algorithm. [Yes] See Appendix \ref{sec:scaling_vines} for the computational complexity of TVineSynth.
   \item (Optional) Anonymized source code, with specification of all dependencies, including external libraries. [Yes] See the submitted zip file containing the anonymized code. 
 \end{enumerate}

 \item For any theoretical claim, check if you include:
 \begin{enumerate}
   \item Statements of the full set of assumptions of all theoretical results. [Yes] See Section \ref{sec:theory_TVineSynth} where we give the full set of assumptions made to theoretically justify TVineSynth.
   \item Complete proofs of all theoretical results. [Yes] See Appendix \ref{sec:proofs_TVineSynth}.
   \item Clear explanations of any assumptions. [Yes] See Section \ref{sec:theory_TVineSynth}.
 \end{enumerate}

 \item For all figures and tables that present empirical results, check if you include:
 \begin{enumerate}
   \item The code, data, and instructions needed to reproduce the main experimental results (either in the supplemental material or as a URL). [Yes] See submitted code and Appendices \ref{sec:model_and_attack_parameters}, \ref{sec:append_simreald20_choice_of_targets} and \ref{sec:append_support2_choice_of_targets}.
   \item All the training details (e.g., data splits, hyperparameters, how they were chosen). [Yes] See Appendices \ref{sec:model_and_attack_parameters}, \ref{sec:simulated_real_data_d20_appendix} and \ref{sec:support2_data}.
    \item A clear definition of the specific measure or statistics and error bars (e.g., with respect to the random seed after running experiments multiple times). [Yes] See captions of figures for description of error bars.
    \item A description of the computing infrastructure used. (e.g., type of GPUs, internal cluster, or cloud provider). [Yes] See Appendix \ref{sec:compute_resources}.
 \end{enumerate}

 \item If you are using existing assets (e.g., code, data, models) or curating/releasing new assets, check if you include:
 \begin{enumerate}
   \item Citations of the creator If your work uses existing assets. [Yes] See Section \ref{subsec:results_support2} and Appendix \ref{sec:support2_data}.
   \item The license information of the assets, if applicable. [Yes] See submitted zip file.
   \item New assets either in the supplemental material or as a URL, if applicable. [Yes] See Appendix \ref{sec:simulated_real_data_d20_appendix} and submitted zip file.
   \item Information about consent from data providers/curators. [Not Applicable] The data set used is published online, see Appendix \ref{sec:support2_data}.
   \item Discussion of sensible content if applicable, e.g., personally identifiable information or offensive content. [Not Applicable] The data used is published and does not contain offensive content.
 \end{enumerate}

 \item If you used crowdsourcing or conducted research with human subjects, check if you include:
 \begin{enumerate}
   \item The full text of instructions given to participants and screenshots. [Not Applicable] No crowdsourcing or research with human subjects was conducted.
   \item Descriptions of potential participant risks, with links to Institutional Review Board (IRB) approvals if applicable. [Not Applicable] No crowdsourcing or research with human subjects was conducted.
   \item The estimated hourly wage paid to participants and the total amount spent on participant compensation. [Not Applicable] No crowdsourcing or research with human subjects was conducted.
 \end{enumerate}

 \end{enumerate}

 \onecolumn

\newpage
\appendix
\addcontentsline{toc}{section}{Appendix} %
\part{Appendix} %
\parttoc %

\section{An Introduction to Vine Copulas}\label{app:intro_to_vines}
 This introduction to vine copulas is based on \citet{griesbauer2022vine} which again is based on \citet{czado2019analyzing}. In the latter more details can be found. Vine copulas build on the concept of copulas.

\subsection{Copulas}

\begin{definition}
Let $d \in \mathbb{N}$. The function $C: [0,1]^d \rightarrow [0,1]^d$ is a \textit{d-dimensional copula} if it is a $d$-dimensional cumulative distribution function with uniform marginal distributions $U[0,1]$.
\end{definition}

So for the random vector $(U_1, \dots U_d)$ taking on values $(u_1, \dots, u_d) \in [0,1]^d$ it is:
\begin{align}
    C(u_1, \dots, u_d) = P(U_1 \leq u_1, \dots, U_d \leq u_d) \; .
\end{align}

\begin{theorem}[Sklar's Theorem] \label{thm:Sklar}
    Let $\bm{X}$ be a $d$-dimensional random vector with distribution function $F$ and marginal distributions $F_1, \dots F_d$. Then $F$ can be expressed as:
    \begin{align} \label{thm:Sklar1}
        F(x_1, \dots, x_d) = C(F_1(x_1), \dots, F_d(x_d)) \; , \quad  (x_1, \dots, x_d) \in \mathbb{R}^d \; . 
    \end{align} 
    where $C$ is a copula. If $F$ is absolutely continuous, the copula $C$ is unique. We then say that the copula $C$ is corresponding to the distribution $F$. In the case of absolute continuity all densities exist and we can express the joint density $f$ of $\bm{X}$ as:
    \begin{align} \label{thm:Sklar1_density}
        f(x_1, \dots, x_d) = c(F_1(x_1), \dots, F_d(x_d)) \cdot f_1(x_1) \cdot ... \cdot f_d(x_d) \; . 
    \end{align} 
    Conversely, let $C$ be the $d$-dimensional copula corresponding to the joint distribution function $F$ of $\bm{X}$ with marginal distributions $F_1, \dots F_d$. Then we can express $C$ as:
    \begin{align}\label{thm:Sklar2}
        C(u_1, \dots, u_d) = F(F_1^{-1}(u_1), \dots, F_d^{-1}(u_d)) 
    \end{align}
    with copula density:
    \begin{align}\label{thm:Sklar2_density}
        c(u_1, \dots, u_d) = \frac{f(F_1^{-1}(u_1), \dots, F_d^{-1}(u_d))}{f_1(F_1^{-1}(u_1)) \cdot ... \cdot f_d(F_d^{-1}(u_d))} \; .
    \end{align}
\end{theorem}

Sklar's Theorem, \cite{sklar1959fonctions} provides the link between the copula on the $d$-dimensional hypercube and the probability distribution of the random vector $(X_1, \dots, X_d)$. Equation \eqref{thm:Sklar1_density} illustrates how the joint density $f$ of a random vector $(X_1, \dots, X_d)$ can be split into the joint copula density, which captures the dependence structure of $X_1, \dots X_d$, and the marginal densities $f_1, \dots f_d$.

The inverse Sklar's Theorem \ref{thm:Sklar} gives the construction of the \textit{elliptical copulas}, to which the Gauss copula belongs.

\begin{definition}[bivariate Gauss copula]
    Let $\Phi_2 ( \cdot, \cdot; \rho)$ be the $2$-dimensional standard normal distribution with mean vector $\bm{\mu} = 0$ and correlation parameter $\rho \in (0, 1)$, and let $\Phi^{-1}(\cdot)$ be the quantile function of the univariate standard normal distribution. Then by Sklar's Theorem  \ref{thm:Sklar} we obtain the \textit{bivariate Gauss copula} by:
    \begin{align}
        C(u_1, u_2; \rho) = \Phi_2(\Phi^{-1}(u_1), \Phi^{-1}(u_2); \rho) \; .
    \end{align}
\end{definition}

The Clayton, Gumbel, Frank and Joe copulas belong to the class of \textit{Archimedean copulas}, which is covered in more detail for example in \cite{nelsen2007introduction}.

\subsubsection*{Pair Copula Construction (PCC)} \label{section:pcc}
The set of multivariate copulas to choose from, i.e. elliptical and Archimedean copulas, is rather limited and constrained in modeling flexibility. However, complex high-dimensional dependence structures call for more flexible multivariate copulas. \citet{aas2009pair}, which the following subsection is based on, decompose a multivariate density by using a cascade of bivariate building blocks: pair copulas. Knowing how to decompose a multivariate distribution, this approach can be reversed in order to construct multivariate copulas and distribution functions respectively. These are flexible and their construction is simple. This is the idea of \textit{pair copula construction}. \\

The following notation is defined:

\begin{definition}\label{def:notation_copula_cond_dist}
    Let $\bm{X}_D \in \mathbb{R}^d$ be a random vector and $\bm{x}_D \in \mathbb{R}^d$, let $i, j, d \in \mathbb{N}$ and $D \subset \mathbb{N}$ with $i, j \notin D$ and $| D | = d$. Let $F_{i j \, | \, D}(\cdot, \cdot \, | \, \bm{X}_D = \bm{x}_D)$ be the conditional distribution of $(X_i, X_j)$ given that $\bm{X}_D = \bm{x}_D $. The copula distribution associated with $F_{i j \, | \, D}(\cdot, \cdot \, | \, \bm{X}_D = \bm{x}_D)$ is denoted by:
    \begin{align*}
        C_{i j ; D}( \cdot, \cdot ; \bm{x}_D) \; .
    \end{align*}
    If existing, its corresponding density is denoted by:
    \begin{align*}
        c_{i j ; D}( \cdot, \cdot ; \bm{x}_D) \; .
    \end{align*}
\end{definition}

    Let $\bm{X} = (X_1, X_2, X_3)$ be a random vector with joint density function $f_{1 2 3}$ and marginal density functions $f_1, f_2$ and $f_3$. Using conditioning we can rewrite the joint density function:
    
    \begin{align}
        f_{1 2 3}(x_1, x_2, x_3) = f_{1 | 2 3}(x_1 \, | \, x_2, x_3) f_{2 | 3}(x_2 \, | \, x_3) f_3(x_3) \; , \label{pcc:start_term}
    \end{align}
    
    with:
    
    \begin{align}
        f_{2 | 3}(x_2 \, | \, x_3) &= \frac{f_{2 3}(x_2, x_3)}{f_3(x_3)} \; , \label{pcc:umformung1}\\
        f_{1 | 2 3}(x_1 \, | \, x_2, x_3) &= \frac{f_{1 2 3}(x_1, x_2, x_3)}{f_{2 3}(x_2, x_3)} = \frac{f_{1 3 | 2}(x_1, x_3 \, | \, x_2)}{f_{3 | 2}(x_3 \, | \, x_2)}\; . \label{pcc:umformung2}
    \end{align}
    
    By Sklar's Theorem \ref{thm:Sklar} we know, that:
    
    \begin{align}
       f_{2 3}(x_2, x_3) = c_{2 3} (F_2(x_2), F_3(x_3)) f_2(x_2) f_3(x_3) \; , 
    \end{align}
    
    and thus (\ref{pcc:umformung1}) becomes:
    
    \begin{equation}
        \begin{aligned}[b]
            f_{2 | 3}(x_2 \, | \, x_3) &:= \frac{f_{2 3}(x_2, x_3)}{f_3(x_3)} = c_{2 3} (F_2(x_2), F_3(x_3)) f_2(x_2) \; . \label{pcc:zwischen_ergebnis1}
        \end{aligned}
    \end{equation}
    
    In the same manner we obtain (\ref{pcc:umformung2}):
    
    \begin{equation}
        \begin{aligned}[b]
            f_{1 | 2 3}(x_1 \, | \, x_2, x_3) &= \frac{f_{1 3 | 2}(x_1, x_3 \, | \, x_2)}{f_{3 | 2}(x_3 \, | \, x_2)} \\
            &= \frac{c_{1 3; 2}(F(x_1 \, | \, x_2), F(x_3 \, | \, x_2); x_2) f_{1 | 2}(x_1 \, | \, x_2) f_{3 | 2}(x_3 \, | \, x_2)}{f_{3 | 2}(x_3 \, | \, x_2)} \\
            &= c_{1 3; 2}(F(x_1 \, | \, x_2), F(x_3 \, | \, x_2); x_2) f_{1 | 2}(x_1 \, | \, x_2) \\
            &= c_{1 3; 2}(F(x_1 \, | \, x_2), F(x_3 \, | \, x_2); x_2) c_{1 2} (F_1(x_1), F_2(x_2)) f_1(x_1) \; . \label{pcc:zwischen_ergebnis2}
        \end{aligned}
    \end{equation}
    
    Combining (\ref{pcc:zwischen_ergebnis1}) and (\ref{pcc:zwischen_ergebnis2}) we can decompose (\ref{pcc:start_term}) into a product of pair copulas and marginal distributions:
    
    \begin{equation}
        \begin{aligned}[b]
            f_{1 2 3}(x_1, x_2, x_3) = \, & c_{1 3; 2}(F(x_1 \, | \, x_2), F(x_3 \, | \, x_2); x_2) \\
            & c_{1 2} (F_1(x_1), F_2(x_2)) \, c_{2 3} (F_2(x_2), F_3(x_3)) \\
            & f_1(x_1) \, f_2(x_2) \, f_3(x_3) \; . \label{pcc:decomposition}
        \end{aligned}
    \end{equation}

\begin{remark}
    
    \begin{itemize}
        \item The decomposition with conditioning in (\ref{pcc:start_term}) is not unique. Neither is therefore (\ref{pcc:decomposition}). In general we could reorder $(X_1, X_2, X_3)$ in $3! = 6$ ways. However, in a pair copula there is no distinction made between the first and the second argument, i.e. $c_{i j}(u_i, u_j) = c_{j i}(u_j, u_i)$. That is why we end up with three distinct decompositions in the form of (\ref{pcc:decomposition}).
        
        \item We see, that $c_{1 3; 2}(\cdot, \cdot; x_2)$, the pair copula associated with the conditional distribution of $(X_1, X_3)$ given $X_2 = x_2$ depends on the value $x_2$ of $X_2$. We stick to the terminology in \cite{czado2019analyzing} and speak of a \textit{pair copula decomposition}, if the copulas associated with conditional distributions are allowed to depend on the value of the conditioning variable, i.e. here $X_2 = x_2$. If we ignore this dependence, which in our case would be equivalent to:
        $$ \forall x_2 \in \mathbb{R}: \quad c_{1 3; 2}(u_1, u_3; x_2) = c_{1 3; 2}(u_1, u_3), \quad u_1 \in [0,1], \; u_3 \in [0,1] \ , $$
        we make the \textit{simplifying assumption} in three dimensions. In general it assumes, that copulas associated with conditional distributions do not depend on the value(s) of the conditioning variable(s).
        If we assume the simplifying assumption, we can reverse the decomposition approach and view the simplified version of (\ref{pcc:decomposition}) as the construction of the three dimensional density $f_{123}$ from pair copula densities, conditional distributions and marginal densities. In this case we speak of \textit{pair copula construction}.

        \item Obviously the construction 3-dimensional example above can be generalized to higher dimensions. There we encounter conditional marginal densities, which can be expressed as:
        \begin{align}\label{pcc:cond_density}
            f(x \, | \, \bm{v}) = c_{x v_j ; \bm{v}_{-j}}(F(x \, | \, \bm{v}_{-j}), F(v_j \, | \ \bm{v}_{-j})) \cdot f(x \, | \, \bm{v}_{-j}) \; ,
        \end{align}
        with $\bm{v} \in \mathbb{R}^d$ and $\bm{v}_{-j}$ the sub-vector of $\bm{v}$ with the $j$th component left out. The second factor of (\ref{pcc:cond_density}) can again be factorized with (\ref{pcc:cond_density}). This illustrates the iterative nature of the construction. Finally, with the result of \cite{joe1996families}, that:
        \begin{equation}
            \begin{aligned}[b]\label{multivariate_conditional_dist}
                \forall j: \quad F(x \, | \, \bm{v}) &= \frac{\partial C_{x, v_j ; \bm{v}_{- j}} \big(F(x \, | \, \bm{v}_{- j}), u \big)}{\partial u} \Bigg\rvert_{u = F(v_j \, | \, \bm{v}_{- j})} \\
                &=: \frac{\partial C_{x, v_j ; \bm{v}_{- j}} \big(F(x \, | \, \bm{v}_{- j}), F(v_j \, | \, \bm{v}_{- j}) \big)}{\partial F(v_j \, | \, \bm{v}_{- j})} \; ,
            \end{aligned}
        \end{equation}
        the construction or decomposition respectively is completed. Here \textit{h-functions} help to simplify the notation of conditional distributions and copulas. 
        
        \begin{definition}\label{def:h-function}
            For a bivariate copula $C_{u v}$ the corresponding \textit{h-function} is defined for all $(u, v) \in [0, 1]^2$ as:
            \begin{align}
                h_{u \, | \, v}(u \, | \, v) := \frac{\partial}{\partial v} C_{u v} (u, v) \; .
            \end{align}
        \end{definition}
        
        Clearly (\ref{multivariate_conditional_dist}) holds for any continuous distribution $F$ and thus also for the bivariate copula distribution $C_{u v}$. With $C(u) = u$ for any $u \in [0,1]$ and the copula $C$ it follows that:
        \begin{align} \label{hfunc_cond_copula}
            C_{u \, | \, v}(u \, | \, v) \stackrel{(\ref{multivariate_conditional_dist})}{=} \frac{\partial}{\partial v} C_{u v} (u, v) \stackrel{\ref{def:h-function}}{=} h_{u \, | \, v}(u \, | \, v) \; .
        \end{align}
    \end{itemize}
\end{remark}

\subsection{Regular Vines}

In Section \ref{section:pcc} we saw that a $d$-dimensional probability distribution function can be constructed from or decomposed into bivariate building-blocks, pair copulas. For a specific $d$-dimensional probability distribution there exist several pair copula constructions, a subset of them satisfying the \textit{proximity condition} introduced in the following. \cite{bedford2001probability} and \cite{bedford2002vines} introduced \textit{regular vines (R-vines)} and the \textit{R-vine specification} to efficiently represent the pair copula constructions satisfying the proximity condition. The \textit{R-vine specification} captures the structure of the pair copula construction: each bivariate copula is associated with an edge in a sequence of nested trees, the \textit{R-vine}. The families, rotations and parameters of the bivariate copulas may be stored in matrices. This compact notation facilitates the estimation and sampling procedures on R-vines. \cite{bedford2001probability} and \cite{bedford2002vines} also show, that each R-vine specification stands for a unique $d$-dimensional distribution $F$.

\begin{definition}[Vine, regular vine, regular vine tree sequence] \label{def:rvine}
A set of trees $\mathcal{V} = (T_1, ..., T_{d-1})$ is a \textit{vine} on $d$ elements if:
\begin{enumerate}
    \item[(i)] $T_1$ is a tree with edge set $E_1$ and node set $V_1 = \{1, ..., d\}$.
    \item[(ii)] For $i \in \{2, ..., (d-1)\}$ it holds that $T_i$ is a tree with edge set $E_i$ and node set $V_i = E_{i-1}$.
\end{enumerate}
$\mathcal{V}$ is an \textit{regular vine} (\textit{R-vine}) or \textit{regular vine tree sequence} (\textit{R-vine tree sequence}) if additionally the so called \textit{proximity condition} holds:
\begin{enumerate}\label{proximity_condition}
    \item[(iii)] For $i \in \{2, ..., (d-1)\}$ and $\{a, b\} \in E_i$ with $a = \{a_1, a_2\}$ and $b = \{b_1, b_2\}$ we have that $|a \cap b| = 1$.
\end{enumerate}

\end{definition}

\begin{remark}
The \textit{proximity condition} makes sure that nodes $a$ and $b$ are only then joined by an edge in tree $T_i$ if they share a common node in tree $T_{i-1}$, where $a, b \in E_{i-1}$. 
\end{remark}

Among the R-vines there are (among others) the two sub-classes of C-vines and D-vine. They distinguish themselves through a special structure each tree in $\mathcal{V}$ takes on.

\begin{definition}[C-vine, D-vine]
An R-vine tree sequence $\mathcal{V}$ on $d$ elements is called:
\begin{itemize}
    \item[(i)] \textit{D-vine}, if for each node $v$ of each tree $T_i \in \mathcal{V} , \; i \in [d-1]$ it holds that $deg(v) \leq 2$,
    \item[(ii)] \textit{C-vine}, if in each tree $T_i \in \mathcal{V} , \; i \in [d-1]$ there is one unique node $v$ with $deg(v) = d - i$ which is called \textit{root node}.
\end{itemize}
\end{definition}

Below is the tree sequence of a C-vine on 5 elements without node and edge labels. The root node in each star-shaped tree is coloured.

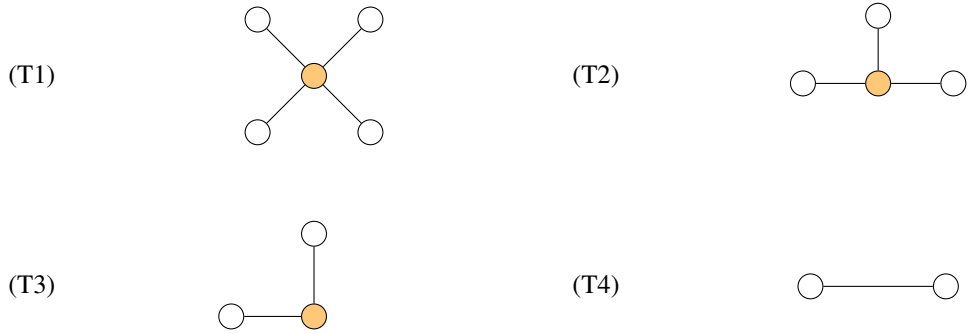
\begin{figure}[H]
    \centering
    \begin{tikzpicture}[node distance=5mm, main/.style = {draw, shape=circle}, baseline=+6mm]
        
        \node[draw=none,fill=none] at (0,0) {(T1)};
        \node[draw=none,fill=none] at (7.5, 0) {(T2)};
        \node[draw=none,fill=none] at (0, -2.8) {(T3)};
        \node[draw=none,fill=none] at (7.5, -2.8) {(T4)};
        \node[main, fill={rgb:orange,1;yellow,2;pink,5}] (1) at (3.75, 0) {};
        \node[main] (2) at (4.5, 0.75) {};
        \node[main] (3) at (3, 0.75) {};
        \node[main] (4) at (3, -0.75) {};
        \node[main] (5) at (4.5, -0.75) {};
        \draw (1) -- (2) node[midway, above] {};
        \draw (1) -- (3) node[midway, above] {};
        \draw (1) -- (4) node[midway, above] {};
        \draw (1) -- (5) node[midway, above] {};

        \node[main, fill={rgb:orange,1;yellow,2;pink,5}] (6) at (11.25, -0.1) {};
        \node[main] (7) at (12.25, -0.1) {};
        \node[main] (8) at (10.25, -0.1) {};
        \node[main] (9) at (11.25, 0.8) {};
        \draw (6) -- (7) node[midway, above] {};
        \draw (6) -- (8) node[midway, above] {};
        \draw (6) -- (9) node[midway, above] {};
        
        \node[main, fill={rgb:orange,1;yellow,2;pink,5}] (10) at (3.75, -3.2) {};
        \node[main] (11) at (3.75, -2.1) {};
        \node[main] (12) at (2.65, -3.2) {};
        \draw (10) -- (11) node[midway, above] {};
        \draw (10) -- (12) node[midway, above] {};
        
        \node[main] (13) at (10.35, -2.8) {};
        \node[main] (14) at (12.15, -2.8) {};
        \draw (13) -- (14) node[midway, above] {};
    \end{tikzpicture}
    \caption{A C-vine on 5 elements.}
\end{figure}

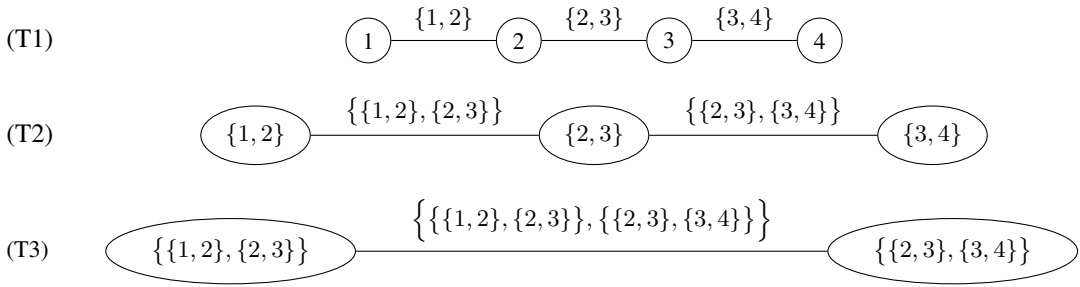
\begin{figure}[H]
    \begin{tikzpicture}[node distance=5mm, main/.style = {draw, shape=circle}, baseline=+6mm]
        
        \node[draw=none,fill=none] at (0,0) {(T1)};
        \node[main] (1) at (4.5, 0) {\footnotesize 1};
        \node[main] (2) at (6.5, 0) {\footnotesize 2};
        \node[main] (3) at (8.5, 0) {\footnotesize 3};
        \node[main] (4) at (10.5, 0) {\footnotesize 4};
        \draw (1) -- (2) node[midway, above] {\footnotesize $\{1, 2\}$};
        \draw (2) -- (3) node[midway, above] {\footnotesize $\{2, 3\}$};
        \draw (3) -- (4) node[midway, above] {\footnotesize $\{3, 4\}$};
    \end{tikzpicture}

\vspace{0.3cm}

    \begin{tikzpicture}[node distance=5mm, main/.style = {draw}, state/.style ={ellipse, draw, minimum width = 0.8 cm}, baseline=+6mm]
    
        \node[draw=none,fill=none] at (0,0) {(T2)};
        \node[state] (1) at (3, 0) {\footnotesize $\{1, 2\}$};
        \node[state] (2) at (7.5, 0) {\footnotesize $\{2, 3\}$};
        \node[state] (3) at (12, 0) {\footnotesize $\{3, 4\}$};
        \draw (1) -- (2) node[midway, above] {\footnotesize $\big\{\{1, 2\}, \{2, 3\}\big\}$};
        \draw (2) -- (3) node[midway, above] {\footnotesize $\big\{\{2, 3\}, \{3, 4\}\big\}$};
    \end{tikzpicture}
 
 \vspace{0.3cm}
 
    \begin{tikzpicture}[font = \footnotesize, node distance=5mm, main/.style = {draw}, state/.style ={ellipse, draw, minimum width = 0.8 cm}, baseline=+6mm]
    
        \node[draw=none,fill=none] at (0,0) {(T3)};
        \node[state] (1) at (2.7, 0) {\footnotesize $\big\{\{1, 2\}, \{2, 3\}\big\}$};
        \node[state] (2) at (12.3, 0) {\footnotesize $\big\{\{2, 3\}, \{3, 4\}\big\}$};
        \draw (1) -- (2) node[midway, above] {\footnotesize $\Big\{ \big\{\{1, 2\}, \{2, 3\}\big\}, \big\{\{2, 3\}, \{3, 4\}\big\} \Big\}$ };
    \end{tikzpicture}
    
    \caption{A D-vine on 4 elements.}
    \label{picture:dvine4}
\end{figure}

\vspace{0.5cm}

In Figure \ref{picture:dvine4} an R-vine tree sequence $\mathcal{V}$ on $d=4$ elements is displayed. Note that $\mathcal{V}$ is a D-vine. 

This notation for edges and nodes is hard to read and use. By \cite{bedford2002vines} and results of \cite{kurowicka2003parameterization} the edges of each tree can be uniquely identified by two \textit{conditioned nodes} and a set of \textit{conditioning nodes}.

\begin{definition}[Complete union, conditioning set, conditioned set] \label{def:complete_union}
    Let $\mathcal{V}$ be an R-vine tree sequence. The \textit{complete union} $U_e$ of the edge $e \in E_i$ is defined as:
    \begin{align}
        U_e := \{ j \in V_1 \ | \ \exists e_1 \in E_1, ..., e_{i-1} \in E_{i-1} \quad  \text{s.th.} \quad  j \in e_1 \in ... \in e_{i-1} \in e \} .
    \end{align}
    The set:
    \begin{align}
        D_e := U_a \cap U_b
    \end{align}
    is called \textit{conditioning set} $D_e$ of an edge $e = \{a, b\}$ and the \textit{conditioned sets} $\mathcal{C}_{e, a}$, $\mathcal{C}_{e, b}$ and $\mathcal{C}_e$ are given by:
    \begin{align}
        \mathcal{C}_{e, a} := U_a \setminus D_e \, , \quad  \mathcal{C}_{e, b} := U_b \setminus D_e \quad \text{and}  \quad \mathcal{C}_e  := \mathcal{C}_{e, a} \cup \mathcal{C}_{e, b} \ .
    \end{align}
\end{definition}

With the notation introduced above we obtain the following conditioning sets:
\begin{align*}
    T_1: & \qquad D_{\{1, 2\}} = \emptyset \; , \quad D_{\{2, 3\}} = \emptyset \; , \quad D_{\{3, 4\}} = \emptyset \; , \\
    T_2: & \qquad D_{\{\{1, 2\}, \{2, 3\}\}} =  \{2\}\; , \quad D_{\{\{2, 3\}, \{3, 4\}\}} =  \{3\} \; , \\
    T_3: & \qquad D_{\{\{\{1, 2\}, \{2, 3\}\}, \{\{2, 3\}, \{3, 4\}\} \}} = \{2, 3\} \; .
\end{align*}
and the following conditioned sets:
\begin{align*}
    T_1: & \qquad \mathcal{C}_{\{1, 2\}} = \{1, 2\} \; , \quad \mathcal{C}_{\{2, 3\}} = \{2, 3\} \; , \quad \mathcal{C}_{\{3, 4\}} = \{3, 4\} \; , \\
    T_2: & \qquad \mathcal{C}_{\{\{1, 2\}, \{2, 3\}\}} =  \{1, 3\}\; , \quad \mathcal{C}_{\{\{2, 3\}, \{3, 4\}\}} =  \{2, 4\} \; , \\
    T_3: & \qquad \mathcal{C}_{\{\{\{1, 2\}, \{2, 3\}\}, \{\{2, 3\}, \{3, 4\}\} \}} = \{1, 4\} \; .
\end{align*}
Figure \ref{picture:dvine4_notation} displays a tree sequence with the new notation which is more readable. It still describes the R-vine tree sequence uniquely up to permutation and order of the elements of the set $\mathcal{C}_e$.\footnote{As a convention we order the elements of $\mathcal{C}_e$ in ascending as done in Figure \ref{picture:dvine4_notation}.}

\vspace{0.2cm}
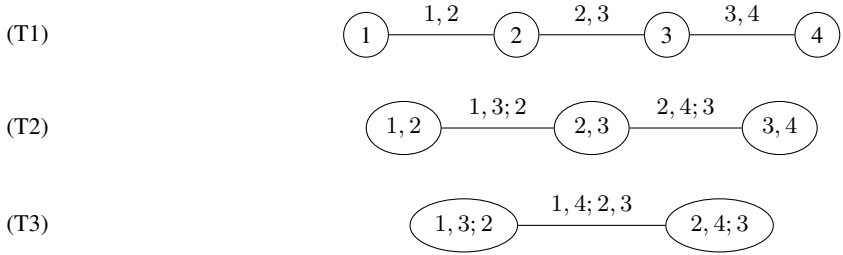
\begin{figure}[H]
    \begin{tikzpicture}[font = \footnotesize, node distance=5mm, main/.style = {draw, shape = circle}, baseline=+6mm]
        
        \node[draw=none,fill=none] at (0,0) {(T1)};
        \node[main] (1) at (4.5, 0) {1};
        \node[main] (2) at (6.5, 0) {2};
        \node[main] (3) at (8.5, 0) {3};
        \node[main] (4) at (10.5, 0) {4};
        \draw (1) -- (2) node[midway, above] {$1, 2$};
        \draw (2) -- (3) node[midway, above] {$2, 3$};
        \draw (3) -- (4) node[midway, above] {$3, 4$};
    \end{tikzpicture}

\vspace{0.3cm}

    \begin{tikzpicture}[font = \footnotesize, node distance=5mm, main/.style = {draw}, state/.style ={ellipse, draw, minimum width = 0.8 cm}, baseline=+6mm]
    
        \node[draw=none,fill=none] at (0,0) {(T2)};
        \node[state] (1) at (5, 0) {$1, 2$};
        \node[state] (2) at (7.5, 0) {$2, 3$};
        \node[state] (3) at (10, 0) {$3, 4$};
        \draw (1) -- (2) node[midway, above] {$1, 3; 2$};
        \draw (2) -- (3) node[midway, above] {$2, 4; 3$};
    \end{tikzpicture}
 
 \vspace{0.3cm}
 
    \begin{tikzpicture}[font = \footnotesize, node distance=5mm, main/.style = {draw}, state/.style ={ellipse, draw, minimum width = 0.8 cm}, baseline=+6mm]
    
        \node[draw=none,fill=none] at (0,0) {(T3)};
        \node[state] (1) at (5.8, 0) {$1, 3; 2$};
        \node[state] (2) at (9.2, 0) {$2, 4; 3$};
        \draw (1) -- (2) node[midway, above] { $1, 4; 2, 3$ };
    \end{tikzpicture}
    \caption{A D-vine on 4 elements with the notation of Definition \ref{def:complete_union}.}
    \label{picture:dvine4_notation}
\end{figure}

\begin{definition}[Constraint set] \label{def:constraint_set}
The \textit{constraint set} $\mathcal{C} \mathcal{V}$ for the R-vine tree sequence $\mathcal{V}$ is defined as:
\begin{align}
    \mathcal{C} \mathcal{V} := \big\{ ( \mathcal{C}_{e, a}, \mathcal{C}_{e, b} ; D_e ) \ | \ e = \{a, b\}, \ e \in E_i \quad \text{for} \quad i \in  [d-1] \big\} \ .
\end{align}
Here the edge $e = ( \mathcal{C}_{e, a}, \mathcal{C}_{e, b} ; D_e )$ of the R-vine tree sequence will often be abbreviated by $e = (e_a, e_b ; D_e)$.

\end{definition}

Here the constraint set is:
\begin{align*}
    \big\{ (1,2), (2, 3), (3, 4), (1, 3; 2), (2, 4; 3), (1, 4; 2, 3) \big\} \; .
\end{align*}
Note that the curly braces of the conditioned and conditioning sets are left out. This is not completely precise nor consistent to Definition \ref{def:constraint_set}, but facilitates notation. In Figure \ref{picture:dvine4_notation} the round braces are left out as well. \\

\citet{kurowicka2003parameterization} show that any edge of an R-vine tree sequence can be identified by its conditioning or conditioned sets.

\begin{definition}[R-vine specification]
The triple $(\bm{F}, \mathcal{V}, B)$ is called \textit{R-vine specification} if:

\begin{enumerate}
    \item[(i)] $\bm{F} = (F_1, ..., F_d)$ is a vector of continuous and invertible distribution functions,
    \item[(ii)] $\mathcal{V}$ is an R-vine tree sequence on $d$ elements and
    \item[(iii)] $B := \big\{C_e \ | \ e \in E_i \quad \text{for} \quad i \in [d-1] \big\}$ is the set of bivariate copulas $C_e$ with $E_i$ the edge set of tree $T_i$ of the R-vine tree sequence $\mathcal{V}$.
\end{enumerate}

\end{definition}

By this definition each edge $e \in E_i$ of a tree $T_i$ in $\mathcal{V}$ corresponds to a bivariate copula $C_e$.

\begin{definition}[Realizing an R-vine specification, Regular vine distribution]

A joint distribution $F$ of the random vector $\bm{X} = (X_1, ..., X_d)$ is said to \textit{realize an R-vine specification} $(\bm{F}, \mathcal{V}, B)$ or have a \textit{regular vine distribution} respectively, if $C_e$ is the bivariate copula of $X_{\mathcal{C}_{e, a}}$ and $X_{\mathcal{C}_{e, b}}$ given $\bm{X}_{D_e}$ for each edge $e = \{a, b\} \in E_i$ and the marginal distribution of $X_i$ is $F_i$ for $i \in [d]$.

\end{definition}

\begin{remark}[Simplifying assumption]
The assumption that for each edge $e$ of $\mathcal{V}$ the bivariate copula $C_e$ does not depend on the value $\bm{x}_{D_e}$ the conditioning random vector $\bm{X}_{D_e}$ takes on is called \textit{simplifying assumption}.
\end{remark}

\begin{theorem} \label{thm:bedford_cooke}

Let $(\bm{F}, \mathcal{V}, B)$ be an R-vine specification on $d$ elements where all pair copulas $C_e \in B$ satisfy the simplifying assumption and have densities $c_e$. There is a unique distribution $F$ that realizes this R-vine specification with density:
\begin{align}
f_{1, ...d}(x_1, ..., x_d) &= \prod_{i = 1}^d f_i(x_i) \; \cdot \\
& \quad \ \prod_{i = 1}^{d-1} \prod_{e \in E_i} c_{ \mathcal{C}_{e, a}, \mathcal{C}_{e, b} ; D_e } \big( F_{\mathcal{C}_{e, a} | D_e}(x_{\mathcal{C}_{e, a}} | \bm{x}_{D_e} ), F_{\mathcal{C}_{e, b} | D_e}(x_{\mathcal{C}_{e, b}} | \bm{x}_{D_e}) \big) \; , 
\end{align}
 
where $f_i$ denote the densities of $F_i$.
 
\end{theorem}

\begin{proof}
The proof of theorem can be found in \cite{bedford2001probability} and \cite{bedford2002vines}.
\end{proof}

\begin{definition}[Regular vine copula]
    A \textit{(regular) vine copula} is a regular vine distribution, where all margins are uniformly distributed on [0, 1].
\end{definition}

\subsubsection*{Estimating Vine Copulas}\label{subsec:estimating_Rvines}
 In order to estimate a vine copula, first the marginal distributions $F_j, \; j \in [d]$ are estimated. Then the vine tree structure $\mathcal{V}$ and the pair copulas $B$ are selected and estimated tree by tree using Di{\ss}mann's Algorithm \cite{dissmann2013selecting} presented in  Algorithm \ref{algo:Dissmann}, which is a maximum spanning tree algorithm to select $\mathcal{V}$ while maximizing the sum of the edge weights - the absolute Kendall's $\tau$ value of the two adjacent random variables.

\begin{algorithm}[tb]
   \caption{Dißmann's algorithm of \cite{dissmann2013selecting}}\label{algo:Dissmann}
    \begin{algorithmic}
        \STATE {\bfseries Input:} $n \in \mathbb{N}$ i.i.d. realizations of the random vector $(X_1, \dots, X_d)$, i.e. $(x_{i 1}, \dots, x_{i d})_{i \in [n]}$
        \STATE {\bfseries Output:} $\mathcal{V}$ and $B$ of R-vine copula specification
    
        \STATE Calculate the empirical Kendall's $\tau$ value $\hat{\tau}_{j, \, k}$ for all possible variable pairs $(j, k)$, $1 \leq j < k \leq d$. 
        \STATE Select the spanning tree that maximizes the sum of absolute empirical Kendalls's $\tau$ values, i.e.:
        $$ T_1 = \arg \max_{T = (V, E) \; \text{in spanning tree}} \ \sum_{e = (j, k) \in E} |\hat{\tau}_{j, \, k}| \; . $$ 
        \STATE For each edge $(j, k)$ in the selected spanning tree, select a copula ad estimate the corresponding parameter(s). Then generate pseudo-observations $\hat{u}_{i, j \, | \, k} := \hat{F}_{j \, | \, k}(x_{i j} \, | \, x_{i k})$ and $\hat{u}_{i, k \, | \, j} := \hat{F}_{k \, | \, j}(x_{i k} \, | \, x_{i j})$, $i \in [n]$ using Equation (\ref{multivariate_conditional_dist}) with the fitted copula $\hat{C}_{j k}$.
        \FOR{$l \in \{2, \dots, d-1\}$}
            \STATE For all conditional variable pairs $(j, \, k \, ; \, D)$ that can be part of tree $T_l$, i.e. all edges fulfilling the proximity condition (iii) of Definition \ref{def:rvine}: calculate the empirical Kendall's $\tau$ value $\hat{\tau}_{j, \, k \, ; \, D}$ 
            $\big(\hat{u}_{i, j \, | \, k \cup D}, \hat{u}_{i, k \, | \, j \cup D} \big)$. Denote these edges in the set $E_l^*$. 
            \STATE Among these edges, select the spanning tree that maximizes the sum of absolute empirical Kendall's $\tau$ values, i.e.:
            $$ T_l = \argmax_{T = (V, E) \; \text{in spanning tree} \; \text{with} \; E \subset E_l^*} \  \sum_{e = (j, \, k \, ; \, D)  \in E} |\hat{\tau}_{j, \, k \, ; \, D}| \; . $$ 
            \STATE  For each edge $(j, \, k \, ; \, D)$ in the selected spanning tree $T_l$, select a conditional copula and estimate the corresponding parameter(s). Then generate pseudo-observations $\hat{u}_{i, j \, | \, k \cup D} := \hat{F}_{j \, | \, k \cup D}(x_{i j} \, | \, x_{i k}, x_{i D})$ and $\hat{u}_{i, k \, | \, j \cup D} := \hat{F}_{k \, | \, j \cup D}(x_{i k} \, | \, x_{i j}, x_{i D})$, $i \in [n]$ using Equation (\ref{multivariate_conditional_dist}) with the fitted copula $\hat{C}_{j k ; D}$.
        \ENDFOR
    \end{algorithmic}
\end{algorithm}

\section{Computational Complexity of TVineSynth}\label{sec:scaling_vines}

\subsection{Estimating a C-Vine} \label{sec:compCost_EstVine}
It is assumed that the user pre-defines the order of the features, thus the full R-vine matrix of the assumed C-vine is given. It is also assumed that a set of $k$ candidate parametric pair copula families to choose from is specified. Pair copula parameters are estimated with MLE and pair copulas are selected using AIC \citep{akaike1998information}. In Dißmann's algorithm \citep{dissmann2013selecting} which is implemented by \citet{rvinecopulib} and most commonly used for R-vine model selection, the maximum spanning tree selection is omitted due to the pre-specified R-vine matrix. This leaves us with $\frac{d \cdot (d-1)}{2}$ edges in the un-truncated vine copula model and thus $\frac{d \cdot (d-1)}{2}$ pair copulas to select and estimate. For each edge all $k$ pair copula candidates are estimated with ML and their AIC is computed. Both involve $n$ terms in the log-likelihood evaluation, the MLE involves an optimization that depends on the number of parameters of the current pair copula family\footnote{The parametric pair copula families implemented by \citet{rvinecopulib} which have been used in this work, have up to 2 parameters.}. The MLE's computational cost also depends on the optimizer chosen and we can assume that it is constant w.r.t. $n$, $d$ and $k$. Then selecting a pair copula out of the $k$ candidates requires finding the minimal AIC among $k$ values which can be solved in $\mathcal{O}(k)$. In total this gives us a computational complexity of $\mathcal{O}(n d^2 k)$. As the number of candidate pair copula families usually is small (\citet{rvinecopulib} implement 10 parametric pair copula families and their rotations excluding the independence copula), $k$ can be considered a constant itself giving a complexity of $\mathcal{O}(n d^2)$.

\subsection{Sampling from a C-Vine} \label{sec:compCost_SamplingVine}
The computational complexity of sampling one observation from a C-vine on $d$ variables using Algorithm 6.4 in \citet{czado2019analyzing} taken from \citet{stoeber2017pair} is $\mathcal{O}(d^2)$. This gives $\mathcal{O}(n d^2)$ for sampling $n$ observations.

\subsection{Computational Complexity of TVineSynth}

Regarding the computational complexity of TVineSynth the following points need to be considered:
\begin{enumerate}
    \item \textbf{TVineSynth is estimated only \textit{once}:}  Let $T \subset [d]$ be the subset of truncation levels considered. Then in a full run of TVineSynth, the C-vine is not estimated $|T|$ times on the real data, but only \textit{once} at the maximal desired truncation level $t_{max} := max T$ (may it be $t_{max} = d$, so no truncation or $t_{max} < d$). For all subsequent $t \in T \setminus \{t_{max}\}$ the vine copula is not re-estimated, but obtained by simply setting all pair copulas in tree levels $t’ \in T \setminus [t]$ to independence. The computational cost comes from fitting a C-vine once for $t_{max}$ and sampling the estimated C-vine at levels $t \in T$. The computational complexity of estimating and sampling from the C-vine is $\mathcal{O}(n d^2)$ each, see Appendices \ref{sec:compCost_EstVine} and \ref{sec:compCost_SamplingVine}.
    \item \textbf{AIA privacy and utility evaluation are cheap, MIA is expensive:} The computational costs of utility evaluation and AIA are unproblematic, since both are based on simple model architectures (e.g. linear regression). It is mainly the MIA that drives the computational cost of the privacy evaluation. However, for sufficiently large real data, the vine copula estimation is robust to adding/removing a single observation in the model estimation, as performed under MIA (see the MIA results in Sections \ref{sec:MIA_support2_results} and Appendix \ref{sec:simreal_results_MIA}). If the MIA PG of a C-vine truncated at level $t_{max}$ is (close to) 1 with little variation, then it will also be so for lower truncation levels. Therefore, the MIA privacy evaluation can be reduced to one truncation level.
    \item \textbf{Limited number of truncation levels considered:} As illustrated in the real data example it is not at all necessary to perform a privacy and utility evaluation for all possible truncation levels $t \in [d]$. It is sufficient to evaluate every 5th or 10th truncation level. In addition, tree levels that model pairwise (conditional) dependencies with a sensitive feature and all other features should not be considered. This means that we can set $t_{max} := d + 1 - j$  where $j$ is the position of the sensitive feature that enters the C-vine first according to order $\mathcal{O}^*$. So for $d = 26$ and $j = 6$ then $t_{max} := 21$. In sum only a limited set of candidate truncation levels $T \subset [t_{max}]$ has to be considered.
    \item \textbf{$\bm{t_{max} << d}$ for high-dimensional real data:} If the real data are high-dimensional, e.g. $d >400$, we recommend truncating the vine copula model early, because of the statistical uncertainty in the model estimation, so for example $t_{max} := 50$.
    \item \textbf{All competitors require human-in-the-loop tuning:} Finally, in TVineSynth there is only the truncation level to tune while for the competitor models, specifically CTGAN and TVAE, several hyperparameters need to be tuned (no. epochs, batch size, dimension of the latent space, \dots).
\end{enumerate}

\section{Graphical Illustration of TVineSynth}\label{sec:TVineSynth_workflow}

Figure \ref{fig:workflow_tvinesynth} depicts a graphical workflow of TVineSynth.
\begin{figure}
    \centering
    \includegraphics[width=\linewidth]{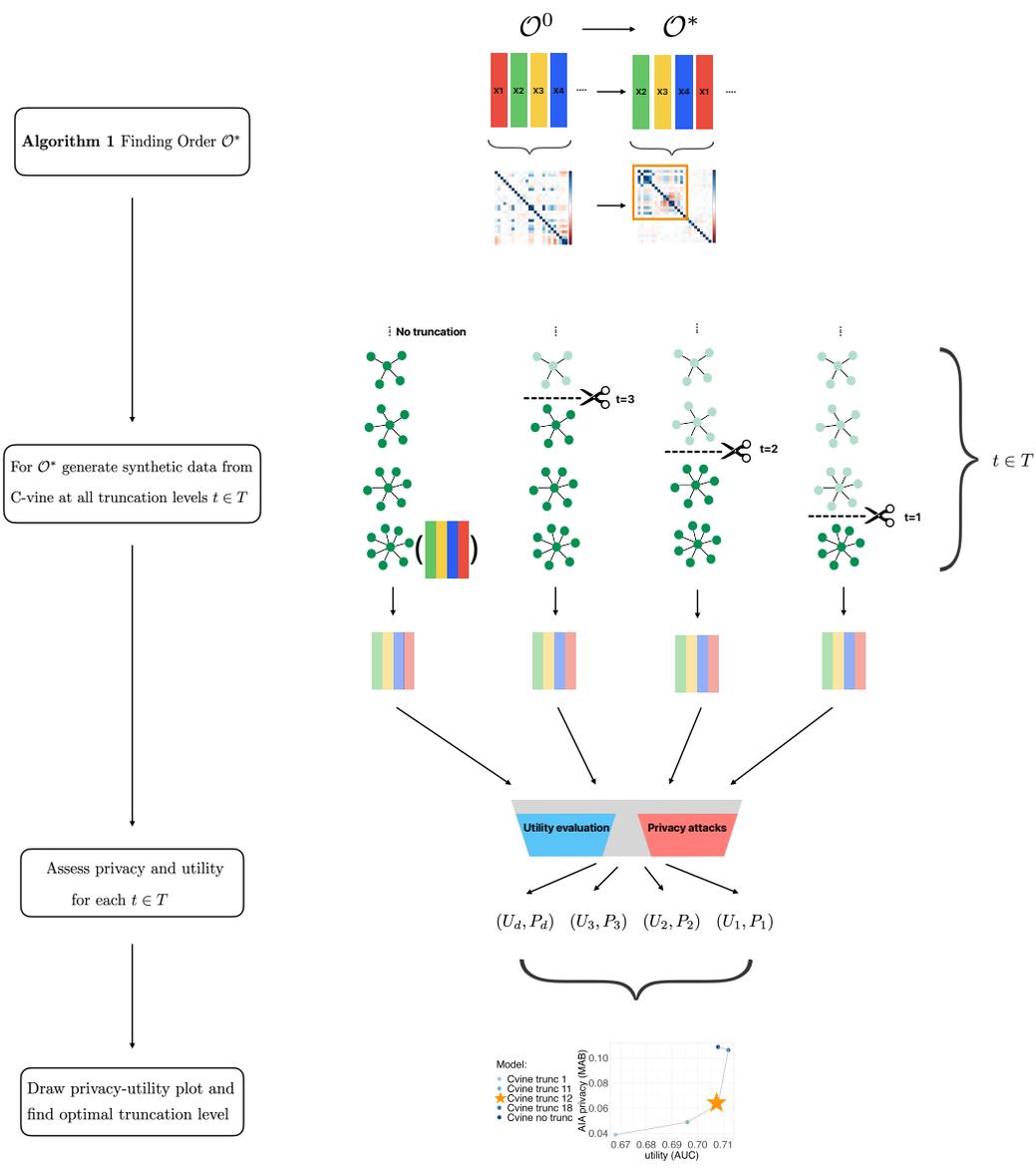}
    \caption{A graphical workflow of TVineSynth.}
    \label{fig:workflow_tvinesynth}
\end{figure}

\section{TVineSynth: Order of Covariates}\label{sec:orderRVs}

Let $X_1, ..., X_d$ be the covariates and $Y$ be the response in a prediction task. We propose Algorithm \ref{alg:find_order} to determine the order $\mathcal{O}^*$ in which the covariates enter the C-vine model, that balances the trade-off between protection against loss of privacy and utility. The order $\mathcal{O}^*$ together with the vine tree structure $\mathcal{V}$ determines the pair copulas and the tree levels they belong to in the C-vine. 

\begin{proposition}\label{prop:order_V}
    Let the order $\mathcal{O}^*$ of the covariates $X_k, \; k \in [d]$ and response $Y$ and let the vine tree structure $\mathcal{V}$ (e.g. in form of a R-vine matrix, see Appendix \ref{sec:RvinematrixCvineSimreal}) be given. Then $\mathcal{O}^*$ together with $\mathcal{V}$ determine which pairwise (conditional) dependencies of $X_k, \; k \in [d]$ and $Y$ are modeled in the C-vine copula.
\end{proposition}

\begin{proof}
    Given the order $\mathcal{O}^*$ and the vine tree structure $\mathcal{V}$, the C-vine is unique.
\end{proof}

This means that $\mathcal{O}^*$ and $\mathcal{V}$ determine which pairwise (conditional) dependencies are cut off from the model when truncating the C-vine at level $t$. Thus, the definition of $\mathcal{O}^*$ should be such that privacy leaking dependencies are cut off early while those important for the prediction task are cut off when truncating at a very low tree level.

\paragraph{Algorithm \ref{alg:find_order}} In any considered order $\mathcal{O}$ the response $Y$ is in the center of the star-shaped tree at level 1, the response $Y$ is placed at position $(d+1)$. First, we compute a matrix of pairwise association measures, using for example Pearson correlation or pairwise Kendall's $\tau$. Let $S \subset [d]$ be the set of sensitive covariates.
For any sensitive features $X_{j^*}, \; j^* \in S$ in turn, we find the covariates $X_k, \; k \in [d] \setminus S$ that show an absolute pairwise association $|\rho_{j^*,k}|$ above a user defined threshold $\rho^* > 0$ and denote their indices in the set $K_{j^*} \subset [d] \setminus S$. Let $K := \bigcup_{j^* \in S} K_{j^*}$.

The set $K$ depends on the threshold $\rho^*$ on the association measure, for example Pearson correlation greater than $\rho^* := 0.6$ in absolute value for all sensitive covariates (though one can use different thresholds for each sensitive covariate). The more conservative, i.e. lower we set this threshold, the more protected the sensitive covariates will be during truncation. Initializing the algorithm with different values of $\rho^*$ results in different orders $\mathcal{O}^*$ that can be compared through pairwise association or privacy plots on the synthetic data resulting from $\mathcal{O}^*$ at different truncation levels. 
 
The covariates $X_k, \; k \in K$ are the ones leaking most private information on the sensitive features $X_{j^*}, \; j^* \in S$. Consequently, disregarding those pairwise dependencies has the highest positive impact on privacy protection. For this reason $X_{j^*}, \; j^* \in S$ and $X_k, \; k \in K$ should  enter the C-vine copula in a group in the final trees, which are the ones that are truncated first. This means placing them on low indices in the order: for the permutation $\sigma: [d] \rightarrow [d]$ giving order $\mathcal{O}^*$ we have $\sigma(k) << d$ for $k \in K$. By grouping the sensitive covariate with the covariates informing it most in the order $\mathcal{O}^*$, we introduce a block structure in the matrix of pairwise association measures. As a result the pairwise (conditional) dependence between $X_{j^*}, \; j^* \in S$ and $X_k, \; k \in K$ is modeled in higher tree levels of the C-vine. Specifically, if the positions of $X_{j^*}$ and $X_k$ are $\sigma(j^*)$ and $\sigma(k)$ in $\mathcal{O}^*$ and w.l.o.g. we assume $\sigma(j^*)< \sigma(k)$ in $\mathcal{O}$, then their pairwise dependence conditioned on $X_{(\sigma(k) + 1)}, ... X_{(d)}$ is modeled in the $d + 2 - \sigma(k)$th tree in the C-vine. It can be truncated away with truncation level $d + 1 - \sigma(k)$ which will be moderate for $\sigma(k) << d$. Thus the pairwise (conditional) dependence between $X_{j^*}, \; j^* \in S$ and $X_k, \; k \in K$ is cut away with a low cost on utility. Results in Sections \ref{subsec:results_simulated_realId20} and \ref{subsec:results_support2} suggest that it suffices to enforce \textit{conditional} independence between $X_{j^*}, \; j^* \in S$ and $X_k, \; k \in K$ to achieve privacy protection. Finally, the appropriateness of the chosen order $\mathcal{O}^*$ is confirmed by plotting the matrix of pairwise association of the C-vine generated synthetic data: the correlation structure is more and more reproduced with increasing truncation level, see Figures \ref{fig:Cvine_synth_data_corr_real_data_I_d20} and \ref{fig:Cvine_synth_data_corr_rrealsupport2}. We summarize our procedure in Algorithm \ref{alg:find_order}.

We illustrate our algorithm with the SUPPORT2 example.  
Figure \ref{fig:support2small_Cvine_order_counterexample} displays the matrices of pairwise Kendall's $\tau$ for covariates of synthetic data generated by a C-vine. Here we use an order $\mathcal{O}_{\text{feature importance}}$ such that a covariate enters the C-vine the earlier the higher its feature importance is. As feature importance measure, we used the mean decrease Gini of a random forest classifier estimated on the real data. We observe that the structure of pairwise association of the real data is almost fully reproduced in the synthetic data already at truncation level 5 of the C-vine. This is a much lower truncation level than compared to when the covariates are ordered following the privacy preserving ordering, as from the Algorithm  \ref{alg:find_order}, see Figure \ref{fig:Cvine_synth_data_corr_rrealsupport2}. AIA results w.r.t. the sensitive feature \textit{totcst} on data ordered according to $\mathcal{O}_{\text{feature importance}}$ in Figure \ref{fig:reordered_support2small_AIA_totcst} further confirm that $\mathcal{O}_{\text{feature importance}}$ is inferior to the approach suggested above. When we compare the MAB in Figure \ref{fig:reordered_support2small_AIA_totcst} obtained from $\mathcal{O}_{\text{feature importance}}$ to the AIA results in Figure \ref{fig:AIA_Cvine_competitors_support2_small_MAB} obtained from an order $\mathcal{O}^*$ as proposed above, we notice that the MAB drastically increases already 10 trees levels earlier (at truncation level 5 as opposed to 15).

\begin{figure}[tb]
    \centering
    \begin{tabular}{cccc}
        \includegraphics[width=0.18\textwidth]{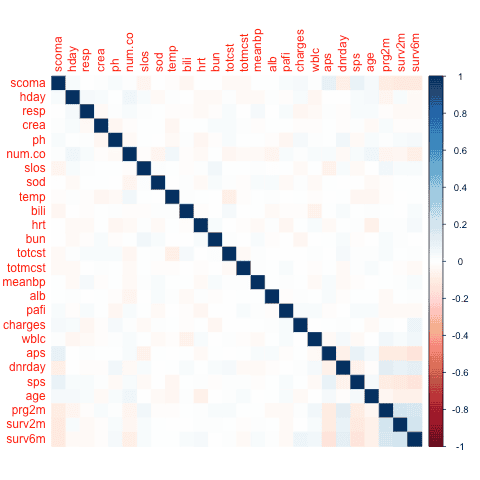} &  
        \includegraphics[width=0.18\textwidth]{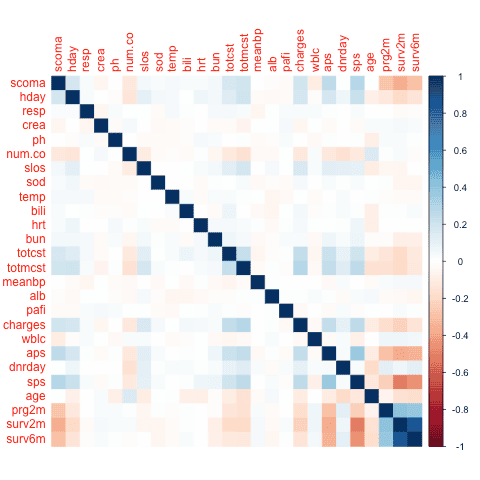} &
        \includegraphics[width=0.18\textwidth]{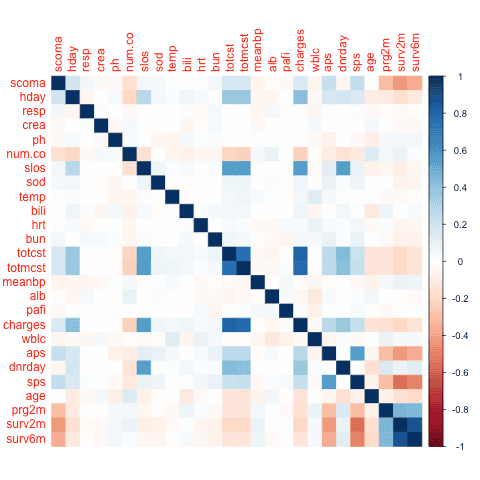} &
        \includegraphics[width=0.18\textwidth]{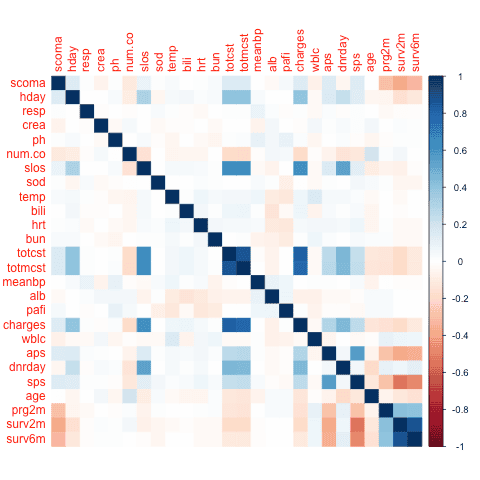}\\
        (a) Truncation at 1. & (b) Truncation at 5. & (c) Truncation at 10. & (d) Truncation at 15. \\
    \end{tabular}
    \begin{tabular}{ccc}
        \includegraphics[width=0.18\textwidth]{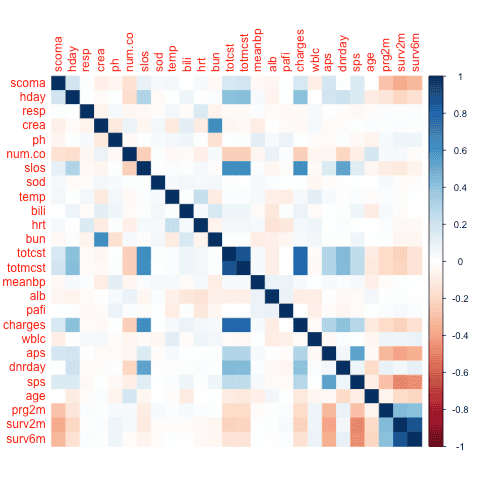} &
        \includegraphics[width=0.18\textwidth]{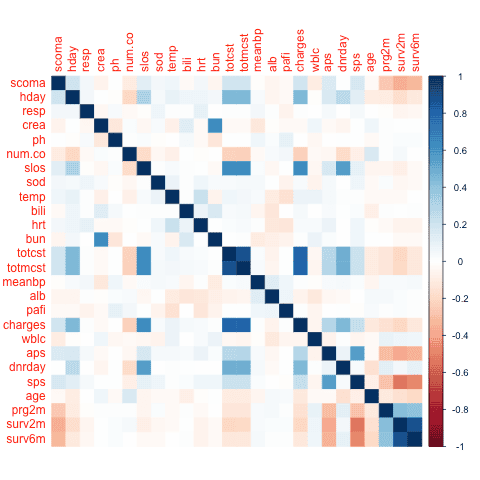} &
        \includegraphics[width=0.18\textwidth]{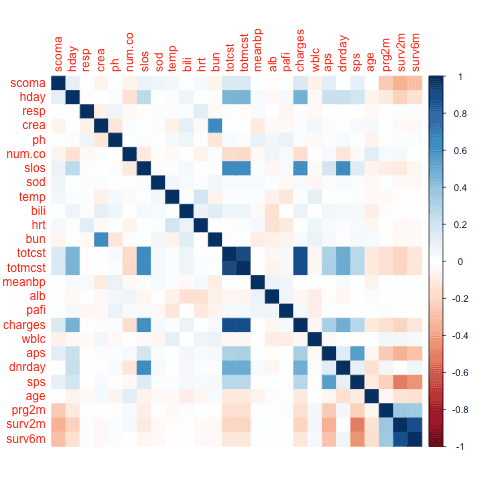} \\
        (e) Truncation at 20. & (f) No truncation. & (g) Real SUPPORT2 data.   \\
    \end{tabular}
    
    \caption{The matrices of pairwise Kendall's $\tau$ of continuous covariates in synthetic data generated with a C-vine for truncation at levels $t \in\{1,5,10,15,20\}$ and no truncation when the covariates in the real data are ordered according to $\mathcal{O}_{\text{feature importance}}$ based on feature importance (mean decrease Gini) in a random forest classifier trained on the real SUPPORT2 data. The structure of pairwise association of the real data is almost fully reproduced in the synthetic data already at truncation level 5 of the C-vine. Details on the estimation of the C-vine can be found in Section \ref{sec:TVineSynth_construction} and Appendix \ref{sec:model_and_attack_parameters}.}
    \label{fig:support2small_Cvine_order_counterexample}
\end{figure}

\begin{figure}[tb]
    \centering
     \includegraphics[width=0.7\columnwidth]{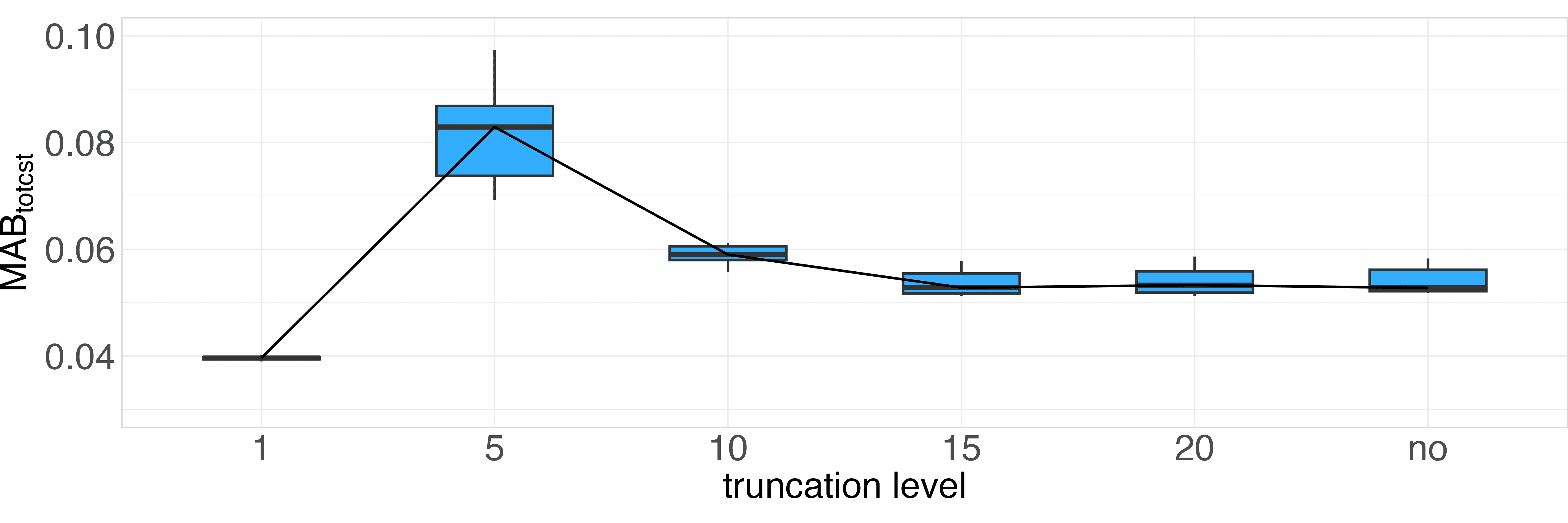}
    \caption{Results of an AIA w.r.t. an order $\mathcal{O}_{\text{feature importance}}$ of the covariates that is based on feature importance (mean decrease Gini) in a random forest classifier trained on the real SUPPORT2 data. The AIA is conducted w.r.t. sensitive covariate \textit{totcst} measured by $MAB_{totcst}$. Synthetic data are generated with a C-vine for different truncation levels. Results are reported as box plots over 10 AIA game iterations. Parameters of the generative models and privacy attacks can be found in Appendix \ref{sec:model_and_attack_parameters}.}\label{fig:reordered_support2small_AIA_totcst}
\end{figure}

\section{TVineSynth: Finding the Best Truncation Level As Optimization Task}\label{sec:TVineSynth_trunc_opt}

The selection of the truncation level $t \in T$ giving the best (in terms of the data holder's demands) privacy-utility balance can be optimized if the data holder has a way to place privacy and utility on the same scale. Let $P_t'$ be the privacy score (MIA or AIA) and $U_t'$ the utility score obtained from a C-vine with truncation level $t \in T$. Normalize $P_t'$ and $U_t'$ to obtain $P_t \in [0,1]$ and $U_t \in [0,1]$ where higher values for $P_t$ and $U_t$ correspond to better privacy and better utility respectively. Build a privacy-utility score $PU_t := \alpha P_t + (1 - \alpha) U_t$ where $\alpha \in [0,1]$ is user defined to trade-off between privacy and utility. Then finding the best truncation level is a discrete optimization problem in the truncation level $t$. This is not an easy optimization, as we lack convexity, but it is feasible.

\section{Limitations}\label{sec:limitations}
We summarize TVineSynth's main limitations:

\begin{itemize}
    \item \textbf{Privacy guarantees vs. empirical evaluation:} We provide an empirical privacy evaluation of TVineSynth and compare to competitor models. We explain how our algorithm aims to balance privacy and utility. However, TVineSynth does not offer theoretical privacy guarantees in the style of DP. The key idea of TVineSynth is to achieve privacy by introducing a targeted bias into the generative model instead of adding noise in a global fashion, which in many cases renders the synthetic data useless for downstream ML applications. For this reason we do not base the design of TVineSynth on DP. 
    \item \textbf{Choice of truncation level:} The choice of the level of truncation that provides the preferred balance between privacy and utility is deliberately left to the data owner's decision. This is to account for the fact that privacy requirements vary by context and application and need to be thoroughly weighed against utility demands by data holders, potential users and policy makers. Appendix \ref{sec:TVineSynth_trunc_opt} discusses how the choice of truncation level can be further automatized. 
    \item \textbf{Controlling privacy-utility trade-off:} It is hard to precisely control the privacy-utility trade-off of TVineSynth generated data with the truncation level of the C-vine. For making a well-informed choice the privacy and utility of the vine copula generated data should be evaluated at all truncation levels $t \in T$. Especially the MIA evaluation is costly. However, the candidate truncation levels $T$ can be chosen to minimize the computational cost: 
    \begin{itemize}
        \item Set $t_{max} := d + 1 - j$ or lower  where $j$ is the position of the sensitive feature that enters the C-vine first according to $\mathcal{O}^*$, as tree levels that model pairwise (conditional) dependencies with a sensitive feature and all other features should not be considered. For $d$ large, set $t_{max} << d$ due to uncertainty in the parameter estimation.
        \item It is not necessary to consider all $t < t_{max}$. Instead it suffices to for example set $T := \{1,5,10,15,25\}$ where $d = 30$. For chosen $T$ it is however necessary to evaluate the vine copula generated data at all $t \in T$.
    \end{itemize}
    \item \textbf{Computational complexity of TVineSynth and its limits for high-dimensional data:} Even though estimating and sampling from a C-vine has moderate computational complexity and TVineSynth is designed to find the optimal privacy-utility balance efficiently, see considerations in Appendix \ref{sec:scaling_vines}, evaluating the utility and especially (MIA) privacy of the synthetic data at truncation levels $t \in T$ is computationally demanding. The computational cost can be reduced by making $T$ smaller, but this gives a less nuanced picture of the privacy-utility trade-off. Additionally, estimating a vine copula on data with dimension $d > 500$ becomes computationally challenging. For such settings vine copulas have to be combined with dimension reduction techniques in TVineSynth.
\end{itemize}

\section{Competitor Models}\label{sec:competitors_appendix}

\subsection{Competitor Models}

We benchmark TVineSynth against the following competitor models:

\paragraph{Private Bayes (PrivBayes)}
\citet{zhang2017privbayes} propose a Bayesian network that satisfies DP guarantees. For a chosen $k$, they first construct a $k$-degree Bayesian network in an $\epsilon_1$-differentially private fashion by introducing a score function in the greedy Bayes algorithm. Then, they generate the conditional distributions corresponding to the Bayesian network by injecting Laplacian noise to obtain $\epsilon_2$-DP. The resulting Private Bayes model is $(\epsilon_1 + \epsilon_2)$-differentially private. We use the implementation provided in \citet{stadler2022synthetic}, which is patched to fulfill its differentially privacy guarantees. 

\paragraph{Private Particle Gradient Descent (PrivPGD)}
\citet{donhauser2024privacy} propose a differentially private marginal based generative model that utilizes particle gradient descent. After privately selecting which marginal distributions to estimate, the selected marginal distributions estimated on the real data are privatized with the Gaussian mechanism and transformed to a compact Euclidean space, the embedded space. In the embedded space particles are propagated such that their empirical marginal distribution minimizes the sliced Wasserstein distance, an optimal-transport based divergence, to the embedded privatized marginal distribution of the real data. Note that model estimation and data generation are done in one go and not in two separate steps. As a consequence, PrivPGD does not require model selection or parameter tuning and makes it robust to hyperparameter variation. At the same time it is not possible to sample additional data from PrivPGD once it was estimated, but the full model has to be run again. PrivPGD requires discrete input data and generates discrete synthetic data guaranteeing $(\epsilon, \delta)$-DP. This means that, if not discrete, the real data need to be discretized before inputting them into the model and the synthetic data need to be reverted to the original scale afterwards. In order to provide DP, either the real data need to be discrete themselves already or, if not the case, discretization and reversion need to be done in a DP manner. In their experiments on continuous real data \citet{donhauser2024privacy} solve the discretization by binning of the real data and the reversion by back-transforming the synthetic data to bin means of the real covariates. For this the covariate ranges are inferred from the real data and stored for the reversion process, through which PrivPGD loses its DP guarantees. 

In order to retain the DP guarantees also for continuous real data, we therefore infer covariate ranges from a source independent of the real data.

\paragraph{Conditional Tabular GAN (CTGAN) and Tabular Variational Autoencoder (TVAE)}
To tackle the numerous problems of tabular data when constructing generative adversarial networks (GANs), such as mixed data types, non-Gaussian and multimodal distributions and class imbalance for discrete covariates, \citet{xu2019modeling} propose the Conditional Tabular GAN (CTGAN). Starting from a GAN, the authors introduce a conditional generator with a modified loss function to account for imbalanced classes and mode-specific normalization to account for non-Gaussian and multimodal distributions.
The authors also introduce Tabular Variational Autoencoders (TVAE) by applying the same preprocessing and modified loss functions to a variational autoencoder. We use the wrapper provided in the \citetalias{sdv} around the implementation in the \citetalias{ctgan}.

\paragraph{R-Vine Copula}
In an R-vine copula \citep{joe1997multivariate, bedford2001probability, bedford2002vines, aas2009pair, joe2014dependence, czado2019analyzing} the vine tree structure is not pre-specified as in a C- or D-vine, but selected with the Dißmann's algorithm proposed by \citep{dissmann2013selecting}. An R-vine is therefore the most general and flexible class of vine copulas. \citet{meyer2021synthia} implement a python package for R-vine copula based synthetic data generation.

\subsection{Discussing the Choice of Competitor Models}
TVineSynth is compared to generative models that focus on preserving privacy of subjects in the real data by providing DP guarantees, and to generative models that focus on reproducing the underlying distribution of the real data closely without offering any formal privacy guarantees. We choose to compare TVineSynth with DP and non-DP competitor models in order to assess which generative model performs best in a context where privacy \textit{and} utility matter.

We compare TVineSynth with CTGAN and TVAE \citep{xu2019modeling}. They are well established and commonly used generative models for tabular data. CTGAN and TVAE do not offer formal privacy guarantees but focus on generating synthetic data that closely resemble the real data. 

PrivBayes is a DP generative model that belongs to the class of graphical probabilistic models as vine copulas do. For this reason we chose PrivBayes as a DP competitor to TVineSynth. Additionally, we chose PrvPGD DP as competitor model as it represents the state-of-the-art for private generative modeling.

Lastly, TVineSynth is compared to an R-vine copula, the most general and flexible vine copula. We do this to assess which impact setting the order of the covariates with Algorithm \ref{alg:find_order} and setting the vine tree structure to be a C-vine in TVineSynth has on the privacy and utility compared to when both are selected freely in an R-vine. 
TVineSynth is not compared to the copula-based approaches proposed by \citet{patki2016sdv}, \citet{kumi2023sleepsynth}, \citet{benali2021mtcopula}, \citet{kamthe2021copula} and \citet{chu2022statistical} as the latter belong to the same model class as R-vines, but are simpler, less flexible models. We do not compare TVineSynth with the models proposed by \cite{coblenz2023learning} and \citet{tagasovska2019copulas} as we are in a setting where dimension reduction using autoencoders is not necessary to enable modeling the data. For the model proposed by \citet{sun2019learning} there is no code available which prohibited a comparison with TVineSynth.

Future work could further compare TVineSynth to tabular denoising diffusion models proposed by \citet{kotelnikov2023tabddpm}.

\section{Measures of Privacy}\label{sec:privacy_measures}

\subsection{Privacy Gain (PG)}
As a measure of privacy preservation of the synthetic data \citet{stadler2022synthetic}
use the PG achieved when publishing a synthetic data set in place of the real given a
target observation. The PG is defined as the 'reduction in the attacker's advantage when given access
to the synthetic data instead of the real data': 

\begin{align}
    PG &:= Adv\big( (X, \bm{y}), (\bm{x}^T_t, y_t) \big) - Adv\big( (Z, \bm{w}), (\bm{x}^T_t, y_t) \big) \; ,
\end{align}
with target observation $(\bm{x}^T_t, y_t)$. For MIA the advantages $Adv^{MIA}(\cdot)$ from real and synthetic data are defined as:
\begin{align}
    Adv^{MIA}\big( (X, \bm{y}), (\bm{x}^T_t, y_t) \big) := P_R(\hat{s}_t = 1 | s_t = 1) - P_R(\hat{s}_t = 1 | s_t = 0) \; ,
\end{align}
and: 
\begin{align}
    Adv^{MIA}\big( (Z, \bm{w}), (\bm{x}^T_t, y_t) \big) := P_S(\hat{s}_t = 1 | s_t = 1) - P_S(\hat{s}_t = 1 | s_t = 0) \; ,
\end{align}
respectively, where:
\begin{align} \label{def:s_t}
    s_t := \begin{cases} 1 \, , & (\bm{x}^T_t, y_t) \; \text{is in} \;  (X, \bm{y}) \\
    0 \, & \text{else}
    \end{cases} \; .
\end{align}
The attacker's guess is $\hat{s}_t$ and $P_R$ ($P_S$) indicates that the attacker's guess is based on the real (synthetic) data. Obviously, $Adv^{MIA}\big( (X, \bm{y}), (\bm{x}^T_t, y_t) \big) = 1$ as the attacker can look up in the real data whether the target observation is present. Together with the theoretical bounds given in \citet{yeom2018privacy}: 
\begin{align}
    Adv^{MIA}\big( (Z, \bm{w}), (\bm{x}^T_t, y_t) \big) &\leq e^{\epsilon} - 1 \; ,
\end{align}
we get that for a differentially private generative model the center is bounded by:
\begin{align}
    PG^{MIA} &\geq 2 - e^{\epsilon} \; .
\end{align}
For AIA, \citet{stadler2022synthetic} define the advantages $Adv^{AIA}(\cdot)$ from real and synthetic data as:
\begin{align}\label{equ:AvdAIA_real}
    Adv^{AIA}\big( (X, \bm{y}), (\bm{x}^T_{t, -j^*}, y_t) \big) := P_R(\hat{x}_{t,j^*} = x_{t,j^*} | s_t = 1) - P_R(\hat{x}_{t,j^*} = x_{t,j^*} | s_t = 0) \; ,
\end{align}
and:
\begin{align}\label{equ:AvdAIA_synth}
    Adv^{AIA}\big( (Z, \bm{w}), (\bm{x}^T_{t, -j^*}, y_t) \big) := P_S(\hat{x}_{t,j^*} = x_{t,j^*} | s_t = 1) - P_S(\hat{x}_{t,j^*} = x_{t,j^*} | s_t = 0) \; ,
\end{align}
respectively, where $(\bm{x}^T_{t, -j^*}, y_t)$ is a sub-vector of $(\bm{x}^T_t, y_t)$ indicating that the sensitive feature value $x_{t,j^*}$ for some $j^* \in S \subset [d]$ of $(\bm{x}^T_t, y_t)$ is unknown to the attacker and $\hat{x}_{t,j^*}$ is the attacker's estimate of $x_{t,j^*}$.

\subsection{Mean Squared Error (MSE)}
The definition of $Adv^{AIA}(\cdot)$ by \citet{stadler2022synthetic} in Equations \ref{equ:AvdAIA_real} and \ref{equ:AvdAIA_synth} is correct for sensitive features $X_{j^*}$ taking on finitely many values. For the continuous case it is wrong, as $P(\hat{x}_{t,j^*} = x_{t,j^*} | s_t = s) = 0$ for any guess $\hat{x}_{t,j^*}$ and taking densities instead, as done in the implementation by \citet{stadler2022synthetic} provided on \citetalias{Stadlercode}, is also incorrect. \citet{olatunji2023does} instead suggest to compute the mean squared error (MSE) to asses the success of an AIA. It may be calculated by generating $K$ samples from the synthetic data generator $g$ and $K$ bootstrap samples of the real data, standardizing them by subtracting the mean and dividing by the standard deviation and then computing:
\begin{align}
    MSE_R(x_{t,j^*} | s_t = s) & := \frac{1}{K} \sum_{k = 1}^K \big(\hat{x}_{t,j^*}(R)^{(k)} - x_{t,j^*}\big)^2 \\
    MSE_S(x_{t,j^*} | s_t = s) & := \frac{1}{K} \sum_{k = 1}^K \big(\hat{x}_{t,j^*}(S)^{(k)} - x_{t,j^*}\big)^2
\end{align}

where $\hat{x}_{t,j^*}(S)^{(k)}$ is the attacker's guess based on the $k$th standardized synthetic data set sampled from the vine copula, $\hat{x}_{t,j^*}(R)^{(k)}$ is the attacker's guess based on the $k$th standardized bootstrap sample from the real data, $k \in [K]$ and $s_t$ is defined as in \eqref{def:s_t}.

However, the MSE gives an incomplete picture of a generative model's AIA privacy: A high MSE indicates that the attacker guesses a value $\hat{x}_{t,j^*}$ which is on average far from the actual sensitive feature value $x_{t,j^*}$ in squared error loss. Thus, privacy protection w.r.t. AIA is high. Concluding from a low MSE that the AIA privacy is low is however not generally correct. This is illustrated in the following example on simulated real data: Figure \ref{fig:AIA_Cvine_competitors_simreal_MSE} shows a low MSE for all sensitive covariates if the target observations are randomly sampled (in blue). This is because the randomly sampled target observations are closer to the center of the marginal distribution of the respective sensitive covariate $X_{j^*}$. We find that the attacker's guess is merely the mean of the corresponding sensitive covariate $X_{j^*}$. In terms of the attacker's regression model estimated on the synthetic data this means that the all the regression coefficients $X_k$ with $k \in [d] \setminus \{j^*\}$ are approximately 0. Hence, $X_k, \; k \in [d] \setminus \{j^*\}$ do not inform the sensitive covariate $X_{j^*}$ and the attacker learns no privacy leaking dependencies but only general statistics from the synthetic data.
This case therefore poses no privacy risk. See further details on the example in Appendix \ref{sec:appendix_simreal_AIA_additional_results_MSE}.

\subsection{Mean Absolute $\beta$-Coefficients (MAB)} \label{sec:def_MAB}

The previous example made clear that a low MSE does not necessarily indicate low AIA privacy. Instead we define the \textit{mean absolute $\beta$-coefficient (MAB)}. The definition of the MAB is based on the AIA game proposed by \citet{stadler2022synthetic}. There it is assumed that the attacker knows the generative model class used to generate synthetic data. In each of the $N$ game iterations the attacker gets access to a subsample of the real data of fixed size. On this subsample the attacker fits the generative model and generates $n_{synth}$ synthetic data sets. On each synthetic data set the attacker then estimates a regression model regressing the sensitive covariate on the non-sensitive covariates. Thus, in a whole AIA privacy evaluation for a fixed sensitive feature $X_{j^*}$ we obtain regression coefficients $\beta^{(j^*)}_{k, m, l}$ with $k \in [d]$ is the index of all other covariates/features (N.B.: The intercept term $\hat{\beta}^{(j^*)}_{0, m, l}$ is not included in the definition of the MAB.), $m \in [N]$ is the index of the game iteration, and $l \in [n_{synth}]$  runs over all generated synthetic data sets. We then define the MAB as in Definition \ref{def:MAB}:
\begin{align}\tag{3}
    MAB_{j^*} := \frac{1}{d N n_{synth}} \sum_{k \in [d]} \sum_{m \in [N]} \sum_{l \in [n_{synth}]}  | \hat{\beta}^{(j^*)}_{k, m, l} | \; .
\end{align}

Lower values for MAB indicate that covariates $X_k, \; k \neq j^*$ inform the sensitive covariate $X_{j^*}$ less. As opposed to the MSE, the definition in Equation \ref{def:MAB} is independent of a target observation and thus quantifies AIA privacy in terms of the generative model.

We discuss how the MAB might behave in the case of collinearity of the covariates $X_k, \; k \neq j^*$. Then we could encounter a scenario in which the MAB \textit{and} the MSE take on a high values. A high MAB value lets us conclude that the covariates $X_k, \; k \neq j^*$ inform the sensitive feature $X_{j^*}$ well hinting on privacy leakage, while a high MSE on the contrary indicates that the attacker's guess is on average far from the actual sensitive feature value in squared error loss. Speaking in hypothesis testing terms this case represents a type II, where the MAB indicates privacy leakage when in fact the MSE confirms that there is not.

\subsection{Worst-Case Absolute $\beta$-Coefficients (WCAB)} \label{sec:def_WCAB}

Exchanging the mean in the MAB with the maximum we obtain the worst-case absolute $\beta$-coefficients (WCAB):

\begin{align}\label{equ:def_WCAB}
    WCAB_{j^*} := \max \{ | \hat{\beta}^{(j^*)}_{k, m, l} |: \; k \in [d], \; m \in [N], \; l \in [n_{synth}] \} \; .
\end{align}

The WCAB gives a worst-case evaluation of the AIA privacy for all individuals following the idea of the worst-case guarantees provided by DP \citep{dwork2014algorithmic}.

\subsection{Mean $R^2$ (MR2)}
Finally, the degree of privacy required has to be decided by the data holder and varies from application to application. The MAB uses estimated $\beta$-coefficients of each feature in the relevant regression model and has therefore a scale which is difficult to interpret. Instead the $R^2$ can be used, which gives the percentage of variance explained by a regression model and is therefore more interpretable. Like for the MAB, an average over the $R^2$ values in all performed regressions can be computed and we call it MR2, Mean $R^2$. It can be shown that regression coefficients of features entering the C-vine late start to vanish with increasing truncation. Therefore the number of degrees of freedom in the regression model varies for different truncation levels of the C-vine. This requires adjusting the MR2 for different number of degrees of freedom according to the truncation level and makes the MR2 harder to compare accross generative models. For this reason we focus on the MAB instead of the MR2.

\section{Proofs of Theoretical Results Concerning the Utility and Privacy of TVineSynth}\label{sec:proofs_TVineSynth}

\begin{proof}[Proof of Theorem \ref{thm:utility}]
For the log-odds ratio, we have:
\begin{align}
    P(Y=1|\bm{X} &=\bm{x}) = \frac{\pi_{Y}f_{\bm{x}|y}(\bm{x}|Y=1;\bm{\theta})}{\pi_{Y}f_{\bm{x}|y}(\bm{x}|Y=1;\bm{\theta})+(1-\pi_{Y})f_{\bm{x}|y}(\bm{x}|Y=0;\bm{\theta})}     
\end{align}
and:

\begin{align}
    P(Y=0|\bm{X} &= \bm{x}) = \frac{(1-\pi_{Y})f_{\bm{x}|y}(\bm{x}|Y=0;\bm{\theta})}{\pi_{Y}f_{\bm{x}|y}(\bm{x}|Y=1;\bm{\theta})+(1-\pi_{Y})f_{\bm{x}|y}(\bm{x}|Y=0;\bm{\theta})} \; ,    
\end{align}

so that:

\begin{align}
    \psi(\pi_{Y},\bm{\theta}_{1},\ldots,\bm{\theta}_{d};\bm{x}) &= \log\frac{\pi_{Y}}{1-\pi_{Y}} + \log\frac{f_{\bm{x}|y}(\bm{x}|Y=1;\bm{\theta})}{f_{\bm{x}|y}(\bm{x}|Y=0;\bm{\theta})} \; .    
\end{align}

Further, we have:
\begin{align}
    f_{\bm{x}|y}(\bm{x}|y) &= f_{d|y}(x_{d}|y)\cdot f_{d-1|d,y}(x_{d-1}|x_{d},y)\cdot\ldots\cdot f_{1|2,\ldots,d,y}(x_{1}|x_{2},\ldots,x_{d},y)    
\end{align}

where, omitting arguments for simplicity:
\begin{align}
    f_{d-1|d,y} &= \frac{f_{d-1,d|y}}{f_{d|y}} = \frac{c_{d-1,d;y}f_{d-1|y}f_{d|y}}{f_{d|y}} = c_{d-1,d;y}f_{d-1|y} \\
    f_{d-2|d-1,d,y} &= \frac{f_{d-2,d-1|d,y}}{f_{d-1|d,y}} = \frac{c_{d-2,d-1;d,y}f_{d-2|d,y}f_{d-1|d,y}}{f_{d-1|d,y}} = c_{d-2,d-1;d,y}f_{d-2|d,y},
\end{align}

where, correspondingly to $f_{d-1|d,y}$, we obtain $f_{d-2|d,y}=c_{d-2,d;y}f_{d-2|y}$, so that:

\begin{align}
    f_{d-2|d-1,d,y} &= c_{d-2,d-1; d,y} c_{d-2,d;y}f_{d-2|y} \; .
\end{align}

Further: 

\begin{align}
    f_{d-3|d-2,d-1,d,y} &= \frac{f_{d-3,d-2|d-1,d,y}}{f_{d-2|d-1,d,y}} = \frac{c_{d-3,d-2;d-1,d,y} f_{d-3|d-1,d,y} f_{d-2|d-1,d,y}}{f_{d-2|d-1,d,y}} \\
    &= c_{d-3,d-2;d-1,d,y}f_{d-3|d-1,d,y} \; ,
\end{align}

where, correspondingly to $f_{d-2|d-1,d,y}$, we obtain 
$f_{d-3|d-1,d,y}=c_{d-3,d-1;d,y}c_{d-3,d;y}f_{d-3|y}$, so
that:

\begin{align}
    f_{d-3|d-2,d-1,d,y} &= c_{d-3,d-2;d-1,d,y}c_{d-3,d-1; d,y}c_{d-3,d;y}f_{d-3|y} \; .   
\end{align}

Continuing this we obtain:

\begin{align}
    f_{1|2,\ldots,d,y} &= c_{1,2;3,\ldots,d,y}\cdot\ldots\cdot c_{1,d;y}f_{1|y} \; .
\end{align}

Hence,

\begin{align}
    f_{\bm{x}|y}(\bm{x}|y) &= \prod_{j=1}^{d}f_{j|y} \prod_{t=2}^{d}\prod_{j=1}^{d+1-t}c_{j,d+2-t;d+3 - t, \ldots ,d,y} \; ,    
\end{align}

which is 
so that:

\begin{align}
    \psi &= \log\frac{\pi_{Y}}{1-\pi_{Y}} + \sum_{j=1}^{d}\log\frac{f_{j|y}(x_{j}|1)}{f_{j|y}(x_{j}|0)} + \sum_{t=2}^{d} \sum_{j=1}^{d + 1 -t}\log\frac{c_{j,d+2-t; d+3-t , \ldots,  d,y}^{1}}{c_{j,d+2-t; d+3-t , \ldots,  d,y}^{0}} = \sum_{t=1}^{d} \psi_{t} \; ,   
\end{align}

where $c_{j,d+2-t;d+3-t, \ldots,  d,y}^{k}$ is evaluated at $(\bm{x},y)=(\bm{x},k)$, with:

\begin{align}
    \psi_{1} &= \log\frac{\pi_{Y}}{1-\pi_{Y}}+\sum_{j=1}^{d}\log\frac{f_{j|y}(x_{j}|1)}{f_{j|y}(x_{j}|0)}
\end{align}

and:

\begin{align}
    \psi_{t} &= \sum_{j=1}^{d+1-t}\log\frac{c_{j,d+2-t; d+3-t , \ldots,  d,y}^{1}}{c_{j,d+2-t; d+3-t , \ldots,  d,y}^{0}} \; , \quad  t \in \{2,\ldots,d \} \; .
\end{align}

When truncating the C-vine at level $\tau$, then:

\begin{align}
    f_{\bm{x}|y}(\bm{x}|y) &= \prod_{j=1}^{d}f_{j|y} \prod_{t=2}^{\tau} \prod_{j=1}^{d+1-t} c_{j,d+2-t;d+3-t, \ldots ,d,y} \; ,    
\end{align}

so that the corresponding log-odds ratio is given by $\psi^{\tau}=\sum_{t=1}^{\tau}\psi_{t}$.

Further:
\begin{align}
    f_{j|y}(x_{j}|y) &= \frac{\partial}{\partial x_{j}}F_{j|y}(x_{j}|y) = \frac{\partial}{\partial x_{j}}\frac{P(X_{j}\leq x_{j}, Y=y)}{P(Y=y)} \\
    &= \frac{\partial}{\partial x_{j}} \frac{P(X_{j}\leq x_{j},Y \leq y) - P(X_{j}\leq x_{j}, Y \leq y-1)}{P(Y=y)} \\
    &= \frac{\partial}{\partial x_{j}}\frac{C_{j,y}\big(F_{j}(x_{j}),F_{Y}(y)\big) - C_{j,y}\big(F_{j}(x_{j}), F_{Y}(y-1)\big)}{P(Y=y)} \\
    &= \frac{1}{P(Y=y)} \Big( h_{y|j} \big( F_{Y}(y)|F_{j}(x_{j}) \big) - h_{y|j} \big( F_{Y}(y-1)|F_{j}(x_{j}) \big) \Big) f_{j}(x_{j} ) \; ,
\end{align}

where $h_{y|j} = \frac{\partial C_{j,y}}{\partial F_{j}(x_{j})}$, and since $X_{j} \sim U(0,1)$, for $j \in [d]$:

\begin{align}
    f_{j|y}(x_{j}|y) &= \begin{cases}
    \frac{1}{1-\pi_{Y}} h_{y|j}\big(1-\pi_{Y}|F_{j}(x_{j})\big) \; , & y=0\\
    \frac{1}{\pi_{Y}} \Big(1-h_{y|j} \big(1-\pi_{Y}|F_{j}(x_{j} \big) \Big) \; , & y=1 \; ,
    \end{cases} 
\end{align}

and the arguments of the pair copulas are given by \citep{joe1997multivariate}:

\begin{align}
    &F_{k|d+2-t,\ldots,d,y}(x_{k}|x_{d+2-t},\ldots, x_d,y) = \\
    &\quad  \frac{\partial C_{k,d+2-t; d+3 -t,\ldots,d,y} \big(u_{k|d+3-t,\ldots, d,y}, F_{d+2-t|d+3-t,\ldots,d,y}(x_{d+2-t}|x_{d+3-t}, \ldots, x_d,y) \big)}{\partial F_{d+2-t|d+3-t, \ldots, d,y}(x_{d+2-t}|x_{d+3-t}, \ldots, x_d,y)} \; ,
\end{align}

for $k \in [d+1 - t]$ and $t \in [d]$ with $u_{k|d+3-t,\ldots, d,y} := F_{k|d+3-t,\ldots, d,y}(x_{k}|x_{d+3-t},\ldots, x_d,y)$.

We see that distributions $f_{j|y}$ depend on $\pi_{Y}$, but also on the parameter of the copula $C_{j,y}$ of the first tree of the C-vine. This means that the first term $\psi_{1}$ of the log-odds ratio depends on $\pi_{Y}$ and the copula parameters $\bm{\theta}_{1}$ of the first tree. The remaining terms $\psi_{t}$, $t \in \{2,\ldots, d \}$ are functions of the pair copulas in trees $2$ to $d$, that have conditional distributions as arguments, which are computed recursively, as shown above. This means that $\psi_{t}$ depends on the parameters $\bm{\theta}_{t}$ of the pair copulas of tree number $t$, but also on the parameters $\bm{\theta}_{1},\ldots,\bm{\theta}_{t-1}$ from the previous trees, as well as $\pi_{Y}$, though the recursion.

Under the usual regularity assumptions, consult for instance \cite{lehmann1999}, we have for large $n$ that the maximum likelihood  estimator is:

\begin{align}
    \begin{pmatrix}
    \hat{\pi}_{Y}\\ \hat{\bm{\theta}}
    \end{pmatrix} = \begin{pmatrix}
    \pi_{Y}\\ \bm{\theta}
    \end{pmatrix}+\bm{J}^{-1}\bar{\bm{U}}_{n} + \smallO_{P}\left(\frac{1}{\sqrt{n}}\right) \; ,
\end{align}

where:

\begin{align}
    \bar{\bm{U}}_{n} &= \frac{1}{n} \sum_{i=1}^{n} \frac{\partial}{\partial(\pi_{Y},\bm{\theta}_{1},\ldots,\bm{\theta}_{d}) } \log f(\bm{X}_{i},Y_{i})
\end{align}

and:

\begin{align}
    \sqrt{n}\bar{\bm{U}}_{n} &\overset{d}{\rightarrow} \mathcal{N}_{|\bm{\theta}|+1}(\bm{0},\bm{J})    
\end{align}

and:

\begin{align}
    \bm{J} &= - E \Big[\frac{\partial^{2}}{\partial(\pi_{Y},\bm{\theta}_{1},\ldots,\bm{\theta}_{d})\partial(\pi_{Y},\bm{\theta}_{1},\ldots,\bm{\theta}_{d})^{T}}\log f(\bm{X},Y) \Big] \; ,
\end{align}

and the delta method gives:

\begin{align}
    \hat{\psi} &= \psi(\hat{\pi}_{Y},\hat{\bm{\theta}}_{1},\ldots,\hat{\bm{\theta}}_{d};\bm{x}) \\
    &= \psi(\pi_{Y},\bm{\theta}_{1},\ldots,\bm{\theta}_{d};\bm{x}) + \frac{1}{\sqrt{n}}\bm{v}^{T}\bm{J}^{-1}\sqrt{n}\bm{\bar{U}}_{n} + \smallO_{P}\left(\frac{1}{\sqrt{n}}\right),
\end{align}

where $\bm{v}=\frac{\partial \psi}{\partial(\pi_{Y},\bm{\theta}_{1},\ldots,\bm{\theta}_{d})}$. 

Hence, for large $n$:

\begin{align}
    MSE(\hat{\psi}) &= E \big[ (\hat{\psi}-\psi)^{2} \big] = \frac{1}{n}\cdot \bm{v}^{T}\bm{J}^{-1}\bm{v} + \smallO\left(\frac{1}{n}\right) \; .  
\end{align}

Further, when we truncate the C-vine at level $\tau \leq d-1$, we simply set all pair copulas from level $\tau+1$ to $d$ to independence, but the models parameters $(\hat{\pi}_{Y},\hat{\bm{\theta}}_{1},\ldots,\hat{\bm{\theta}}_{\tau})$ of the truncated model are not re-estimated. Thus:

\begin{align}
    \tilde{\psi}^{\tau} &= \sum_{t=1}^{\tau}\psi_{t}(\hat{\pi}_{Y},\hat{\bm{\theta}}_{1},\ldots,\hat{\bm{\theta}}_{t};\bm{x}) \\
    &= \sum_{t=1}^{\tau}\psi_{t}(\pi_{Y},\bm{\theta}_{1},\ldots,\bm{\theta}_{t};\bm{x}) + \frac{1}{\sqrt{n}}\left(\frac{\partial \sum_{t=1}^{\tau}\psi_{t}}{\partial (\pi_{Y},\bm{\theta}_{1},\ldots,\bm{\theta}_{d})}\right)^{T}\bm{J}^{-1}\sqrt{n}\bm{U}_{n} + \smallO_{P}\left(\frac{1}{\sqrt{n}}\right) \\
    &= \psi + \left(\sum_{t=1}^{\tau}\psi_{t}(\pi_{Y},\bm{\theta}_{1},\ldots,\bm{\theta}_{t};\bm{x})-\psi\right) \\
    & + \frac{1}{\sqrt{n}}\begin{pmatrix}
    \bm{v}^{1\ldots \tau} \\
    \bm{0}
    \end{pmatrix}^{T} \begin{pmatrix}
    \bm{J}^{1\ldots \tau,1\ldots \tau} & \bm{J}^{1\ldots \tau, \tau+1 \ldots d} \\
    \bm{J}^{\tau+1 \ldots d,1\ldots \tau} & \bm{J}^{\tau + 1 \ldots d,\tau+1\ldots d}
    \end{pmatrix}\sqrt{n}\bm{U}_{n} + \smallO_{P}\left(\frac{1}{\sqrt{n}}\right) \; ,
\end{align}

where:

\begin{align}
    \bm{v}^{1\ldots \tau} &= \frac{\partial}{\partial (\pi_{Y},\bm{\theta}_{1},\ldots,\bm{\theta}_{\tau})} \sum_{t=1}^{\tau}\psi_{t}
\end{align}

and the diagonal blocks $\bm{J}^{1\ldots \tau,1\ldots \tau}$ and $\bm{J}^{\tau+1\ldots d,\tau+1\ldots d}$ of $\bm{J}^{-1}$ correspond to the double derivatives with  respect to $(\pi_{Y},\bm{\theta}_{1},\ldots,\bm{\theta}_{\tau})$ and $(\bm{\theta}_{\tau+1},\ldots,\bm{\theta}_{d})$, respectively, and the off-diagonal blocks $\bm{J}^{1\ldots \tau,\tau+1\ldots d}$ and $\bm{J}^{\tau+1\ldots d,1\ldots \tau}$ to the derivative with respect to $(\pi_{Y},\bm{\theta}_{1},\ldots,\bm{\theta}_{\tau})$ and $(\bm{\theta}_{\tau+1},\ldots,\bm{\theta}_{d})$. This means that for large $n$:

\begin{align}
    MSE(\tilde{\psi}^{\tau}) &= \left(\sum_{t=1}^{\tau} \psi_{t}(\pi_{Y},\bm{\theta}_{1},\ldots,\bm{\theta}_{i};\bm{x})-\psi\right)^{2}+\frac{1}{n}\cdot \left(\bm{v}^{1\ldots \tau}\right)^{T}\bm{J}^{1\ldots \tau,1\ldots \tau}\bm{v}^{1\ldots \tau} + \smallO\left(\frac{1}{n}\right) \\
    &= \left(\sum_{t=\tau+1}^{d} \psi_{t}(\pi_{Y},\bm{\theta}_{1},\ldots,\bm{\theta}_{i};\bm{x})\right)^{2}+\frac{1}{n}\cdot \left(\bm{v}^{1\ldots \tau}\right)^{T}\bm{J}^{1\ldots \tau,1\ldots \tau}\bm{v}^{1\ldots \tau} + \smallO\left(\frac{1}{n}\right) \; .
\end{align}

\end{proof}

\begin{proof}[Proof of Theorem \ref{thm:mab1}]
Since the rows of $\bm{V}$ are independent and follow a standard $(d+1)$-variate 
normal distribution with correlation matrix $\bm{\rho}$, we know that for each
row $i$:

\begin{align}
    &V_{ij^*}|\bm{V}_{i,[d+1]\setminus\{j^*\}}=\bm{v}_{i,[d+1]\setminus\{j^*\}}\\ 
    &\quad \sim \mathcal{N}(\bm{\rho}_{[d+1] \setminus \{j^*\},j^*}^{T}\bm{\rho}_{[d+1]\setminus \{j^*\},[d+1]\setminus \{j^*\}}^{-1}\bm{v}_{i,[d+1]\setminus\{j^*\}},\\
    &\qquad \quad \quad 1-\bm{\rho}_{[d+1]\setminus \{j^*\},j^*}^{T}\bm{\rho}_{[d+1]\setminus \{j^*\},[d+1]\setminus \{j^*\}}^{-1}\bm{\rho}_{[d+1]\setminus \{j^*\},j^*}) \; ,
\end{align}

so that we may write:

\begin{align}
    V_{ij^*} &= \bm{\rho}_{[d+1] \setminus \{j^*\},j^*}^{T}\bm{\rho}_{[d+1]\setminus \{j^*\},[d+1]\setminus \{j^*\}}^{-1}\bm{v}_{i,[d+1]\setminus\{j^*\}} + \varepsilon_{i} = (\bm{\beta}^{(j^*)})^{T}\bm{v}_{i,[d+1]\setminus\{j^*\}} + \varepsilon_{i},    
\end{align}

with $\varepsilon_{i}\overset{iid}{\sim} \mathcal{N}(0,(\sigma^{(j^*)})^{2})$ and:

\begin{align} \label{eq:sigma2}
    (\sigma^{(j^*)})^{2} &= 1-\bm{\rho}_{[d+1]\setminus \{j^*\},j^*}^{T}\bm{\rho}_{[d+1]\setminus \{j^*\},[d+1]\setminus \{j^*\}}^{-1}\bm{\rho}_{[d+1]\setminus \{j^*\},j^*} \; .    
\end{align}

Then,
it follows from the properties of the ordinary least squares estimator that $\hat{\bm{\beta}}^{(j^*)}$ follows a $d$-variate normal distribution with mean:

\begin{align}
    \bm{\beta}^{(j^*)} &= \bm{\rho}_{[d+1]\setminus \{j^*\},[d+1]\setminus \{j^*\}}^{-1}\bm{\rho}_{[d+1] \setminus \{j^*\},j^*} \; , \label{eq:beta_calc}
\end{align}

and covariance matrix:

\begin{align}
    (\sigma^{(j^*)})^{2}(\bm{V}_{[d+1]\setminus \{j^*\}}^{T}\bm{V}_{[d+1]\setminus \{j^*\}})^{-1} \; .
\end{align}

Now, assume first that the C-vine is truncated at level $\tau = d+1-j^*$. This means that all pair copulas in tree levels $t \in \{d+2-j^*,\ldots,d\}$ are set to independence, and since they are all Gaussian, this is the same as setting the corresponding partial correlations to $0$, i.e.:

\begin{align}
    \rho_{12\cdot 3\ldots d+1}=\rho_{13\cdot 4\ldots d+1}=\rho_{23\cdot 4 \ldots d+1} = \ldots = \rho_{1 j^* \cdot j^* + 1 \ldots d+1} = \ldots = \rho_{j^*-1,j^*\cdot j^*+1 \ldots d+1}=0 \; .
\end{align}

These partial correlations may be expressed in terms of the partial variance-covariance matrix (consult for instance \cite{baba2004partial}). For this, let $k \in [j-1]$ for some $j \in \{2, ..., j^*\}$. Then partial variance-covariance matrix is given by:

\begin{align}
    &\bm{\rho}_{kj\cdot j+1\ldots d+1}\\ 
    &= \begin{pmatrix}
    a_{11} & a_{12}\\ a_{12} & a_{22}
    \end{pmatrix}\\  
    &= \begin{pmatrix}
    1 & \rho_{kj}\\ \rho_{kj} & 1
    \end{pmatrix} - \begin{pmatrix}
    \bm{\rho}_{j+1\ldots d+1,k}^{T}\\ \bm{\rho}_{j+1\ldots d+1,j}^{T}
    \end{pmatrix} \bm{\rho}_{j+1\ldots d+1,j+1\ldots d+1}^{-1} \begin{pmatrix}
    \bm{\rho}_{j+1\ldots d+1,k} & \bm{\rho}_{j+1\ldots d+1,j}\end{pmatrix} \\
    &= {\begin{pmatrix}\scriptscriptstyle
        1- \bm{\rho}_{j+1\ldots d+1,k}^{T}\bm{\rho}_{j+1\ldots d+1,j+1\ldots d+1}^{-1}\bm{\rho}_{j+1\ldots d+1,k}& \scriptscriptstyle\rho_{kj}-\bm{\rho}_{j+1\ldots d+1,k}^{T}\bm{\rho}_{j+1\ldots d+1,j+1\ldots d+1}^{-1}\bm{\rho}_{j+1\ldots d+1,j}\\ \scriptscriptstyle\rho_{kj}-\bm{\rho}_{j+1\ldots d+1,j}^{T}\bm{\rho}_{j+1\ldots d+1,j+1\ldots d+1}^{-1}\bm{\rho}_{j+1\ldots d+1,k} & \scriptscriptstyle1-\bm{\rho}_{j+1\ldots d+1,j}^{T}\bm{\rho}_{j+1\ldots d+1,j+1\ldots d+1}^{-1}\bm{\rho}_{j+1\ldots d+1,j}
    \end{pmatrix}} \; .
\end{align}

The partial correlation is then:

\begin{align}\label{eq:equality0}
    &\rho_{kj\cdot j+1\ldots d+1} = \frac{a_{12}}{\sqrt{a_{11}a_{22}}}\\
    &= \scriptscriptstyle\frac{\rho_{kj}-\bm{\rho}_{j+1\ldots d+1,k}^{T}\bm{\rho}_{j+1\ldots d+1,j+1\ldots d+1}^{-1}\bm{\rho}_{j+1\ldots d+1,j}}{\sqrt{(1- \bm{\rho}_{j+1\ldots d+1,k}^{T}\bm{\rho}_{j+1\ldots d+1,j+1\ldots d+1}^{-1}\bm{\rho}_{j+1\ldots d+1,k})(1-\bm{\rho}_{j+1\ldots d+1,j}^{T}\bm{\rho}_{j+1\ldots d+1,j+1\ldots d+1}^{-1}\bm{\rho}_{j+1\ldots d+1,j})}}\\
    &= 0 \; ,
\end{align}

which is equivalent to the numerator being 0, i.e.:

\begin{align}\label{eq:equality}
    \rho_{kj} &= \bm{\rho}_{j+1\ldots d+1,k}^{T}\bm{\rho}_{j+1\ldots d+1,j+1\ldots d+1}^{-1}\bm{\rho}_{j+1\ldots d+1,j}  \; . 
\end{align} 

This holds specifically for $j = j^*$ and any $k \in [j^* - 1]$, hence:

\begin{align}\label{eq:equals0}
    \bm{\rho}_{1\ldots j^*-1,j^*} -\bm{\rho}_{j^*+1\ldots d+1,1\ldots j^*-1}^{T}\bm{\rho}_{j^*+1\ldots d+1,j^*+1\ldots d+1}^{-1}\bm{\rho}_{j^*+1\ldots d+1,j^*} = \bm{0}  \; .
\end{align}
Further, if we express:

\begin{align}
    \bm{\rho}_{[d+1]\setminus \{j^*\},[d+1]\setminus \{j^*\}}
    &= \begin{pmatrix}
        \bm{\rho}_{1\ldots j^*-1,1\ldots j^*-1} & \bm{\rho}_{j^*+1\ldots d+1,1\ldots j^*-1}^{T}\\
        \bm{\rho}_{j^*+1\ldots d+1,1\ldots j^*-1} & \bm{\rho}_{j^*+1\ldots d+1,j^*+1\ldots d+1}
    \end{pmatrix}
    = \begin{pmatrix}
    \bm{B} & \bm{C}^{T}\\
    \bm{C} & \bm{D}
    \end{pmatrix} \; ,
\end{align}

then we have:

\begin{align}\label{eq:reexpress_rho-1}
    \bm{\rho}_{[d+1]\setminus \{j^*\},[d+1]\setminus \{j^*\}}^{-1}	= &\begin{pmatrix}
        \bm{M}^{-1} & -\bm{M}^{-1}\bm{C}^{T}\bm{D}^{-1}\\
        -\bm{D}^{-1}\bm{C}\bm{M}^{-1} & \bm{D}^{-1}+\bm{D}^{-1}\bm{C}\bm{M}^{-1}\bm{C}^{T}\bm{D}^{-1} 
    \end{pmatrix} \; ,
\end{align}

with $\bm{M}=\bm{B}-\bm{C}^{T}\bm{D}^{-1}\bm{C}$. We use this to obtain that the regression coefficient of $V_{j^*}$ when the C-vine is truncated at level $\tau$:

\begin{align}
    &\bm{\beta}^{(j^*)}_{(\tau)} = \begin{pmatrix}
        \bm{\beta}_{(\tau) \; 1\ldots d - \tau}^{(j^*)}\\
        \bm{\beta}_{(\tau) \; d+1 - \tau \ldots d}^{(j^*)}
    \end{pmatrix}  \stackrel{\eqref{eq:beta_calc}}{=} \bm{\rho}_{[d+1]\setminus \{j^*\},[d+1]\setminus \{j^*\}}^{-1}\bm{\rho}_{[d+1]\setminus \{j^*\},j^*}\\
    &\stackrel{\eqref{eq:reexpress_rho-1}}{=} \begin{pmatrix}
    \bm{M}^{-1} & -\bm{M}^{-1}\bm{C}^{T}\bm{D}^{-1}\\
    -\bm{D}^{-1}\bm{C}\bm{M}^{-1} & \bm{D}^{-1}+\bm{D}^{-1}\bm{C}\bm{M}^{-1}\bm{C}^{T}\bm{D}^{-1}
    \end{pmatrix}
    \begin{pmatrix}
        \bm{\rho}_{1\ldots j^*-1,j^*} \\
        \bm{\rho}_{j^*+1\ldots d+1,j^*}
    \end{pmatrix} \\
    &= \begin{pmatrix} \scriptscriptstyle
        \bm{M}^{-1}(\bm{\rho}_{1\ldots j^*-1,j^*} -\bm{\rho}_{j^*+1\ldots d+1,1\ldots j^*-1}^{T}\bm{\rho}_{j^*+1\ldots d+1,j^*+1\ldots d+1}^{-1}\bm{\rho}_{j^*+1\ldots d+1,j^*}) \\
        \scriptscriptstyle\bm{D}^{-1}\bm{\rho}_{j^*+1\ldots d+1,j^*}-\bm{D}^{-1}\bm{C}\bm{M}^{-1}(\bm{\rho}_{1\ldots j^*-1,j^*} -\bm{\rho}_{j^*+1\ldots d+1,1\ldots j^*-1}^{T}\bm{\rho}_{j^*+1\ldots d+1,j^*+1\ldots d+1}^{-1}\bm{\rho}_{j^*+1\ldots d+1,j^*})
    \end{pmatrix} \\
    &\stackrel{\eqref{eq:equals0}}{=} \begin{pmatrix}
        \bm{0}\\
        \bm{\rho}_{j^*+1\ldots d+1,j^*+1\ldots d+1}^{-1}\bm{\rho}_{j^*+1\ldots d+1,j^*}
    \end{pmatrix} \; , \label{eq:beta_subvector_0}
\end{align}

where the subscript $(\tau)$ indicates the specific truncation level of the C-vine. Note that the regression coefficients $\bm{\beta}^{(j^*)}_{(\tau)} \in \mathbb{R}^d$ are indexed from 1 to $d$. This means that for any $k > j^*$ the coefficient for $v_{k}$ is $\beta^{(j^*)}_{(\tau) \; k-1}$, i.e.:

\begin{align}
    V_{j^*} &= \sum_{k=1}^{j^* - 1} \beta^{(j^*)}_{(\tau) \; k} v_{k} + \sum_{k= j^* + 1}^d \beta^{(j^*)}_{(\tau) \; k-1} v_{k} + \varepsilon \; .    
\end{align}

Finally, we have:

\begin{align}
    (\sigma^{(j^*)}_{(\tau)})^{2} &\stackrel{\eqref{eq:sigma2}}{=}  1-\bm{\rho}_{[d+1]\setminus \{j^*\},j^*}^{T}\bm{\rho}_{[d+1]\setminus \{j^*\},[d+1]\setminus \{j^*\}}^{-1}\bm{\rho}_{[d+1]\setminus \{j^*\},j^*}  \\
    &\stackrel{\eqref{eq:beta_subvector_0}}{=} 1-\begin{pmatrix}
        \bm{\rho}_{1\ldots j^*-1,j^*}^{T} & \bm{\rho}_{j^*+1\ldots d+1,j^*}^{T}
    \end{pmatrix} \begin{pmatrix}
        \bm{0} \\
        \bm{\rho}_{j^*+1\ldots d+1,j^*+1\ldots d+1}^{-1}\bm{\rho}_{j^*+1\ldots d+1,j^*}
    \end{pmatrix} \\
    &= 1-\bm{\rho}_{j^*+1\ldots d+1,j^*}^{T}\bm{\rho}_{j^*+1\ldots d+1,j^*+1\ldots d+1}^{-1}\bm{\rho}_{j^*+1\ldots d+1,j^*} \; . \label{eq:sigma_calc}
\end{align}

If we assume that the C-vine is truncated at level $\tau < d+1-j^*$, we can reformulate the results of \eqref{eq:beta_subvector_0} and \eqref{eq:sigma_calc} in terms of a general truncation level $\tau$ and obtain:

\begin{align} \label{eq:beta_sub_0}
    \bm{\beta}_{(\tau) \; 1\ldots d-\tau}^{(j^*)} &= \bm{0} \; ,\\ \label{eq:beta_sub_general}
    \bm{\beta}_{(\tau) \; d+1-\tau\ldots d}^{(j^*)} &= \bm{\rho}_{d-\tau+2\ldots d+1,d-\tau+2\ldots d+1}^{-1}\bm{\rho}_{d-\tau+2\ldots d+1,j^*} \; , \\
    (\sigma^{(j^*)}_{(\tau)})^{2} &= 1-\bm{\rho}_{d-\tau+2\ldots d+1,j^*}^{T}\bm{\rho}_{d-\tau+2\ldots d+1,d-\tau+2\ldots d+1}^{-1}\bm{\rho}_{d-\tau+2\ldots d+1,j^*} \; . \label{eq:sigma2_general}
\end{align}

Let us now assume that we truncate away one more tree, i.e. truncate the C-vine at level $\tau-1$. Then among others the partial correlation:

\begin{align}
    \rho_{j^* d+2-\tau \cdot d+3-\tau \ldots d+1} &= 0 \; .
\end{align}

Proceeding in the same way as for $\rho_{kj\cdot j+1\ldots d+1}$ in Equations \eqref{eq:equality0} to \eqref{eq:equality}, it is easily shown that:

\begin{align}
    \rho_{j^*,d+2-\tau} &= \bm{\rho}_{d+3-\tau\ldots d+1,j^*}^{T}\bm{\rho}_{d+3-\tau\ldots d+1,d+3-\tau\ldots d+1}^{-1}\bm{\rho}_{d+3-\tau\ldots d+1,d+2-\tau} \; . \label{eq:rho_new_0}
\end{align}

Further, we have:
\begin{align}
    \bm{\rho}_{d-\tau+2\ldots d+1,d-\tau+2\ldots d+1}
    &= \begin{pmatrix}
        1 & \bm{\rho}_{d+3 - \tau\ldots d+1,d+2-\tau}^{T}\\
        \bm{\rho}_{d+3-\tau\ldots d+1,d+2-\tau} & \bm{\rho}_{d+3-\tau\ldots d+1,d+3-\tau\ldots d+1}
    \end{pmatrix}
    = \begin{pmatrix}
        1 & \bm{E}^{T}\\
        \bm{E} & \bm{F}
    \end{pmatrix} \; ,
\end{align}

so that:

\begin{align}
    \bm{\rho}_{d-\tau+2\ldots d+1,d-\tau+2\ldots d+1}^{-1} &= \begin{pmatrix}
        \frac{1}{m} & -\frac{1}{m}\bm{E}^{T}\bm{E}^{-1}\\
        -\frac{1}{m}\bm{F}^{-1}\bm{E} & \bm{F}^{-1}+\frac{1}{m}\bm{F}^{-1}\bm{E}\bm{E}^{T}\bm{F}^{-1}
    \end{pmatrix} \; , \label{eq:new_block_inv}
\end{align}

with $m=1-\bm{E}^{T}\bm{F}^{-1}\bm{E}$. Now we know from \eqref{eq:beta_sub_0} that is has to be:

\begin{align}
    \bm{\beta}_{(\tau-1)}^{(j^*)} &= \begin{pmatrix}
        \bm{\beta}_{(\tau -1) \; 1 \ldots d-\tau}^{(j^*)} \\
        \bm{\beta}_{(\tau-1) \; d+1 - \tau \ldots d}^{(j^*)}
    \end{pmatrix} = \begin{pmatrix}
        \bm{0} \\
        \bm{\beta}_{(\tau-1) \; d+1 - \tau \ldots d}^{(j^*)}
    \end{pmatrix}  \; ,
\end{align}

and we know that then the remaining sub-vector $\bm{\beta}_{(\tau-1) \; d+1 - \tau \ldots d}^{(j^*)}$ is:

\begin{align}
    &\bm{\beta}_{(\tau-1) \; d+1 - \tau \ldots d}^{(j^*)} = \begin{pmatrix}
        \beta_{(\tau-1) \; d+1-\tau}^{(j^*)} \\
        \bm{\beta}_{(\tau-1) \; d+2 - \tau\ldots d}^{(j^*)}
    \end{pmatrix} \stackrel{\eqref{eq:beta_sub_general}}{=} \bm{\rho}_{d-\tau+2\ldots d+1,d-\tau+2\ldots d+1}^{-1}\bm{\rho}_{d-\tau+2\ldots d+1,j^*} \\
    &\stackrel{\eqref{eq:new_block_inv}}{=} \begin{pmatrix}
        \frac{1}{m} & -\frac{1}{m}\bm{E}^{T}\bm{E}^{-1}\\
        -\frac{1}{m}\bm{F}^{-1}\bm{E} & \bm{F}^{-1}+\frac{1}{m}\bm{F}^{-1}\bm{E}\bm{E}^{T}\bm{F}^{-1}
    \end{pmatrix}
    \begin{pmatrix}
        \rho_{j^*,d+2-\tau}\\
        \bm{\rho}_{d+3-\tau\ldots d+1,j^*}
        \end{pmatrix}\\
    &= \begin{pmatrix}
        \scriptscriptstyle\frac{1}{m}(\rho_{j^*,d+2-\tau}-\bm{\rho}_{d+3-\tau\ldots d+1,j^*}^{T}\bm{\rho}_{d+3-\tau\ldots d+1,d+3-\tau\ldots d+1}^{-1}\bm{\rho}_{d+3-\tau\ldots d+1,d+2-\tau})\\
        \scriptscriptstyle\bm{F}^{-1}\bm{\rho}_{d+3-\tau\ldots d+1,j^*}+\frac{1}{m}\bm{F}^{-1}\bm{E}(\rho_{j^*,d+2-\tau}-\bm{\rho}_{d+3-\tau\ldots d+1,j^*}^{T}\bm{\rho}_{d+3-\tau\ldots d+1,d+3-\tau\ldots d+1}^{-1}\bm{\rho}_{d+3-\tau\ldots d+1,d+2-\tau})
    \end{pmatrix}\\
    &\stackrel{\eqref{eq:rho_new_0}}{=} \begin{pmatrix}
        0\\
        \bm{\rho}_{d+3-\tau\ldots d+1,d+3-\tau\ldots d+1}^{-1}\bm{\rho}_{d+3-\tau\ldots d+1,j^*}
    \end{pmatrix} \; , \label{eq:beta_tau-1_0}
\end{align}

where due to symmetry of the covariance matrix $\rho_{j^*, d+2-\tau} = \rho_{d+2-\tau, j^*}$. Finally, we have:

\begin{align}
    (\sigma^{(j^*)}_{(\tau-1)})^{2}
        &\stackrel{\eqref{eq:sigma2_general}}{=} 1-\bm{\rho}_{d-\tau+2\ldots d+1,j^*}^{T}\bm{\rho}_{d-\tau+2\ldots d+1,d-\tau+2\ldots d+1}^{-1}\bm{\rho}_{d-\tau+2\ldots d+1,j^*} \\
        &\stackrel{\eqref{eq:beta_tau-1_0}}{=} 1-\begin{pmatrix}
        \rho_{d+2-\tau,j^*} & \bm{\rho}_{d+3-\tau\ldots d+1,j^*}^{T}
    \end{pmatrix} \begin{pmatrix}
        0\\
        \bm{\rho}_{d+3-\tau\ldots d+1,d+3-\tau\ldots d+1}^{-1}\bm{\rho}_{d+3-\tau\ldots d+1,j^*}
    \end{pmatrix}\\
    &= 1-\bm{\rho}_{d+3-\tau\ldots d+1,j^*}^{T}\bm{\rho}_{d+3-\tau\ldots d+1,d+3-\tau\ldots d+1}^{-1}\bm{\rho}_{d+3-\tau\ldots d+1,j^*} \; .
\end{align}

\end{proof}  

\begin{proof}[Proof of Theorem \ref{thm:mab2}]
We know from Theorem \ref{thm:mab1} that when the C-vine is truncated at level:

\begin{align}
    \tau \leq d+1-|K|-|S| \leq d+1-j^* \; ,
\end{align}

we have:

\begin{align}
    \bm{\beta}_{(\tau) \; 1\ldots d-\tau}^{(j^*)} &\stackrel{\eqref{eq:beta_sub_0}}{=} \bm{0} \; , \\
    \bm{\beta}_{(\tau) \; d+1-\tau\ldots d}^{(j^*)} &\stackrel{\eqref{eq:beta_sub_general}}{=} \bm{\rho}_{d-\tau+2\ldots d,d-\tau+2\ldots d}^{-1}\bm{\rho}_{d-\tau+2\ldots d+1,j^*} \; .    
\end{align}
 
Further, since $\rho_{kl}=0$, $\forall (k,l)$ with $k \in (K \cup S)$ and $l \in [d+1]\setminus (K \cup S)$, then $\bm{\rho}_{d-\tau+2\ldots d+1,j^*} = \bm{0}$, and it follows directly that $\bm{\beta}^{(j^*)}_{(\tau)}= \bm{0}$.
\end{proof}

\section{Statistical Discrepancy}\label{sec:statistical_fidelity}

\citet{alaa2022faithful} introduce $\alpha$-precision, $\beta$-recall and authenticity $(P_{\alpha}, R_{\beta}, A)$, a three-dimensional, domain- and model-agnostic measure to evaluate fidelity, diversity and generalization of generative models on the sample level. 
Precision and recall for comparing two distributions were introduced in \citet{sajjadi2018assessing}, and measure the degree of overlap of the supports of two distributions. On the contrary, $\alpha$-precision and $\beta$-recall only give high scores if typical regions of the support of the two distributions (in our case: real and synthetic one), holding a certain probability mass, overlap. By this, $\alpha$-precision and $\beta$-recall are able to diagnose different types of failures of the generative distribution, such as mode invention, mode drop or density shift. Hence, they give a more nuanced picture of the performance of a generative model. 

For some $\alpha \in [0,1]$ the $\alpha$-support of the distribution $P$ is defined as the minimum volume subset of $A \subset supp(P)$ that supports a probability mass of $\alpha$ \citep{alaa2022faithful}, i.e.:
\begin{align}
    \mathcal{S}^{\alpha} := {\arg\min}_{A \subset supp(P)} V(A) \quad s.th. \quad P(A) = \alpha \; ,
\end{align}
where $V(A)$ is the volume (Lebesgue measure) of $A$. Thus, the $\alpha$-precision and $\beta$-recall are given by:
\begin{align}
    P_{\alpha} &:= P(Z \in \mathcal{S}_R^{\alpha} ) \; ,
\end{align}
and:
\begin{align}
    R_{\beta} &:= P(X \in \mathcal{S}_S^{\beta} ) \; ,
\end{align}
respectively, with $\mathcal{S}_R^{\alpha}$ the $\alpha$-support of the real distribution $P_R$ and $\mathcal{S}_S^{\beta}$ the $\beta$-support of the generative distribution $P_S$ and $\alpha, \beta \in [0,1]$. 
For finding $\mathcal{S}_R^{\alpha}$ and $\mathcal{S}_S^{\beta}$ and evaluating $P_{\alpha}$ and $R_{\beta}$ on data, \cite{alaa2022faithful} embed $X$ and $Z$ with an evaluation embedding. Letting $\alpha$ and $\beta$ go from 0 to 1 we obtain curves for $P_{\alpha}$ and $R_{\beta}$.
\cite{alaa2022faithful} show that $P_{\alpha} / \alpha = R_{\beta} / \beta = 1$ for all $\alpha, \beta \in [0,1]$ if and only if $P_S = P_R$. Therefore it makes sense to define the integrated $\alpha$-precision and integrated $\beta$-recall:
\begin{align}
    IP_{\alpha} &:= 1 - 2 \cdot \int_0^1 | P_{\alpha} - \alpha| d\alpha \; ,\\
    IR_{\beta} &:= 1 - 2 \cdot \int_0^1 | R_{\beta} - \beta| d\beta \; , \\
\end{align}
both in $[0,1]$, where values closer to 1 indicate a better generative model.

The authenticity score $A$ measures to which percentage the generative model invents genuinely new samples rather than just copying real samples with some noise added. Consequently:
\begin{align}
    P_S = A \cdot P_S' + (1 - A) \cdot \delta_{S, \epsilon} \; ,
\end{align}
where $P_S'$ is the generative distribution conditioned on the synthetic samples not being copied. In the second summand $\delta_{S, \epsilon} = \delta_S * \mathcal{N}(0, \epsilon^2)$ is a convolution of the discrete distribution $\delta_S$ placing an unknown probability mass on each real sample in $X$ and the noise distribution $\mathcal{N}(0, \epsilon^2)$ with arbitrarily small noise variance $\epsilon$.

\cite{alaa2022faithful} estimate $\alpha$-precision ($\beta$-recall) of a single synthetic (real) sample to be 1 if it resides within the estimate of $\mathcal{S}_R^{\alpha}$ ($\mathcal{S}_S^{\beta}$) and 0 otherwise. The mean of all sample-wise $P_{\alpha}$ ($R_{\beta}$) scores gives the $\alpha$-precision ($\beta$-recall) of the synthetic (real) data set. The authenticity score of a synthetic sample is estimated through a likelihood ratio test and averaged to obtain the authenticity of the whole synthetic data set.
In our analysis we estimate $(P_{\alpha}, R_{\beta}, A)$ in terms of the unlabeled real data $X$ and  unlabeled synthetic data $Z$.

\section{Model and Attack Parameters}\label{sec:model_and_attack_parameters}

Parameters of the AIA and MIA are given in Table \ref{tab:parameters_privacyAttacks_genModels}, parameters of the generative models are given in the following. 

\begin{table}[t]
    \caption{Parameters of AIAs and MIAs.}
    \label{tab:parameters_privacyAttacks_genModels}
    \vskip 0.15in
    \begin{center}
        \begin{small}
            \begin{sc}
            \scriptsize
                \begin{tabular}{lcr}
                    \toprule
                    Attack & Parameter & Value \\
                    \midrule
                    AIA & no. game iterations $N$ ($\mathtt{nIter}$): & 10 \\
                     AIA & size of reference data $\mathtt{sizeRawT}$: & 500 \\
                     AIA & size of synthetic data $\mathtt{sizeSynT}$: & 500 \\
                     AIA & no. bootstraped/synthetic data sets $n_{synth}$ ($\mathtt{nSynT}$): & 50 \\
                     AIA & size of bootstrap samples $\mathtt{bootstrapSize}$: & 500 \\
                     \midrule
                     MIA & no. game iterations $N$ ($\mathtt{nIter}$): & 10 \\ 
                     MIA & size of real reference data set for attacker's training  $\mathtt{sizeRawA}$: & 500 \\
                     MIA & no. of shadow models during attacker's training $\mathtt{nShadows}$: & 10 \\ 
                     MIA & no. synthetic data sets sampled during attacker's training $\mathtt{nSynA}$: & 10 \\
                    MIA & size of real reference data set for attacker's evaluation $\mathtt{sizeRawT}$: & 400 \\ 
                     MIA & size of the synthetic data set generated during attacker's evaluation  $\mathtt{sizeSynT}$: & 400 \\
                     MIA & no. synthetic data sets evaluated $\mathtt{nSynT}$: & 50 \\ 
                    \bottomrule
                \end{tabular}
            \end{sc}
        \end{small}
    \end{center}
    \vskip -0.1in
\end{table}

\paragraph{Parameters of the generative models:}\label{sec:parameters_generative_models}
\begin{itemize}
    \item \textbf{Vine Copula:} Parametric pair copula families and their rotations are estimated with maximum likelihood and selected with AIC as selection criterion.
    \item \textbf{PrivBayes:} Histogram bins 25 and degree 1. Privacy parameter $\epsilon \in \{0.1 , 1, 5\}$.
    \item \textbf{CTGAN:} Number of epochs and batch size were tuned with random search to 1000 and 150 respectively for results on simulated real data in Section \ref{subsec:results_simulated_realId20} and Appendix \ref{sec:simulated_real_data_d20_appendix}.
    The remaining parameters are set to default values as provided in the \citetalias{ctgan} implementing \cite{xu2019modeling}. For results on SUPPORT2 data in Section \ref{subsec:results_support2} the random search resulted in 400 epochs and a batch size of 100.
    \item \textbf{TVAE:} Number of epochs, batch size and the dimension of the latent space were tuned with random search to 1500, 400 and 2 respectively for results on simulated real data in Section \ref{subsec:results_simulated_realId20} and Appendix \ref{sec:simulated_real_data_d20_appendix}.
    The remaining parameters are set to default values as provided in the \citetalias{ctgan} implementing \cite{xu2019modeling}. For results on SUPPORT2 data in Section \ref{subsec:results_support2} the random search resulted in 800 epochs, a batch size of 100 and latent space dimension of 4.
    \item \textbf{PrivPGD:} Parameters are kept to their default values, as \citet{donhauser2024privacy} state that PrivPGD does not require specific parameter tuning due to the data being represented as particles. We choose the authors' proposed DP parameter default values, i.e. $\epsilon=2.5$ and $\delta=10^{-5}$.
\end{itemize}

\section{Compute Resources}\label{sec:compute_resources}

All experiments on the SUPPORT2 data with results in Section \ref{subsec:results_support2} and Appendix \ref{sec:support2_data} and utility and statistical fidelity results on simulated real data of Section \ref{subsec:results_simulated_realId20} and Appendix \ref{sec:simulated_real_data_d20_appendix} were conducted on an Apple Macbook Pro with macOS Sonoma 14.4.1, Apple M2 Pro chip and 16 GB RAM using 10 cores. AIA and MIA experiments on simulated real data of Section \ref{subsec:results_simulated_realId20} and Appendix \ref{sec:simulated_real_data_d20_appendix} were conducted on an hpc cluster with the following specs: 
\begin{itemize}
    \item CPU: 256 threads (2 × AMD EPYC 7713 Milan: 64 cores, 128 threads per CPU)
    \item RAM: 4 TB	(32 × 128 GB DDR4)
    \item OS: Linux (Red Hat Enterprise Linux 7)
\end{itemize}
Experiments were conducted in parallel on 20 cores.

Software used for experiments on Apple Macbook Pro:
\begin{itemize}
    \item Python 3.10.13
    \item R version 4.3.1 (2023-06-16) -- "Beagle Scouts"
    \item tmux 3.3a
    \item conda 23.10.0
\end{itemize}

Execution time measured with \textit{time} command of AIA on SUPPORT2 data for C-vine per truncation level:
\begin{itemize}
    \item truncation at level 1: 489.20s user 14.44s system 100\% cpu 8:21.01 total
    \item truncation at level 5: 1168.23s user 14.72s system 99\% cpu 19:44.55 total
    \item truncation at level 10:  1836.25s user 14.78s system 99\% cpu 30:56.40 total
    \item truncation at level 15: 2427.97s user 15.03s system 99\% cpu 40:49.45 total
    \item truncation at level 20: 2877.82s user 15.27s system 99\% cpu 48:22.73 total
    \item no truncation: 3059.69s user 15.21s system 99\% cpu 51:26.67 total
\end{itemize}

Execution time measured with \textit{time} command of MIA on SUPPORT2 data for C-vine per truncation level:
\begin{itemize}
    \item truncation at level 1: 1540.75s user 84.16s system 63\% cpu 42:28.99 total
    \item truncation at level 5: 3180.88s user 89.65s system 103\% cpu 52:51.39 total
    \item truncation at level 10:  4783.33s user 92.23s system 128\% cpu 1:03:04.40 total
    \item truncation at level 15: 6077.78s user 93.04s system 143\% cpu 1:11:34.51 total
    \item truncation at level 20: 6994.86s user 91.37s system 151\% cpu 1:18:11.24 total
    \item no truncation: 7385.50s user 90.99s system 154\% cpu 1:20:47.09 total
\end{itemize}

\section{Simulated Real Data}\label{sec:simulated_real_data_d20_appendix}
We simulate real data with $n = 1000$ and $n_{test} = 250$ realizations of the random vector $(X_1, X_2, \dots , X_{20}, Y) \in \mathbb{R}^{20} \times \{0,1\}$ following a distribution $F$. The joint distribution $F$ of $(\bm{X}^T, Y) := (X_1, X_2, \dots, X_{20}, Y)$ is composed the following way: $Y \sim Bernoulli (0.5)$, $\bm{X} | Y = 0 \sim \mathcal{N}\big( \bm{\mu}_0, \Sigma_0 \big)$ and $\bm{X} | Y = 1 \sim \mathcal{N}\big( \bm{\mu}_1, \Sigma_1 \big)$ with $\bm{\mu}_0, \bm{\mu}_1, \Sigma_0$ and $\Sigma_1$ defined in Equations \eqref{equ:mu0I_d20_I}, \eqref{equ:mu1I_d20_I}, \eqref{equ:Sigma0I_d20_I} and \eqref{equ:Sigma1I_d20_I}.
As can be observed from parameters of $F$, the dependence structure of $F$ is of block form where the three blocks $(X_1, \dots, X_5)$, $(X_6, \dots, X_{10})$  and $(X_{11}, \dots, X_{20}, Y)$ are independent. The distribution $F$ was chosen deliberately such that, if we simulate data, we obtain three approximately uncorrelated blocks in the correlation matrix of the real data to investigate the effect of truncation according to Theorem \ref{thm:mab2}. For this reason we do not apply Algorithm \ref{alg:find_order} additionally.

The estimated correlation matrix of the real data shown Figure \ref{fig:corr_I_d20_append} exhibits this block structure to imitate a scenario where some covariates (third block: $X_{11}, ..., X_{20}$) are important for classifying $Y$ while others (first block: $X_1, ..., X_5$ and second block: $X_6, ..., X_{10}$) are less so. Let us assume that covariates $X_1, X_6$ and $X_{11}$ are sensitive. In this experiment we know that the dependencies in the first and second block do not contribute to the classification of $Y$ but provide information on the sensitive covariates $X_1$ and $X_6$ which may result in impaired privacy. Hence, in TVineSynth we choose the structure and truncation level of the vine copula such that these dependencies are not reflected in the synthetic data. Specifically, we use an ordering $\mathcal{O}^*$ of the covariates that corresponds to their indices, i.e. $X_j = X_{(j)}$, and $Y$ as the center of $T_1$, see Appendix \ref{sec:RvinematrixCvineSimreal}. 
In Figure \ref{fig:Cvine_synth_data_corr_real_data_I_d20} we can observe how the correlation structure of the real data in Figure \ref{fig:corr_I_d20_append} is more and more reproduced in the synthetic data generated by a C-vine with increasing truncation level.
Specifically, we note that dependencies in the first block containing sensitive covariate $X_1$ start to be represented in synthetic data from truncation level 17 and more closely from truncation level 19 onward. For $X_6$ in the second block this is the case from truncation level 12 onward. In the third block containing sensitive covariate $X_{11}$ this happens already from truncation level 1 onward. This indicates the effect of truncation combined with the C-vine structure.

\begin{figure}
    \centering
    \rotatebox{90}{
    \begin{minipage}{\textheight}
        {\tiny
        \begin{align} \label{equ:mu0I_d20_I}
            \bm{\mu}_0^{(I)} &:= (-2.42,  5.84, 20.10, 12.66,  0.35, 12.64, 12.29, 21.29, 1.11, 24.69, 25.27, -3.53, 6.10, -4.52, 3.37, 19.73,  5.78, 12.80, -3.19, 14.76)^T \; , \\ \label{equ:mu1I_d20_I}
            \bm{\mu}_1^{(I)} &:= (-2.42,  5.84, 20.10, 12.66,  0.35, 12.64, 12.29, 21.29,  1.11, 24.69, 24.44, -4.78,  6.51, -4.73,  2.08, 20.63,  5.26, 13.57, -2.94, 15.39)^T \; , %
        \end{align}
        }%
        \setcounter{MaxMatrixCols}{20}
        {\tiny
        \begin{align} \label{equ:Sigma0I_d20_I}
            \Sigma_0^{(I)} &= 
            \begin{pmatrix}
                2.57 & -2.14 & 1.33 & 0.04 & -0.76 & 0 & 0 & 0 & 0 & 0 & 0 & 0 & 0 & 0 & 0 & 0 & 0 & 0 & 0 & 0 \\ 
                -2.14 & 6.12 & -2.99 & -0.36 & 1.25 & 0 & 0 & 0 & 0 & 0 & 0 & 0 & 0 & 0 & 0 & 0 & 0 & 0 & 0 & 0 \\ 
                1.33 & -2.99 & 6.36 & -1.85 & 2.05 & 0 & 0 & 0 & 0 & 0 & 0 & 0 & 0 & 0 & 0 & 0 & 0 & 0 & 0 & 0 \\ 
                0.04 & -0.36 & -1.85 & 3.29 & -0.97 & 0 & 0 & 0 & 0 & 0 & 0 & 0 & 0 & 0 & 0 & 0 & 0 & 0 & 0 & 0 \\ 
                -0.76 & 1.25 & 2.05 & -0.97 & 6.07 & 0 & 0 & 0 & 0 & 0 & 0 & 0 & 0 & 0 & 0 & 0 & 0 & 0 & 0 & 0 \\ 
                0 & 0 & 0 & 0 & 0 & 3.80 & -1.97 & 1.69 & -0.29 & -1.01 & 0 & 0 & 0 & 0 & 0 & 0 & 0 & 0 & 0 & 0 \\ 
                0 & 0 & 0 & 0 & 0 & -1.97 & 7.77 & -1.69 & -2.07 & 2.00 & 0 & 0 & 0 & 0 & 0 & 0 & 0 & 0 & 0 & 0 \\ 
                0 & 0 & 0 & 0 & 0 & 1.69 & -1.69 & 3.86 & 1.67 & -1.46 & 0 & 0 & 0 & 0 & 0 & 0 & 0 & 0 & 0 & 0 \\ 
                0 & 0 & 0 & 0 & 0 & -0.29 & -2.07 & 1.67 & 4.12 & -1.47 & 0 & 0 & 0 & 0 & 0 & 0 & 0 & 0 & 0 & 0 \\ 
                0 & 0 & 0 & 0 & 0 & -1.01 & 2.00 & -1.46 & -1.47 & 1.92 & 0 & 0 & 0 & 0 & 0 & 0 & 0 & 0 & 0 & 0 \\ 
                0 & 0 & 0 & 0 & 0 & 0 & 0 & 0 & 0 & 0 & 5.82 & -1.29 & -2.52 & 1.80 & -1.93 & -2.13 & -2.74 & -1.84 & -0.09 & -2.72 \\ 
                0 & 0 & 0 & 0 & 0 & 0 & 0 & 0 & 0 & 0 & -1.29 & 8.60 & -0.98 & -3.48 & -1.80 & -1.33 & 2.56 & -2.21 & -1.09 & -0.62 \\ 
                0 & 0 & 0 & 0 & 0 & 0 & 0 & 0 & 0 & 0 & -2.52 & -0.98 & 5.44 & 0.48 & -1.02 & 0.63 & -1.18 & 1.90 & -1.13 & 2.84 \\ 
                0 & 0 & 0 & 0 & 0 & 0 & 0 & 0 & 0 & 0 & 1.80 & -3.48 & 0.48 & 5.67 & -1.13 & -1.80 & -2.98 & 0.89 & 0.28 & -0.37 \\ 
                0 & 0 & 0 & 0 & 0 & 0 & 0 & 0 & 0 & 0 & -1.93 & -1.80 & -1.02 & -1.13 & 7.80 & 2.92 & 3.83 & 3.01 & 1.11 & 2.81 \\ 
                0 & 0 & 0 & 0 & 0 & 0 & 0 & 0 & 0 & 0 & -2.13 & -1.33 & 0.63 & -1.80 & 2.92 & 4.44 & 3.05 & 2.73 & -0.03 & 1.93 \\ 
                0 & 0 & 0 & 0 & 0 & 0 & 0 & 0 & 0 & 0 & -2.74 & 2.56 & -1.18 & -2.98 & 3.83 & 3.05 & 7.16 & 2.72 & 2.21 & 2.05 \\ 
                0 & 0 & 0 & 0 & 0 & 0 & 0 & 0 & 0 & 0 & -1.84 & -2.21 & 1.90 & 0.89 & 3.01 & 2.73 & 2.72 & 5.21 & 1.16 & 3.65 \\ 
                0 & 0 & 0 & 0 & 0 & 0 & 0 & 0 & 0 & 0 & -0.09 & -1.09 & -1.13 & 0.28 & 1.11 & -0.03 & 2.21 & 1.16 & 4.99 & 0.23 \\ 
              0 & 0 & 0 & 0 & 0 & 0 & 0 & 0 & 0 & 0 & -2.72 & -0.62 & 2.84 & -0.37 & 2.81 & 1.93 & 2.05 & 3.65 & 0.23 & 5.88
            \end{pmatrix} , \\[1em] \label{equ:Sigma1I_d20_I}
            \Sigma_1^{(I)} &= 
            \begin{pmatrix}
                2.57 & -2.14 & 1.33 & 0.04 & -0.76 & 0 & 0 & 0 & 0 & 0 & 0 & 0 & 0 & 0 & 0 & 0 & 0 & 0 & 0 & 0 \\ 
                -2.14 & 6.12 & -2.99 & -0.36 & 1.25 & 0 & 0 & 0 & 0 & 0 & 0 & 0 & 0 & 0 & 0 & 0 & 0 & 0 & 0 & 0 \\ 
                1.33 & -2.99 & 6.36 & -1.85 & 2.05 & 0 & 0 & 0 & 0 & 0 & 0 & 0 & 0 & 0 & 0 & 0 & 0 & 0 & 0 & 0 \\ 
                0.04 & -0.36 & -1.85 & 3.29 & -0.97 & 0 & 0 & 0 & 0 & 0 & 0 & 0 & 0 & 0 & 0 & 0 & 0 & 0 & 0 & 0 \\ 
                -0.76 & 1.25 & 2.05 & -0.97 & 6.07 & 0 & 0 & 0 & 0 & 0 & 0 & 0 & 0 & 0 & 0 & 0 & 0 & 0 & 0 & 0 \\ 
                0 & 0 & 0 & 0 & 0 & 3.80 & -1.97 & 1.69 & -0.29 & -1.01 & 0 & 0 & 0 & 0 & 0 & 0 & 0 & 0 & 0 & 0 \\ 
                0 & 0 & 0 & 0 & 0 & -1.97 & 7.77 & -1.69 & -2.07 & 2.00 & 0 & 0 & 0 & 0 & 0 & 0 & 0 & 0 & 0 & 0 \\ 
                0 & 0 & 0 & 0 & 0 & 1.69 & -1.69 & 3.86 & 1.67 & -1.46 & 0 & 0 & 0 & 0 & 0 & 0 & 0 & 0 & 0 & 0 \\ 
                0 & 0 & 0 & 0 & 0 & -0.29 & -2.07 & 1.67 & 4.12 & -1.47 & 0 & 0 & 0 & 0 & 0 & 0 & 0 & 0 & 0 & 0 \\ 
                0 & 0 & 0 & 0 & 0 & -1.01 & 2.00 & -1.46 & -1.47 & 1.92 & 0 & 0 & 0 & 0 & 0 & 0 & 0 & 0 & 0 & 0 \\ 
                0 & 0 & 0 & 0 & 0 & 0 & 0 & 0 & 0 & 0 & 6.01 & -3.24 & 2.67 & -0.55 & 3.89 & 1.34 & 1.22 & -1.22 & 2.02 & -2.53 \\ 
                0 & 0 & 0 & 0 & 0 & 0 & 0 & 0 & 0 & 0 & -3.24 & 7.78 & -0.19 & 2.07 & -3.66 & 0.89 & -0.01 & 0.03 & -0.35 & -0.28 \\ 
                0 & 0 & 0 & 0 & 0 & 0 & 0 & 0 & 0 & 0 & 2.67 & -0.19 & 6.12 & 0.90 & 3.53 & 1.65 & -0 & -1.67 & 3.01 & -1.95 \\ 
                0 & 0 & 0 & 0 & 0 & 0 & 0 & 0 & 0 & 0 & -0.55 & 2.07 & 0.90 & 5.60 & 1.52 & 2.00 & 1.41 & 1.80 & 1.06 & -1.91 \\ 
                0 & 0 & 0 & 0 & 0 & 0 & 0 & 0 & 0 & 0 & 3.89 & -3.66 & 3.53 & 1.52 & 8.20 & 1.75 & 1.01 & -0.11 & 2.97 & -1.88 \\ 
                0 & 0 & 0 & 0 & 0 & 0 & 0 & 0 & 0 & 0 & 1.34 & 0.89 & 1.65 & 2.00 & 1.75 & 3.06 & 1.48 & -0.59 & 1.15 & -0.50 \\ 
                0 & 0 & 0 & 0 & 0 & 0 & 0 & 0 & 0 & 0 & 1.22 & -0.01 & -0 & 1.41 & 1.01 & 1.48 & 5.03 & -1.08 & 1.93 & -1.31 \\ 
                0 & 0 & 0 & 0 & 0 & 0 & 0 & 0 & 0 & 0 & -1.22 & 0.03 & -1.67 & 1.80 & -0.11 & -0.59 & -1.08 & 4.02 & -1.14 & 0.08 \\ 
                0 & 0 & 0 & 0 & 0 & 0 & 0 & 0 & 0 & 0 & 2.02 & -0.35 & 3.01 & 1.06 & 2.97 & 1.15 & 1.93 & -1.14 & 5.13 & -1.57 \\ 
                0 & 0 & 0 & 0 & 0 & 0 & 0 & 0 & 0 & 0 & -2.53 & -0.28 & -1.95 & -1.91 & -1.88 & -0.50 & -1.31 & 0.08 & -1.57 & 5.40 
            \end{pmatrix} .
        \end{align} %
        } %
    \end{minipage}
   }
\end{figure}

\begin{figure}[ht]
    \centering
    \includegraphics[width=0.2\columnwidth]{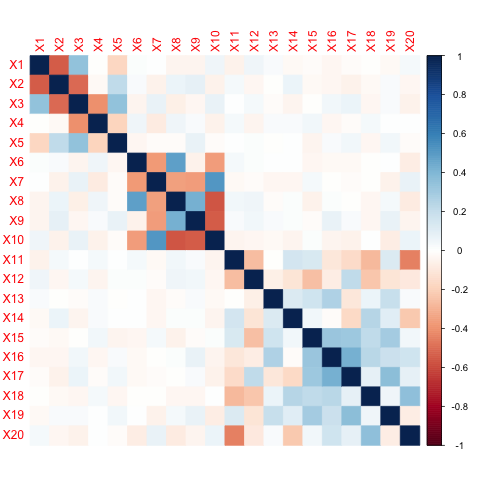}
    \caption{Pearson correlation matrix estimated on simulated real data. It exhibits a block structure to imitate a scenario where some covariates, i.e. the ones in the third block $X_{11}, ..., X_{20}$ are important for classification, while others in the first block, $X_1, ..., X_5$ and in the second block $X_6, ..., X_{10}$ are not.}
    \label{fig:corr_I_d20_append}
\end{figure}

\begin{figure}[ht]
    \centering
    \begin{tabular}{cccc}
        \includegraphics[width=0.18\textwidth]{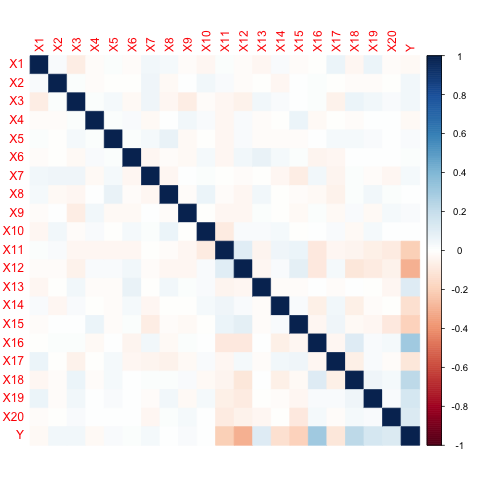} &  
        \includegraphics[width=0.18\textwidth]{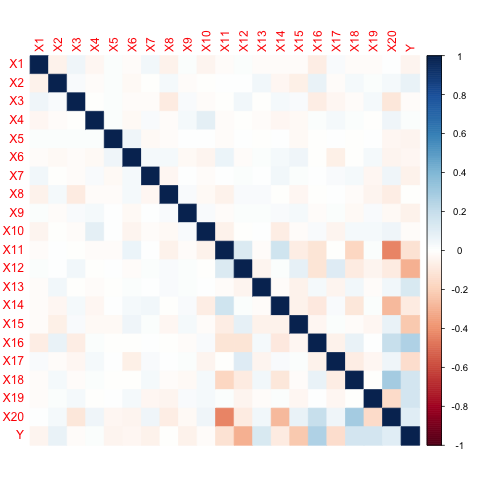} &
        \includegraphics[width=0.18\textwidth]{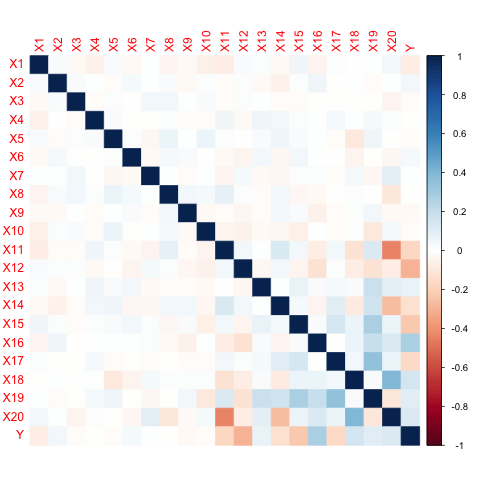} &
        \includegraphics[width=0.18\textwidth]{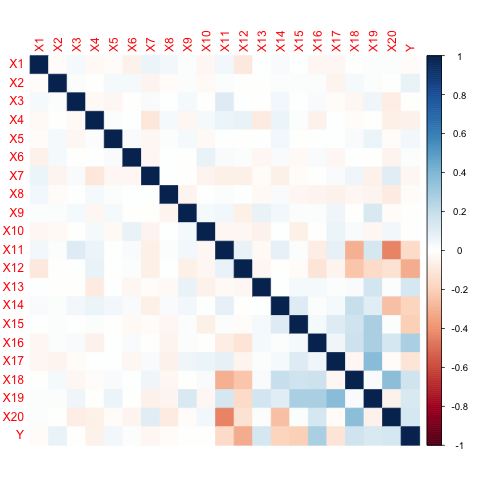} \\
        (a) Truncation at 1. & (b) Truncation at 2. & (c) Truncation at 3. & (d) Truncation at 4. \\
        \includegraphics[width=0.18\textwidth]{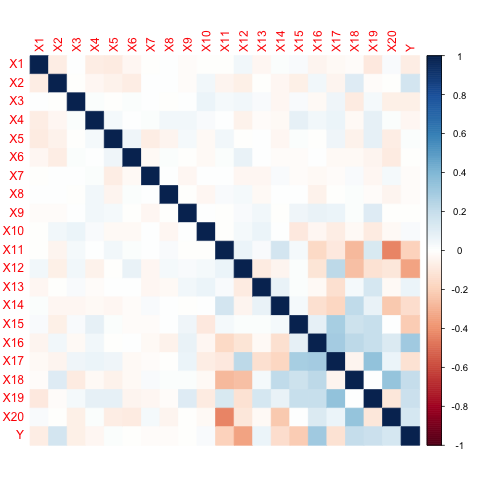} &
        \includegraphics[width=0.18\textwidth]{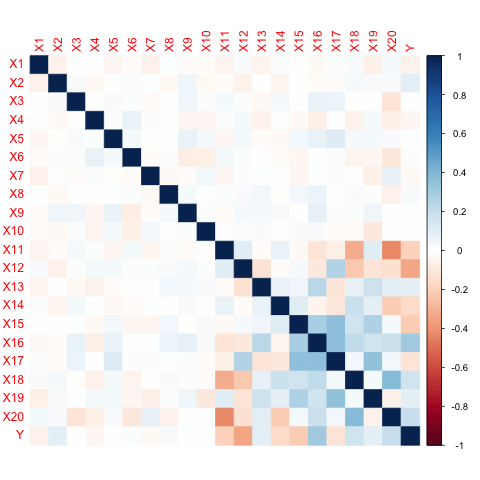} &
        \includegraphics[width=0.18\textwidth]{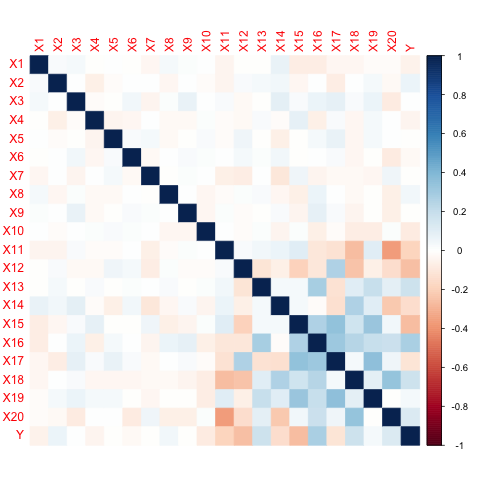} &
        \includegraphics[width=0.18\textwidth]{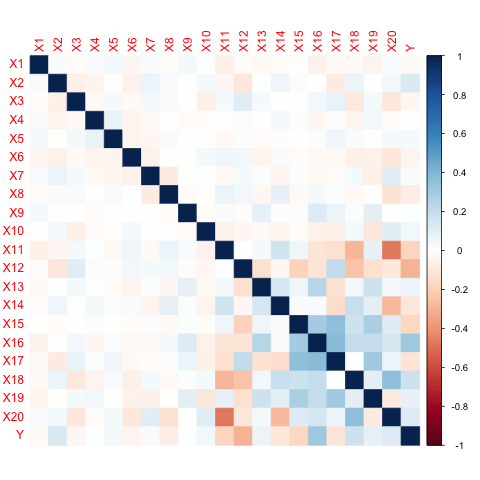} \\
        (e) Truncation at 5. & (f) Truncation at 6. & (g) Truncation at 7. & (h) Truncation at 8. \\
        \includegraphics[width=0.18\textwidth]{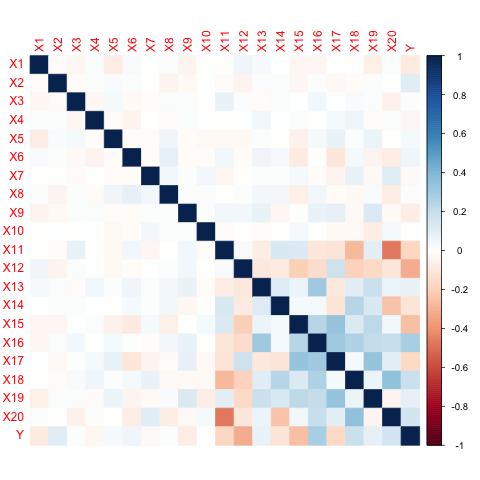} &
        \includegraphics[width=0.18\textwidth]{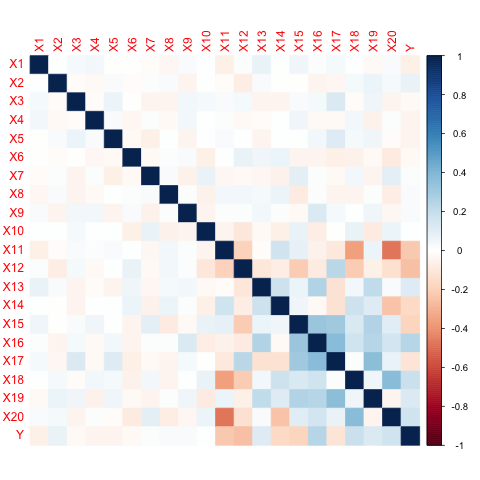} &
         \includegraphics[width=0.18\textwidth]{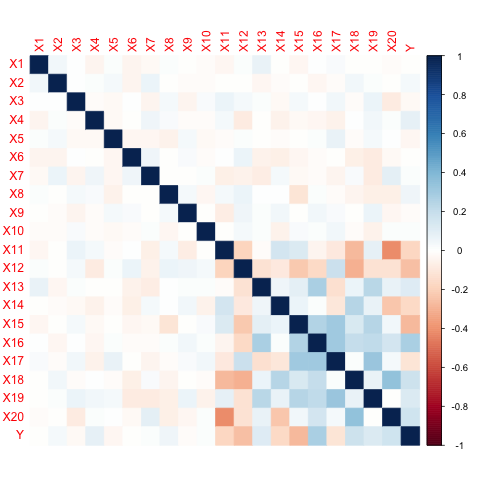} &
        \includegraphics[width=0.18\textwidth]{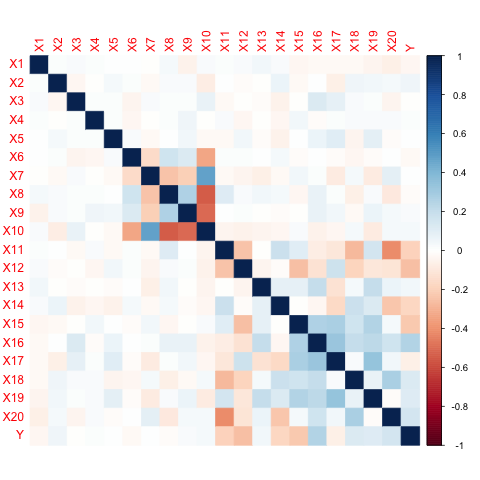} \\
        (i) Truncation at 9. & (j) Truncation at 10. & (k) Truncation at 11. & (l) Truncation at 12. \\
        \includegraphics[width=0.18\textwidth]{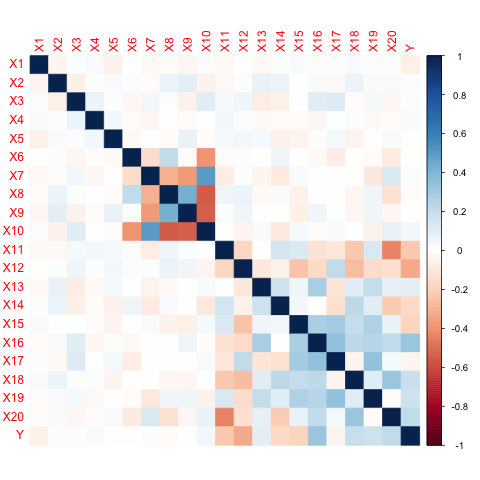} &
        \includegraphics[width=0.18\textwidth]{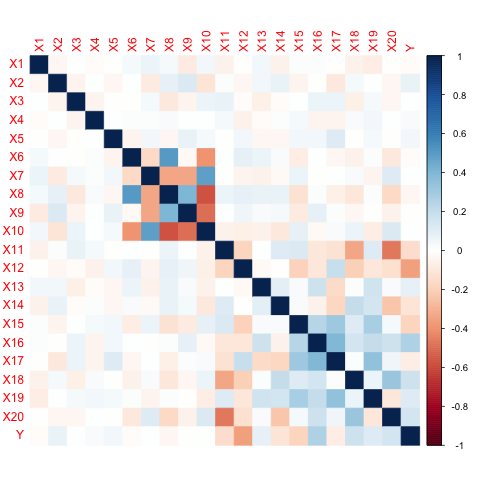} &
        \includegraphics[width=0.18\textwidth]{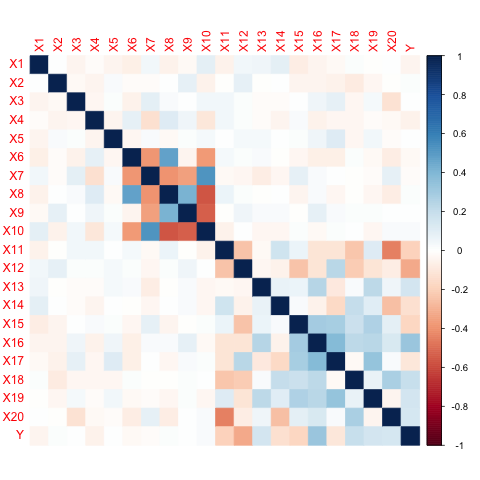} &
        \includegraphics[width=0.18\textwidth]{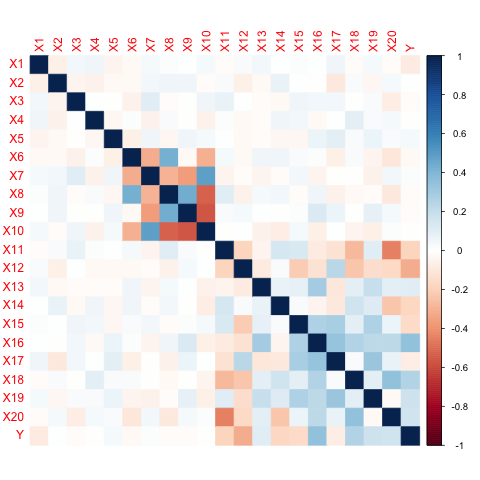} \\
        (m) Truncation at 13. & (n) Truncation at 14. & (o) Truncation at 15. & (p) Truncation at 16. \\
        \includegraphics[width=0.18\textwidth]{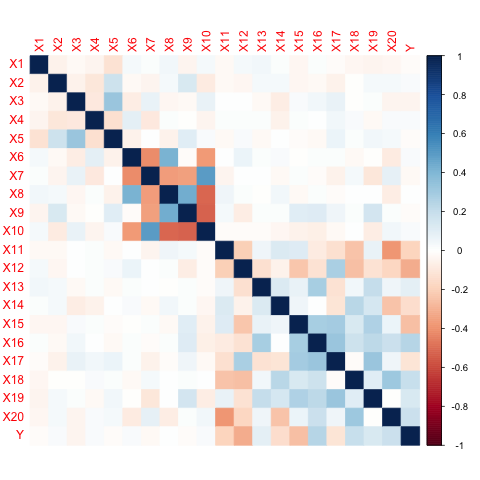} &
        \includegraphics[width=0.18\textwidth]{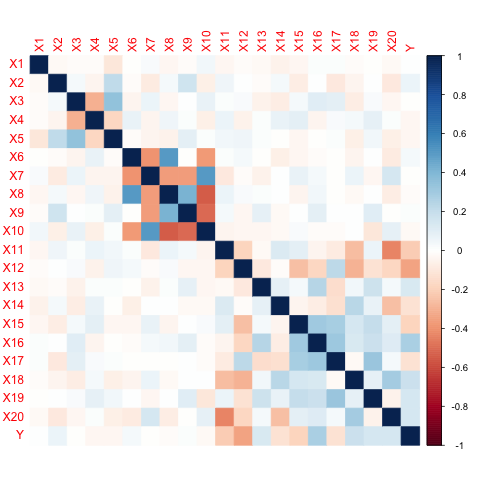} &
        \includegraphics[width=0.18\textwidth]{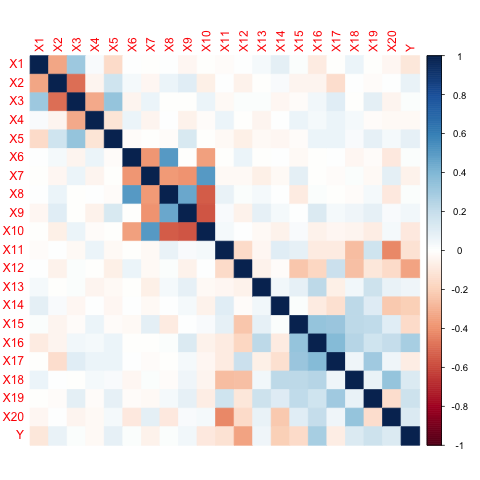} &
        \includegraphics[width=0.18\textwidth]{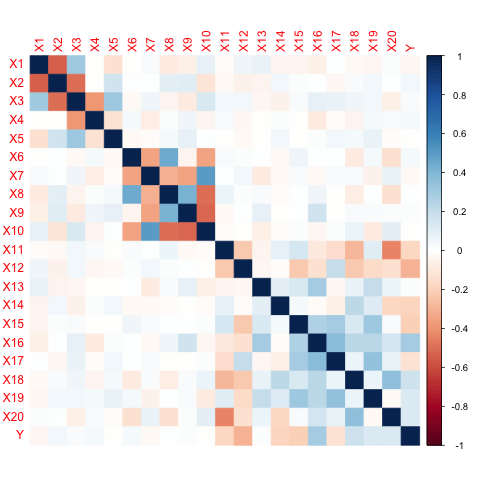} \\
         (q) Truncation at 17. & (r) Truncation at 18. & (s) Truncation at 19. & (t) No truncation. \\
    \end{tabular}
    
    \caption{The (Pearson) correlation matrices of the synthetic data generated with a C-vine for truncation levels from 1 to 19 and no truncation illustrate how the correlation structure of the real data is more and more reproduced with increasing truncation level. Note that correlations in the first and second block are only reflected in the synthetic data from truncation level 12 and 17 respectively. Details on the estimation of the C-vine can be found in Section \ref{sec:TVineSynth_construction} and Appendix \ref{sec:model_and_attack_parameters}.}
    \label{fig:Cvine_synth_data_corr_real_data_I_d20}
\end{figure}

\FloatBarrier

\subsection{R-Vine Matrix of the C-Vine Used as a Generative Model on Simulated Real Data}\label{sec:RvinematrixCvineSimreal}
R-vine matrices are a compact way to represent the vine tree structure $\mathcal{V}$. They indicate which pairwise conditional dependencies between covariates are modeled through an edge in a tree in $\mathcal{V}$. For a thorough introduction, the reader can consult for example \citet{czado2019analyzing}. The R-vine matrix of the C-vine used as a generative model on simulated real data is as follows with $Y$ on index 21 and index $j \in [20]$ corresponding to covariate $X_j$:
\begin{align}
    \begin{pmatrix}
        21 & 21 & 21 & \cdots & 21 & 21 \\
        & 20 & 20 & \cdots & 20 & 20 \\
        & & 19 & \cdots & 19 & 19 \\
        & & & \ddots & \vdots & \vdots \\
        & & & & 2 & 2 \\
        & & & & &  1 
    \end{pmatrix} \; .
\end{align}

\subsection{Choice of Target Observations for Privacy Evaluation on Simulated Real Data} \label{sec:append_simreald20_choice_of_targets}

We conduct an AIA and MIA on simulated real data described in \ref{sec:simulated_real_data_d20_appendix}. The parameter setup of the privacy attacks can be found in Table \ref{tab:parameters_privacyAttacks_genModels}. For the attacks for each sensitive covariate four target observations are handpicked outside the 95\%-quantile of the regarding sensitive covariate, see Table \ref{tab:handpicked_targets095_simreald20}.

\begin{table}[t]
    \caption{Target observations of the simulated real data set, Section \ref{sec:simulated_real_data_d20_appendix} that are handpicked to lie outside the 95\%-quantile of the respective sensitive covariate $X_1$, $X_6$ and $X_{11}$.}
    \label{tab:handpicked_targets095_simreald20}
    \vskip 0.15in
    \begin{center}
        \begin{small}
            \begin{sc}
                \begin{tabular}{llccr}
        \toprule
         & quantiles & & & \\
         sensitive covariate & 0.01 & 0.025 & 0.975 & 0.99 \\ 
         \midrule
         $X_1$ & ID202 & ID164 & ID179 & ID843 \\
         $X_6$ & ID127 & ID4 & ID353 & ID326 \\
         $X_{11}$ & ID970 & ID949 & ID392 & ID862 \\ \bottomrule
    \end{tabular}
            \end{sc}
        \end{small}
    \end{center}
    \vskip -0.1in
\end{table}

Additionally, five target observations, namely ID123, ID507, ID589, ID740 and ID922\footnote{NB: ID$k$ corresponds to the $(k+1)$th observation in the real data set with $k \in \{0, ..., (n-1)\}$.} are randomly sampled from the real data set. They correspond to the quantiles w.r.t. the respective sensitive covariate given in Table \ref{tab:quantiles_randomly_sampled_targets_AIA_realI_d20}.

\begin{table}[t]
    \caption{Randomly sampled target observations from the simulated real data set of Section \ref{sec:simulated_real_data_d20_appendix} and their corresponding quantiles w.r.t. covariates $X_1, X_6$ and $X_{11}$.}
    \label{tab:quantiles_randomly_sampled_targets_AIA_realI_d20}
    \vskip 0.15in
    \begin{center}
        \begin{small}
            \begin{sc}
                \begin{tabular}{lccccr}
        \toprule
        & \multicolumn{3}{l}{target IDs} & & \\
         sensitive covariate & ID123 & ID507 & ID589 & ID740 & ID922 \\ 
         \midrule
         $X_1$ & 0.475 & 0.858 & 0.284 & 0.469 & 0.302 \\
         $X_6$ & 0.517 & 0.473 & 0.945 & 0.512 & 0.309 \\
         $X_{11}$ & 0.628 & 0.592 & 0.838 & 0.204 & 0.549 \\ \bottomrule
    \end{tabular}
            \end{sc}
        \end{small}
    \end{center}
    \vskip -0.1in
\end{table}

\FloatBarrier

\subsection{Simulated Real Data: Results}\label{app:sim_real_results_appendix}

Parameters of the privacy attacks and of the generative models can be found in Appendix \ref{sec:model_and_attack_parameters}. As outlined in Section \ref{sec:choice_of_targets}, we pick four target observations outside the 95\% quantile for each sensitive covariate $X_1, X_6$ and $X_{11}$ and randomly sample five more target observations from the real data for the privacy analysis, see Appendix \ref{sec:append_simreald20_choice_of_targets}.

\subsubsection{Privacy: Attribute Inference Attack}\label{sec:AIA_realI_d20_results}

The top row of Figure \ref{fig:AIA_MAB_Cvine_competitors_simreal_X1_X6_X11} corresponds to the case where the sensitive covariate $X_1$ is less important for classifying $Y$ correctly. In this situation the star shaped C-vine combined with truncation at level 18 or lower is able to cut away sensitive dependencies that harm privacy but do not contribute to utility. This is in accordance with our observations from Figures \ref{fig:corr_I_d20_append} and \ref{fig:Cvine_synth_data_corr_real_data_I_d20}. If the sensitive covariate is $X_6$, again playing a less important role for classifying $Y$ correctly, we observe in the second row of Figure
\ref{fig:AIA_MAB_Cvine_competitors_simreal_X1_X6_X11} that truncating a C-vine at level 11 or lower offers a high level of privacy, which again complies with our observations from Figures \ref{fig:corr_I_d20_append} and \ref{fig:Cvine_synth_data_corr_real_data_I_d20}.
Thus, the C-vine offers a high level privacy w.r.t. AIA, which is comparable to the one of the DP PrivBayes model at a very strict privacy budget of $\epsilon = 0.1$, and outperforms CTGAN, TVAE and PrivPGD in terms of AIA privacy with some margin for low truncation levels. Simultaneously, the C-vine achieves high utility for all truncation levels, outperforming PrivBayes by far, see Figure \ref{fig:utility_Cvine_competitors_simreal}. Sensitive covariate $X_{11}$ on the other hand shows pairwise association with $Y$, see Figure \ref{fig:corr_I_d20_append}. In this case it is necessary to truncate the C-vine at level 1 to provide privacy w.r.t. AIA, see bottom row of Figure \ref{fig:AIA_MAB_Cvine_competitors_simreal_X1_X6_X11}. The AIA results in terms of WCAB in Appendix \ref{sec:appendix_simreal_AIA_additional_results_WCAB} and in terms of the MSE in Appendix \ref{sec:appendix_simreal_AIA_additional_results_MSE} confirm these findings.

As a proof of concept for why we base TVineSynth on a C-vine we generate synthetic data with an R-vine where the vine tree structure is not pre-specified, but selected as described in \citep{dissmann2013selecting}, and an R-vine star1 model and compare it to C-vine generated synthetic data. An R-vine star1 model is equal to an R-vine except that we exchange its first tree with $T_1$ of the C-vine. Even for truncation at a very low level, an R-vine struggles to offer effective protection against AIAs. The same holds for an R-vine star1 for truncation level 2 and higher. If it consists only of its star-shaped first tree, an R-vine star1 is equivalent to a C-vine and thus grants the same high level of privacy, see Figure \ref{fig:AIA_MAB_Rvine_Rvinestar1_simreal}.

\begin{figure}[ht]
    \centering
    \includegraphics[width=0.7\columnwidth]{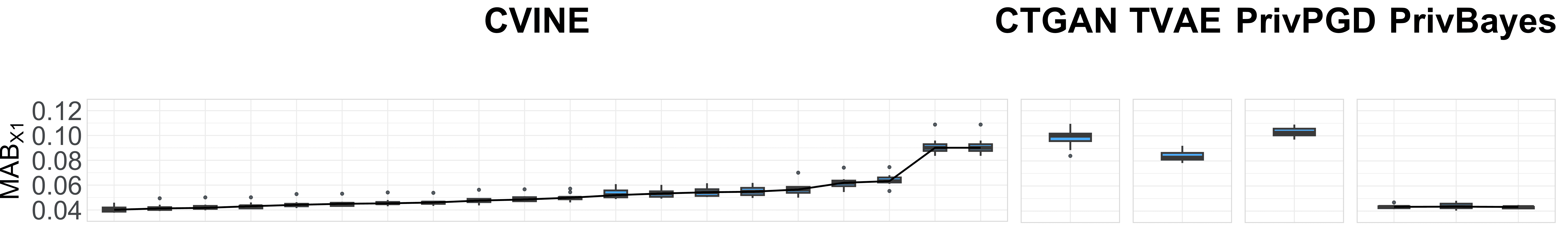}
    \includegraphics[width=0.7\columnwidth]{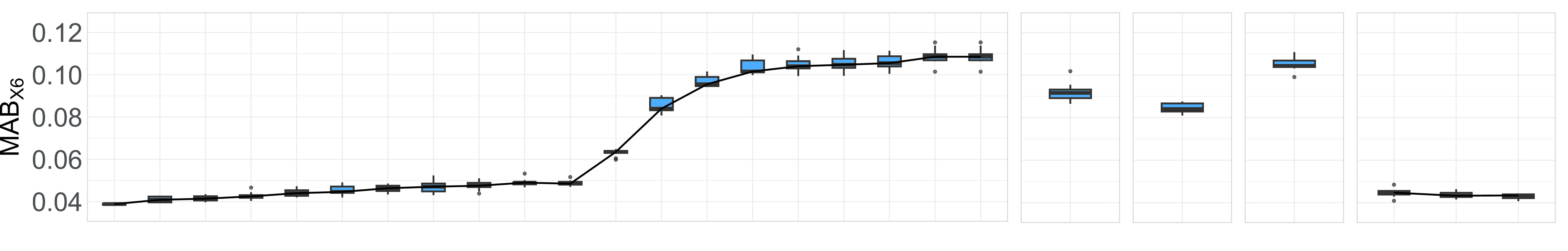}
     \includegraphics[width=0.7\columnwidth]{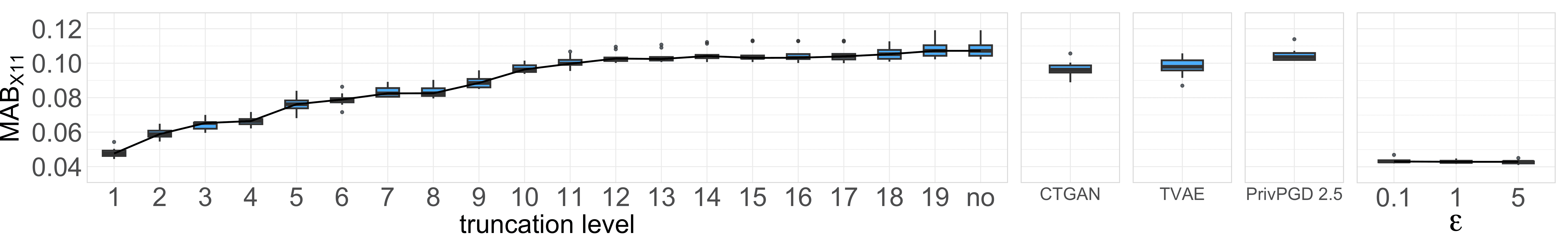}
    \caption{Simulated real data: Results of an AIA w.r.t. sensitive covariate $X_1$ (top row), $X_6$ (middle row) and $X_{11}$ (bottom row) measured by $MAB_j$. Synthetic data are generated with a C-vine for different truncation levels (left), CTGAN (2nd), TVAE (3rd), PrivPGD  with $\epsilon = 2.5$ and $ \delta= 10^{-5}$ (4th) and PrivBayes (right) for privacy parameter $\epsilon \in \{0.1, 1, 5\}$. Results are reported as box plots over 10 AIA game iterations. Parameters of the generative models and privacy attacks can be found in Appendix \ref{sec:model_and_attack_parameters}.
   }
    \label{fig:AIA_MAB_Cvine_competitors_simreal_X1_X6_X11}
\end{figure}

\begin{figure}[ht]
    \centering
    \includegraphics[width=0.7\columnwidth]{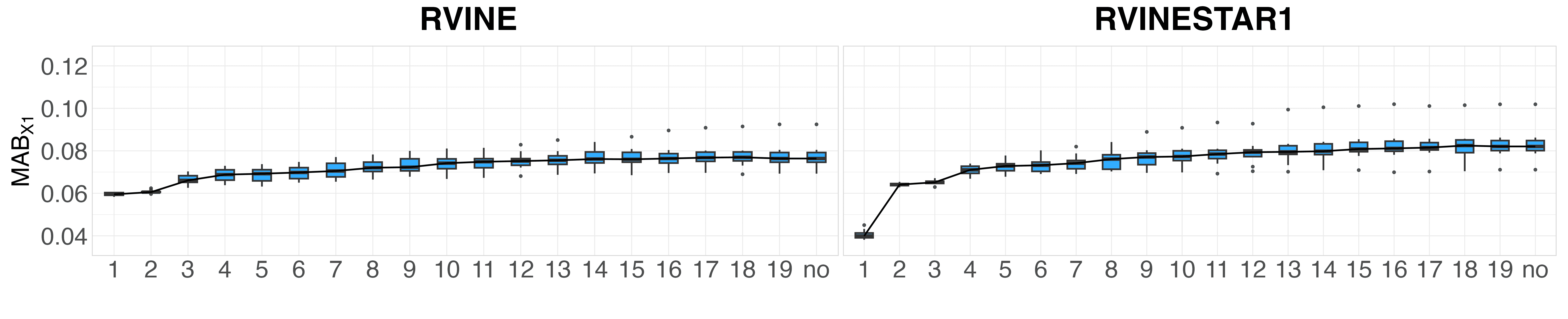}
    \includegraphics[width=0.7\columnwidth]{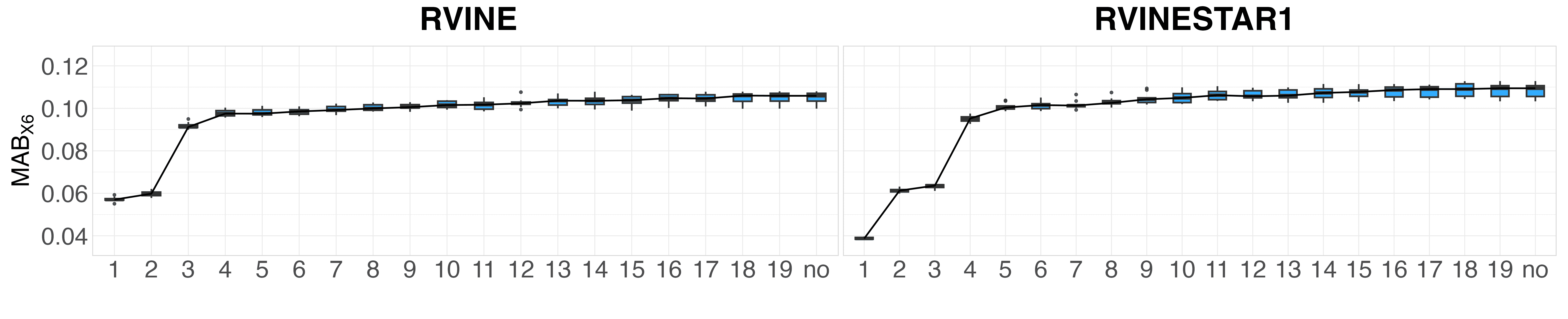}
    \includegraphics[width=0.7\columnwidth]{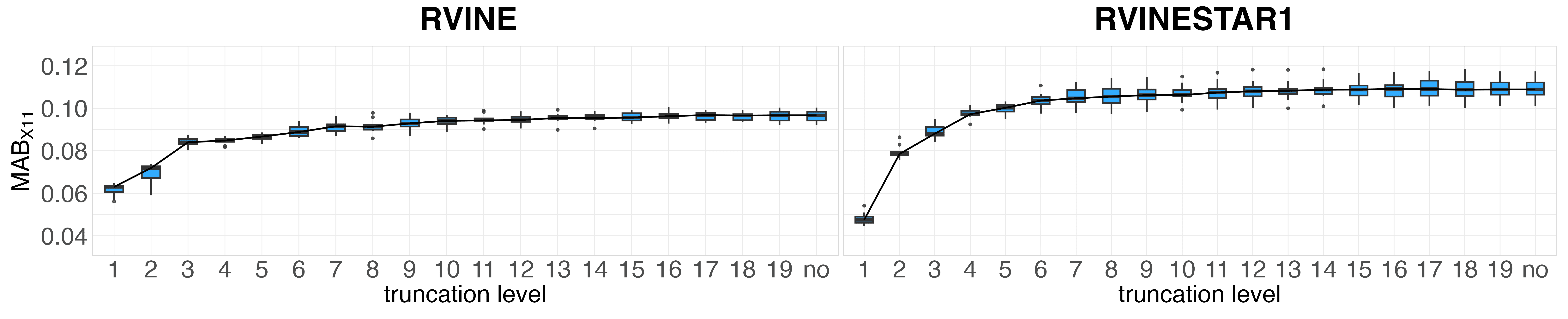}
    \caption{The lower the $MAB_j$ of AIA w.r.t. sensitive covariate $X_1$, $X_6$ and $X_{11}$, the more private the synthetic data generated by a R-vine and R-vine star1 for different truncation levels. Results are reported as box plots over 10 AIA game iterations. Parameters of the generative models and privacy attacks can be found in Appendix \ref{sec:model_and_attack_parameters}.}
    \label{fig:AIA_MAB_Rvine_Rvinestar1_simreal}
\end{figure}

\subsubsection{Privacy: Membership Inference Attack}\label{sec:simreal_results_MIA}

In Figure \ref{fig:MIA_Cvine_competitors_simreal} we observe that the PG of C-vine generated synthetic data is around 1 with low variation for \textit{all} truncation levels, indicating optimal privacy w.r.t MIA, independent of whether the target observation is randomly sampled (in blue) or an outlier (in orange). 

The PG of the C-vine is seemingly independent of truncation level because the estimation of the un-truncated C-vine done with Maximum Likelihood is robust w.r.t. adding/removing a single observation to the real data. The robustness of Maximum Likelihood estimation (MLE) depends on the sample size of the (real) data. Thus, for lower sample sizes than the ones we use here, we would expect to see a MIA PG that varies more with truncation level. 
As a consequence also a C-vine truncated at level $t < d$ shows the same robustness, because it results from the un-truncated C-vine by setting pair copulas in tree levels $t + 1$ and higher to independence. 

The results of the C-vine are similar to PrivBayes model for privacy parameter $\epsilon \in \{0.1, 1, 5\}$. CTGAN also gives average PG of about 1, but exhibits a high variation in the PG over different observations and repetitions of the MIA. Synthetic data generated with a TVAE provide very low protection against MIAs with a PG of around 0, hinting on that they include too much details of the real data which are harmful for privacy. The PrivPGD model performs poorly in terms of PG. Even though the covariate ranges are not directly inferred from the sensitive real data, this might have an impact on MIA privacy. 

Similar to the C-vine, the R-vine and R-vine star1 score a PG of 1 at median with little variation over different observations and repetitions of the attack, see Figure \ref{fig:MIA_Rvine_Rvinestar1_simreald20}.

\begin{figure}[ht]
    \centering
    \includegraphics[width=0.7\columnwidth]{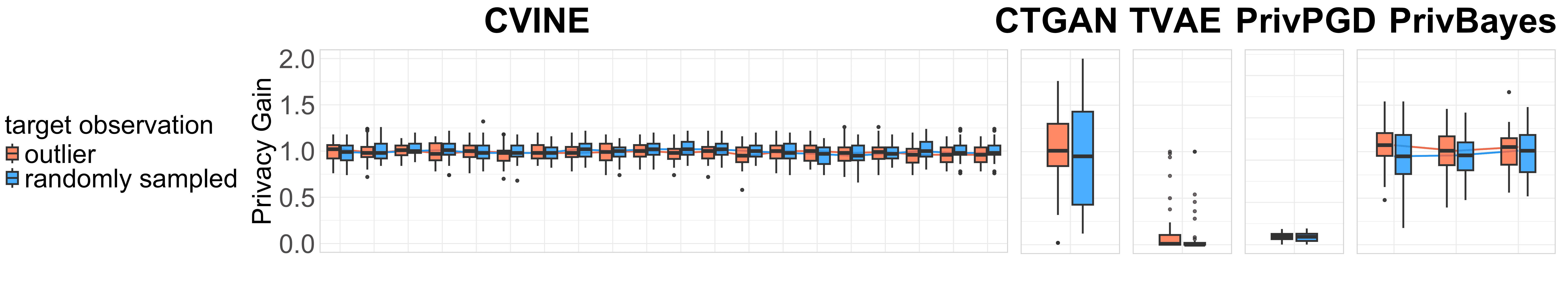}
    \includegraphics[width=0.7\columnwidth]{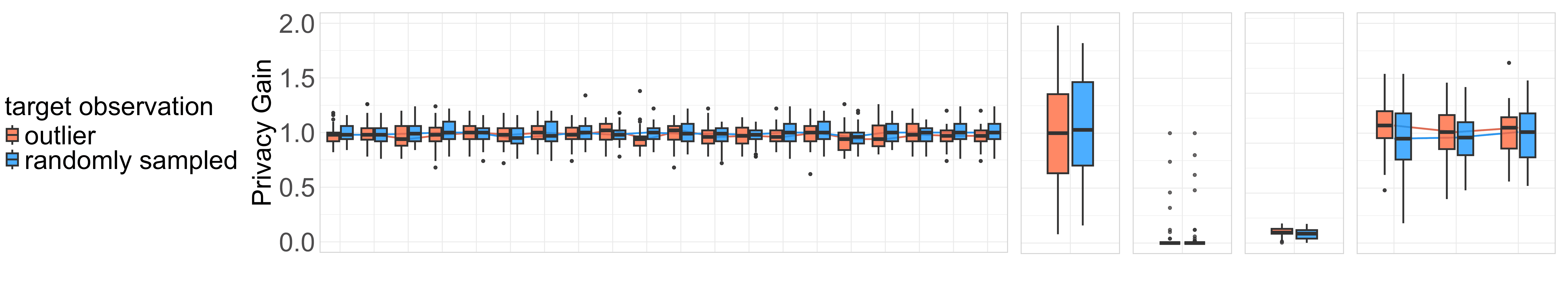}
     \includegraphics[width=0.7\columnwidth]{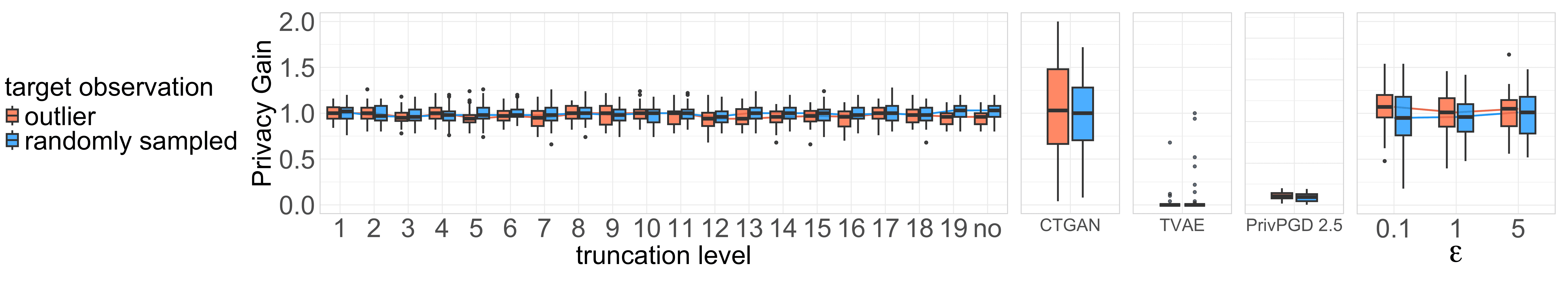}
    \caption{Simulated real data: PG under a MIA w.r.t randomly sampled target observations (in blue) and targets that are outliers  (in orange) w.r.t. $X_1$ (top row), $X_6$ (middle row) and $X_{11}$ (bottom row) of synthetic data are generated with a C-vine for different truncation levels (left), CTGAN (2nd), TVAE (3rd), PrivPGD with $\epsilon=2.5$ and $\delta=10^{-5}$ (4th) and PrivBayes for privacy parameter $\epsilon \in \{0.1, 1, 5\}$ (right). Results are reported as box plots over 10 MIA game iterations and 4 outlying (orange) and 5 randomly sampled (blue) target observations respectively. Parameters of the generative models and privacy attacks can be found in Appendix \ref{sec:model_and_attack_parameters}.
    }\label{fig:MIA_Cvine_competitors_simreal}
\end{figure}

\begin{figure}[ht]
    \centering
    \includegraphics[width=0.7\columnwidth]{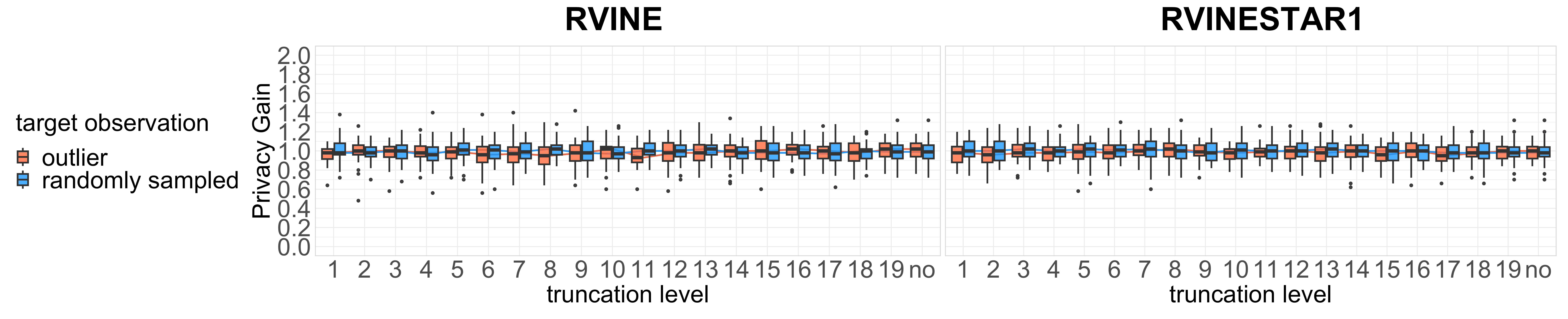}
    \includegraphics[width=0.7\columnwidth]{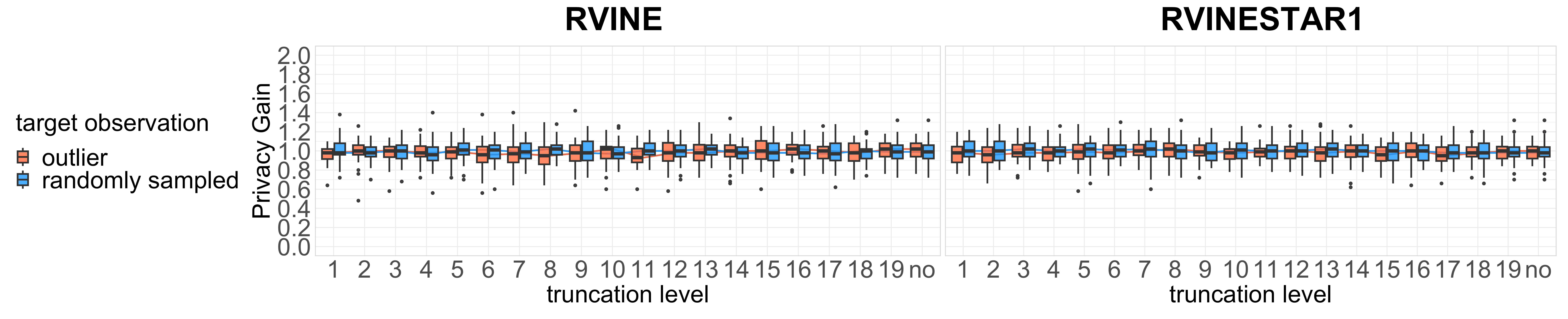}
    \includegraphics[width=0.7\columnwidth]{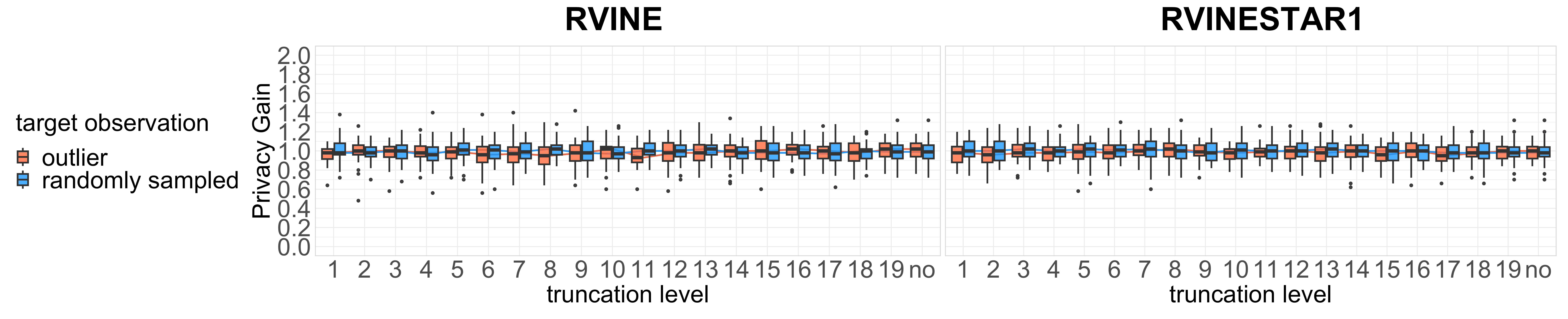}
    \caption{Simulated real data: Results of a MIA w.r.t randomly sampled target observations (in blue) and targets that are outliers  (in orange) w.r.t. $X_1$ (top row), $X_6$ (middle row) and $X_{11}$ (bottom row), measured by the $PG$. Synthetic data are generated with an R-vine (left) and an R-vine star1 (right) for different truncation levels. Results are reported as box plots over 10 MIA game iterations and 4 outlying (orange) and 5 randomly sampled (blue) target observations respectively Parameters of the generative models and privacy attacks can be found in Appendix \ref{sec:model_and_attack_parameters}. Parameters of the generative models and privacy attacks can be found in Appendix \ref{sec:model_and_attack_parameters}.}
    \label{fig:MIA_Rvine_Rvinestar1_simreald20}
\end{figure}

\subsubsection{Utility}\label{sec:simreal_results_utility}
For evaluating utility, 50 synthetic data sets of the same size as the simulated real data ($n=1000$) are generated from each model, a random forest classifier is trained on each of them and tested on a hold-out test data set of size $n_{test} = 250$.
From Figure \ref{fig:utility_Cvine_competitors_simreal} we observe that the C-vine generated synthetic data consistently outperform synthetic data generated from a CTGAN, PrivPGD and a PrivBayes model for all truncation levels. Only the TVAE scores a higher $AUC(\bm{y}^*, \hat{\bm{w}}^*)$ and comes closest to the performance of the classifier trained on real data of $AUC(\bm{y}^*, \hat{\bm{y}}^*) = 0.908$. Considering its high MAB under an AIA in Figure \ref{fig:AIA_MAB_Cvine_competitors_simreal_X1_X6_X11} and low PG under a MIA in Figure \ref{fig:MIA_Cvine_competitors_simreal}, the TVAE seems to model the real data too closely, thus violating privacy.

R-vine and R-vine star 1 generated synthetic data are as useful as C-vine generated synthetic data, see Figure \ref{fig:utility_RvineRvinestar1_simreal}. As they perform worse in terms of AIA privacy, we see ourselves confirmed in our choice of a C-vine as the core of TVineSynth.

\begin{figure}[ht]
    \centering
    \includegraphics[width=0.7\columnwidth]{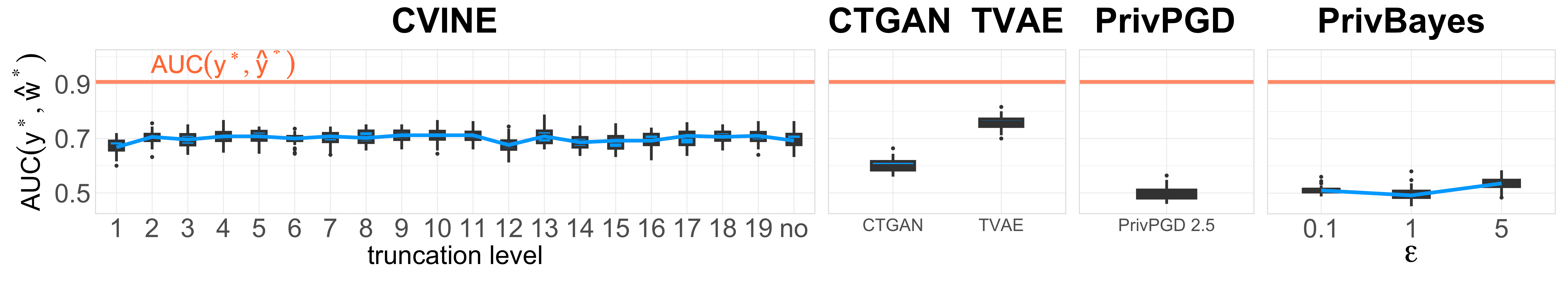}
    \caption{
    Simulated real data: Utility of synthetic data generated with a C-vine for different truncation levels (left), CTGAN (2nd), TVAE (3rd), PrivPGD with $\epsilon=2.5$ and $\delta=10^{-5}$ (4th) and PrivBayes for privacy parameter $\epsilon \in \{0.1, 1, 5\}$ (right) measured with $AUC(\bm{y}^*, \hat{\bm{w}}^*)$ (blue) w.r.t.  a random forest classifier and compared to $AUC(\bm{y}^*, \hat{\bm{y}}^*)$ (orange). Results are reported as box plots over 50 AUC values obtained from 50 synthetic data sets per generative model. Parameters of the generative models can be found in Appendix \ref{sec:model_and_attack_parameters}.}
    \label{fig:utility_Cvine_competitors_simreal}
\end{figure}

\begin{figure}[ht]
    \centering
    \includegraphics[width=0.7\columnwidth]{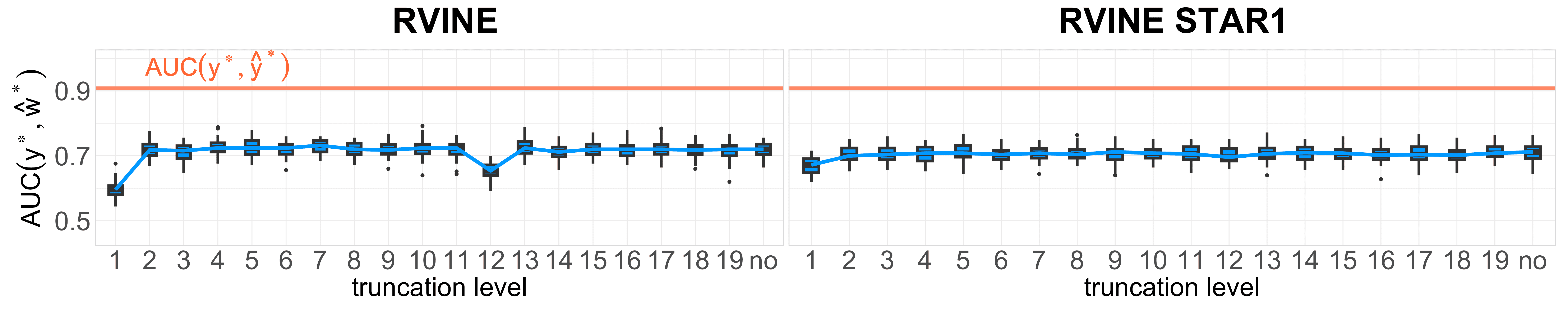}
    \caption{Simulated real data: Enforcing the first vine tree to be a star with $Y$ as the center as for R-vine star1 yields higher utility at truncation level 1 than compared to R-vine. Utility is consistently high for all truncation levels. It is measured with $AUC(\bm{y}^*, \hat{\bm{w}}^*)$ w.r.t.  a random forest classifier. Results are reported as box plots over 50 AUC values obtained from 50 synthetic data sets per generative model. Parameters of the generative models can be found in Appendix \ref{sec:model_and_attack_parameters}.}
    \label{fig:utility_RvineRvinestar1_simreal}
\end{figure}

\subsubsection{Privacy-Utility Plots}\label{sec:2dPrivUt_AIAX1_simreal}

Figures \ref{fig:2dPrivUt_simreal_X6AIA_MAB_Cvine}, \ref{fig:2d_priv-ut_simreal_AIA_MAB} and \ref{fig:2d_priv-ut_simreal_X1X6X11_MIA} illustrate the privacy-utility balance per model.
In the case of sensitive feature $X_1$ in Figure \ref{fig:2d_priv-ut_simreal_AIA_MAB}, the C-vine offers a well balanced privacy-utility trade-off for truncation between levels 11 and 18.
If we truncate at level 11 for sensitive covariate $X_6$ and at level 1 for sensitive covariate $X_{11}$, the TVineSynth generated synthetic data offer a privacy-utility balance superior to the one of the competitor models, see Figure \ref{fig:2d_priv-ut_simreal_AIA_MAB}.
Figure \ref{fig:2d_priv-ut_simreal_X1X6X11_MIA} displays the privacy-utility plots w.r.t. a MIA and and sensitive features $X_1$, $X_6$ and $X_{11}$. As already observed in Sections \ref{sec:simreal_results_MIA} and \ref{sec:simreal_results_utility}, all models except for TVAE and PrivPGD score a PG of around 1. Compared to the competitor models that are able to protect sensitive covariates against a MIA, the C-vine scores the highest utility.

\begin{figure}[ht]
    \centering
    \includegraphics[width=0.367\columnwidth]{images/2d_PrivUt_X1_AIA_MAB_simreal_PrivPGD_DP.png}
    \includegraphics[width=0.2685\columnwidth]{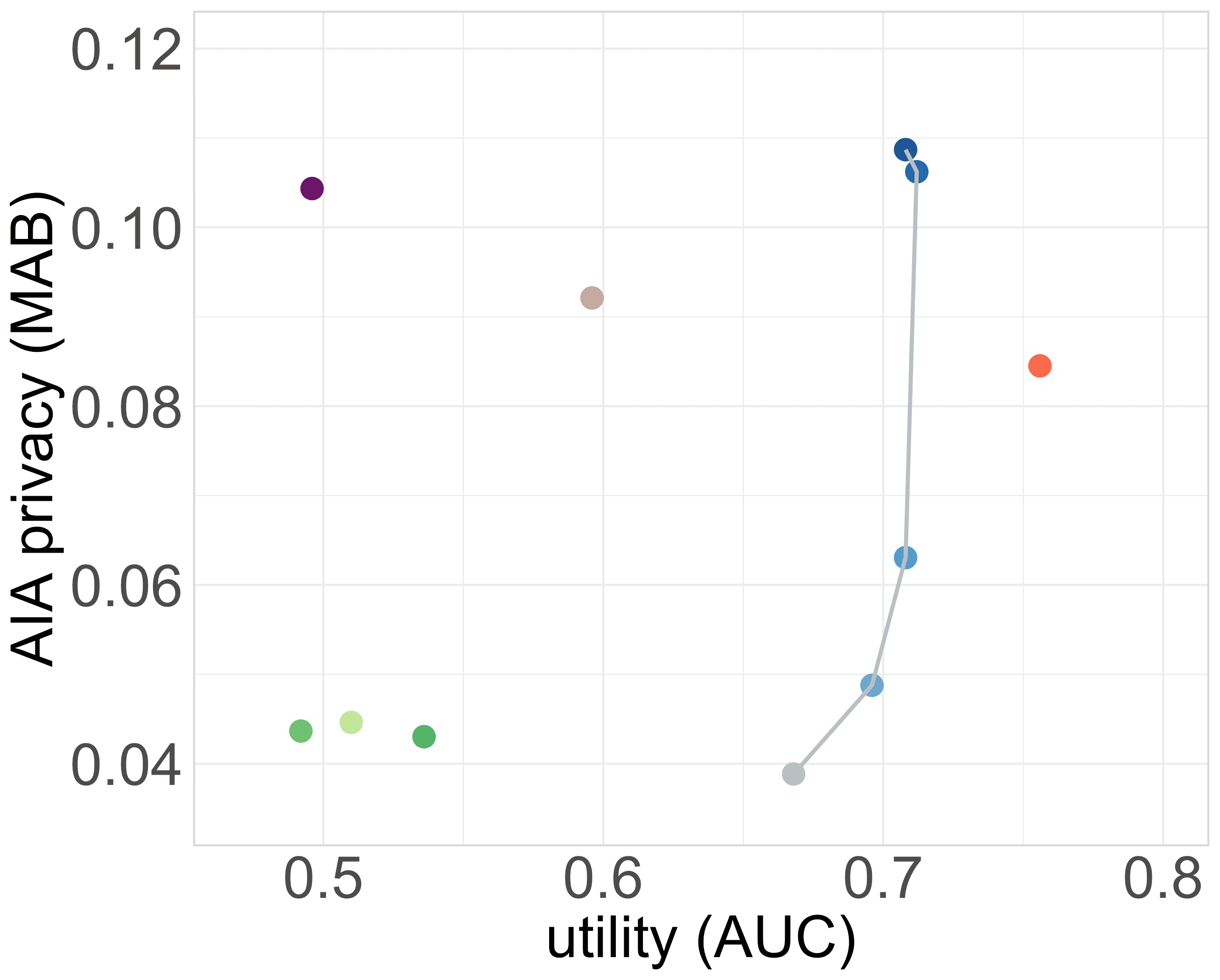}
    \includegraphics[width=0.2685\columnwidth]{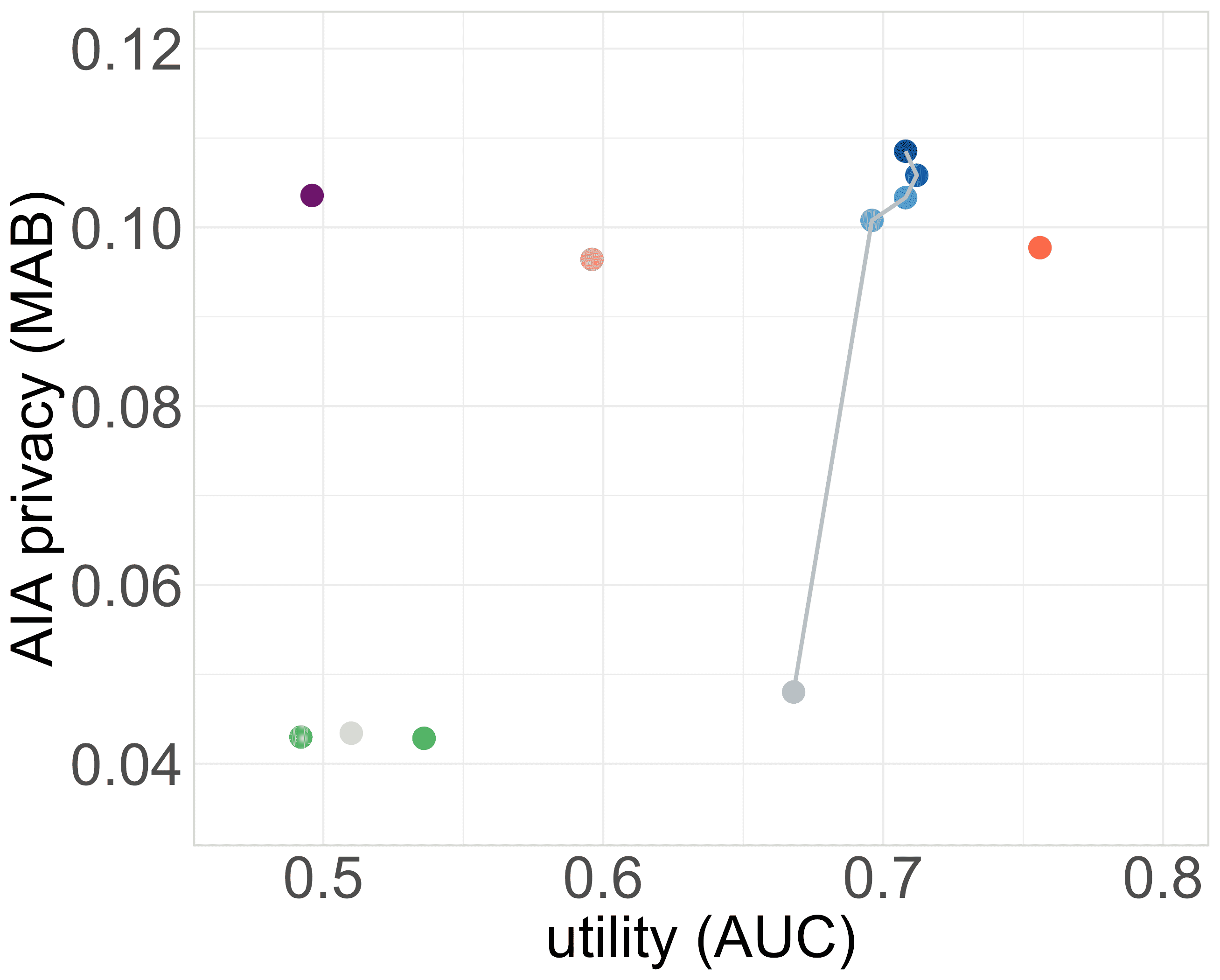}
    \caption{Simulated real data: Privacy-utility plot w.r.t. AIA and sensitive features $X_1$ (left), $X_6$ (middle) and $X_{11}$ (right) of a C-vine with truncation levels $t \in \{1, 11, 12, 18\}$ and no truncation, PrivPGD  with $\epsilon = 2.5$ and $ \delta= 10^{-5}$, PrivBayes model with $\epsilon \in \{0.1,1,5\}$, CTGAN and TVAE on simulated real data of Section \ref{subsec:results_simulated_realId20}. For AIA privacy the $MAB_j$ is reported, for utility the median over 50 synthetic data sets is reported. Parameters of the generative models and privacy attacks can be found in Appendix \ref{sec:model_and_attack_parameters}.}
    \label{fig:2d_priv-ut_simreal_AIA_MAB}
\end{figure}

\begin{figure}[ht]
    \centering
    \includegraphics[width=0.367\columnwidth]{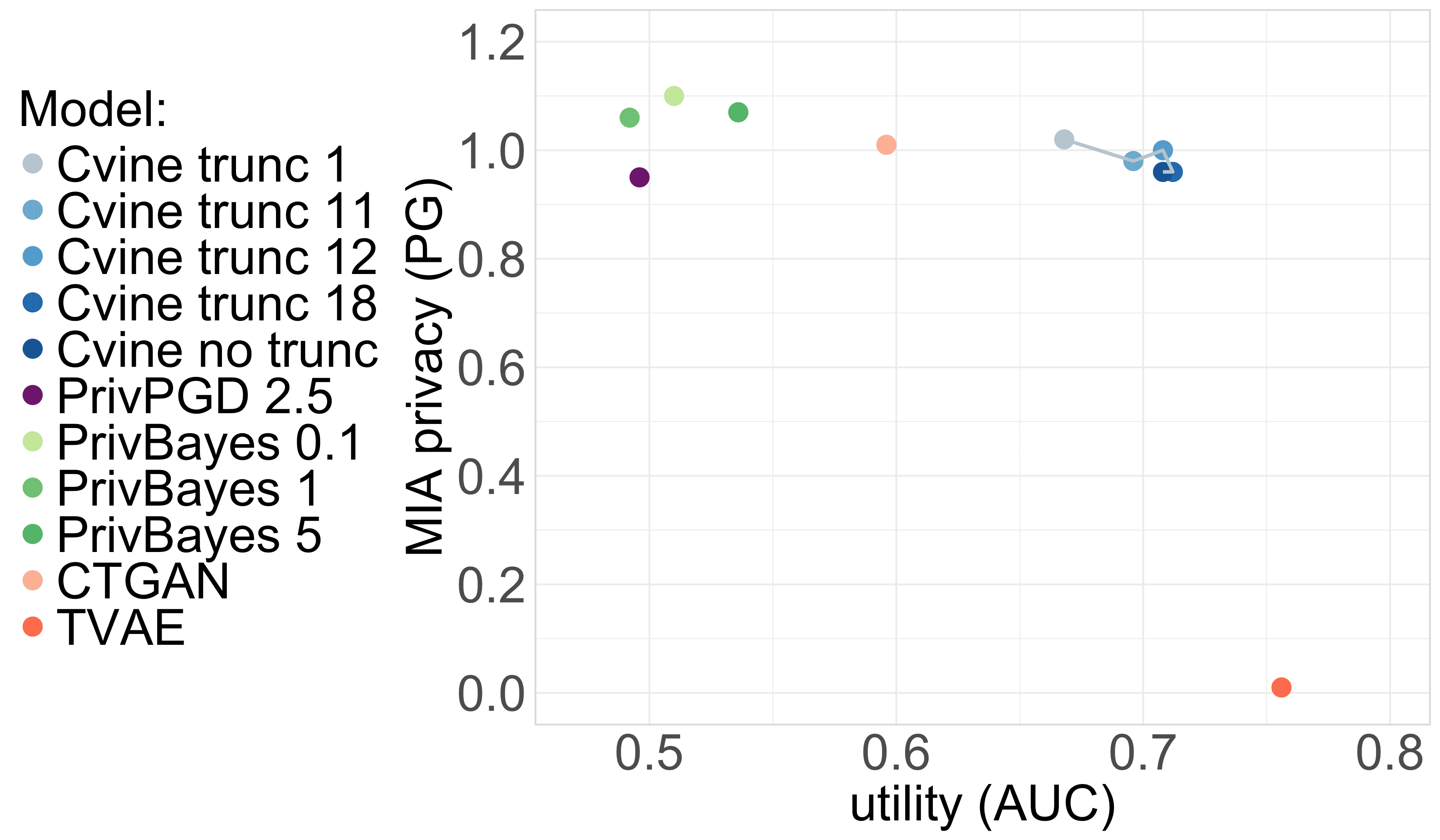}
    \includegraphics[width=0.2685\columnwidth]{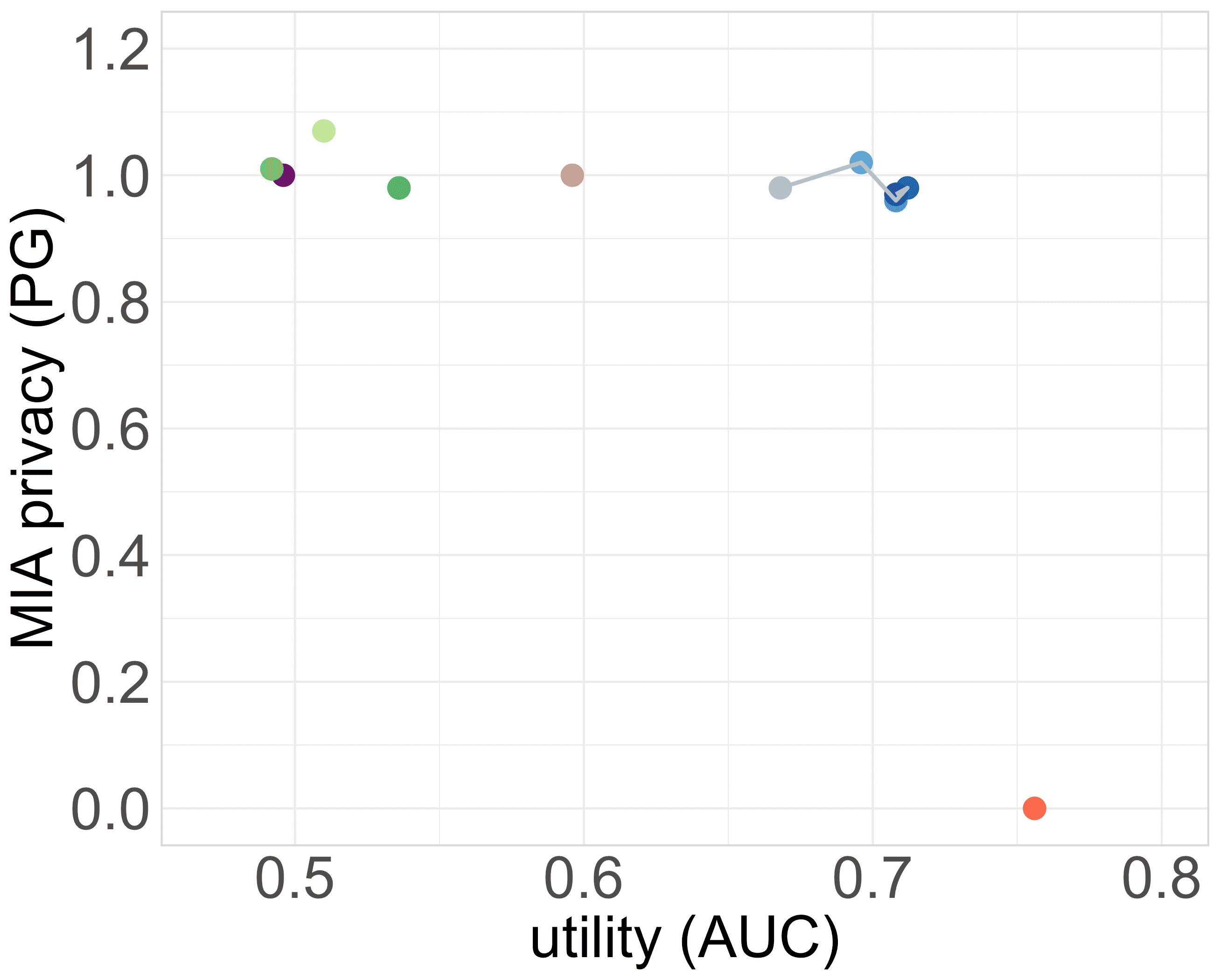}
    \includegraphics[width=0.2685\columnwidth]{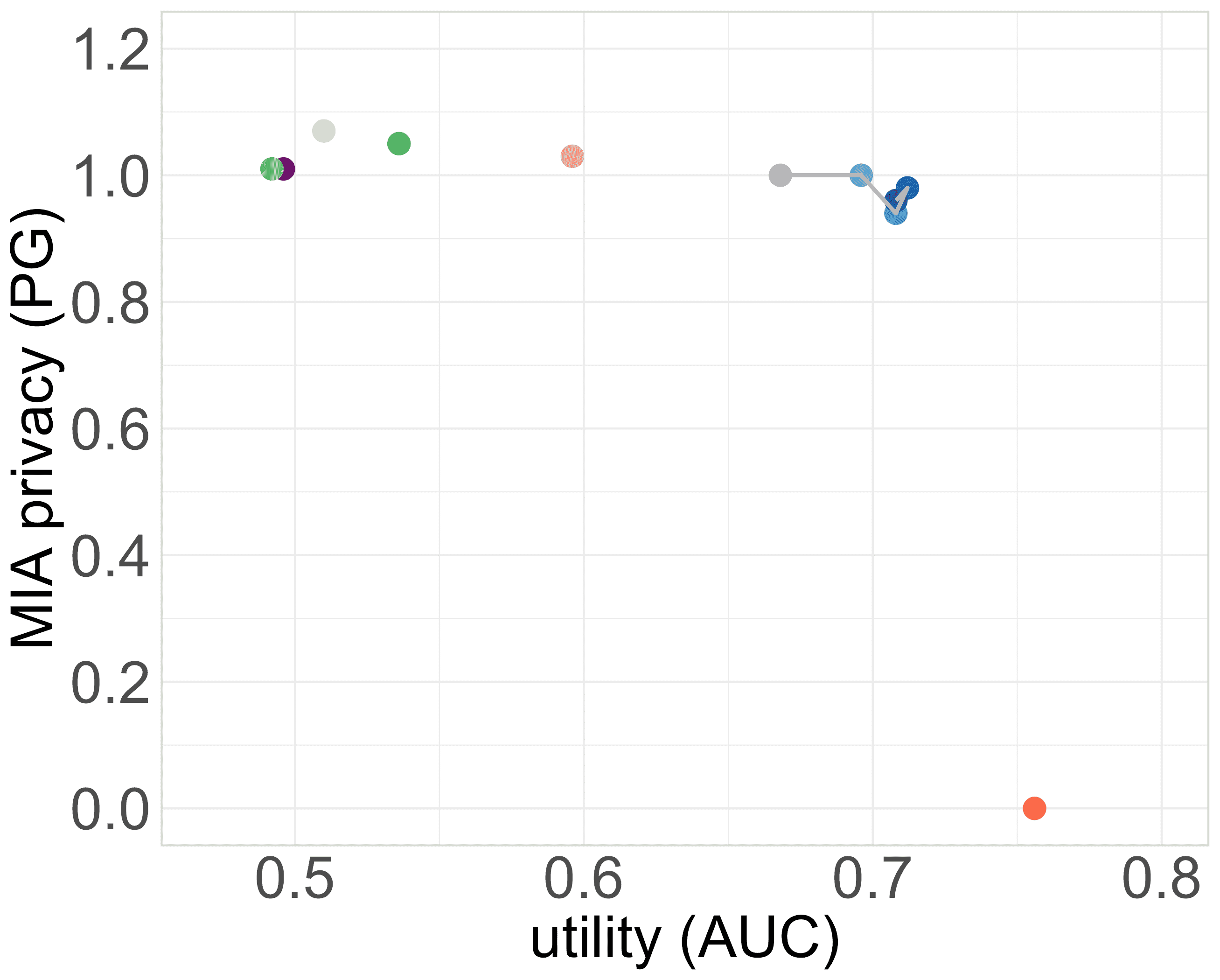}
    \caption{Simulated real data: Privacy-utility plot w.r.t. MIA and sensitive features $X_1$ (left), $X_6$ (middle) and $X_{11}$ (right) of a C-vine with truncation levels $t \in \{1,11,12,18\}$ and no truncation, PrivBayes model with $\epsilon \in \{0.1,1,5\}$, CTGAN and TVAE on simulated real data of Section \ref{sec:simulated_real_data_d20_appendix}. The median MIA PG over all game iterations and utility for 50 synthetic data sets are reported. Parameters of the generative models and privacy attacks can be found in Appendix \ref{sec:model_and_attack_parameters}.} 
    \label{fig:2d_priv-ut_simreal_X1X6X11_MIA}
\end{figure}

\subsubsection{Statistical Discrepancy}\label{sec:statFidelity_simreal}
We measure the statistical discrepancy between joint real and synthetic distribution with $\alpha$-precision, $\beta$-recall and authenticity $(P_{\alpha}, R_{\beta}, A)$ introduced by \citet{alaa2022faithful}. Their definition can be found in Appendix \ref{sec:statistical_fidelity}.

From Figure \ref{fig:statfidelity_allmodels_simreal} it can be observed that increasing truncation level of the C-vine improves fidelity and diversity of the synthetic data while it decreases their generalization. While PrivBayes generated synthetic data score very poorly in diversity (around 0) and moderately in fidelity (0.63 - 0.67), they achieve very high authenticity of around 1, indicating that the synthetic data do not reflect the real data sufficiently well. PrivPGD's diversity and authenticity (0.7 and 0.63) compare to the one of the C-vine truncated at level 1 (0.71 and 0.66), but achieves a fidelity of 0.76 that is considerably lower than the one of the C-vine truncated at level 1 (0.94). The TVAE performs very comparably to the C-vine truncated at level 10 in terms of statistical fidelity. The rather high generalization and rather low diversity of CTGAN generated data appear plausible w.r.t. the the model's results of Section \ref{sec:simreal_results_utility}.

\begin{figure}[ht]
    \centering
    \includegraphics[width=\columnwidth]{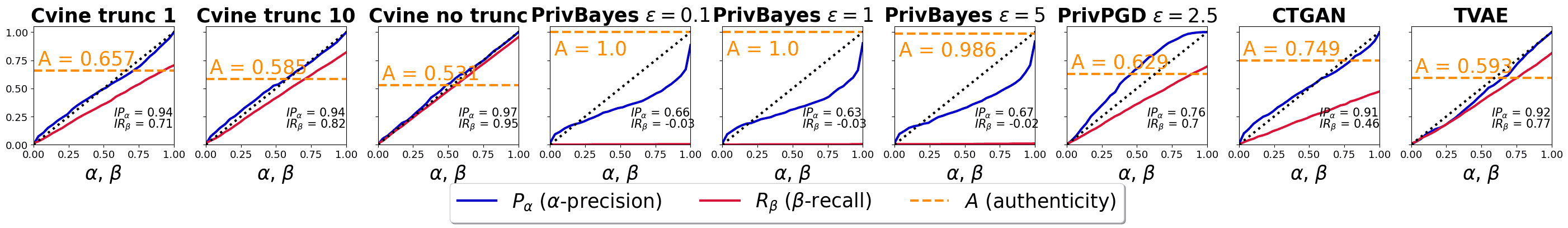}
    \caption{Simulated real data: Fidelity ($\alpha$-precision), diversity ($\beta$-recall) and generalization (authenticity) of synthetic data generated in the order C-vine for truncation at levels 1 and 10 and no truncation, PrivBayes for $\epsilon \in \{0.1, 1, 5\}$, PrivPGD with $\epsilon=2.5$, $\delta=10^{-5}$, CTGAN and TVAE. Parameters of the generative models can be found in Appendix \ref{sec:model_and_attack_parameters}.
    } \label{fig:statfidelity_allmodels_simreal}
\end{figure}

\subsubsection{AIA Results in Terms of WCAB}\label{sec:appendix_simreal_AIA_additional_results_WCAB}

Figure \ref{fig:AIA_WCAB_Cvine_competitors_simreal_X1_X6_X11} displays the AIA results in terms of WCAB that give a worst-case assessment of how much information covariates in the synthetic data leak on the sensitive feature. They support and further strengthen the observations on the MAB of Figure \ref{fig:AIA_MAB_Cvine_competitors_simreal_X1_X6_X11}. For sensitive feature $X_1$ truncating the C-vine at level 18 or lower providing a worst-case privacy superior to the differentially private PrivBayes. The same holds if the sensitive covariate is $X_6$ and we truncate the C-vine at level 11 or lower. Even for the sensitive feature $X_{11}$, which informs $Y$, the WCAB of the C-vine is comparable or lower than that of the competitors.

\begin{figure}[ht]
    \centering
    \includegraphics[width=0.9\columnwidth]{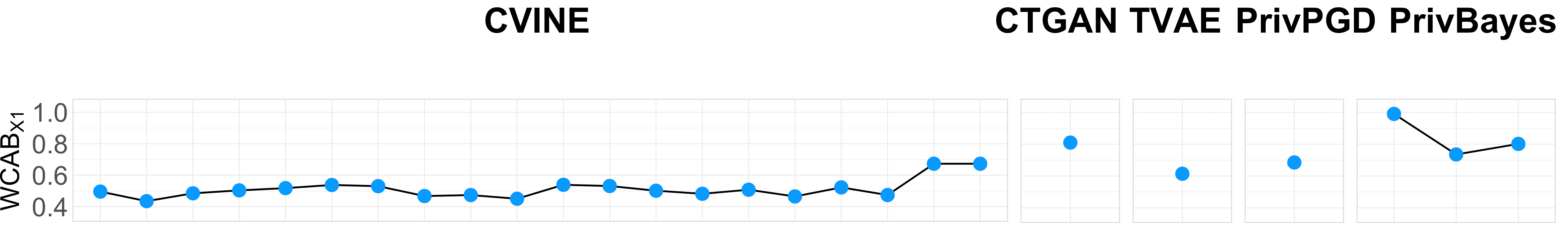}
    \includegraphics[width=0.9\columnwidth]{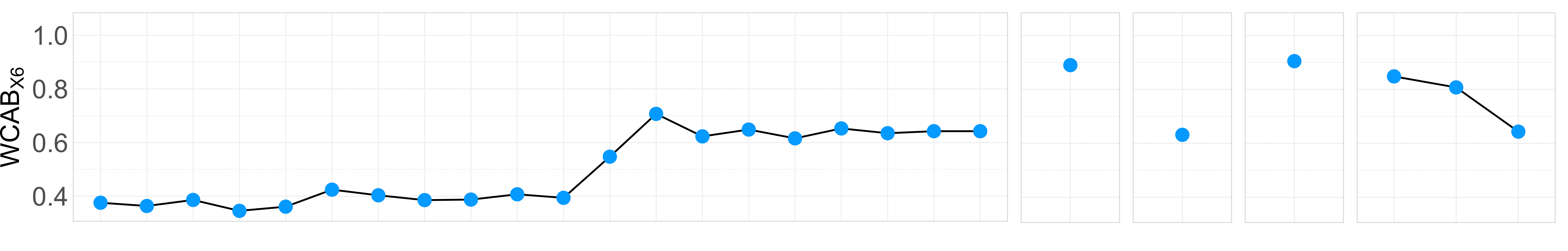}
     \includegraphics[width=0.9\columnwidth]{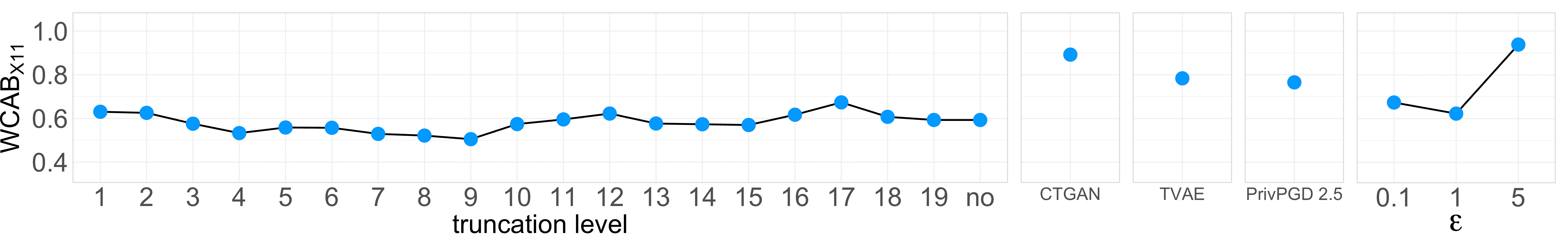}
    \caption{Simulated real data: Results of an AIA w.r.t. sensitive covariate $X_1$ (top row), $X_6$ (middle row) and $X_{11}$ (bottom row) measured by $WCAB_{j^*}$. Synthetic data are generated with a C-vine for different truncation levels (left), CTGAN (2nd), TVAE (3rd), PrivPGD  with $\epsilon = 2.5$ and $ \delta= 10^{-5}$ (4th) and PrivBayes (right) for privacy parameter $\epsilon \in \{0.1, 1, 5\}$. Results are reported over 10 AIA game iterations. Parameters of the generative models and privacy attacks can be found in Appendix \ref{sec:model_and_attack_parameters}.
   }
    \label{fig:AIA_WCAB_Cvine_competitors_simreal_X1_X6_X11}
\end{figure}

\begin{figure}[ht]
    \centering
    \includegraphics[width=1\columnwidth]{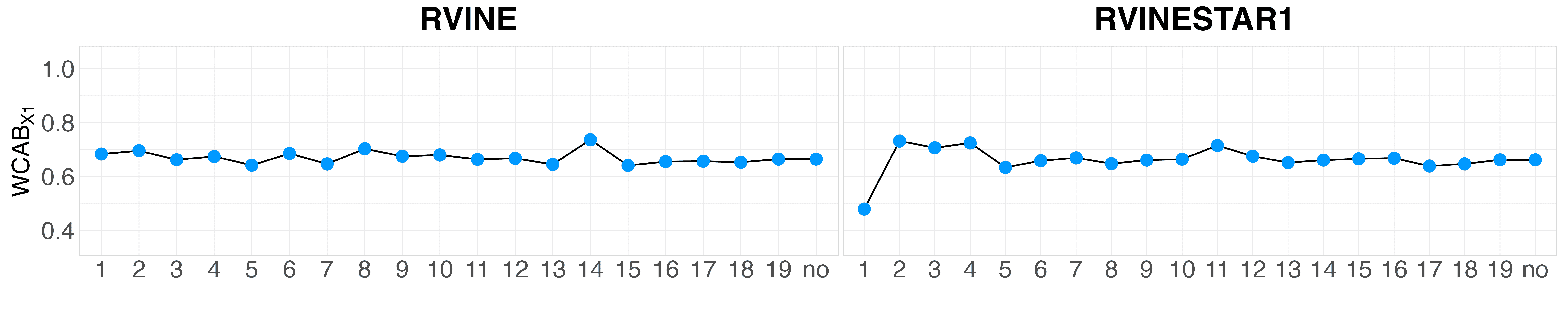}
    \includegraphics[width=1\columnwidth]{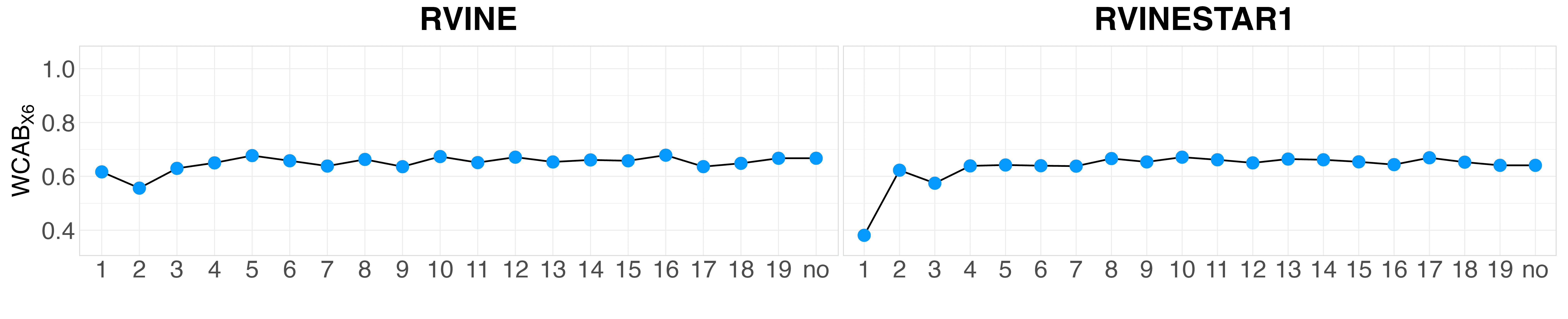}
    \includegraphics[width=1\columnwidth]{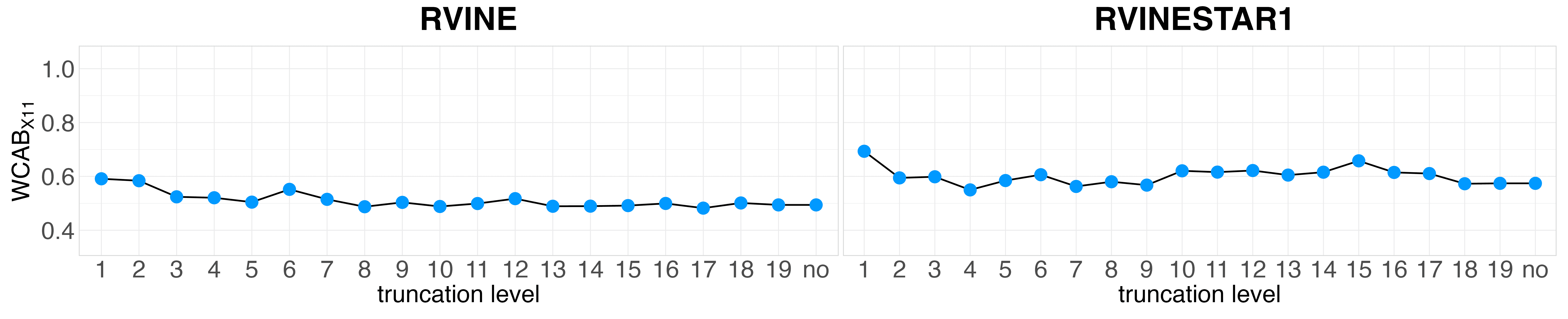}
    \caption{Simulated real data: The lower the $WCAB_{j^*}$ of AIA w.r.t. sensitive covariate $X_1$, $X_6$ and $X_{11}$, the more private the synthetic data generated by a R-vine and R-vine star1 for different truncation levels. Results are reported over 10 AIA game iterations. Parameters of the generative models and privacy attacks can be found in Appendix \ref{sec:model_and_attack_parameters}.}
    \label{fig:AIA_WCAB_Rvine_Rvinestar1_simreal}
\end{figure}

\subsubsection{AIA Results in Terms of MSE}\label{sec:appendix_simreal_AIA_additional_results_MSE}

The top row of Figure \ref{fig:AIA_Cvine_competitors_simreal_MSE} corresponds to the case where the sensitive covariate $X_1$ is less important for classifying $Y$ correctly. In this situation the star shaped C-vine combined with truncation at level 18 or lower is able to cut away sensitive dependencies that harm privacy but do not contribute to utility. If the sensitive covariate is $X_6$, again playing a less important role for classifying $Y$ correctly, we observe in the second row of Figure
\ref{fig:AIA_Cvine_competitors_simreal_MSE} that truncating a C-vine at level 11 or lower offers a high level of privacy for outliers (in orange). These findings are consistent with our observations from the results in terms of the MAB in Figure \ref{fig:AIA_MAB_Cvine_competitors_simreal_X1_X6_X11} and the correlation structure in Figures \ref{fig:corr_I_d20_append} and \ref{fig:Cvine_synth_data_corr_real_data_I_d20}.
Thus, the C-vine offers a high level of privacy for outliers (in orange) w.r.t. AIA, which is comparable to the one of the DP PrivBayes model at a very strict privacy budget of $\epsilon = 0.1$, and better than the one of DP PrivPGD. Simultaneously, the C-vine achieves high utility for all truncation levels, outperforming PrivBayes and PrivPGD by far, see Figure \ref{fig:utility_Cvine_competitors_simreal}. Sensitive covariate $X_{11}$ on the other hand shows pairwise association with $Y$, see Figure \ref{fig:corr_I_d20_append}. In this case it is necessary to truncate the C-vine at level 1 to provide privacy w.r.t. AIA, see bottom row of Figure \ref{fig:AIA_Cvine_competitors_simreal_MSE}.

\begin{figure}[ht]
    \centering
    \includegraphics[width=0.9\columnwidth]{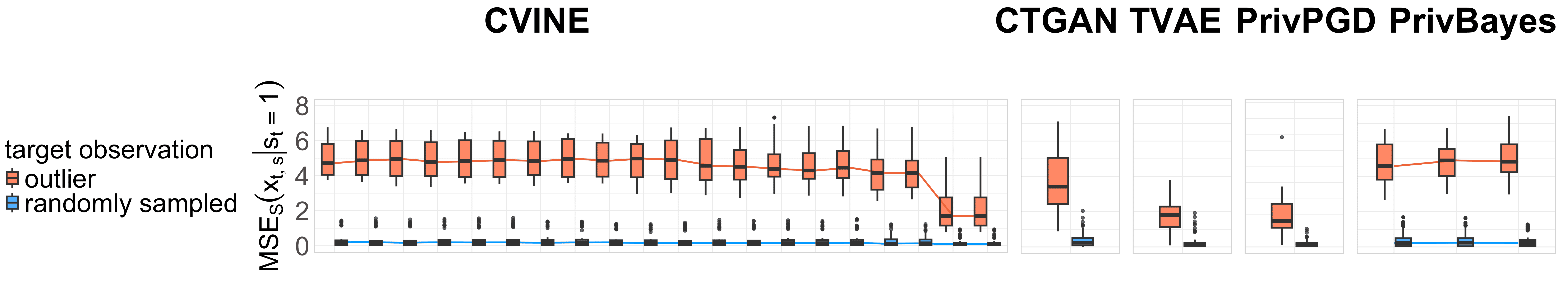}
    \includegraphics[width=0.9\columnwidth]{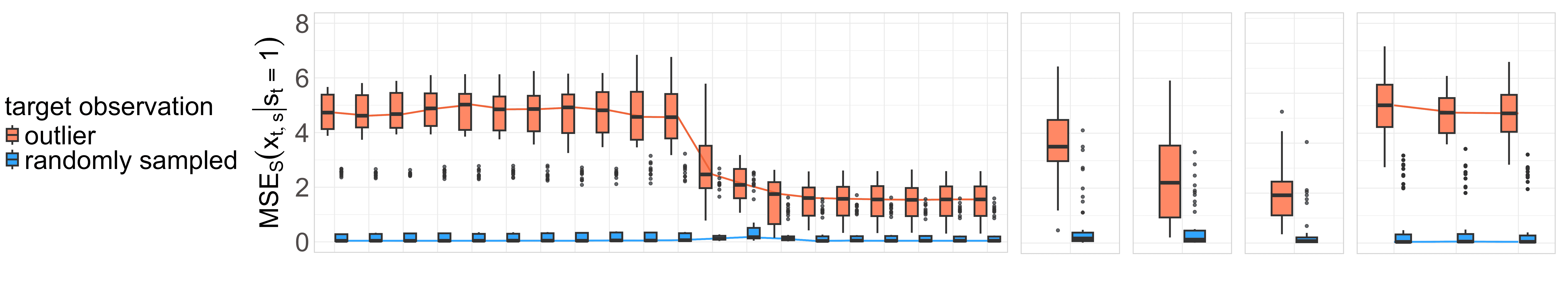}
     \includegraphics[width=0.9\columnwidth]{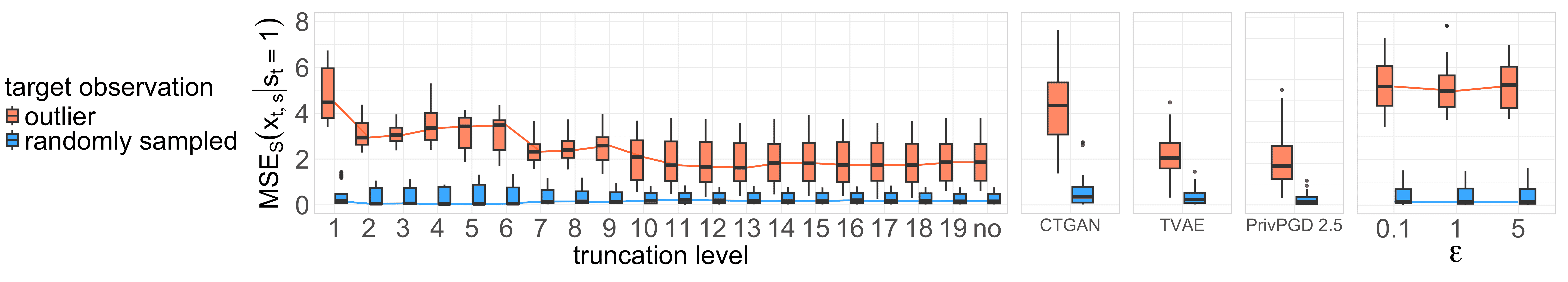}
    \caption{Simulated real data: Results of an AIA w.r.t. sensitive covariate $X_1$ (top row), $X_6$ (middle row) and $X_{11}$ (bottom row) and randomly sampled (blue) and handpicked, outlying target observations (orange) measured by $MSE_S(x_{t,s} | s_t = 1)$. Synthetic data are generated with a C-vine for different truncation levels (left), CTGAN (2nd), TVAE (3rd), PrivPGD  with $\epsilon = 2.5$ and $ \delta= 10^{-5}$ (4th) and PrivBayes (right) for privacy parameter $\epsilon \in \{0.1, 1, 5\}$. Results are reported as box plots over 10 AIA game iterations and 4 outlying (orange) and 5 randomly sampled (blue) target observations respectively. Parameters of the generative models and privacy attacks can be found in Appendix \ref{sec:model_and_attack_parameters}.
   }
    \label{fig:AIA_Cvine_competitors_simreal_MSE}
\end{figure}

From Figure \ref{fig:AIA_Cvine_competitors_simreal_MSE} we observe that randomly sampled target observations (in blue) that are close to the median of the sensitive covariate show very low $MSE_S(x_{t,s} | s_t = 1)$.
This raises the question of whether this actually presents a privacy breach. It does \textit{not}, if it suffices for the attacker to merely guess the mean of the respective sensitive covariate without regarding the other non-sensitive covariates! In other words, if in the attacker's regression model, the coefficients of the respective non-sensitive covariates are (close to) 0, the synthetic data does not offer more information on the sensitive covariate value than what we really wish to learn from the synthetic data, i.e. aggregate information such as the mean of a covariate. For this reason, we assess the regression coefficients in the AIA model for C-vine generated synthetic data and sensitive covariates $X_1$, see Figure \ref{fig:AIA_Cvine_X1_realI_d20_regCoeff}. 
There we indeed find that the results of Figure \ref{fig:AIA_Cvine_competitors_simreal_MSE} do not present an impairment of privacy. For sensitive covariate $X_1$, the regression coefficients of non-sensitive covariates displayed in Figure \ref{fig:AIA_Cvine_X1_realI_d20_regCoeff} are at median 0 for all target observations up to truncation level 16, as we would expect from Figure \ref{fig:Cvine_synth_data_corr_real_data_I_d20}. This means that even though $MSE_S(x_{t,s} | s_t = 1)$ is low in those cases, the attacker's guess is merely based on the mean of the respective sensitive covariate and guessing a covariate's mean correctly does not leak private information but confirms the synthetic data still allow to learn aggregate information about the real data as it is our goal.

These considerations build the basis for Definition \ref{def:MAB} of the MAB.

\begin{figure}[ht]
    \centering
    \begin{tabular}{cccc}
        \includegraphics[width=0.18\textwidth]{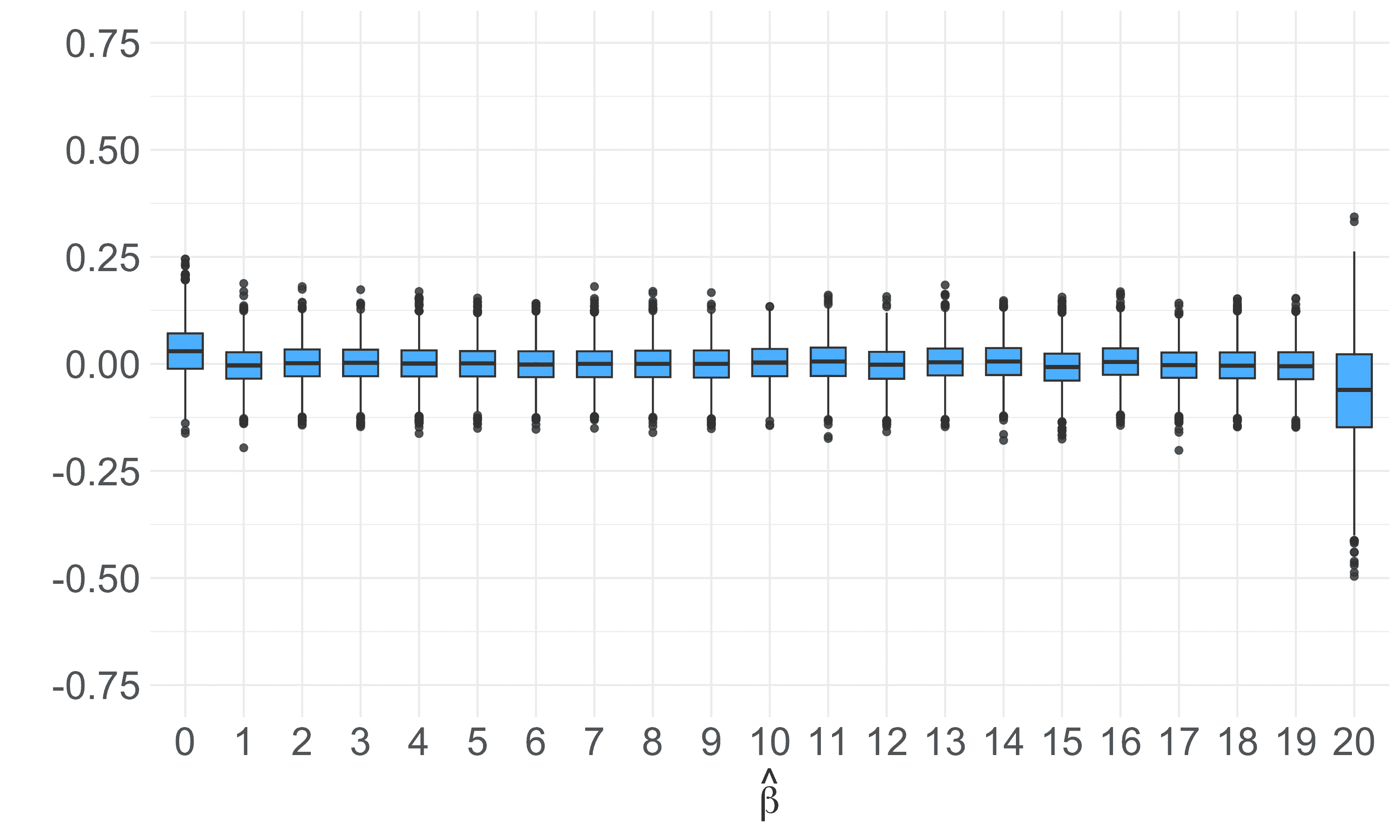} &  
        \includegraphics[width=0.18\textwidth]{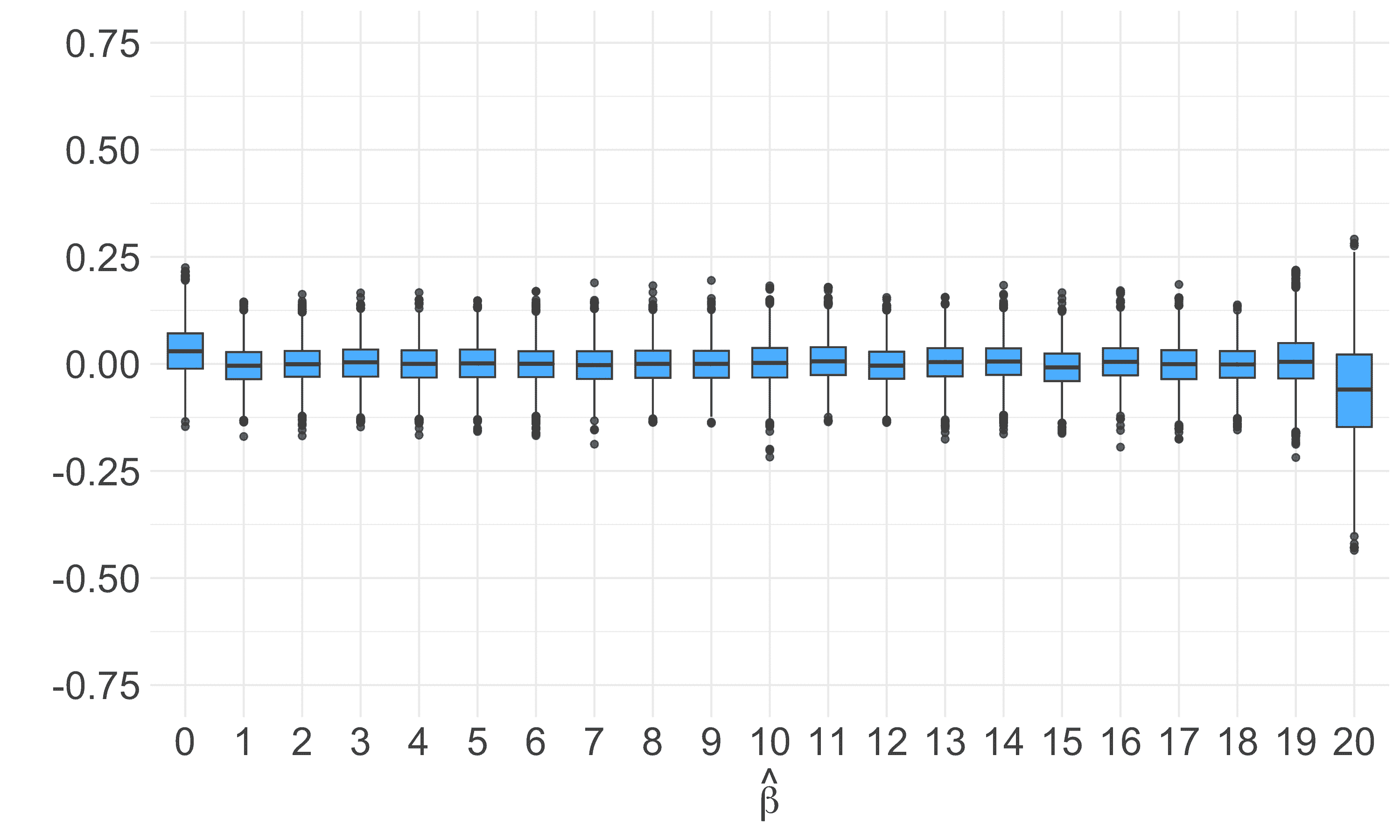} &
        \includegraphics[width=0.18\textwidth]{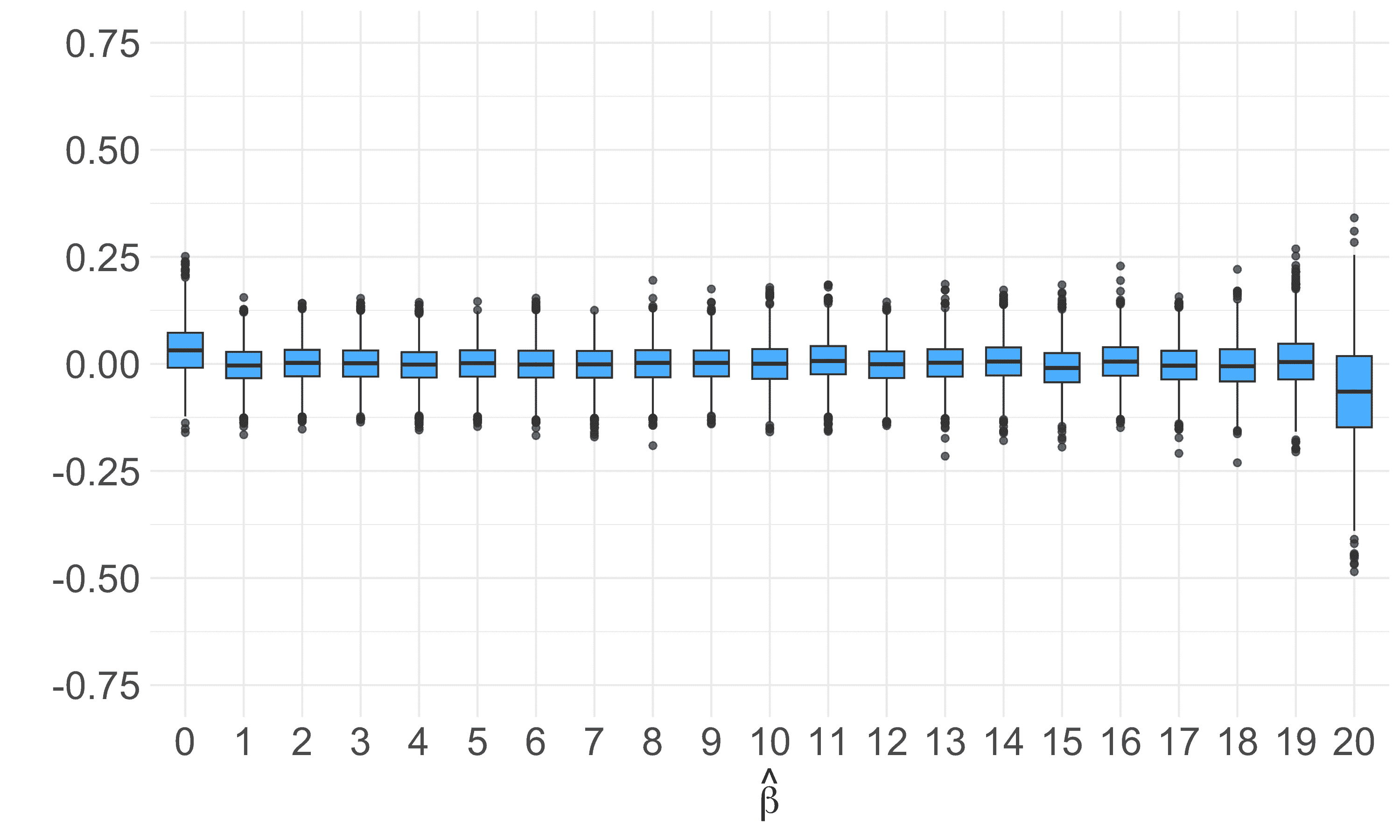} &
        \includegraphics[width=0.18\textwidth]{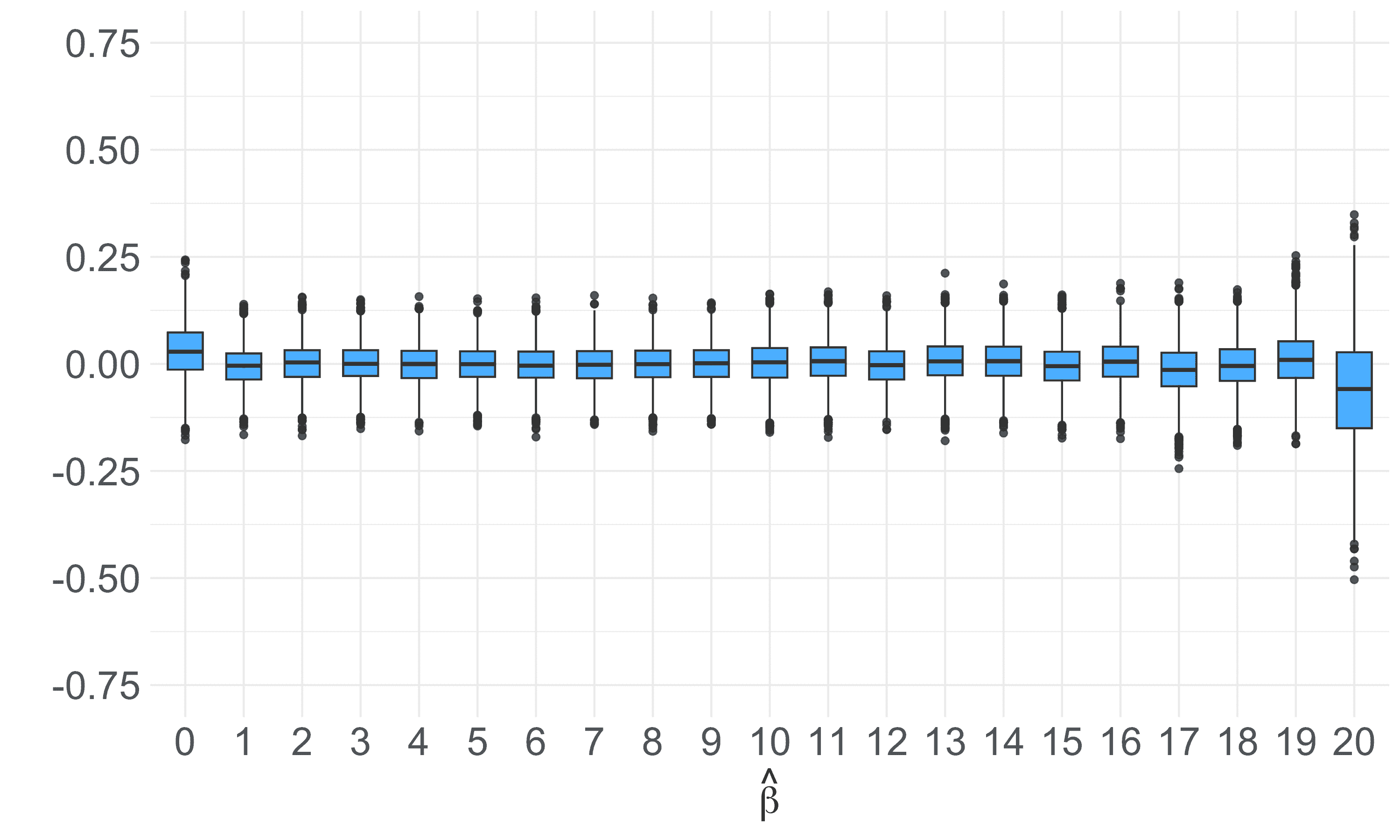} \\
        (a) Truncation at 1. & (b) Truncation at 2. & (c) Truncation at 3. & (d) Truncation at 4. \\
        \includegraphics[width=0.18\textwidth]{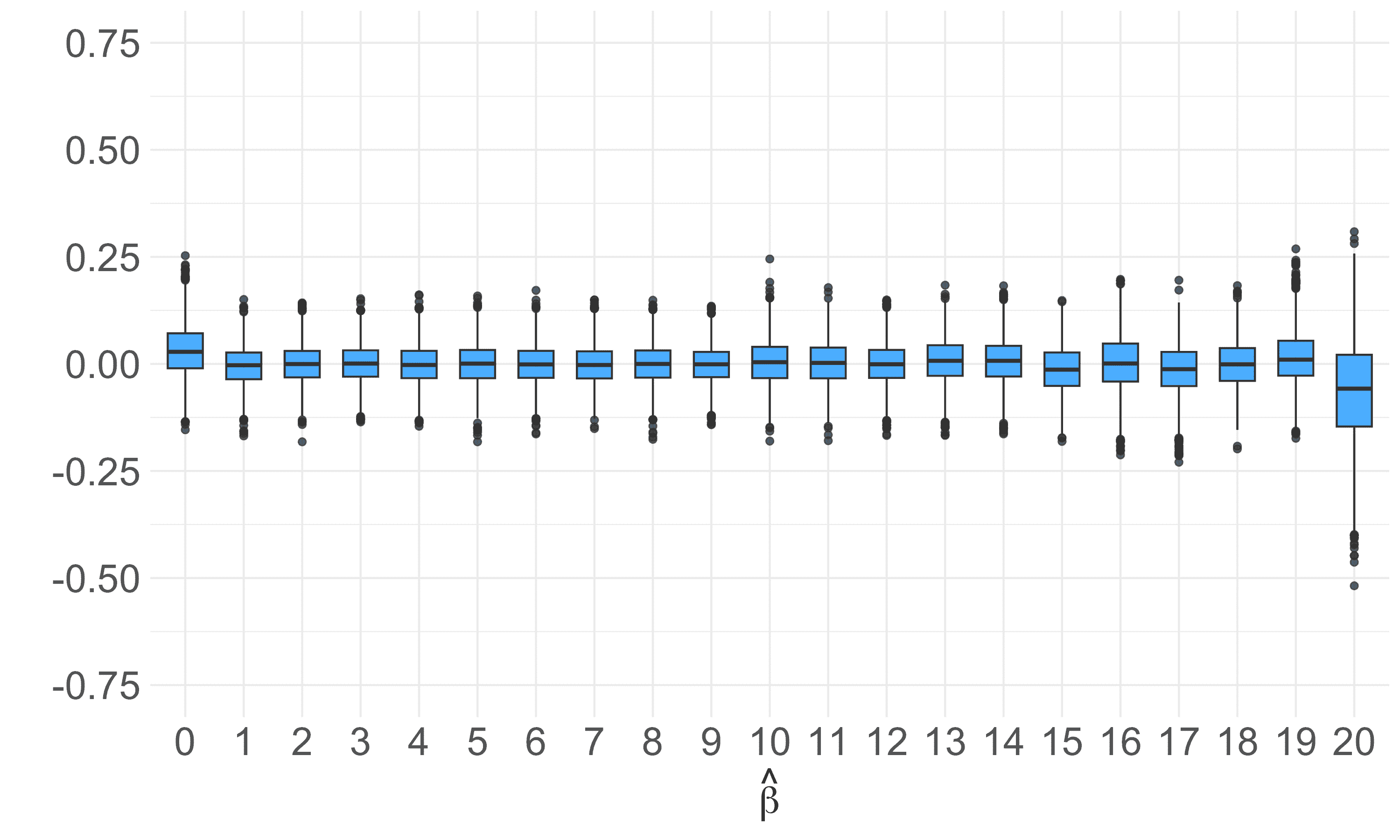} &
        \includegraphics[width=0.18\textwidth]{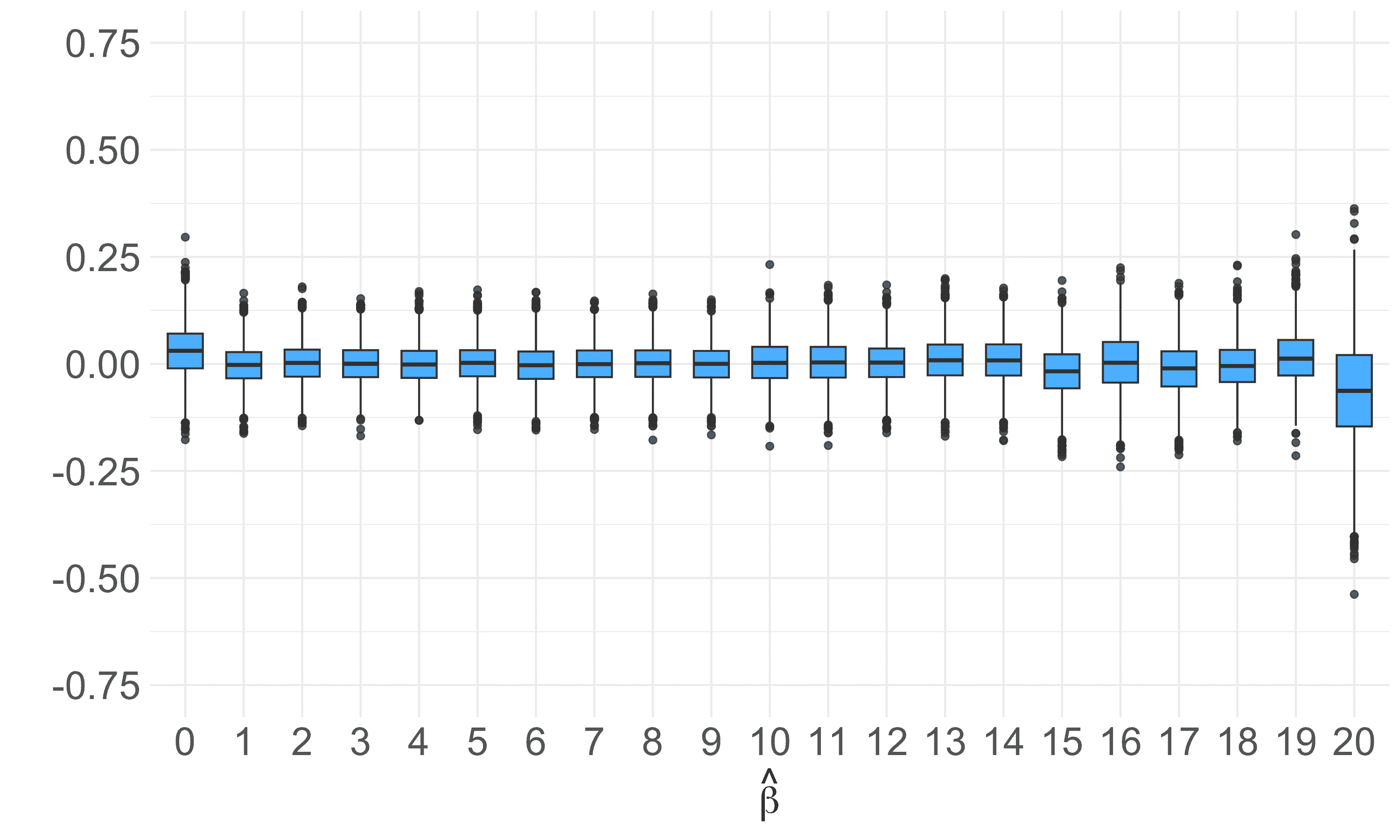} &
        \includegraphics[width=0.18\textwidth]{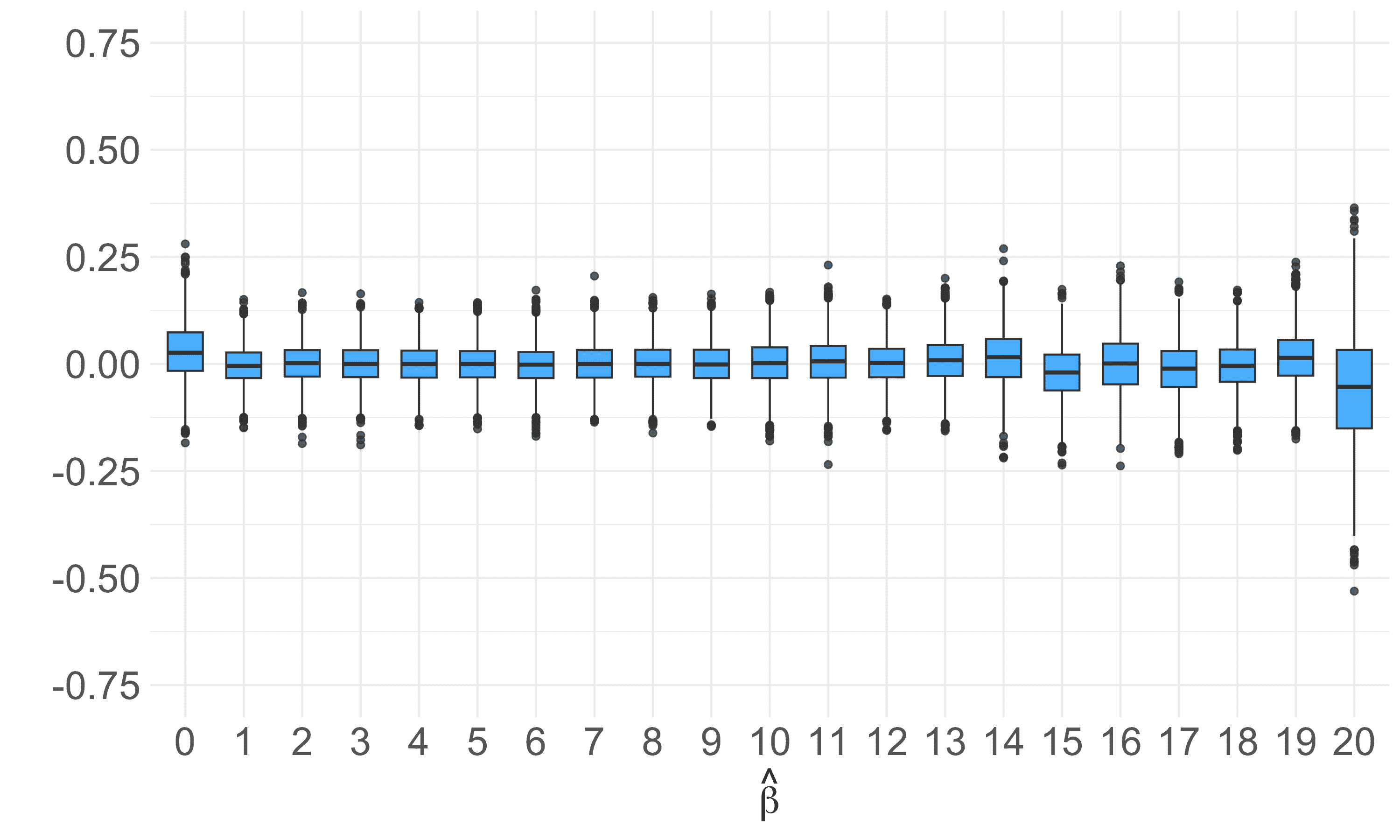} &
        \includegraphics[width=0.18\textwidth]{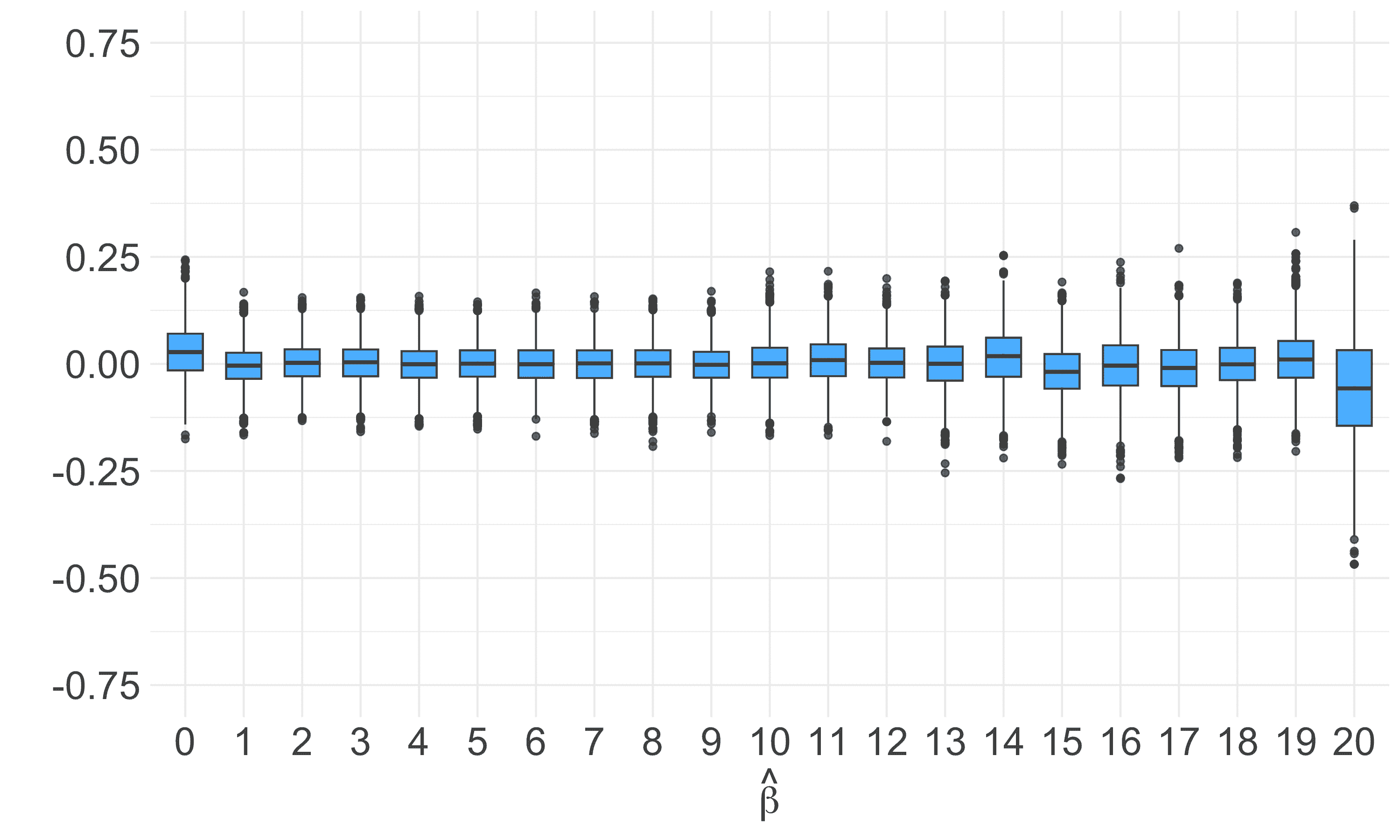} \\
        (e) Truncation at 5. & (f) Truncation at 6. & (g) Truncation at 7. & (h) Truncation at 8. \\
        \includegraphics[width=0.18\textwidth]{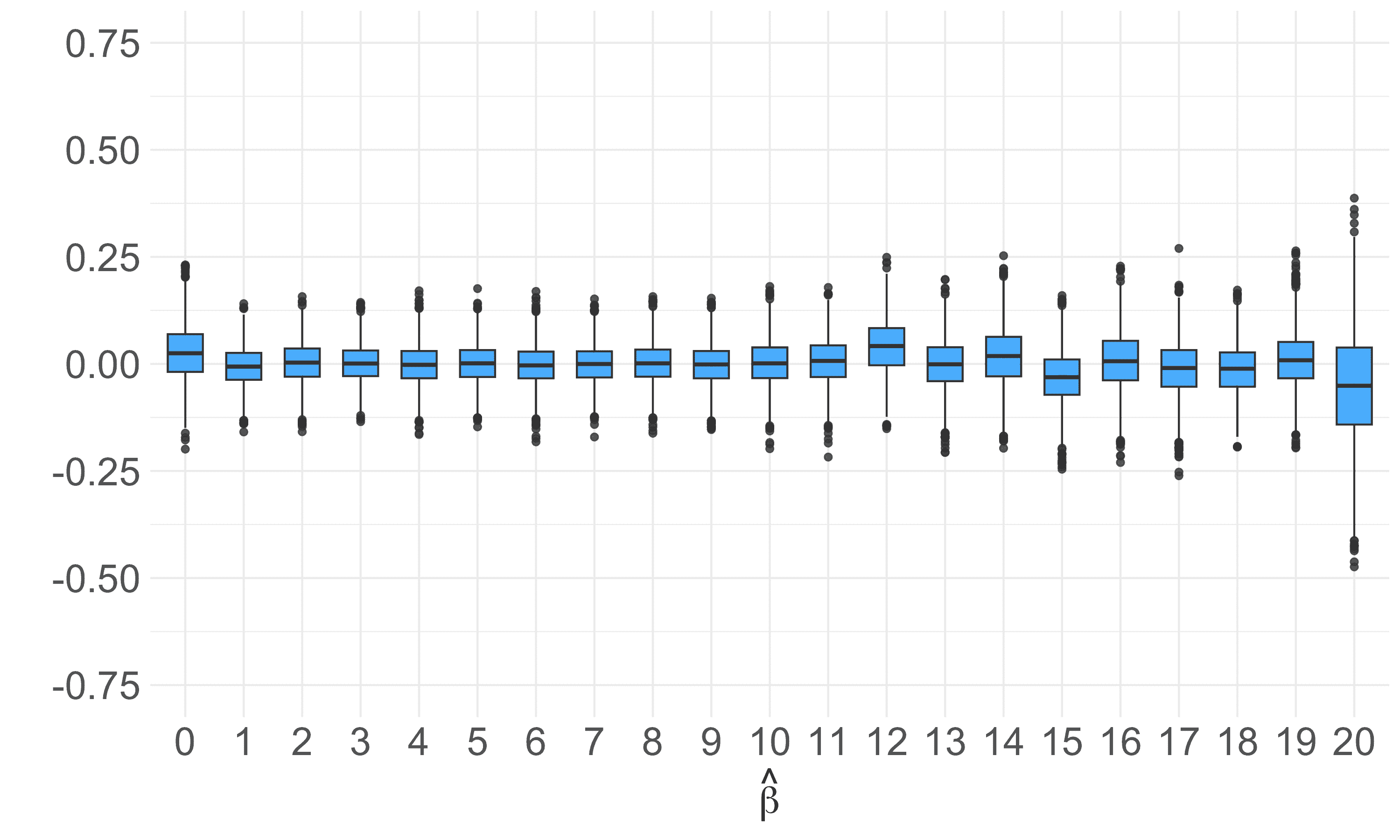} &
        \includegraphics[width=0.18\textwidth]{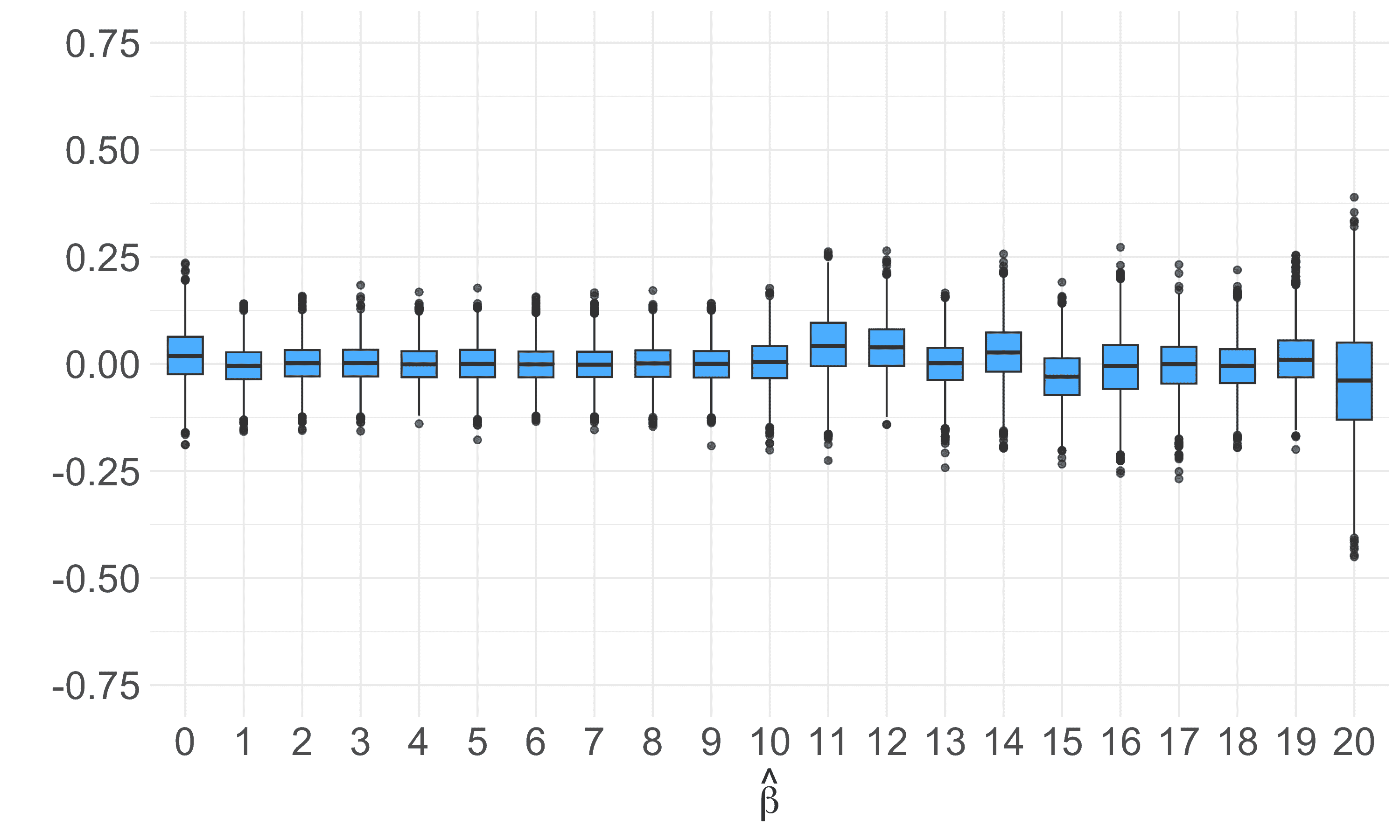} &
         \includegraphics[width=0.18\textwidth]{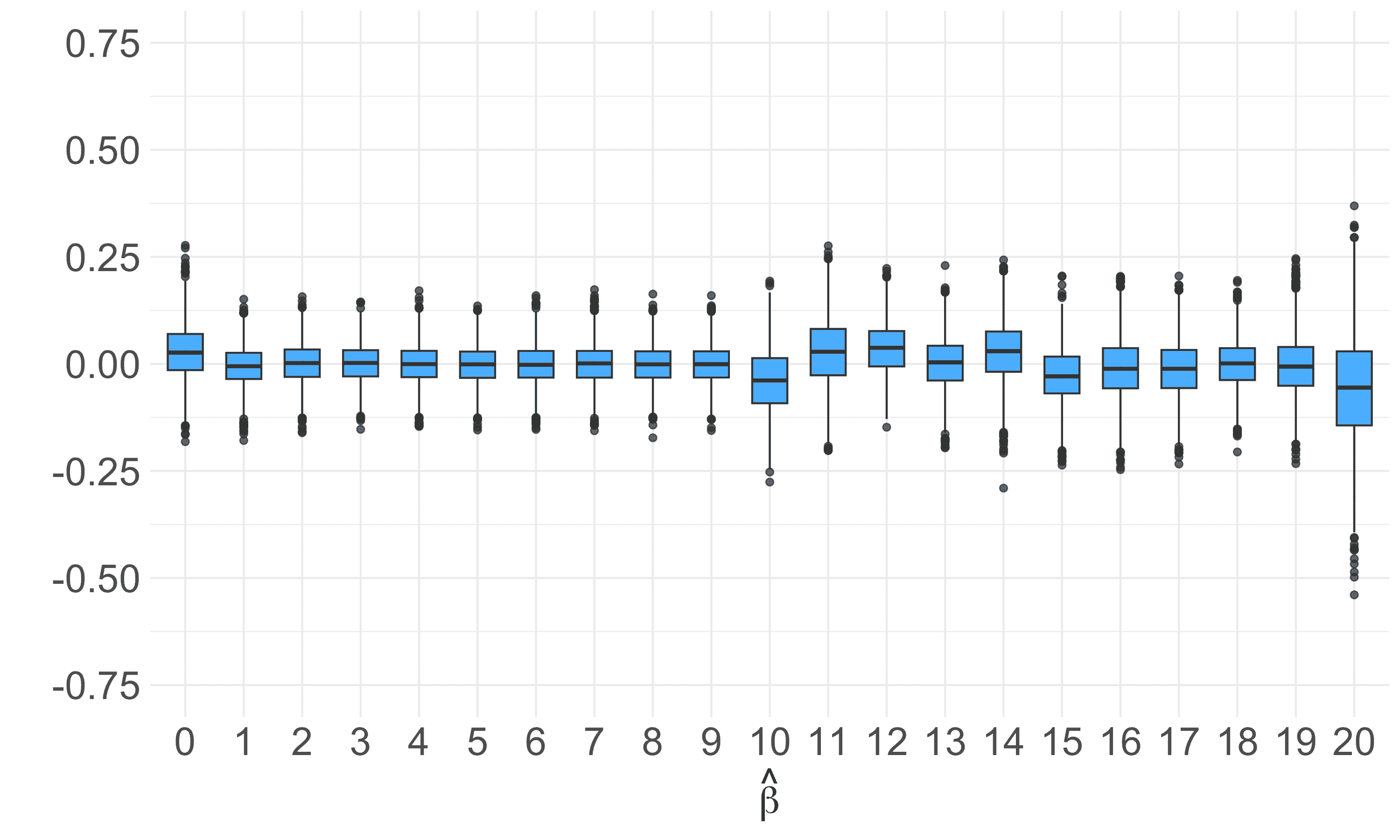} &
        \includegraphics[width=0.18\textwidth]{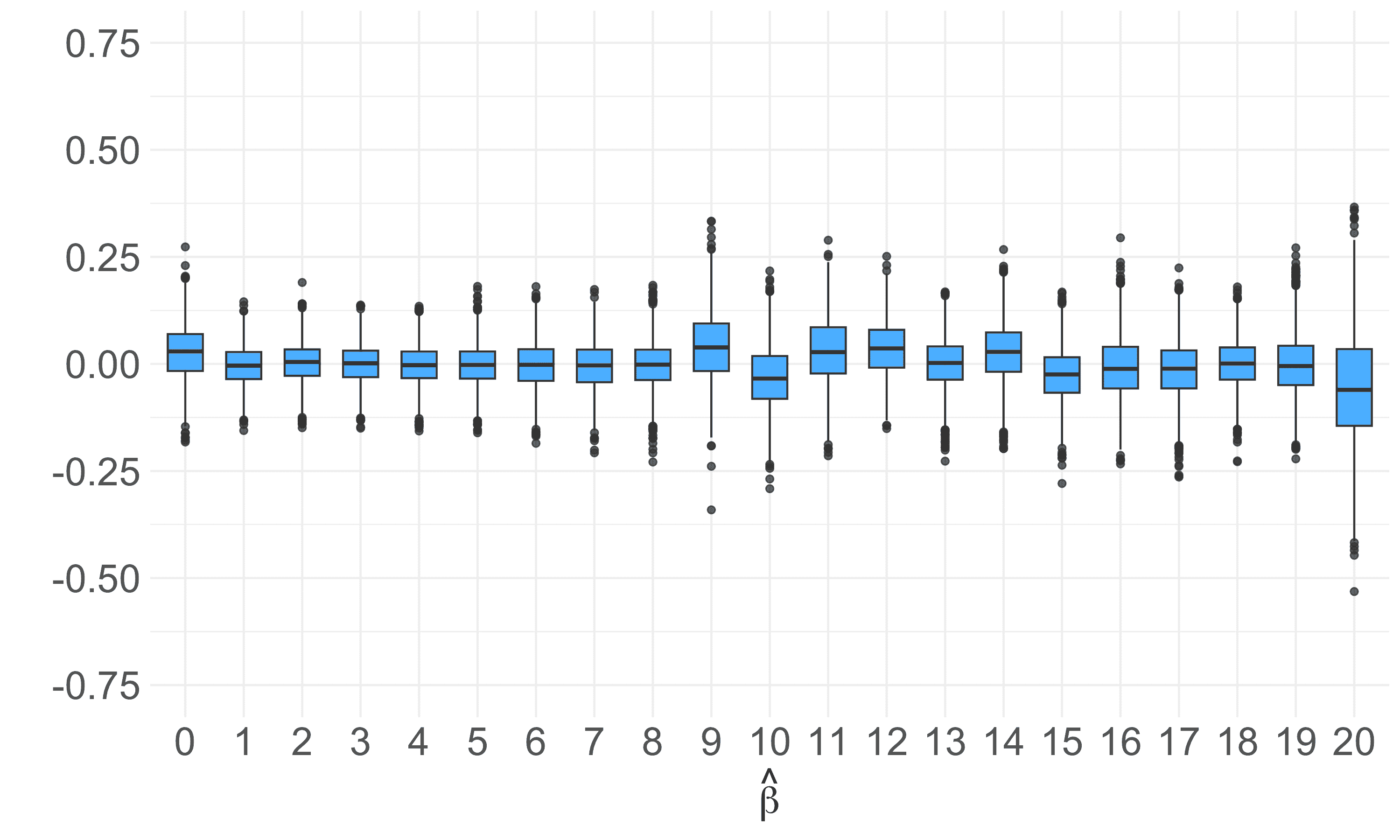} \\
        (i) Truncation at 9. & (j) Truncation at 10. & (k) Truncation at 11. & (l) Truncation at 12. \\
        \includegraphics[width=0.18\textwidth]{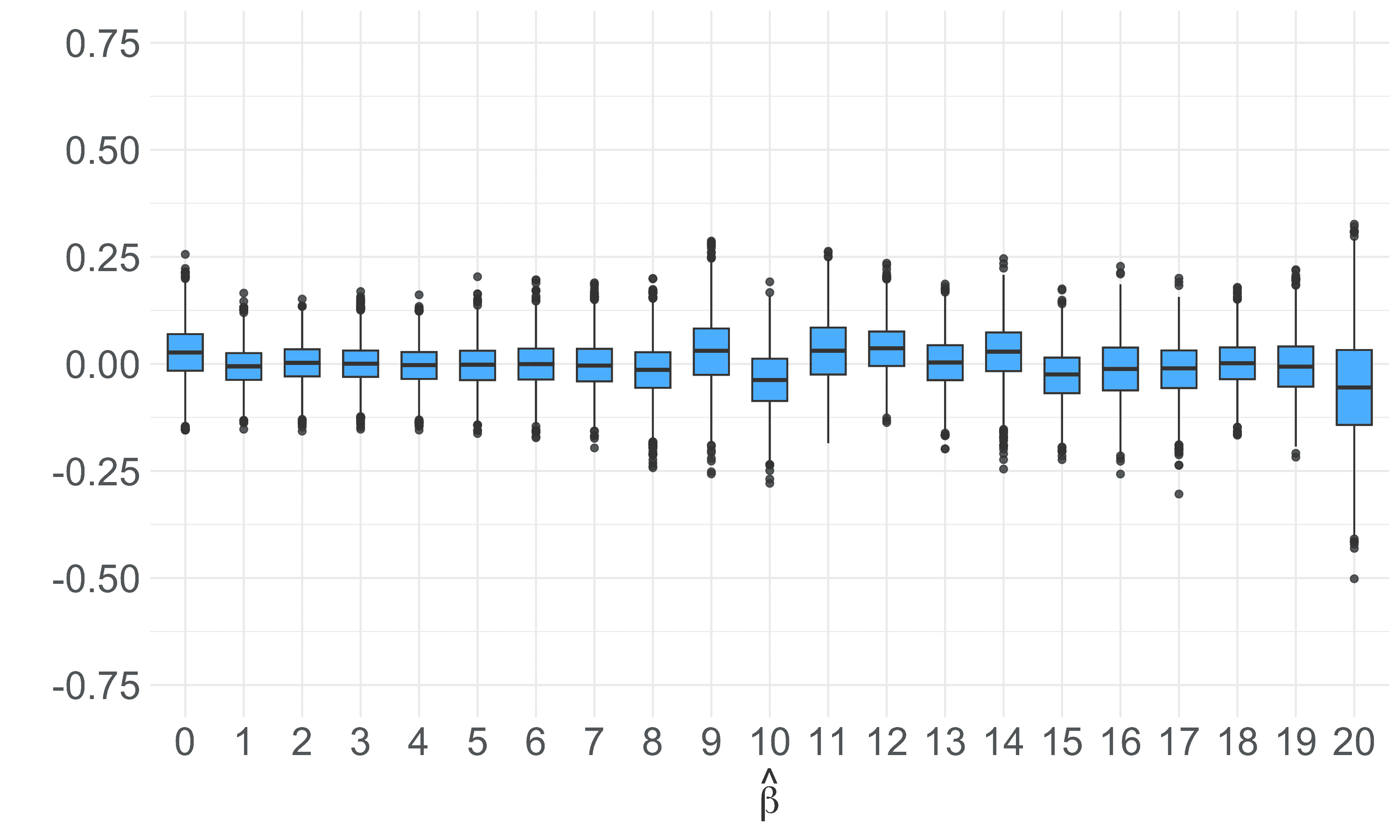} &
        \includegraphics[width=0.18\textwidth]{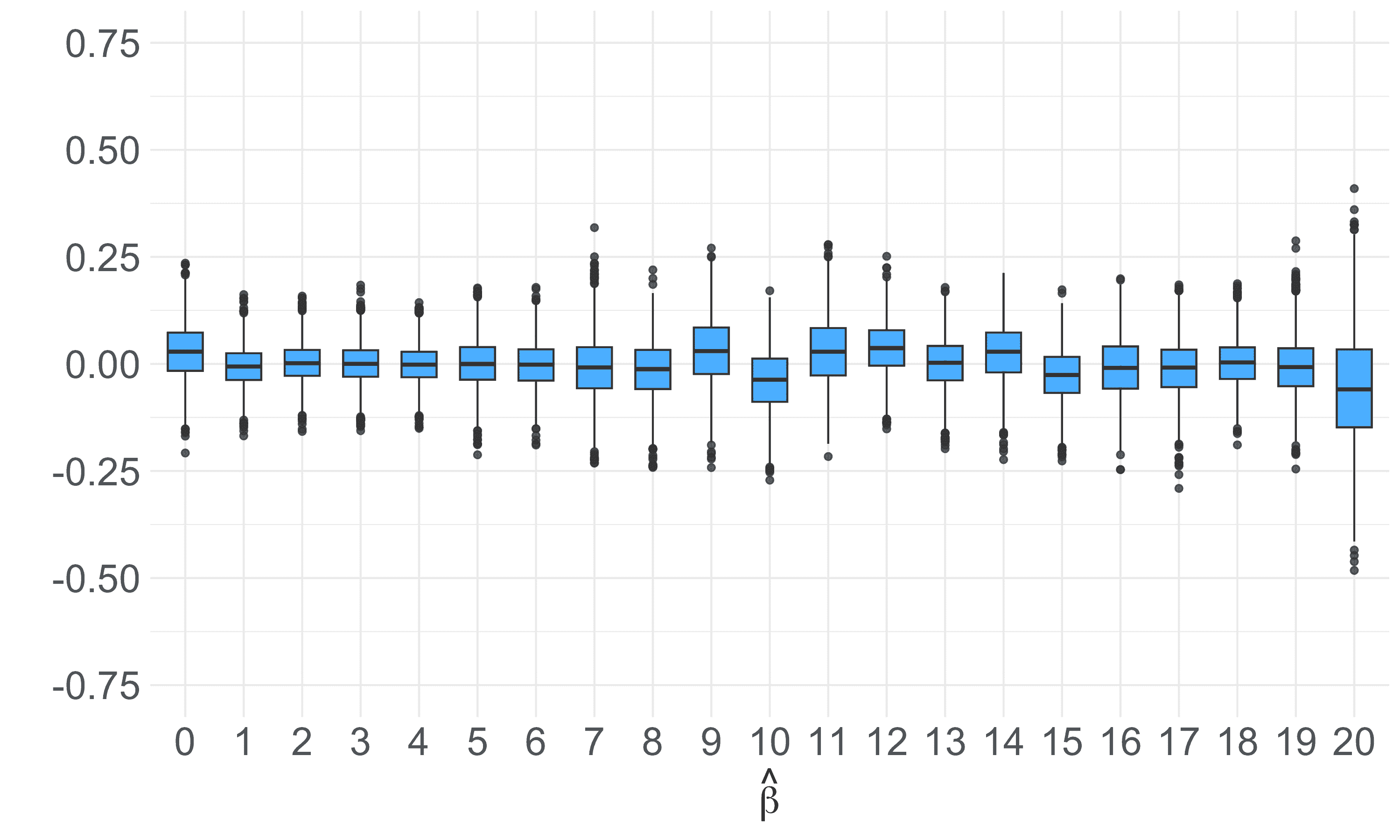} &
        \includegraphics[width=0.18\textwidth]{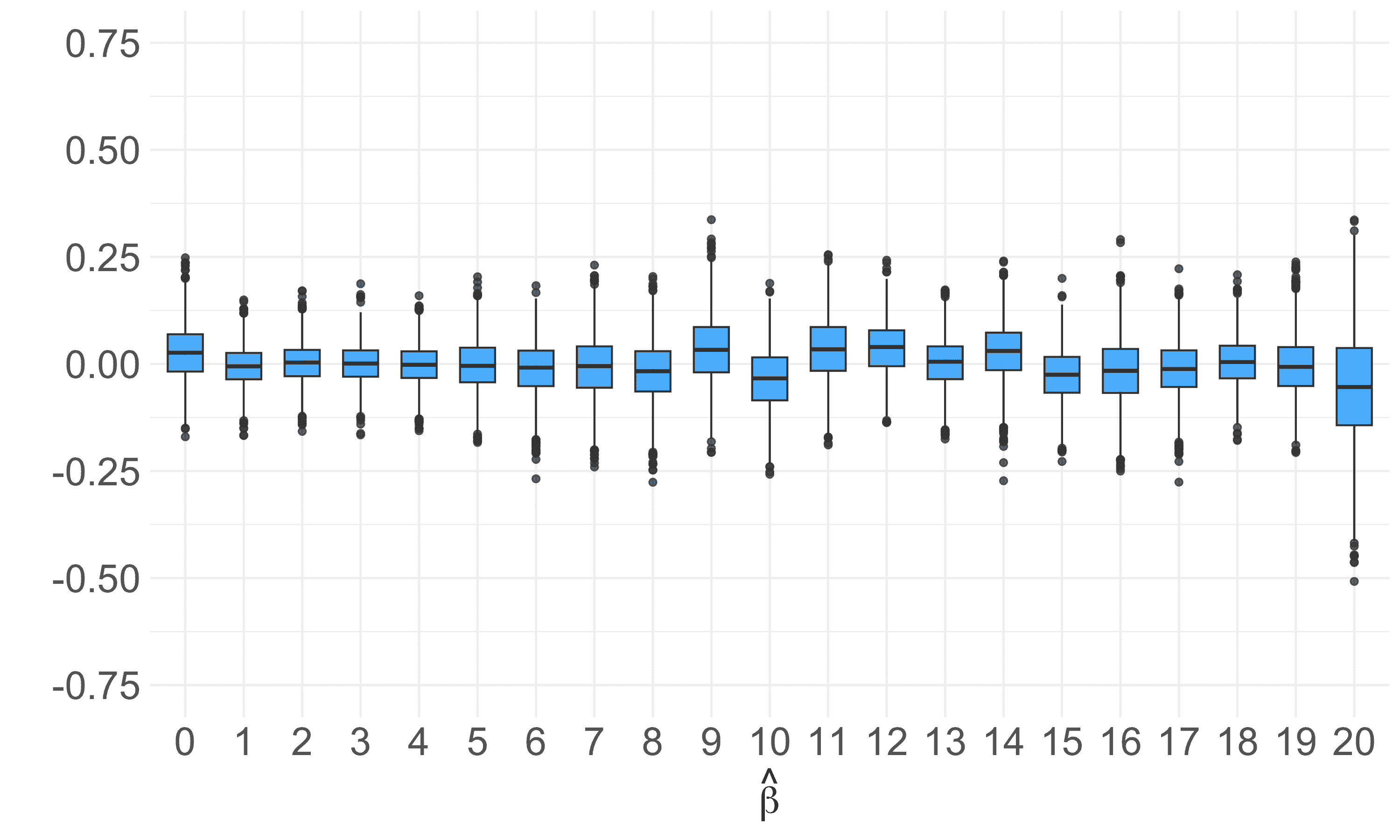} &
        \includegraphics[width=0.18\textwidth]{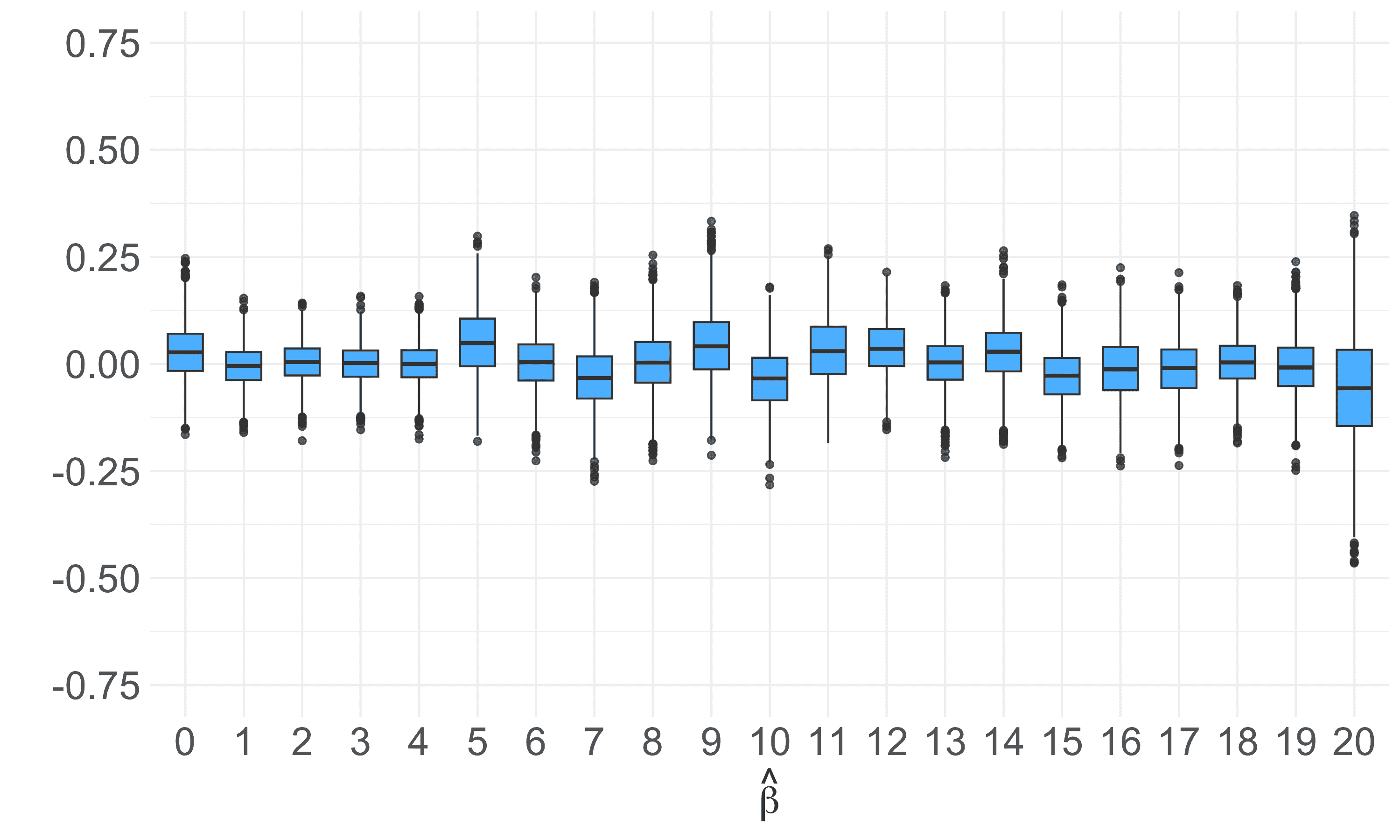} \\
        (m) Truncation at 13. & (n) Truncation at 14. & (o) Truncation at 15. & (p) Truncation at 16. \\
        \includegraphics[width=0.18\textwidth]{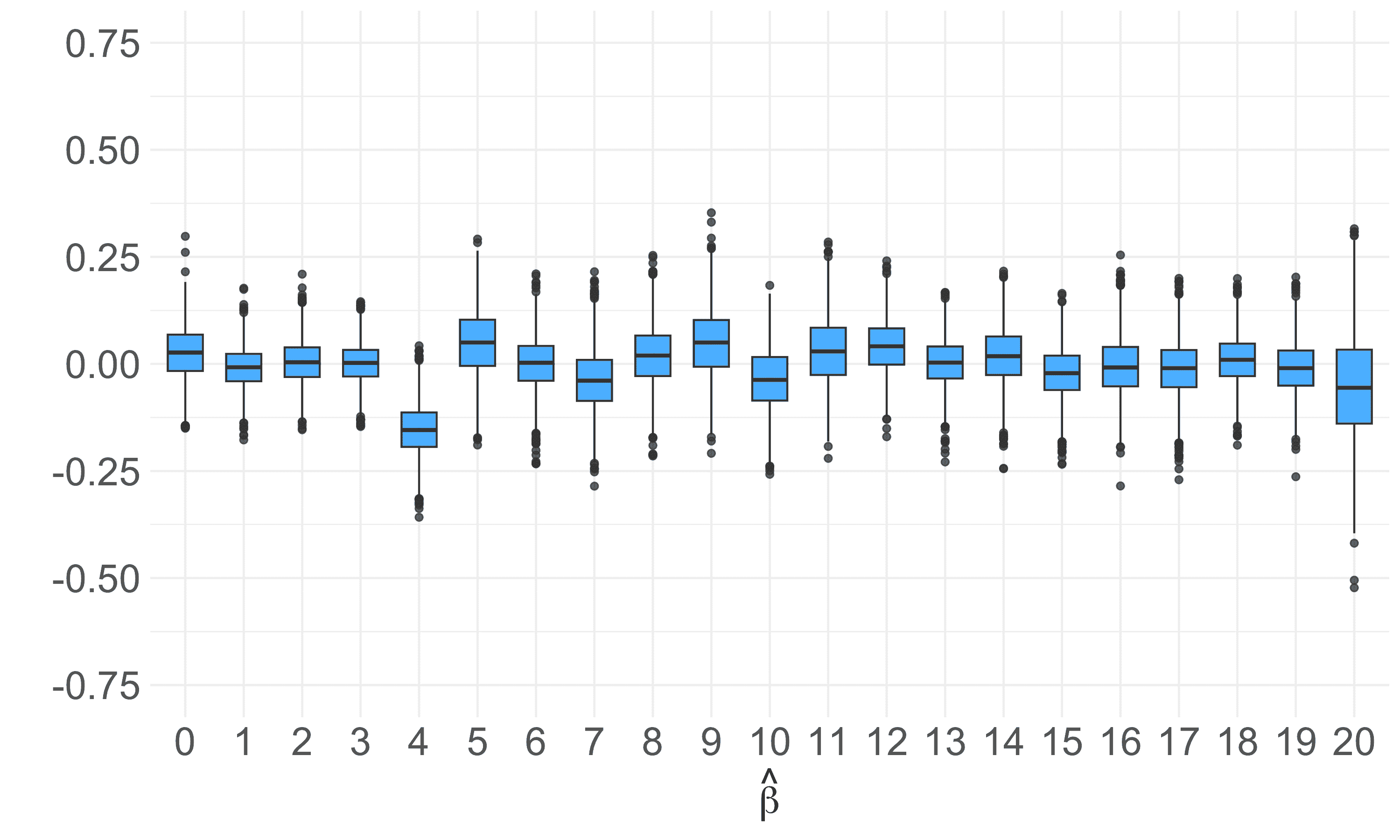} &
        \includegraphics[width=0.18\textwidth]{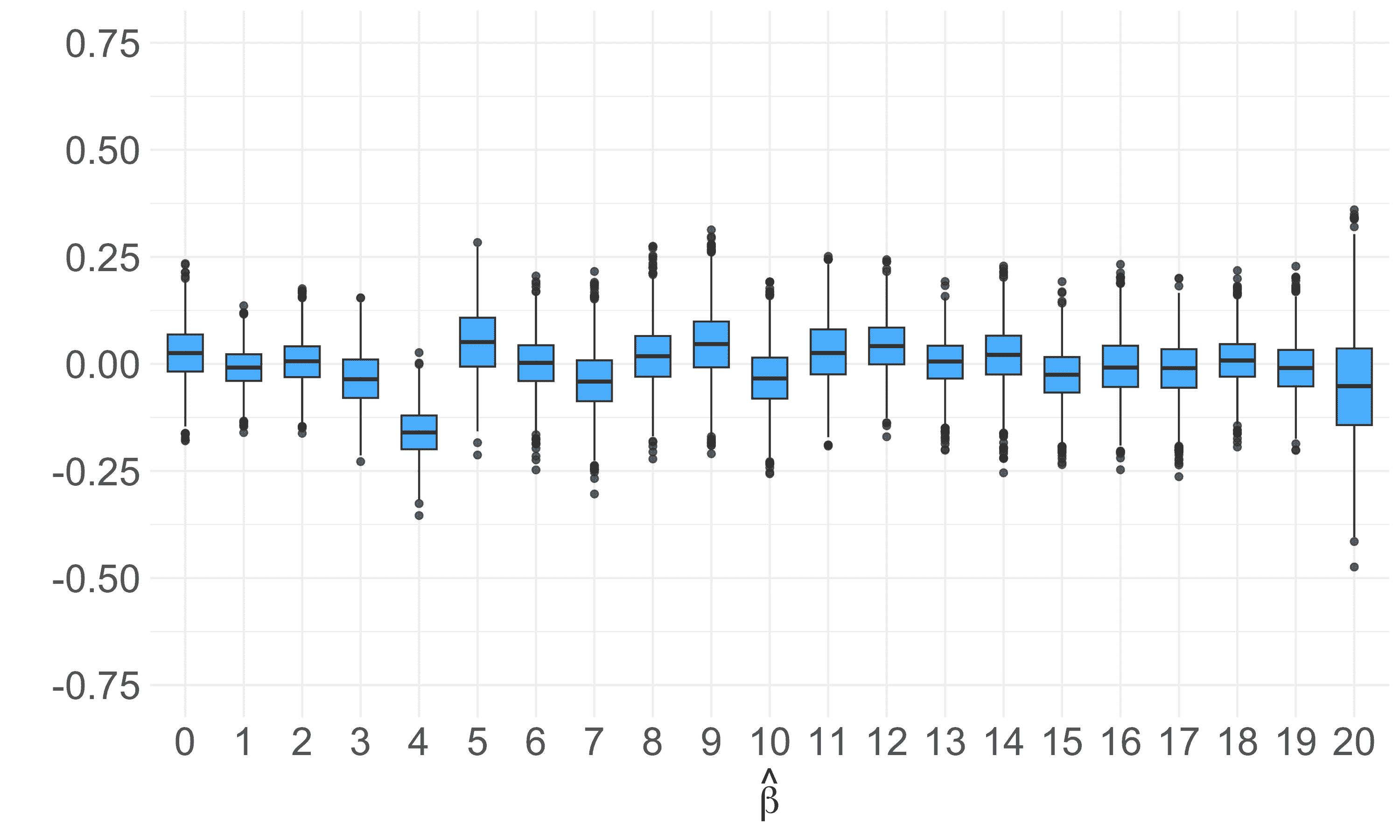} &
        \includegraphics[width=0.18\textwidth]{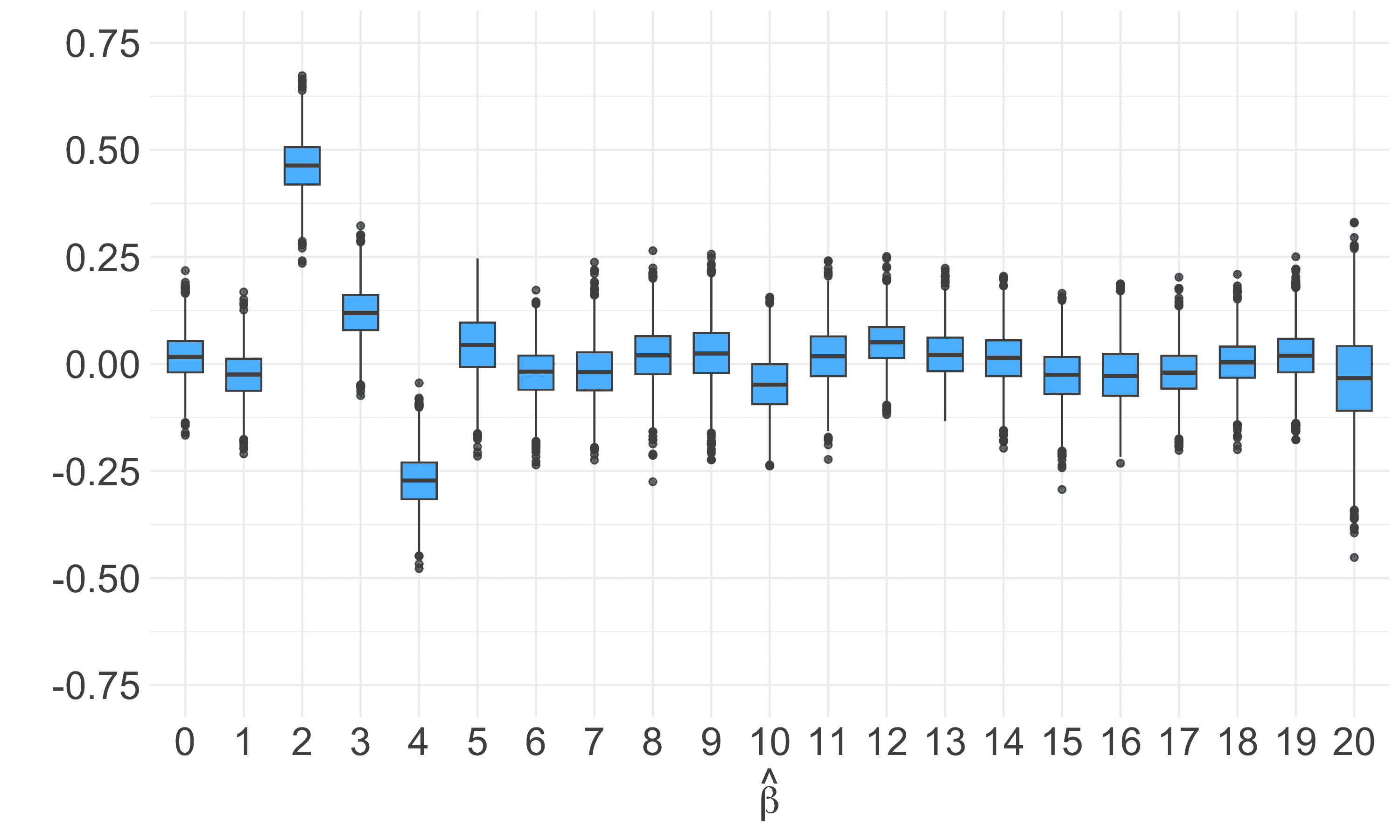} &
        \includegraphics[width=0.18\textwidth]{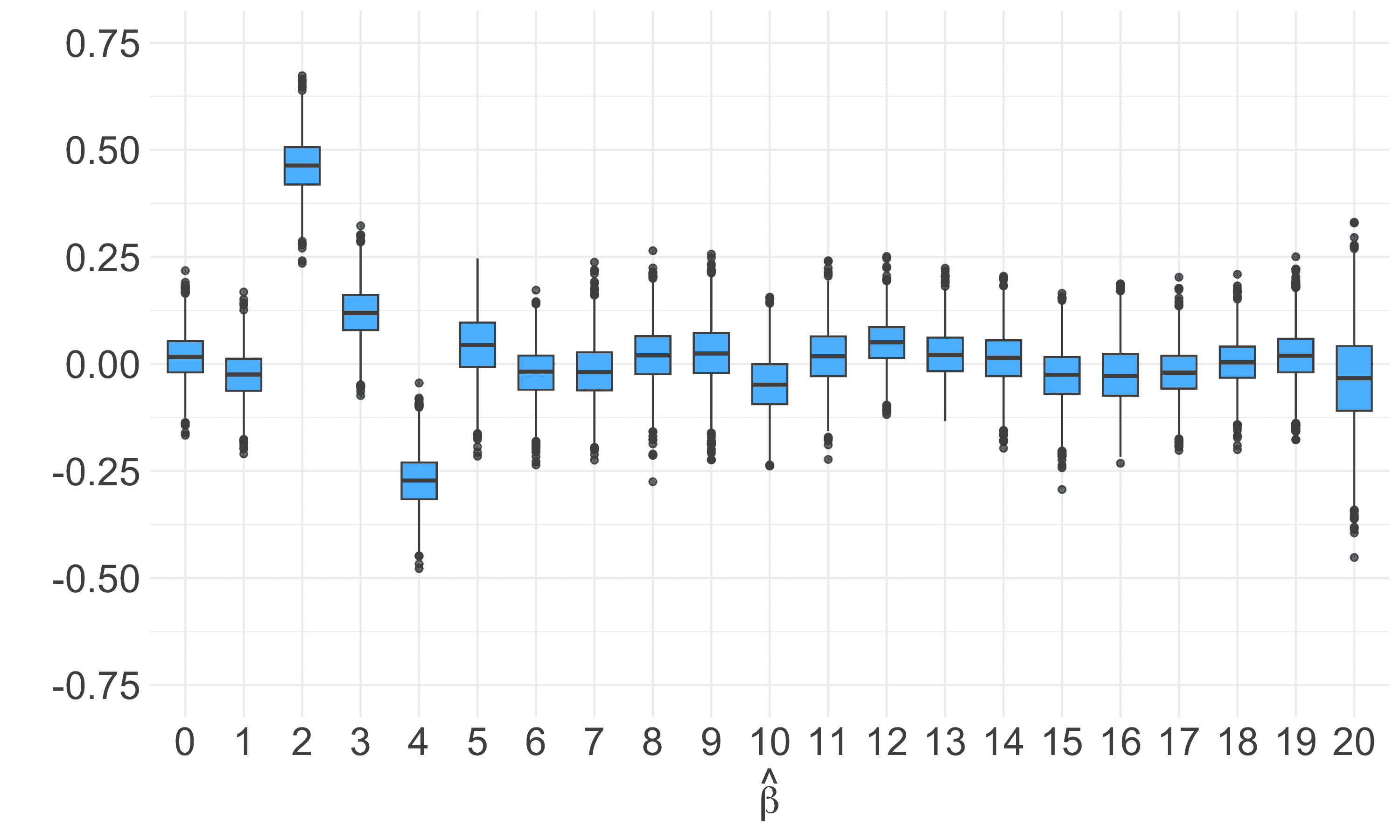} \\
        (q) Truncation at 17. & (r) Truncation at 18. & (s) Truncation at 19. & (t) No truncation. \\
    \end{tabular}
    
    \caption{Regression coefficients of an AIA on simulated real data from \ref{sec:simulated_real_data_d20_appendix} with randomly sampled target observations w.r.t. sensitive covariate $X_1$.}
    \label{fig:AIA_Cvine_X1_realI_d20_regCoeff}
\end{figure}

\FloatBarrier

\section{Real-World Data: SUPPORT2}\label{sec:support2_data}

The SUPPORT2 data set used in Section \ref{subsec:results_support2} is processed version of the raw SUPPORT2 data set by \citet{support2data}. The raw SUPPORT2 data 'comprises 9105 individual critically ill patients across 5 United States medical centers, accessioned throughout 1989-1991 and 1992-1994. Each row concerns hospitalized patient records who met the inclusion and exclusion criteria for nine disease categories: acute respiratory failure, chronic obstructive pulmonary disease, congestive heart failure, liver disease, coma, colon cancer, lung cancer, multiple organ system failure with malignancy, and multiple organ system failure with sepsis',  \citep{UCIsupport2data}.

In the processing step covariates \textit{age}, \textit{slos}, \textit{num.co}, \textit{scoma}, \textit{charges}, \textit{totcst}, \textit{totmcst}, \textit{sps}, \textit{aps}, \textit{surv2m}, \textit{surv6m}, \textit{hday}, \textit{prg2m}, \textit{dnrday}, \textit{meanbp}, \textit{wblc}, \textit{hrt}, \textit{resp}, \textit{temp}, \textit{pafi}, \textit{alb}, \textit{bili}, \textit{crea}, \textit{sod}, \textit{ph}, \textit{bun} and \textit{death} from the raw data are kept where the bivariate covariate \textit{death} is considered as response variable and renamed to $Y$. Additionally, all rows containing missing data are left out. The resulting SUPPORT2 data set contains $n=1104$ observations and $d=27$ covariates including response $Y$. Of these data, 220 randomly selected observations (equalling 20\%) are stored away and only accessed later for assessing the utility of the synthetic data. For parameter tuning of the generative models, synthetic data generation, as well as privacy attacks, the remaining 884 observations are used.

For the PrivPGD model to hold its DP guarantees, the ranges of all covariates need to be inferred from a source that is independent of the actual SUPPORT2 data. This is well possible for covariates describing features that inherently have limits outside which they lose their meaning (for example age or respiratory rate cannot be negative) and more difficult for other covariates. The following ranges have been inferred together with an MD:
\begin{itemize}
    \item \textit{totcst}: Total ratio of costs to charges (RCC) cost. Range $[0,500 000]$.
    \item \textit{crea}: Serum creatinine levels measured at day 3. Range $[0,13]$, assuming measurements in milligrams per liter \citep{finney2000adult}.
    \item \textit{totmcst}: Total micro cost. Range $[0,500 000]$.
    \item \textit{charges}: Hospital charges (in \$). Range $[0,500 000]$.
    \item \textit{slos}: Days from Study Entry to Discharge. Range $[0,180]$ \citep{knaus1995support}.
    \item \textit{bun}: Blood urea nitrogen levels measured at day 3. Range $[0,120]$ \citep{BUN}.
    \item \textit{age}: Age of the patients in years. Range $[18,115]$ \citep{knaus1995support}.
    \item \textit{num.co}: The number of simultaneous diseases (or comorbidities) exhibited by the patient. Range $[0,9]$ \cite{UCIsupport2data}.  
    \item \textit{scoma}: SUPPORT day 3 Coma Score based on Glasgow scale. Range $[0, 15]$
     \citep{teasdale1974assessment}.
    \item \textit{sps}: SUPPORT physiology score on day 3. Range $[0,163]$ \citep{le1993new, sapsII}.
    \item \textit{aps}: APACHE III day 3 physiology score. Range: $[0,299]$ \citep{knaus1991apache}.
    \item \textit{surv2m}: SUPPORT model 2-month survival estimate at day 3. Range $[0,1]$ \citep{knaus1995support}.
    \item \textit{surv6m}: SUPPORT model 6-month survival estimate at day 3. Range $[0,1]$ \citep{knaus1995support}.
    \item \textit{hday}: Day in hospital at which patient entered study. Range $[0,180]$ \citep{knaus1995support}.
    \item \textit{prg2m}: Physician’s 2-month survival estimate for patient. Range $[0,1]$.
    \item \textit{dnrday}: Day of DNR (Do Not Resuscitate) order (<0 if before study). Range $[-30,180]$ \citep{knaus1995support}.
    \item \textit{meanbp}: Mean arterial blood pressure of the patient, measured at day 3. Range $[0,180]$ \citep{mcevoy20242024}.
    \item \textit{wblc}: Counts of white blood cells (in thousands) measured at day 3. Range $[0, 100]$ \citep{riley2015evaluation}.
    \item \textit{hrt}: Heart rate of the patient measured at day 3. Range $[0,200]$ \citep{ACC}.
    \item \textit{resp}: Respiration rate of the patient measured at day 3. Range $[0,40]$ \citep{ATS}.
    \item \textit{temp}: Temperature in Celsius degrees measured at day 3. Range $[27,42]$.
    \item \textit{pafi}: $PaO_2/FiO_2$ ratio measured at day 3. The ratio of arterial oxygen partial pressure (PaO2 in mmHg) to fractional inspired oxygen (FiO2 expressed as a fraction). Range $[0,500]$.
    \item \textit{alb}: Serum albumin levels measured at day 3. Range $[0,5]$ \citep{RefInt}.
    \item \textit{bili}: Bilirubin levels measured at day 3. Range $[0,21]$ \citep{RefInt}.
    \item \textit{sod}: Serum sodium concentration measured at day 3. Range $[120, 160]$ \citep{RefInt}. 
    \item \textit{ph}: Arterial blood pH. Range $[6.9, 7.8]$ \citep{ignatavicius2017medical}.
\end{itemize}

\subsection{Estimated Correlation Matrix of C-Vine Generated Synthetic Data Generated Per Truncation Level on Real-World SUPPORT2 Data}\label{sec:CvineCorrTruncSupport2}

In Figure \ref{fig:Cvine_synth_data_corr_rrealsupport2} we observe how the correlation structure of the real SUPPORT2 data presented is more and more reproduced in the synthetic data generated by a C-vine with increasing truncation level.
Specifically, we note that dependencies in the first block containing sensitive covariates \textit{crea} and \textit{totcst} start to be represented in synthetic data generated by a C-vine from truncation level 15.

\begin{figure}[ht]
    \centering
    \begin{tabular}{cccc}
        \includegraphics[width=0.18\textwidth]{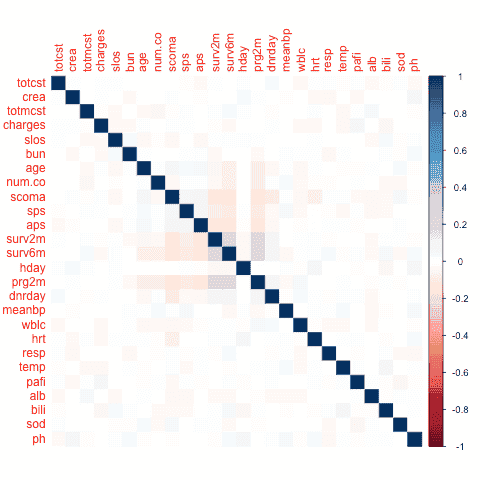} &  
        \includegraphics[width=0.18\textwidth]{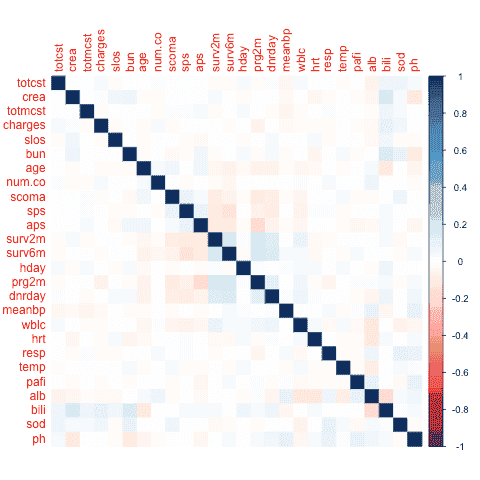} &
        \includegraphics[width=0.18\textwidth]{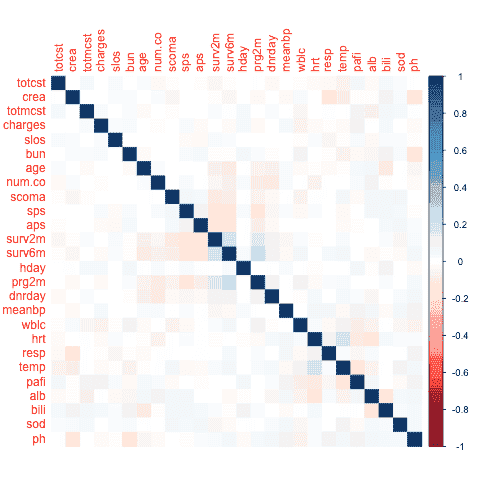} &
        \includegraphics[width=0.18\textwidth]{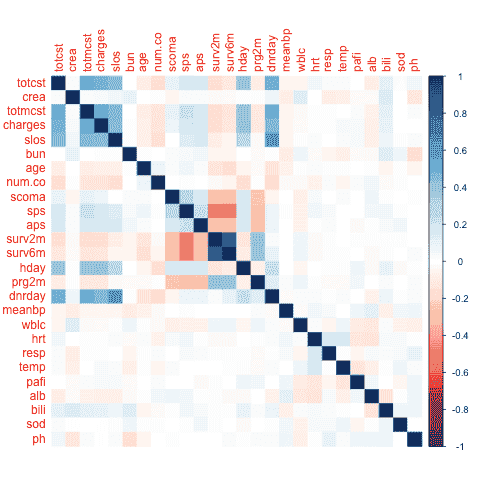}\\
        (a) Truncation at 1. & (b) Truncation at 5. & (c) Truncation at 10. & (d) Truncation at 15. \\
    \end{tabular}
    \begin{tabular}{ccc}
        \includegraphics[width=0.18\textwidth]{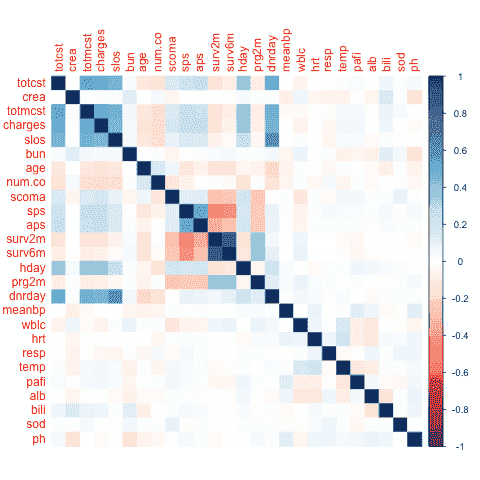} &
        \includegraphics[width=0.18\textwidth]{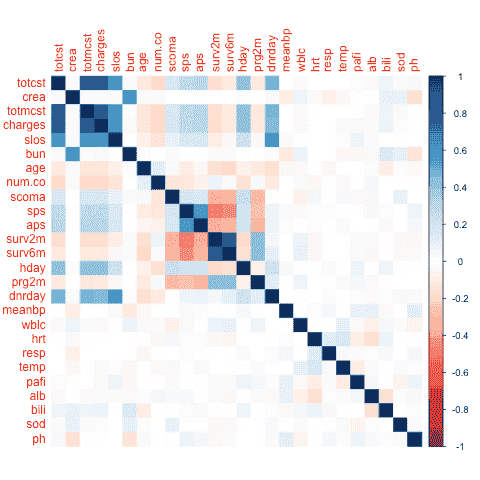} &
        \includegraphics[width=0.18\textwidth]{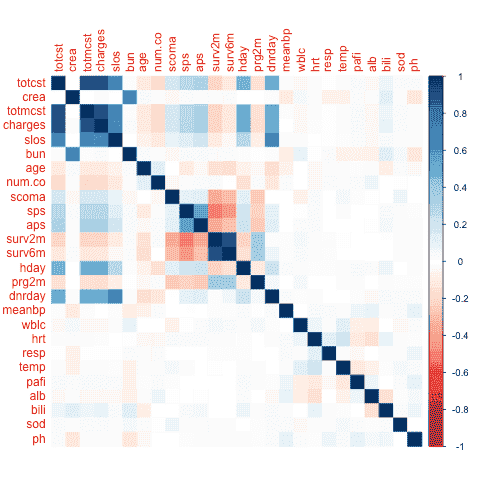} \\
        (e) Truncation at 20. & (f) No truncation. & (g) Real SUPPORT2 data.   \\
    \end{tabular}
    
    \caption{SUPPORT2 data: The matrices of pairwise Kendall's $\tau$ of continuous covariates in synthetic data generated with a C-vine for truncation at levels $t \in\{1,5,10,15,20\}$ and no truncation illustrate how the rank correlation structure of the SUPPORT2 data in (g) is more and more reproduced in the synthetic data with increasing truncation level.}
    \label{fig:Cvine_synth_data_corr_rrealsupport2}
\end{figure}

\subsection{R-Vine Matrix of the C-Vine Used as a Generative Model on Real-World SUPPORT2 Data}\label{sec:RvinematrixCvineSUPPORT2}
The R-vine matrix of the C-vine used as a generative model on real-world SUPPORT2 data is as follows with $Y$ on index 27 and index $j \in [26]$ corresponding to covariate $X_j$:
\begin{align}
    \begin{pmatrix}
        27 & 27 & 27 & \cdots & 27 & 27 \\
        & 26 & 26 & \cdots & 26 & 26 \\
        & & 25 & \cdots & 25 & 25 \\
        & & & \ddots & \vdots & \vdots \\
        & & & & 2 & 2 \\
        & & & & &  1 
    \end{pmatrix} \; .
\end{align}
In the first vine tree the response $Y$ is in the center. By this it is enforced that pairwise dependencies between $Y$ and the covariates are modeled.

\FloatBarrier

\subsection{Choice of Target Observations for Privacy Evaluation on Real-World SUPPORT2 Data} \label{sec:append_support2_choice_of_targets}

We conduct an AIA and MIA on SUPPORT2 data described in \ref{sec:support2_data}. The parameter setup of the privacy attacks can be found in Table \ref{tab:parameters_privacyAttacks_genModels}. For the attacks, four target observations are handpicked outside the 95\%-quantile of the regarding sensitive covariate for each sensitive covariate, see Table \ref{tab:handpicked_targets095_support2}.

\begin{table}[t]
    \caption{Target observations of the SUPPORT2 data set, Section \ref{sec:support2_data} that are handpicked to lie outside the 95\%-quantile of the respective sensitive covariate \textit{crea} and \textit{totcst}.}
    \label{tab:handpicked_targets095_support2}
    \vskip 0.15in
    \begin{center}
        \begin{small}
            \begin{sc}
                \begin{tabular}{llccr}
        \toprule
         & quantiles & & & \\
         sensitive covariate & 0.01 & 0.025 & 0.975 & 0.99 \\ 
         \midrule
         \textit{crea} & ID820 & ID45 & ID403 & ID447 \\
         \textit{totcst} & ID806 & ID823 & ID31 & ID41 \\ \bottomrule
    \end{tabular}
            \end{sc}
        \end{small}
    \end{center}
    \vskip -0.1in
\end{table}

Additionally, four target observations, namely ID123, ID507, ID589 and ID740\footnote{NB: ID$k$ corresponds to the $(k+1)$th observation in the real data set with $k \in \{0, ..., (n-1)\}$.}, are randomly sampled from the real data set. They correspond to the quantiles w.r.t. the respective sensitive covariate given in Table \ref{tab:quantiles_randomly_sampled_targets_AIA_support2}.

\begin{table}[t]
    \caption{Randomly sampled target observations from the SUPPORT2 data set of Section \ref{sec:support2_data} and their corresponding quantiles w.r.t. covariates \textit{crea} and \textit{totcst}.}
    \label{tab:quantiles_randomly_sampled_targets_AIA_support2}
    \vskip 0.15in
    \begin{center}
        \begin{small}
            \begin{sc}
                \begin{tabular}{lcccr}
        \toprule
         & \multicolumn{3}{l}{target IDs} & \\
         sensitive covariate & ID123 & ID507 & ID589 & ID740 \\ 
         \midrule
         \textit{crea} & 0.966 & 0.984 & 0.506 & 0.957 \\
         \textit{totcst} & 0.374 & 0.924 & 0.683 & 0.615 \\ \bottomrule
    \end{tabular}
            \end{sc}
        \end{small}
    \end{center}
    \vskip -0.1in
\end{table}

\subsection{Attribute Inference Attack: Results in Terms of MSE}\label{sec:realsupport2_small_AIA_MSE}

Figure \ref{fig:MSE_Cvine_competitors_support2_small_totcst} displays the MSE under an AIA w.r.t randomly sampled and outlying targets  w.r.t. sensitive covariate \textit{totcst}. 

\begin{figure}[t]{}
    \centering
    \includegraphics[width=\columnwidth]{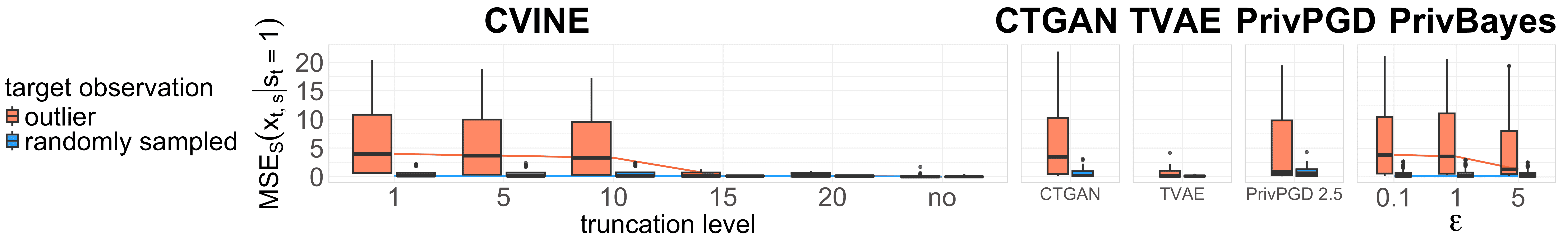}
    \caption{SUPPORT2 data: MSE under an AIA w.r.t randomly sampled (blue) and outlying targets (orange) w.r.t. sensitive covariate \textit{totcst}.}\label{fig:MSE_Cvine_competitors_support2_small_totcst}
\end{figure}%

\subsection{Attribute Inference Attack: Results in Terms of WCAB}\label{sec:realsupport2_small_AIA_WCAB}

Figure \ref{fig:AIA_Cvine_competitors_support2_small_WCAB} displays the AIA results in terms of WCAB. We observe that the WCAB approximately replicates the trend of the MAB for the C-vine, CTGAN, TVAE and PrivPGD, see Figure \ref{fig:AIA_Cvine_competitors_support2_small_MAB}. The WCAB of PrivBayes for $\epsilon \in \{0.1, 5\}$ on the other hand lies above the one of the C-vine truncated at level 10 or lower and the one of CTGAN for sensitive attributes \textit{crea} and \textit{totcst}). This indicates that even though the PrivBayes provides formal guarantees on privacy leakage on a single individual that translate to theoretical bounds on the PG in an MIA, the PrivBayes might in the worst case leak dependencies that inform the sensitive covariate in an AIA from the real into the synthetic data, even for low $\epsilon$.

\begin{figure}[ht]
    \centering
    \includegraphics[width=0.8\columnwidth]{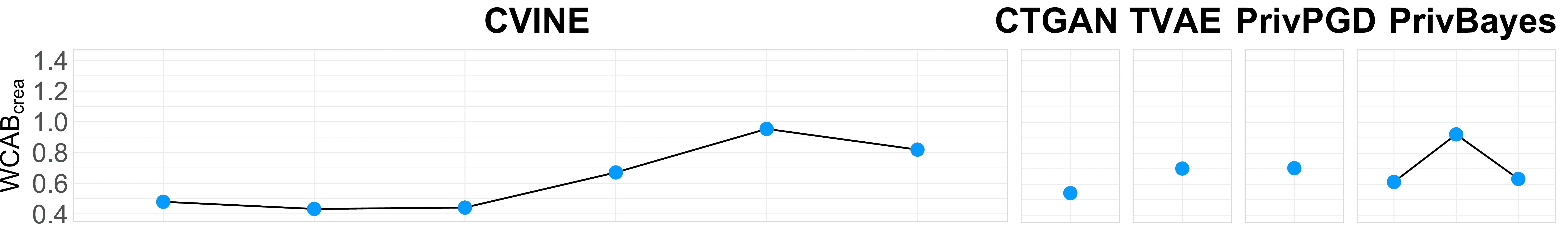}
     \includegraphics[width=0.8\columnwidth]{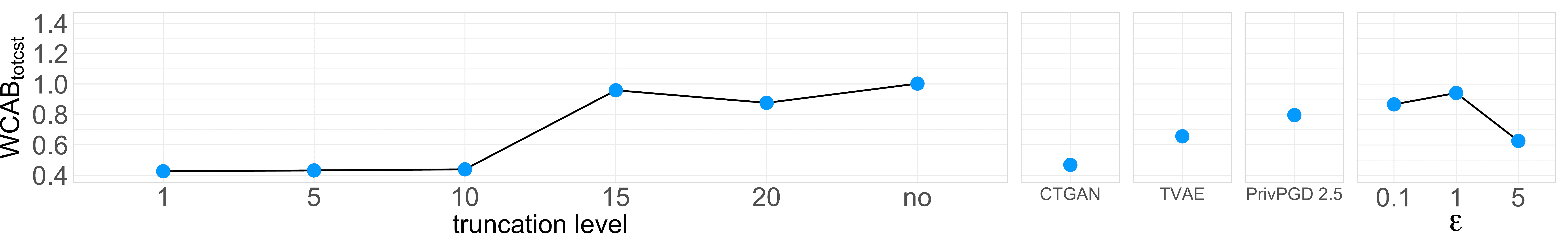}
    \caption{SUPPORT2 data: Results of an AIA w.r.t. sensitive covariate \textit{crea} (top row) and \textit{totcst} (bottom row) measured by $WCAB_{j^*}$ . Synthetic data are generated with a C-vine for different truncation levels (left), CTGAN (2nd), TVAE (3rd), PrivPGD  with $\epsilon = 2.5$ and $ \delta= 10^{-5}$ (4th) and PrivBayes (right) for privacy parameter $\epsilon \in \{0.1, 1, 5\}$. Results are reported over 10 AIA game iterations.}
    \label{fig:AIA_Cvine_competitors_support2_small_WCAB}
\end{figure}

\subsection{Privacy-Utility Plots on Real-Wolrd SUPPORT2 Data: Additional Plots}\label{sec:2d_priv-ut_support2_additional}

Figure \ref{fig:2d_priv-ut_support2_totcst_crea_MIA} displays the privacy-utility plots w.r.t. a MIA and and sensitive features \textit{totcst} and \textit{crea} based on the results of Sections \ref{sec:MIA_support2_results} and \ref{sec:utility_support2}. As already observed in Sections \ref{sec:MIA_support2_results} and \ref{sec:utility_support2}, all models except for TVAE score a PG of around 1. Compared to the competitor models that are able to protect sensitive covariates against a MIA, the C-vine scores the highest utility.

\begin{figure}[ht]
    \centering
    \includegraphics[width=0.3\columnwidth]{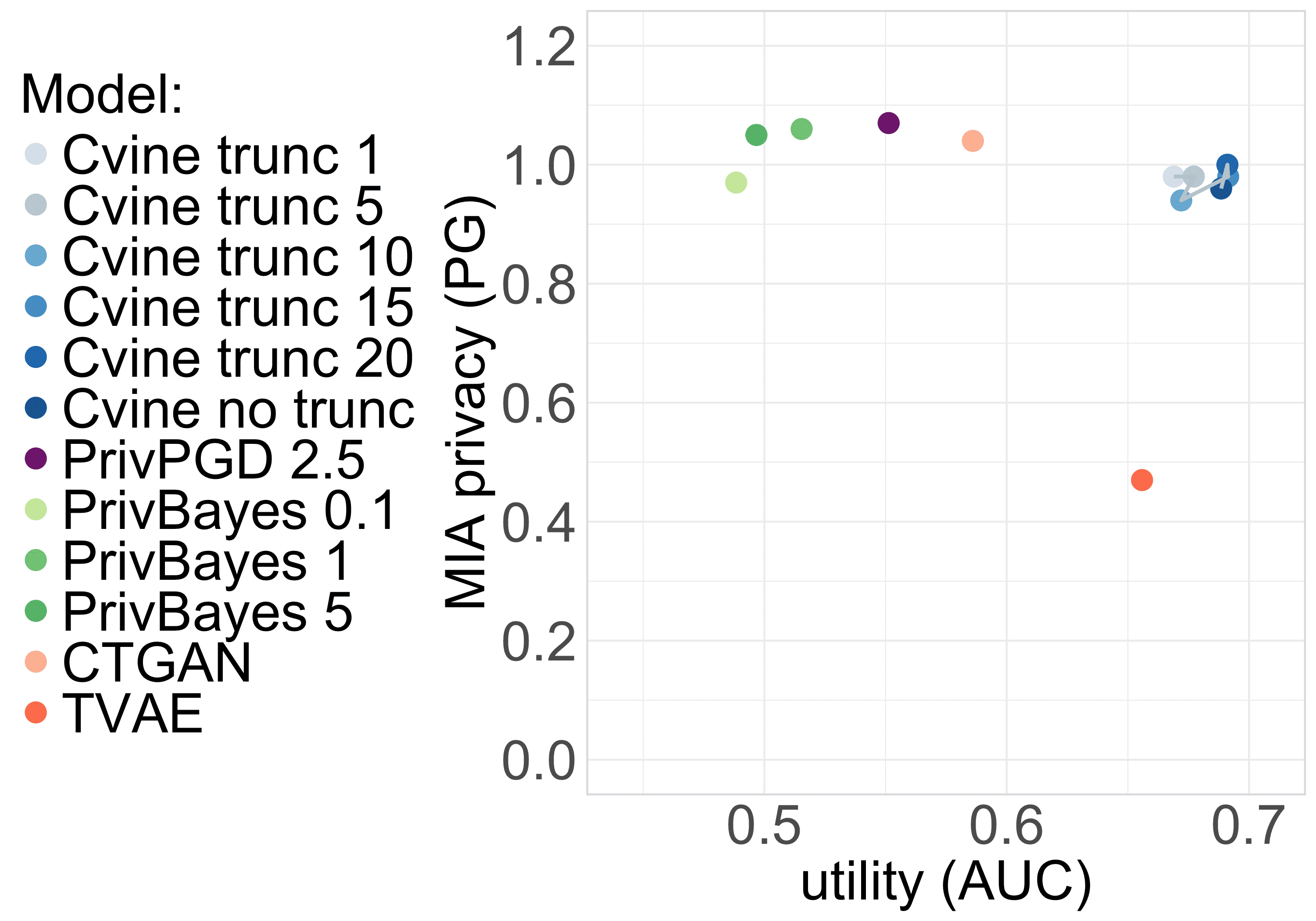}
    \includegraphics[width=0.198\columnwidth]{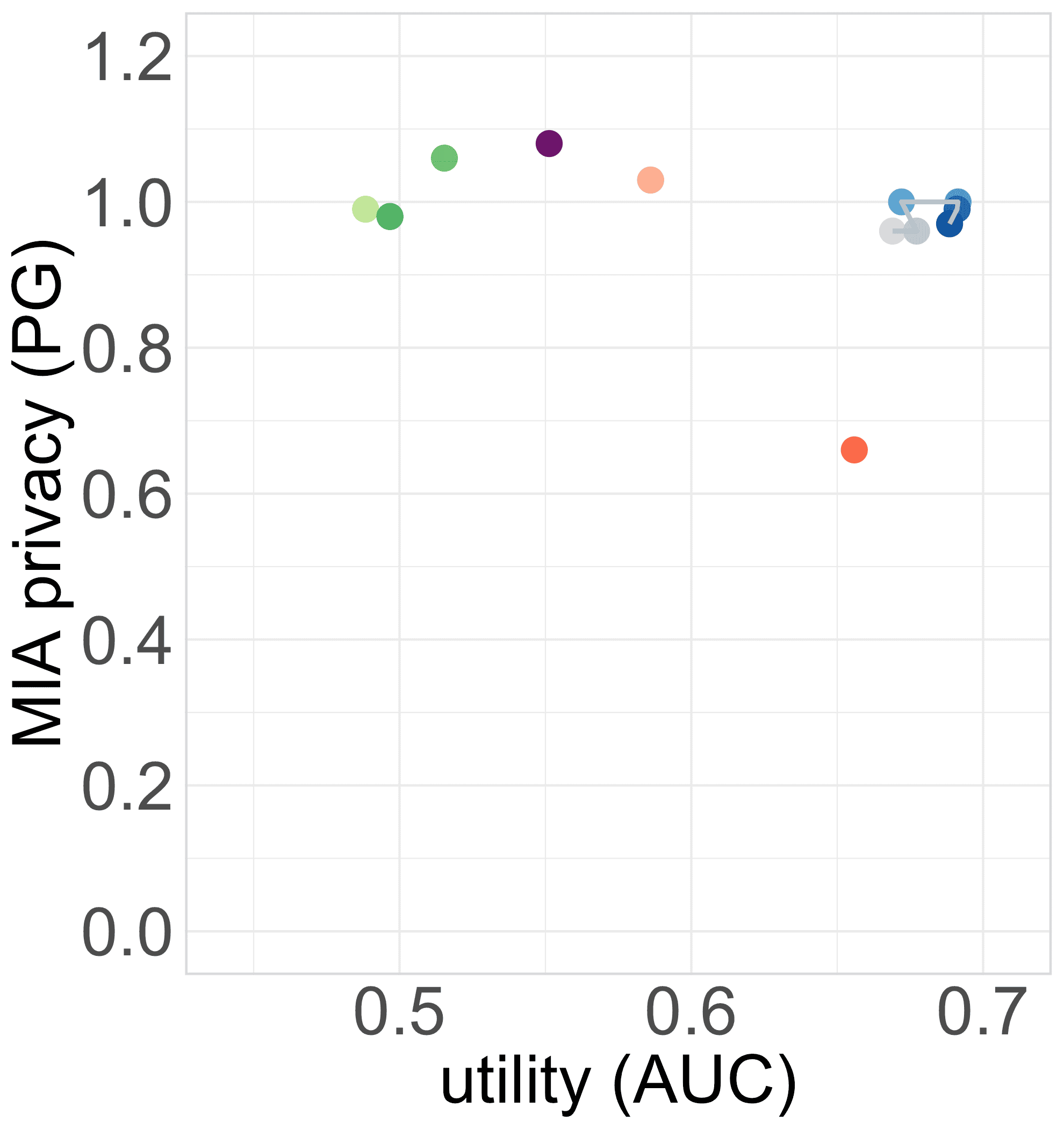}
    \caption{SUPPORT2 data: Privacy-utility plot w.r.t. MIA and sensitive features \textit{totcst} (left) and \textit{crea} (right) of a C-vine with truncation levels $t \in \{1,5,10,15, 20\}$ and no truncation, PrivPGD  with $\epsilon = 2.5$ and $ \delta= 10^{-5}$, PrivBayes model with $\epsilon \in \{0.1,1,5\}$, CTGAN and TVAE on SUPPORT2 data of Section \ref{sec:support2_data}. The median MIA PG over all game iterations and utility for 50 synthetic data sets are reported.}
    \label{fig:2d_priv-ut_support2_totcst_crea_MIA}
\end{figure}

\subsection{Statistical Fidelity}\label{app:stat_discrepancy}

\subsubsection{Statistical Discrepancy}\label{sec:statFidelity_support2}

We measure the statistical discrepancy between joint real and synthetic distribution with $\alpha$-precision, $\beta$-recall and authenticity $(P_{\alpha}, R_{\beta}, A)$ introduced by \citet{alaa2022faithful}. Their definition can be found in Appendix \ref{sec:statistical_fidelity}.

As for the simulated real data, an increasing truncation level of the C-vine leads to an increase in fidelity and diversity while it decreases the generalization, see Figure \ref{fig:statfidelity_allmodels_support2}. Highest fidelity and diversity of an un-truncated C-vine results in lowest generalization compared to competitor models. The PrivBayes model reaches a generalization of up to 0.99 but its generated samples fail to resemble and cover the real data. The CTGAN generated synthetic data achieve a high generalization but struggle to be diverse enough to cover the real data. The TVAE generates synthetic data with very high fidelity, high generalization and moderate diversity.

Additionally, we evaluate the generative models by comparing empirical marginal histograms on real and synthetic data in Appendix \ref{sec:support2_margHistOverlap}.

\begin{figure}[ht]
    \centering
    \includegraphics[width=1\columnwidth]{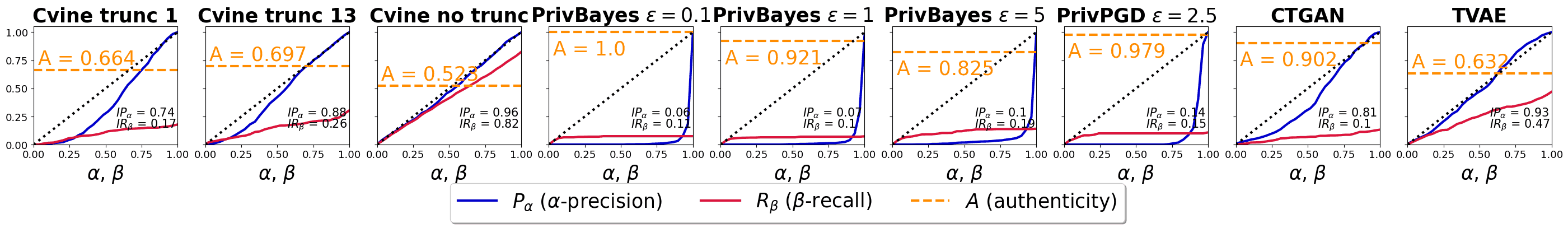}
    \caption{SUPPORT2 data: Fidelity ($\alpha$-precision), diversity ($\beta$-recall) and generalization (authenticity) of synthetic data generated in the order C-vine for truncation at levels 1 and 10 and no truncation, PrivBayes for $\epsilon \in \{0.1, 1, 5\}$, PrivPGD with $\epsilon=2.5$, $\delta=10^{-5}$, CTGAN and TVAE from SUPPORT2 data. 
    } \label{fig:statfidelity_allmodels_support2}
\end{figure}

\subsubsection{Comparing Marginal Histograms} \label{sec:support2_margHistOverlap}

On the SUPPORT2 data we generate synthetic data with the C-vine for different truncation levels and with the competitor models (CTGAN, TVAE, PrivBayes with $\epsilon \in \{0.1, 1, 5\}$). We select the 6 covariates \textit{age}, \textit{aps}, \textit{surv2m}, \textit{resp}, \textit{alb}, and \textit{ph} for which we want to show to empirical marginal histograms for real data (in blue) and for synthetic data (in red) superimposed, an overlap of the empirical marginal histograms results in a purple color. The 6 covariates are selected to represent the most challenging marginal distributions according to visual analysis of the empirical marginal histogram plots. In Figure \ref{fig:margHist_support2small_Cvine} we compare the empirical marginal histograms of synthetic data generated by a C-vine with truncation levels 1 and 10 and with no truncation with the real data. It shows almost perfect overlap independent of the truncation level. The reason for this extraordinary good marginal overlap is the fact the vine copulas allow to model marginal distribution separately from the joint dependence structure. Thus, independent of how close the joint dependence structure present in the real data is modeled through the vine copula (e.g. if we truncate at an early tree level this will be worse than for the un-truncated vine), the marginal distributions are always captured well (if there is enough data to model them).

In Figures \ref{fig:margHist_support2small_CTGAN_TVAE} and \ref{fig:margHist_support2small_PrivBayes_PrivPGD} the empirical marginal histograms of the competitor models are compared. Particularly, we observe that the PrivBayes model (for various $\epsilon$s) considerably struggles to reproduce marginal distributions of the real data confirming its low values of $P_{\alpha}$ and $R_{\beta}$ and high authenticity in Figure \ref{fig:statfidelity_allmodels_support2} in Section \ref{sec:statFidelity_support2}.

\begin{figure}[ht]
    \centering
    \begin{tabular}{cccccc}
            \includegraphics[width=0.14\textwidth]{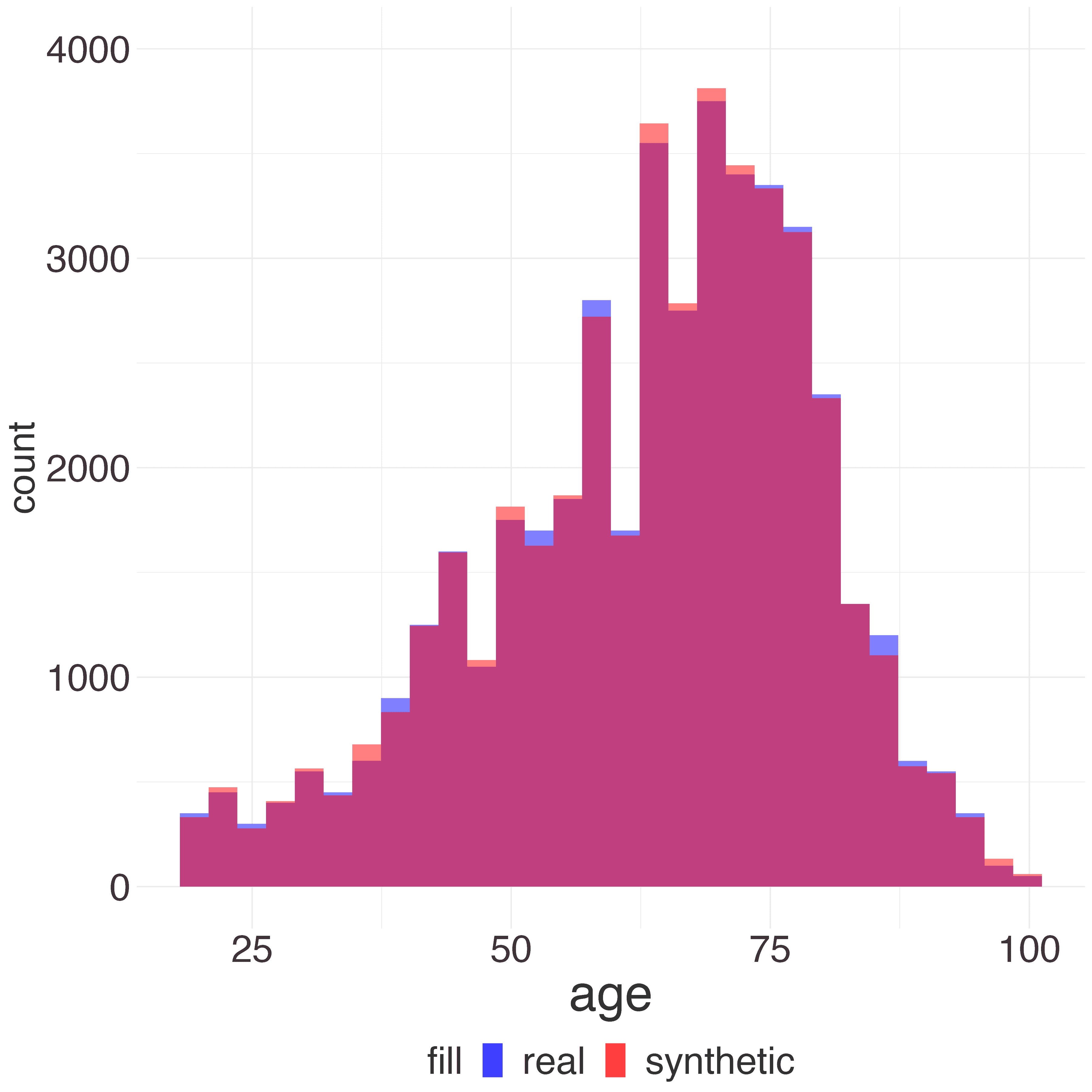} &  
            \includegraphics[width=0.14\textwidth]{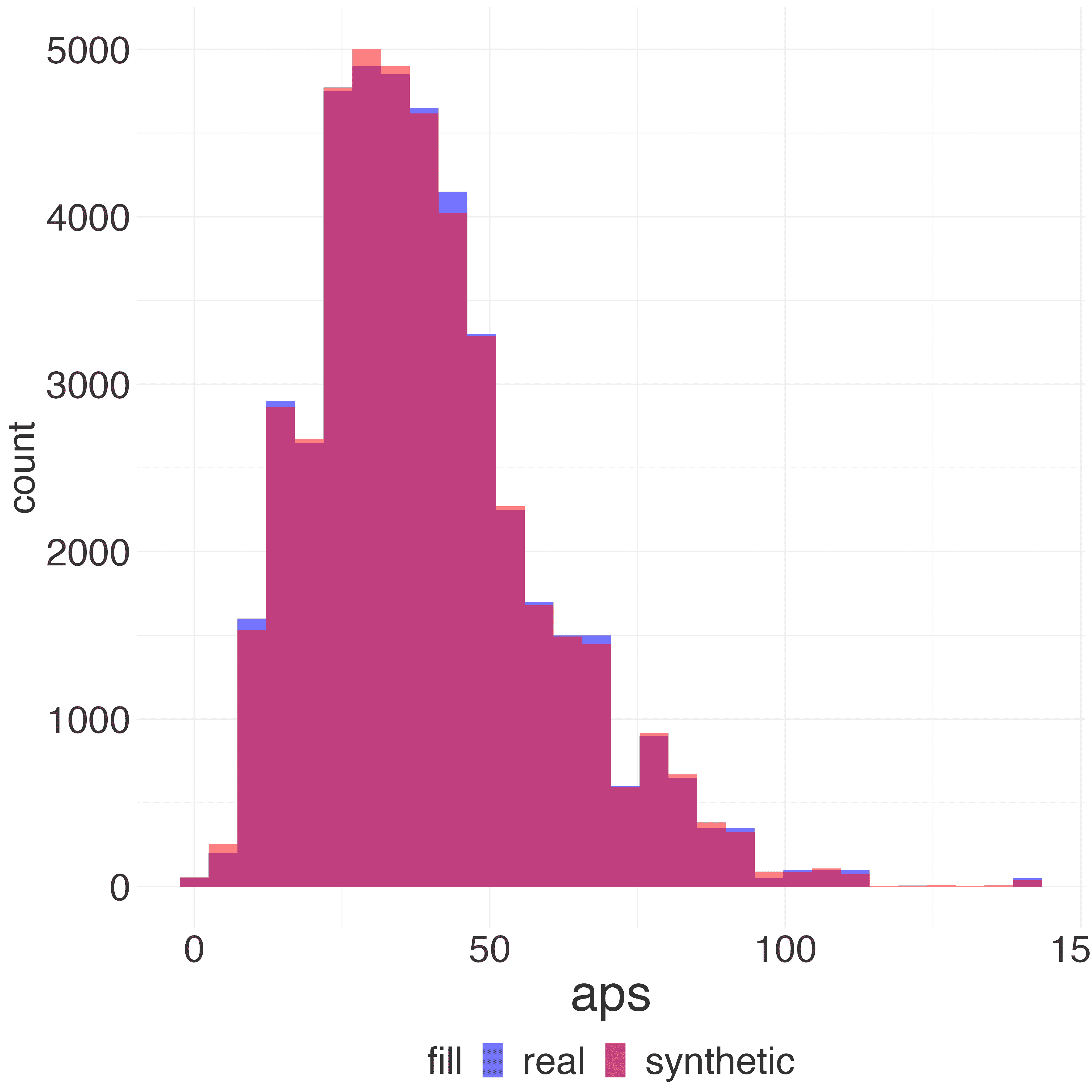} &
            \includegraphics[width=0.14\textwidth]{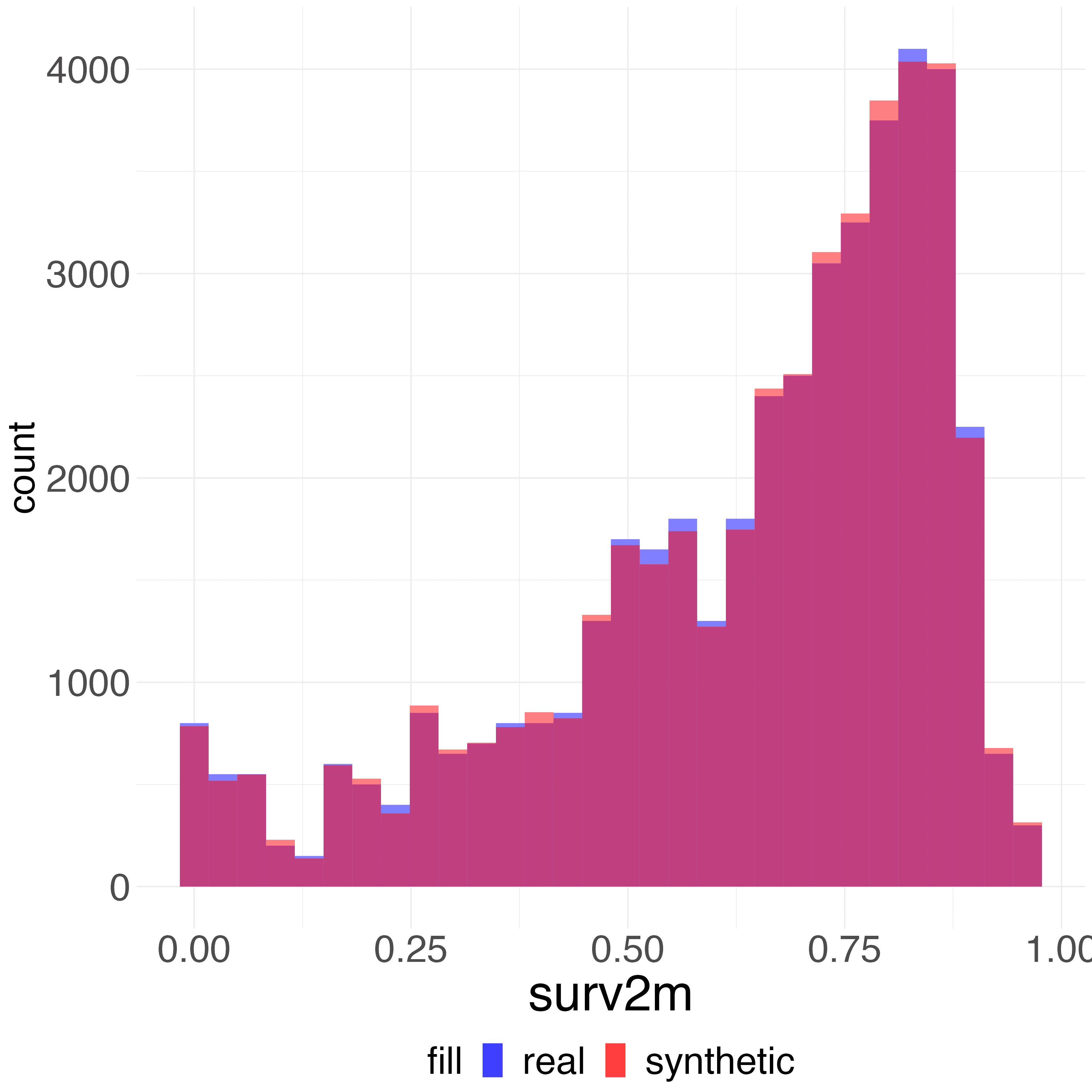} &
            \includegraphics[width=0.14\textwidth]{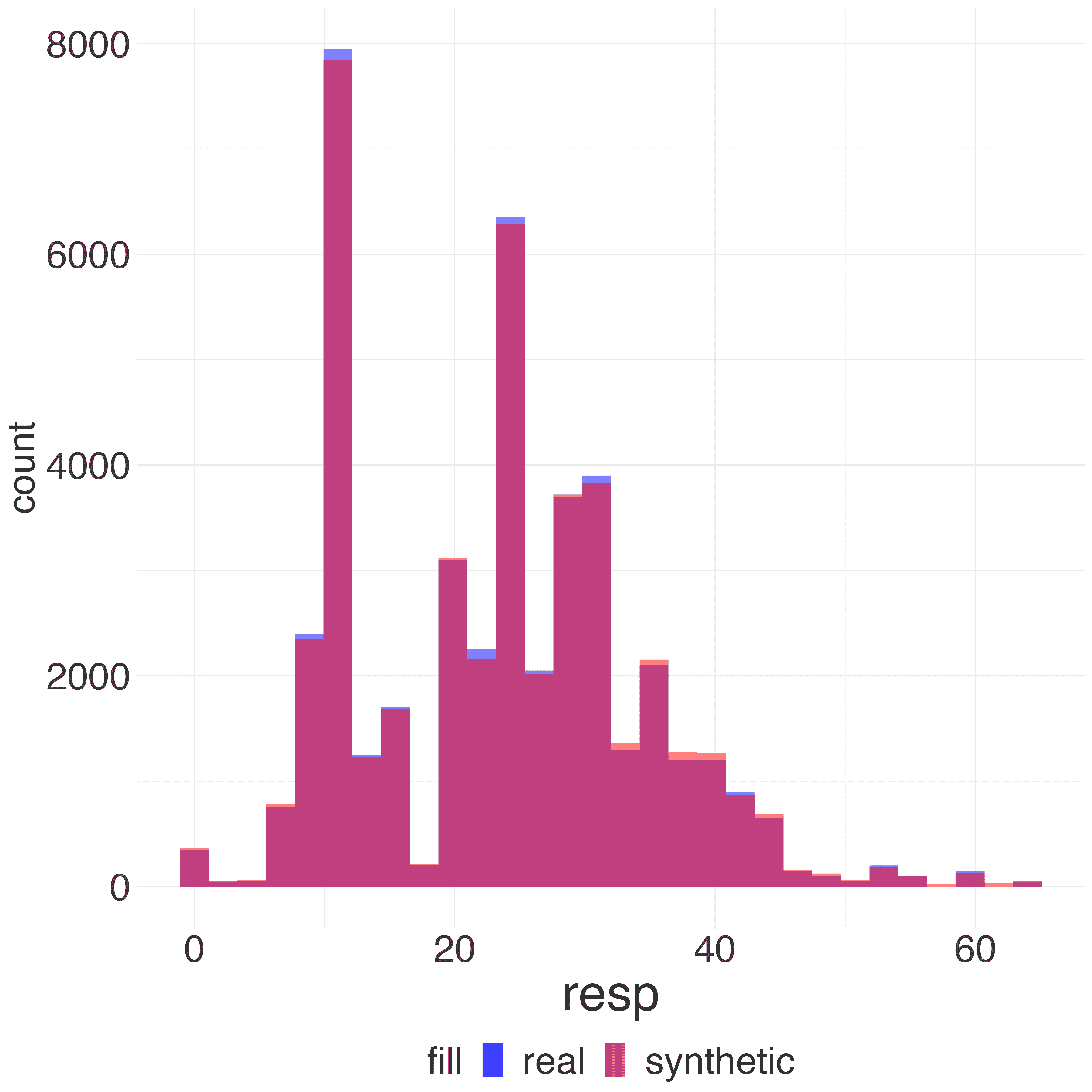} &
            \includegraphics[width=0.14\textwidth]{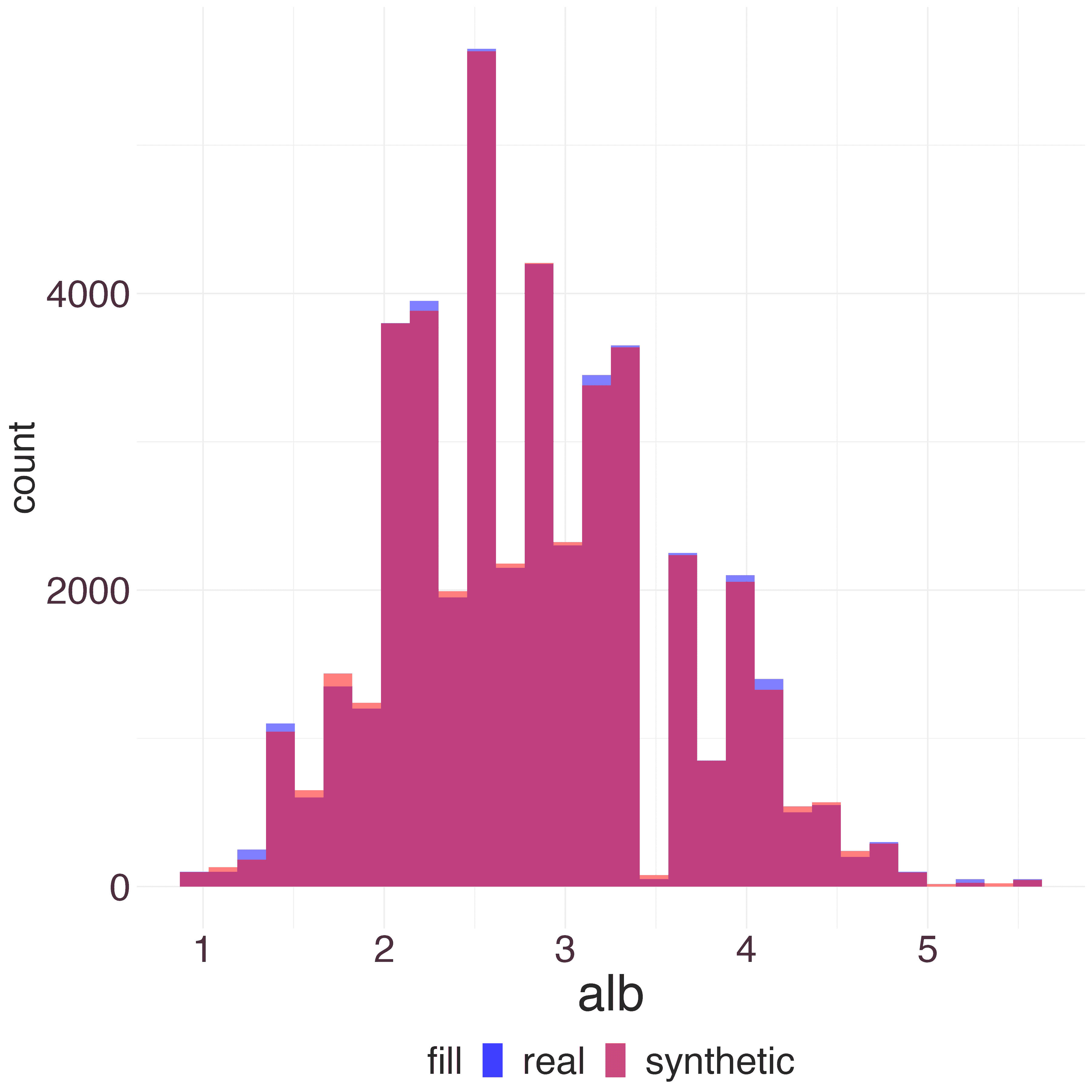} &
            \includegraphics[width=0.14\textwidth]{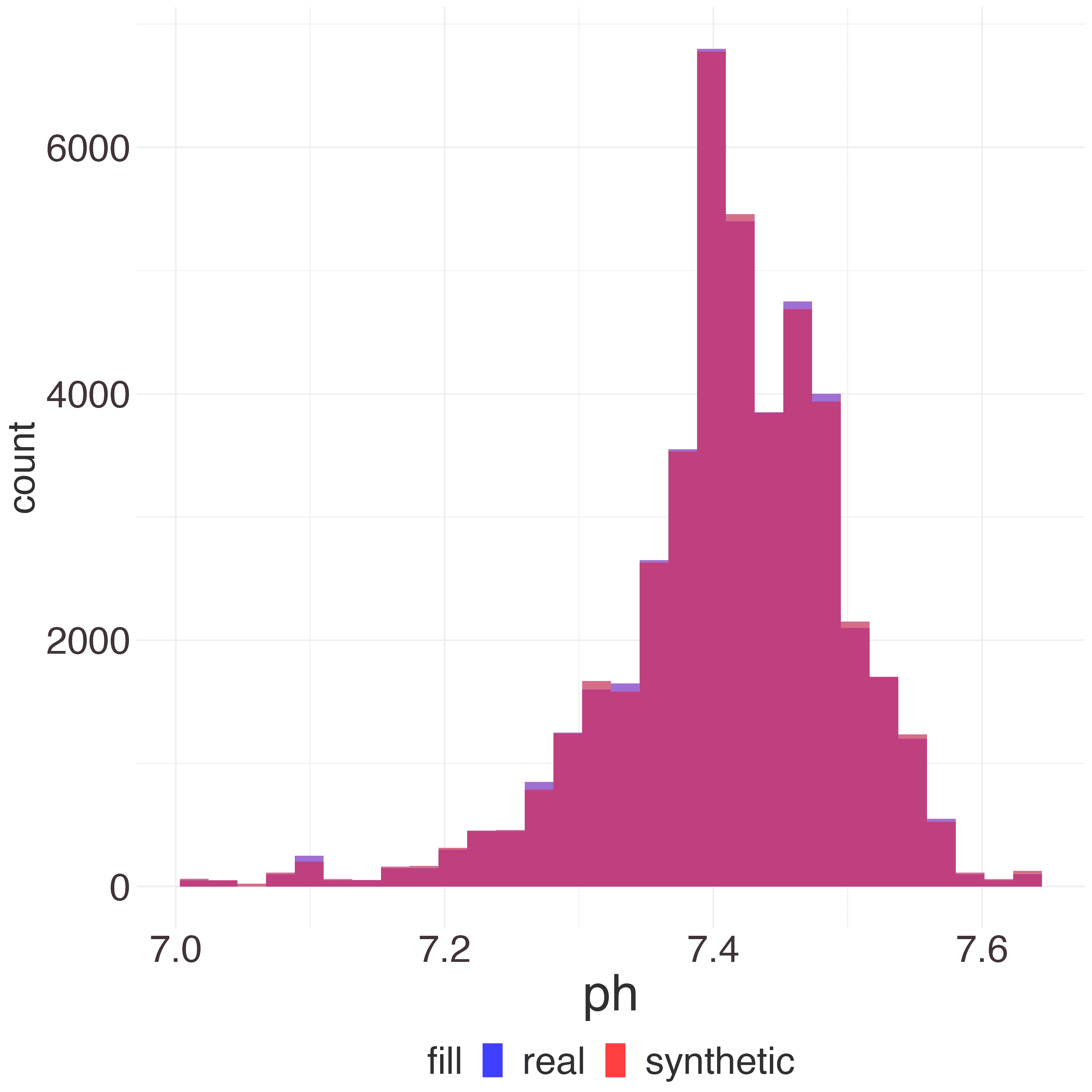}  \\
            (a) age & (b) aps & (c) surv2m & (d) resp & (e) alb & (f) ph \\
            \multicolumn{6}{c}{(1) C-vine, truncation at 1.} \\
             &&&&& \\
            \includegraphics[width=0.14\textwidth]{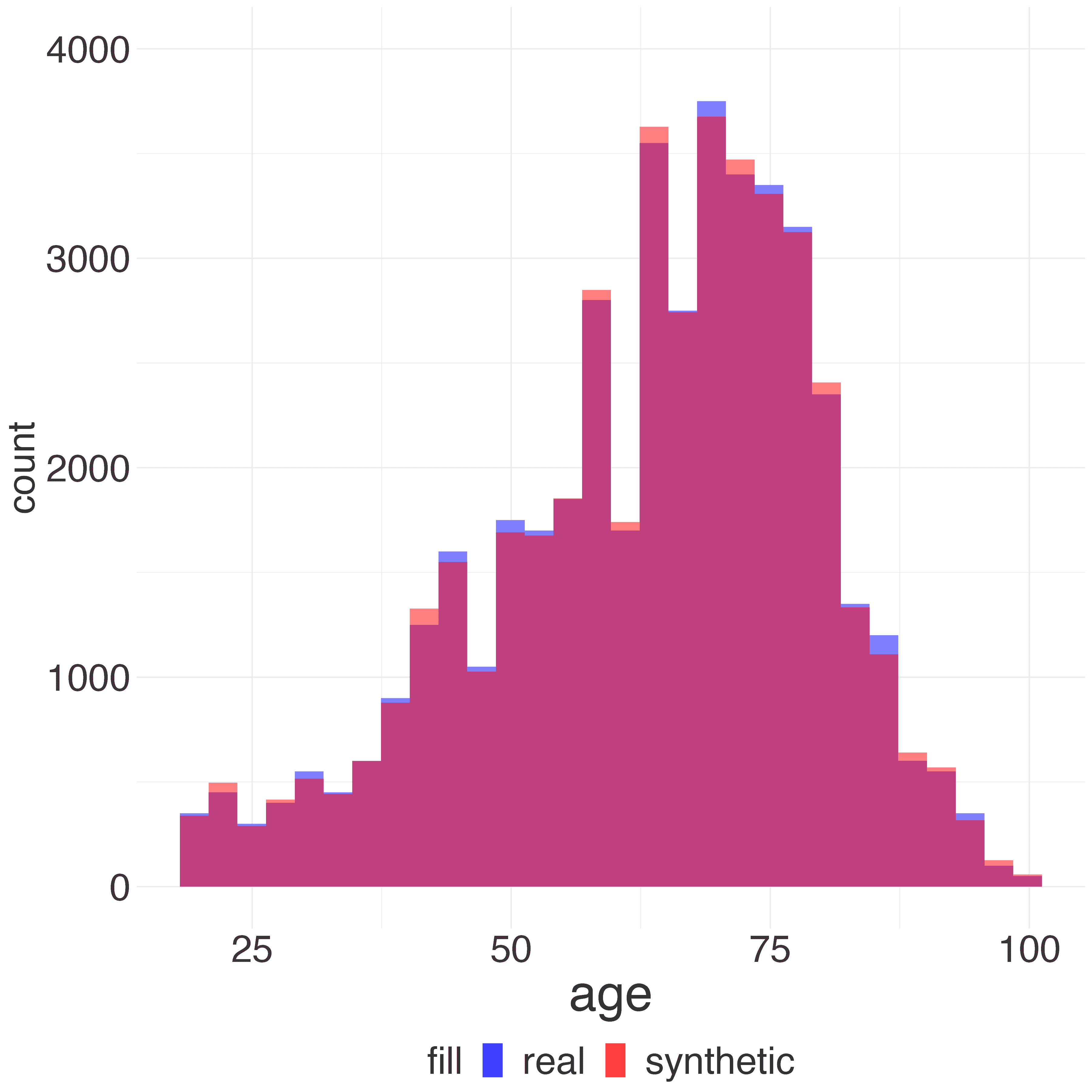} &  
            \includegraphics[width=0.14\textwidth]{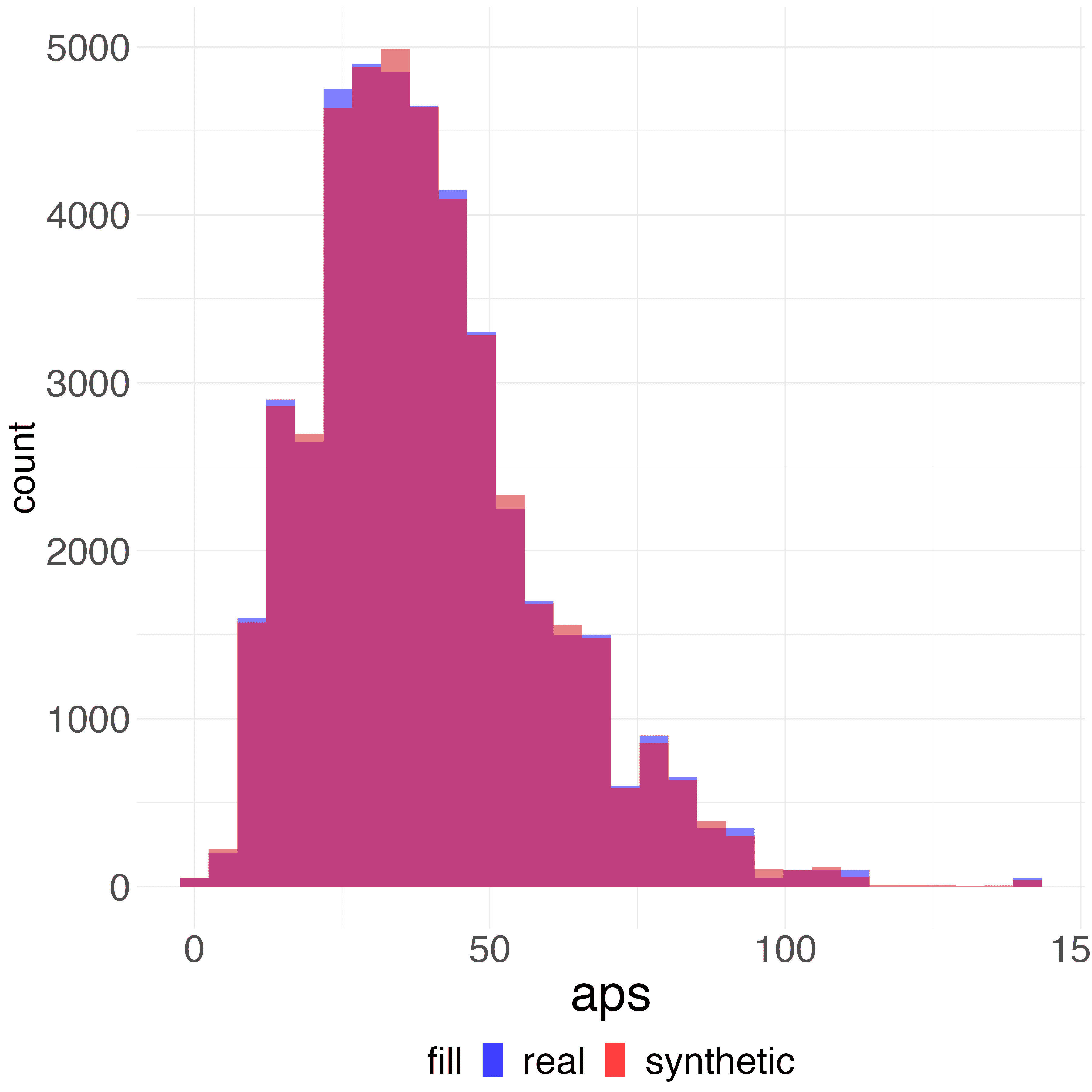} &
            \includegraphics[width=0.14\textwidth]{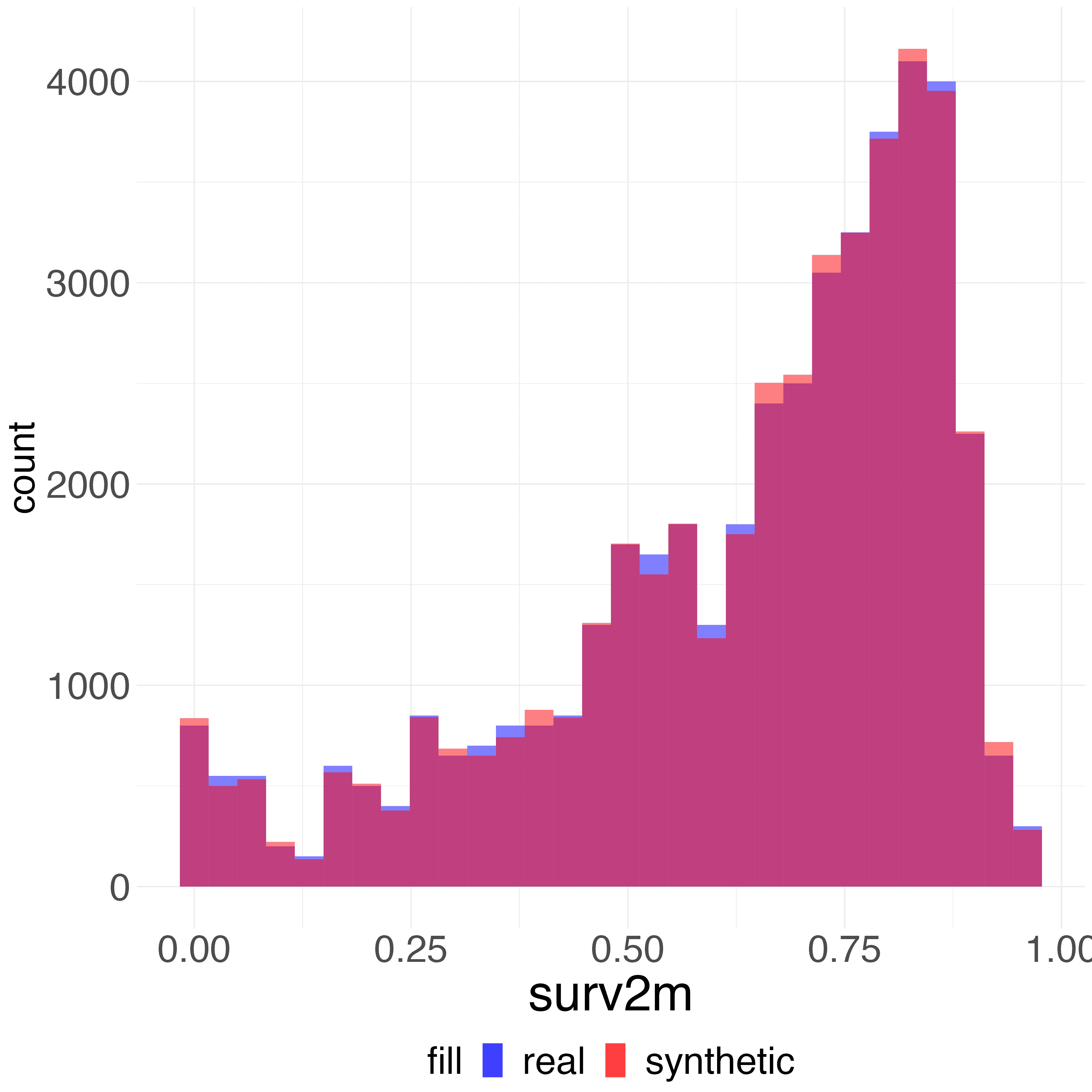} &
            \includegraphics[width=0.14\textwidth]{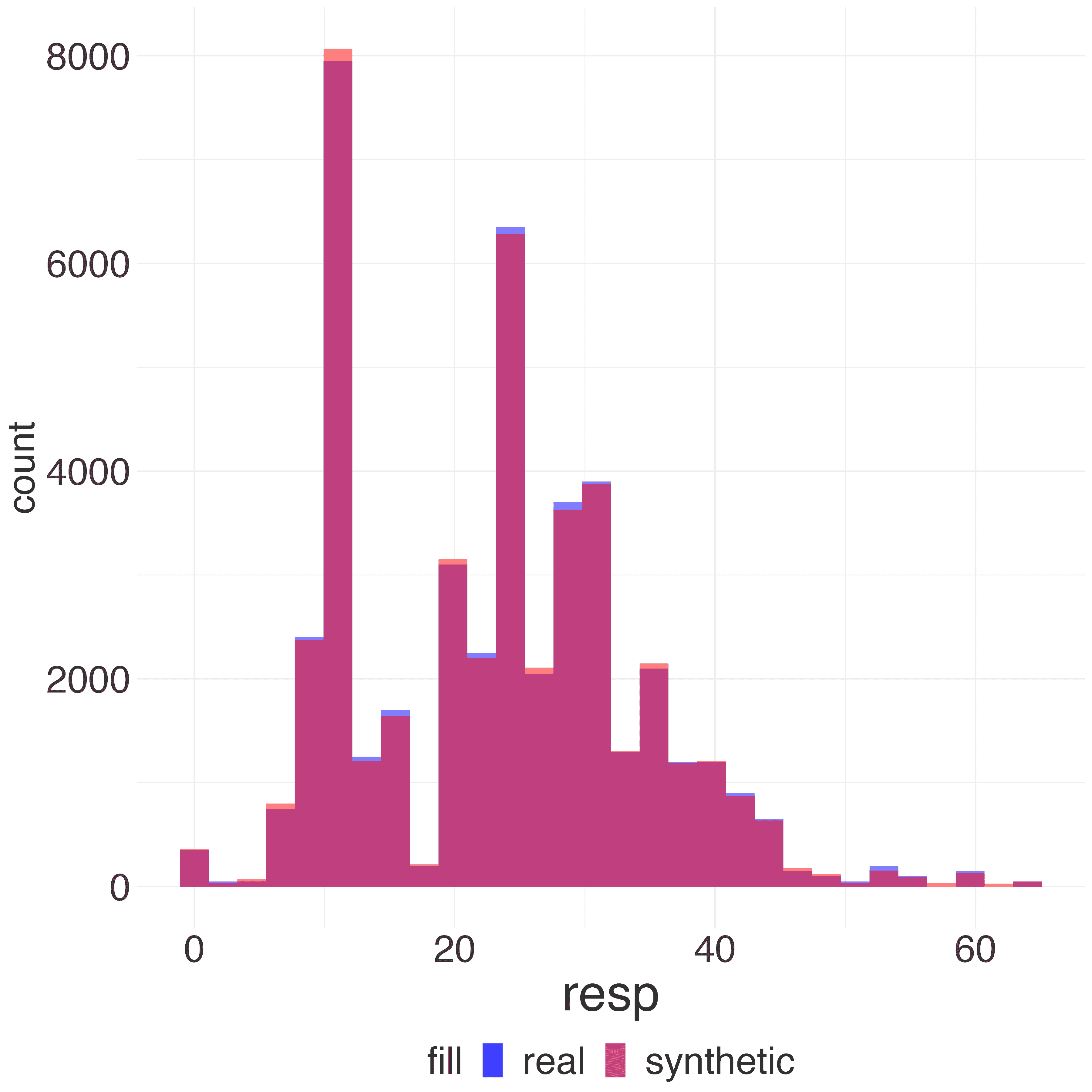} &
            \includegraphics[width=0.14\textwidth]{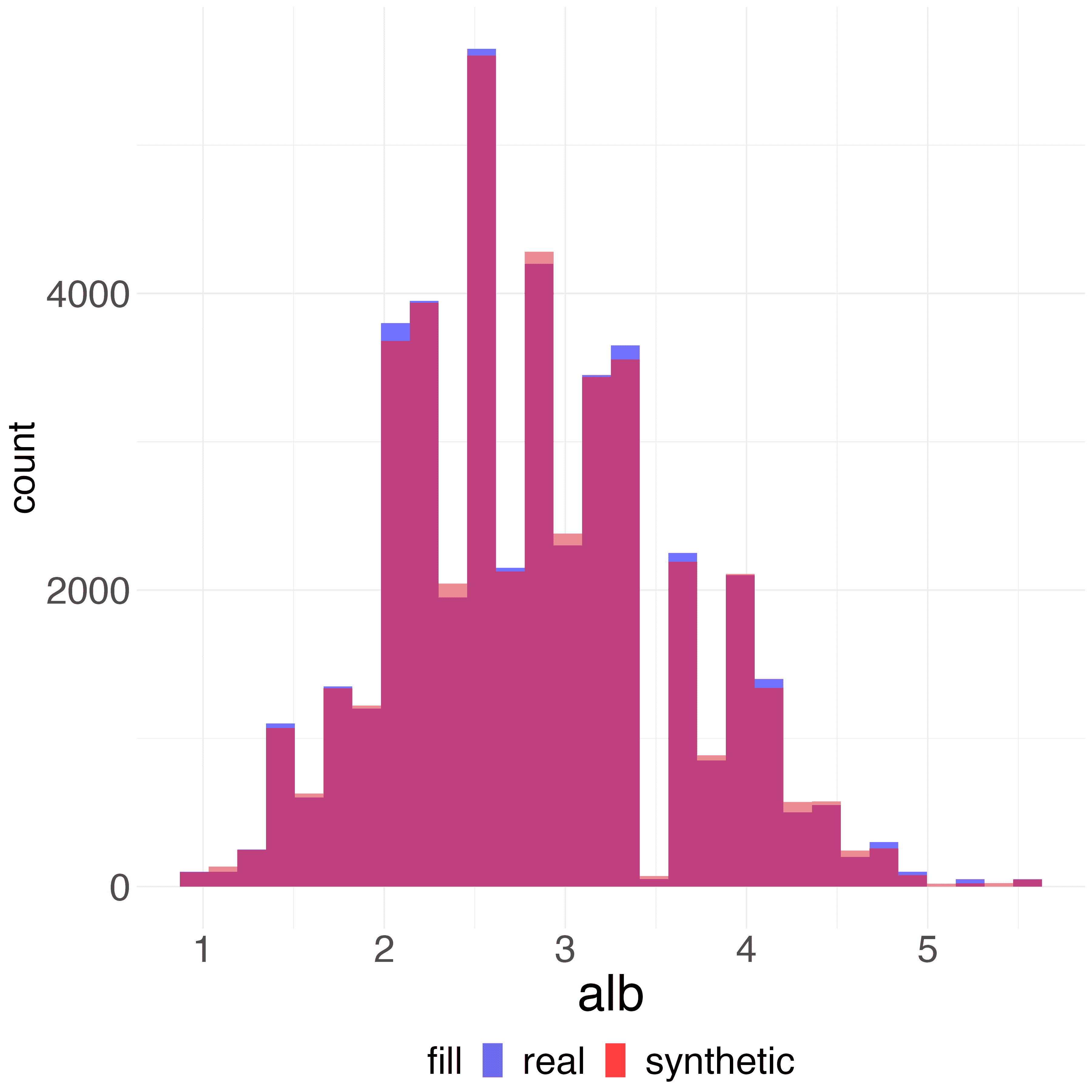} &
            \includegraphics[width=0.14\textwidth]{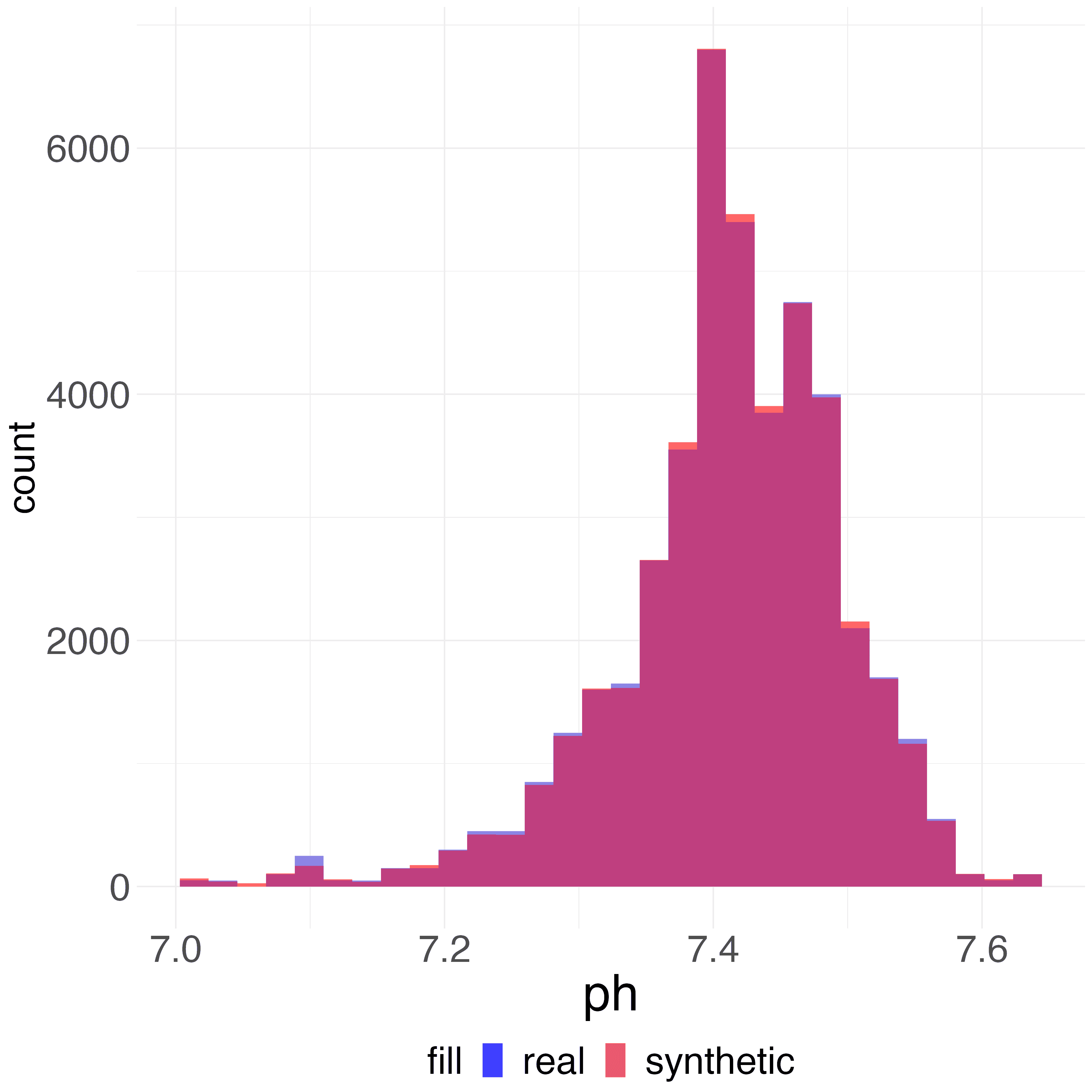}  \\
            (a) age & (b) aps & (c) surv2m & (d) resp & (e) alb & (f) ph \\
            \multicolumn{6}{c}{(2) C-vine, truncation at 10.} \\
             &&&&& \\             \includegraphics[width=0.14\textwidth]{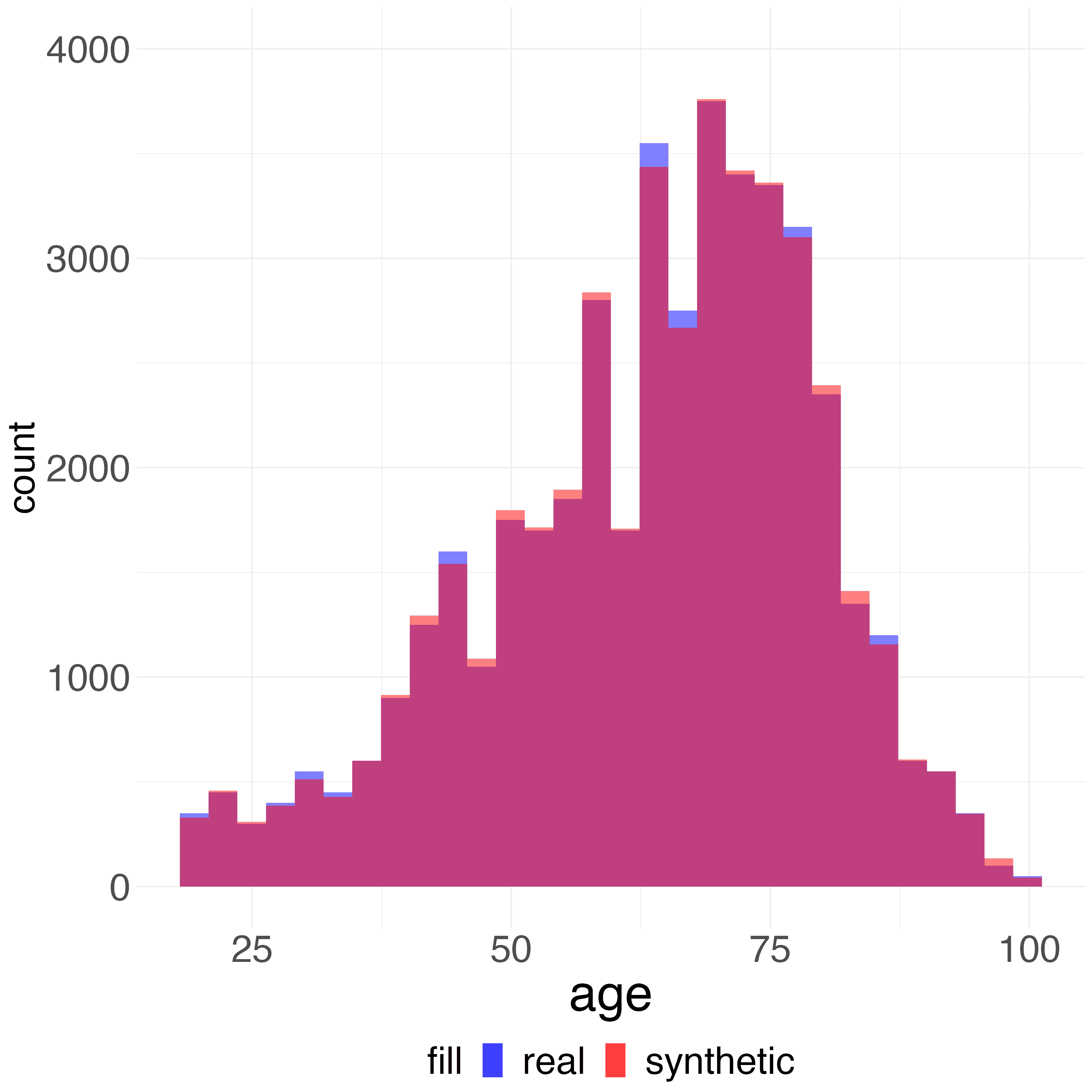} &  
            \includegraphics[width=0.14\textwidth]{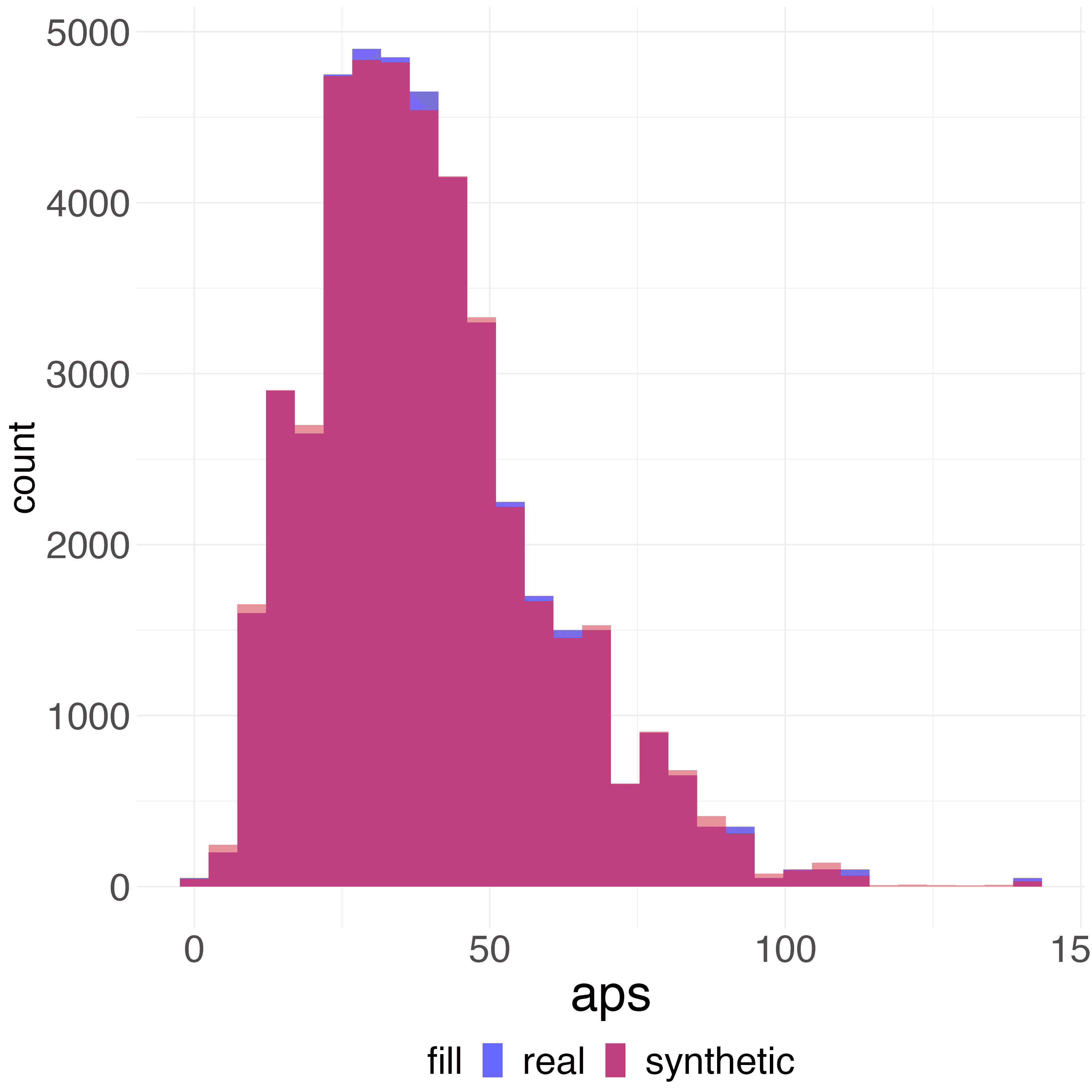} &
            \includegraphics[width=0.14\textwidth]{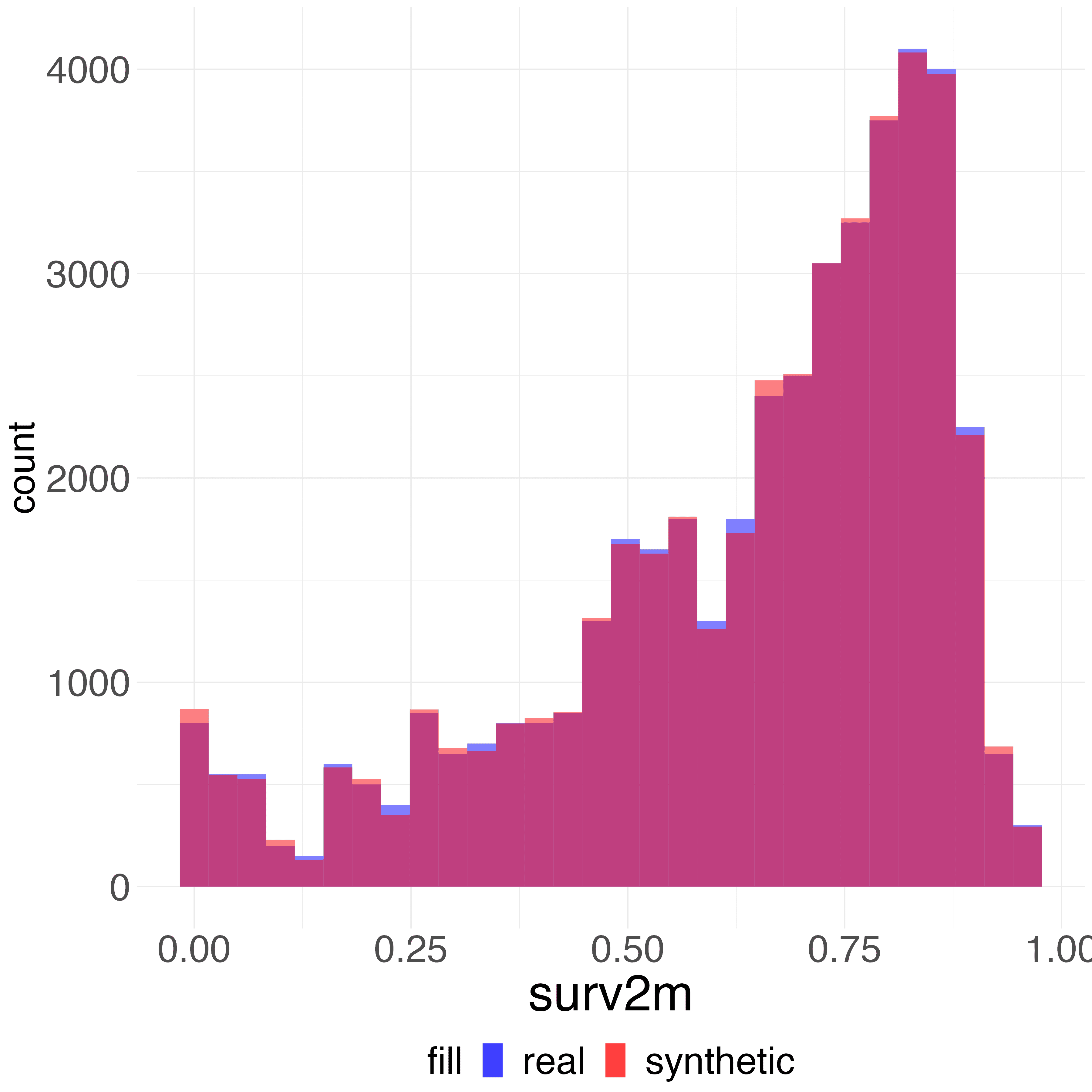} &
            \includegraphics[width=0.14\textwidth]{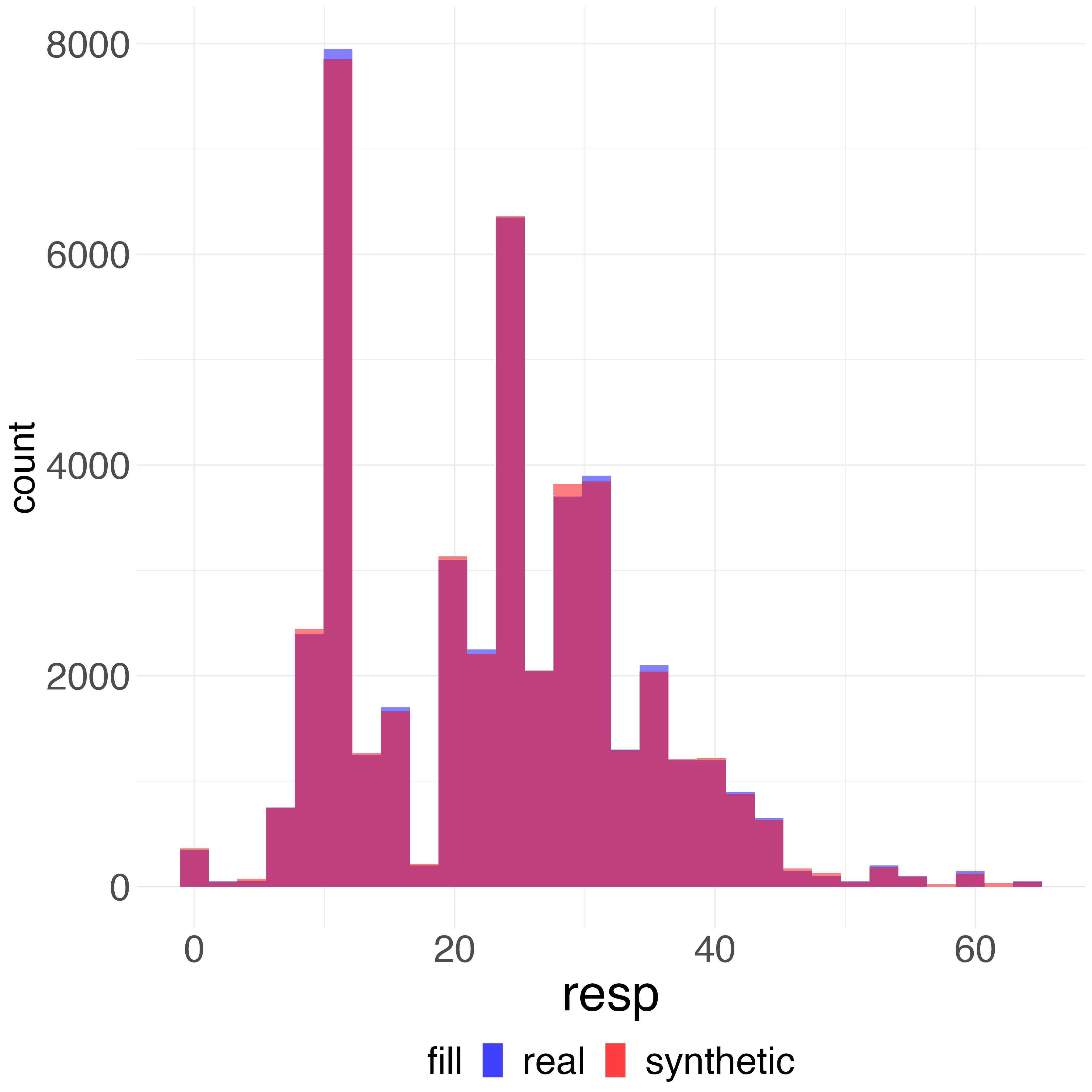} &
            \includegraphics[width=0.14\textwidth]{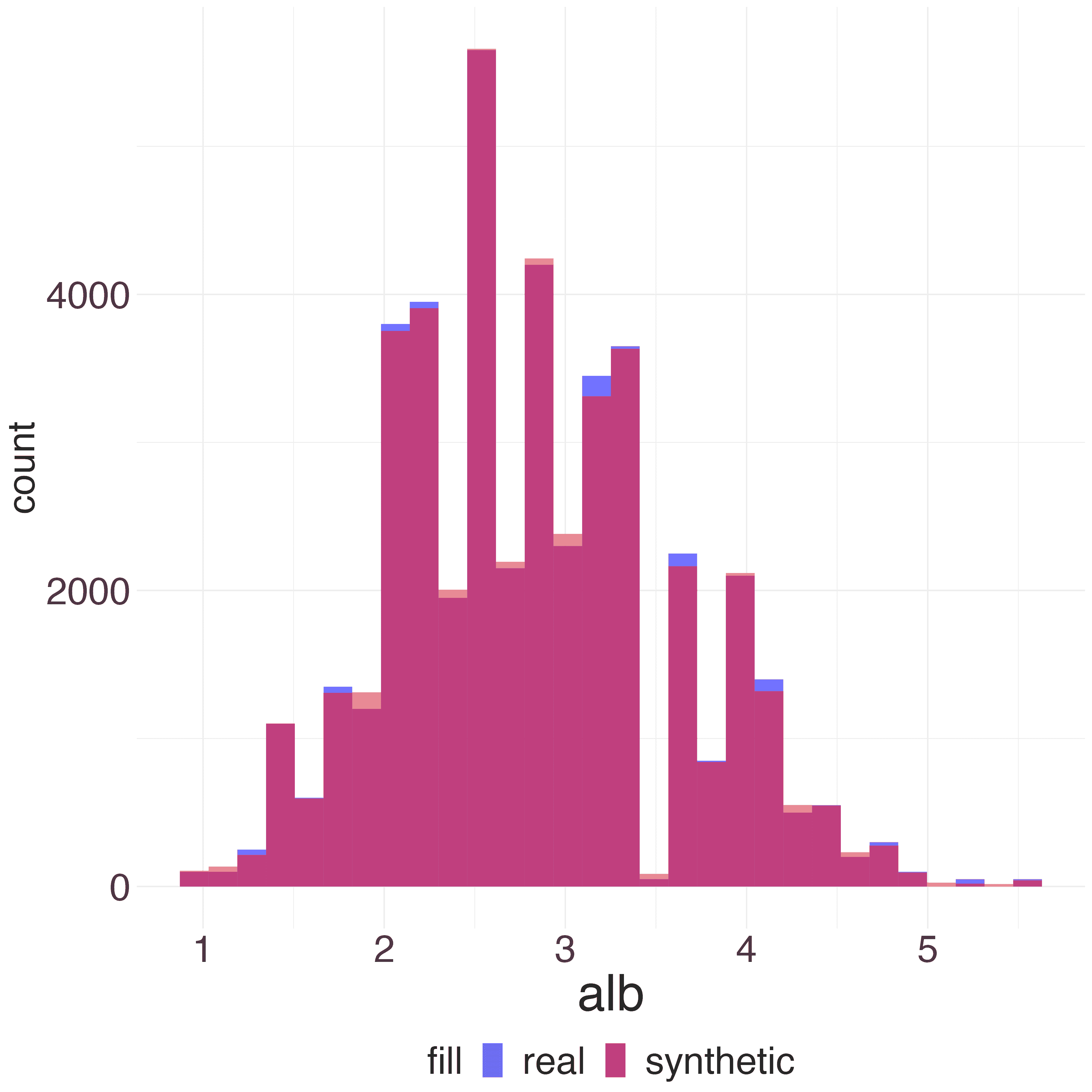} &
            \includegraphics[width=0.14\textwidth]{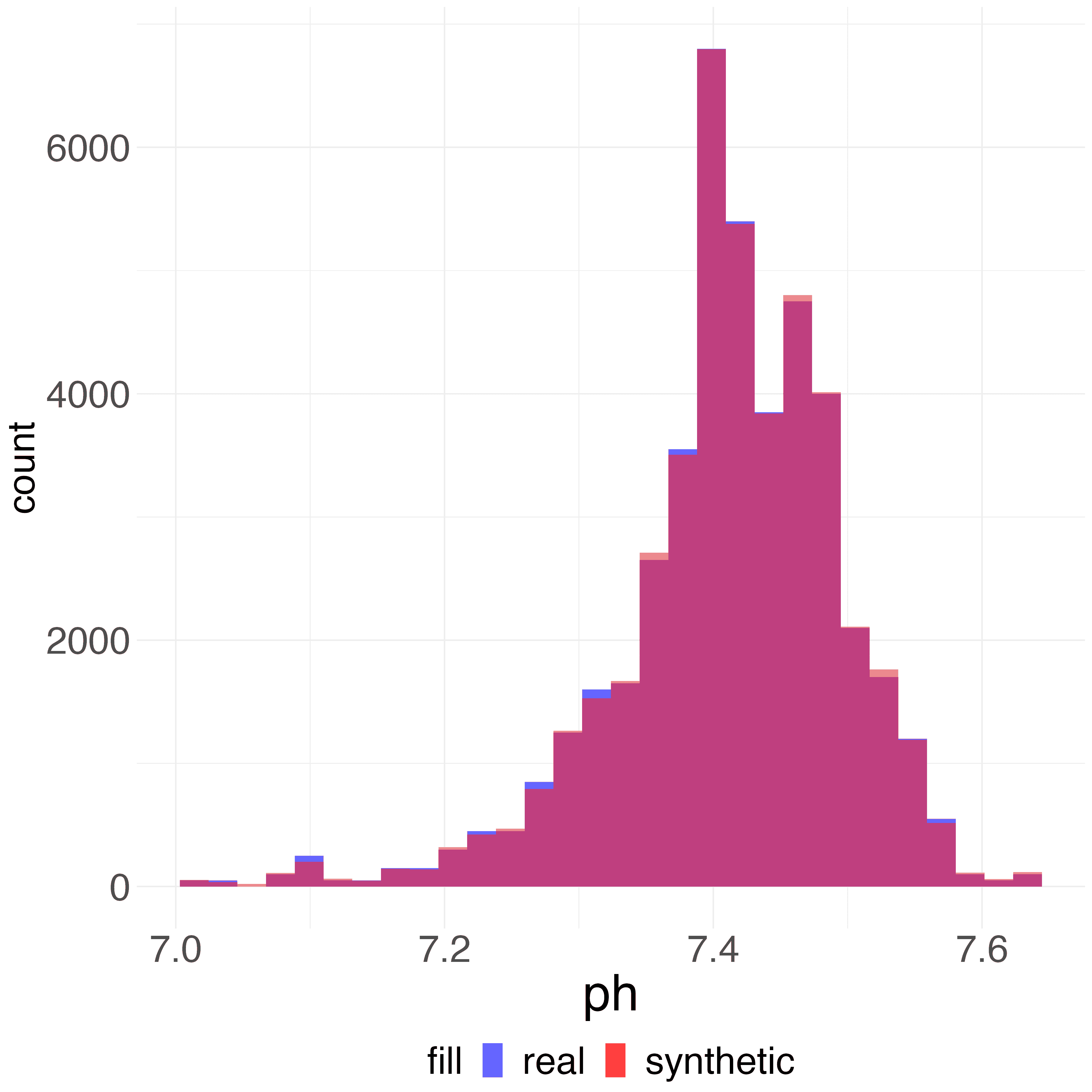}  \\
            (a) age & (b) aps & (c) surv2m & (d) resp & (e) alb & (f) ph \\
            \multicolumn{6}{c}{(3) C-vine, no truncation.} 
    \end{tabular}
    \caption{SUPPORT2 data: Overlapping empirical marginal histograms of covariates \textit{age, aps, surv2m, resp, alb} and \textit{ph} estimate on the real data (blue) and synthetic data (red) generated by a C-vine with truncation level 1 and 10 and no truncation.}
    \label{fig:margHist_support2small_Cvine}
\end{figure}

\begin{figure}[ht]
    \centering
    \begin{tabular}{cccccc}
            \includegraphics[width=0.14\textwidth]{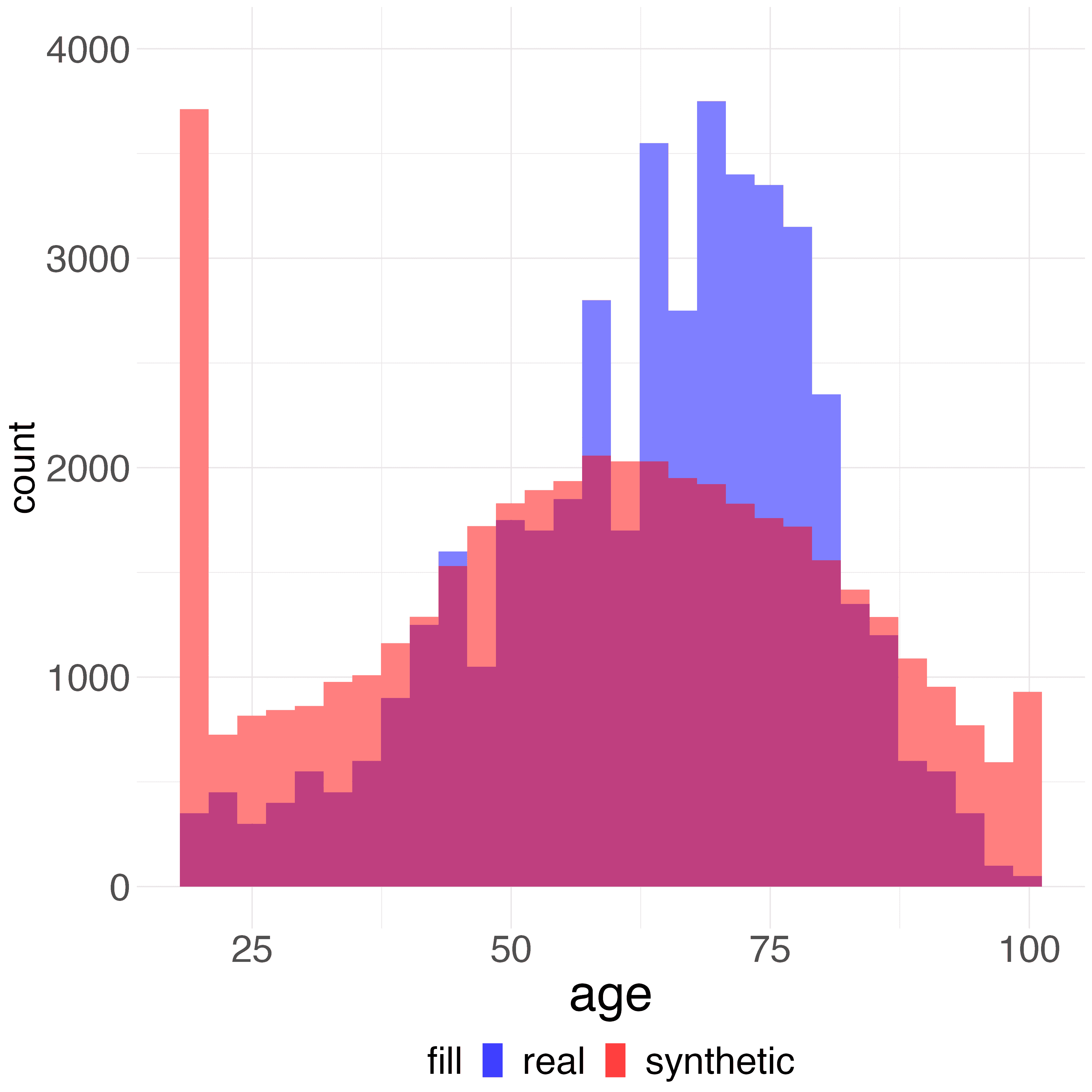} &  
            \includegraphics[width=0.14\textwidth]{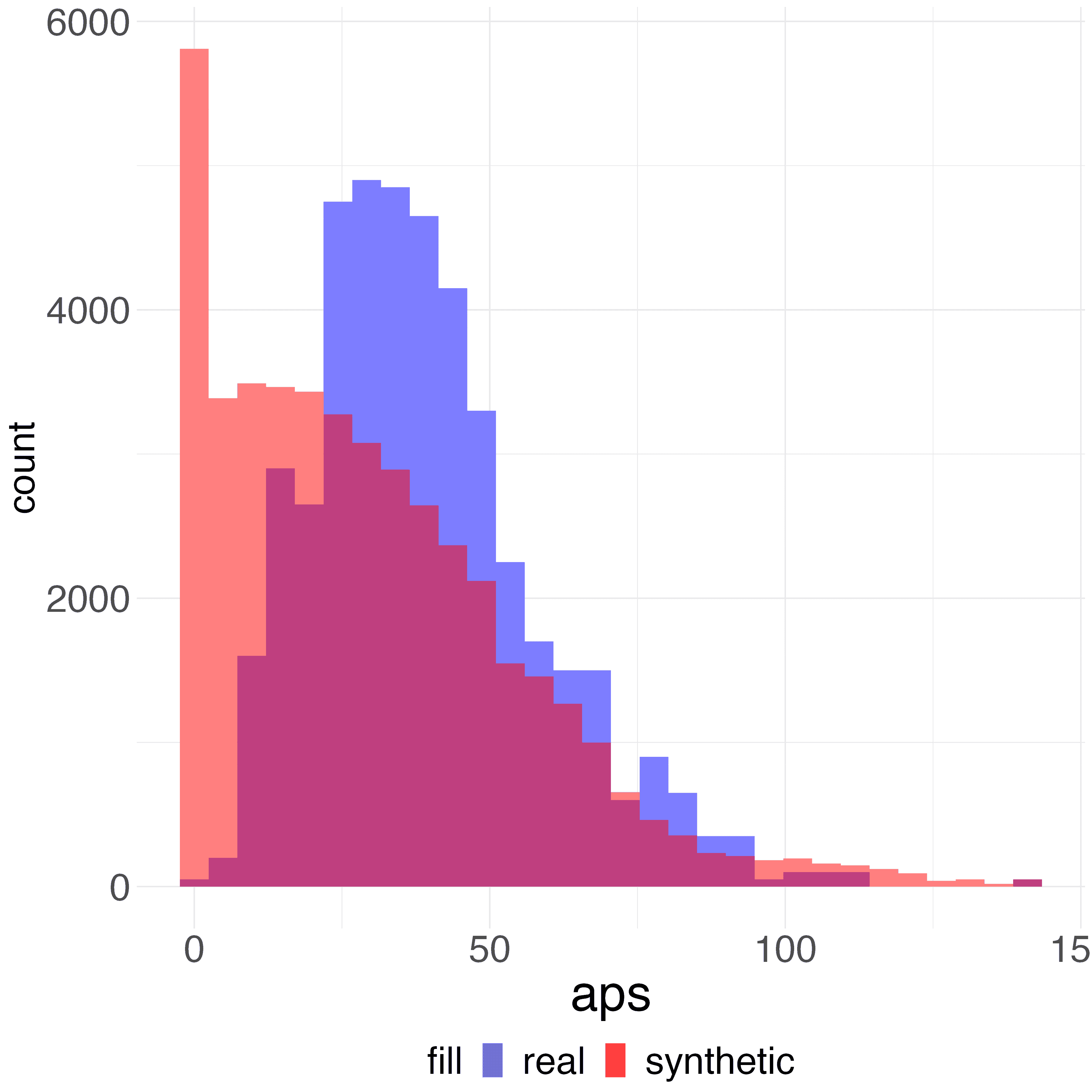} &
            \includegraphics[width=0.14\textwidth]{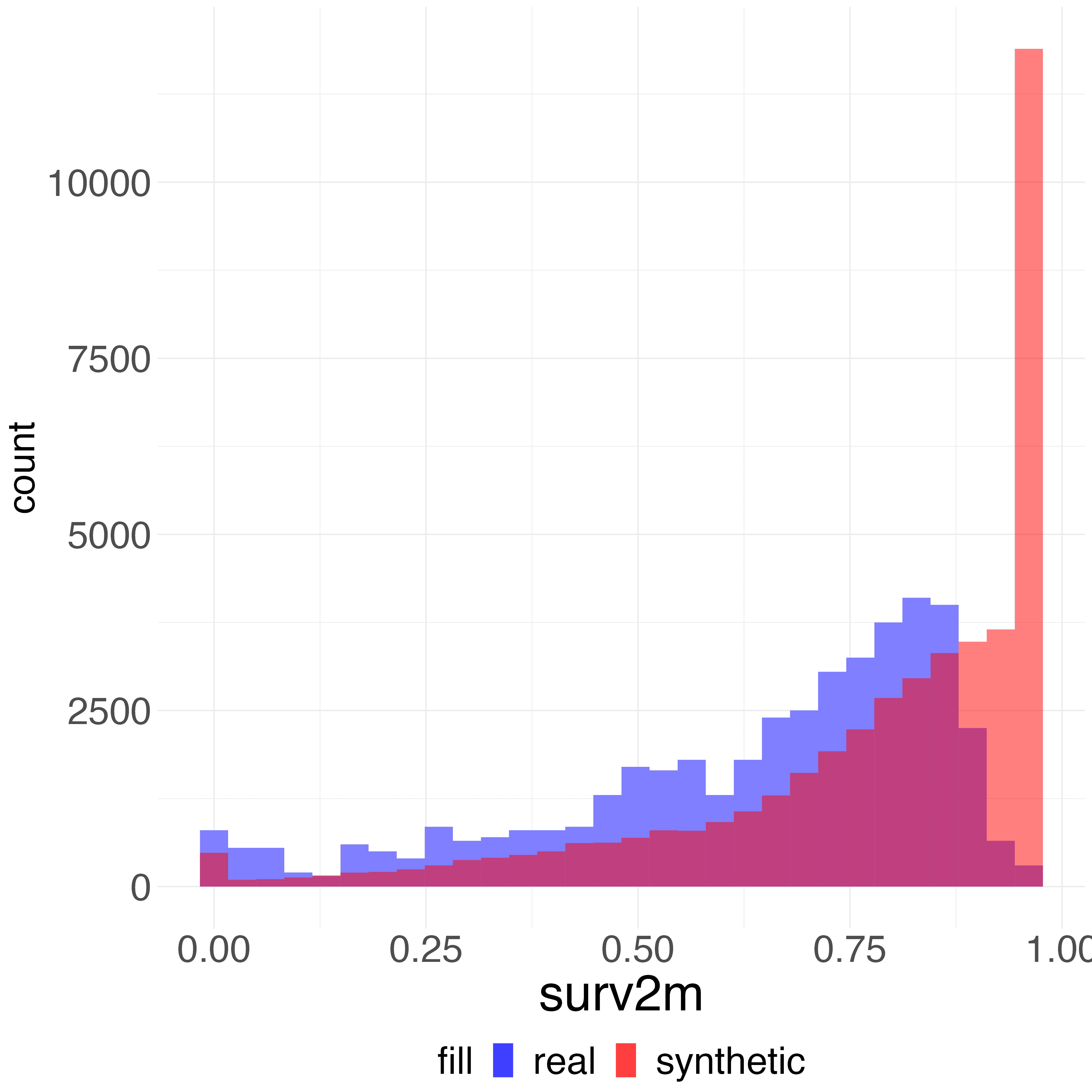} &
            \includegraphics[width=0.14\textwidth]{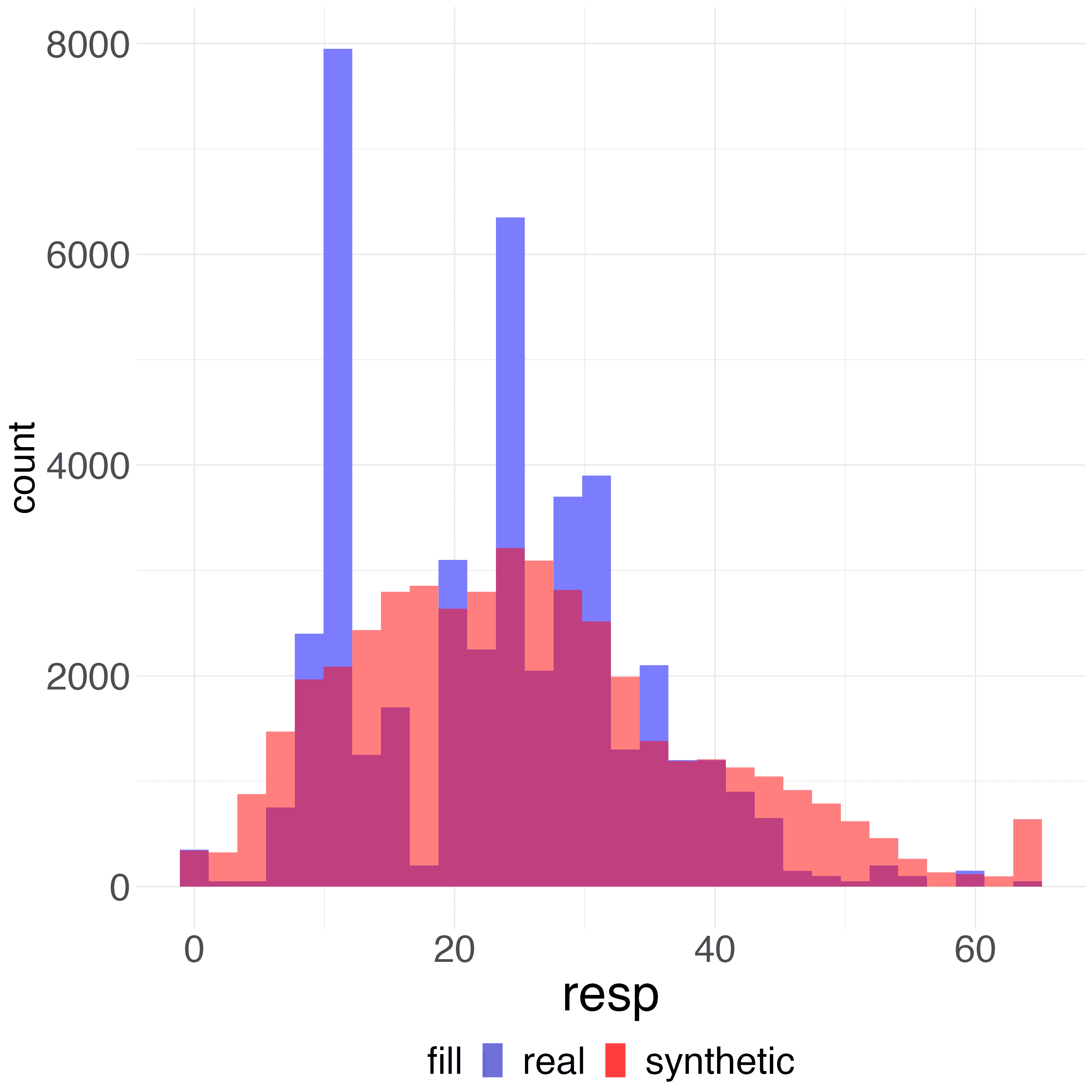} &
            \includegraphics[width=0.14\textwidth]{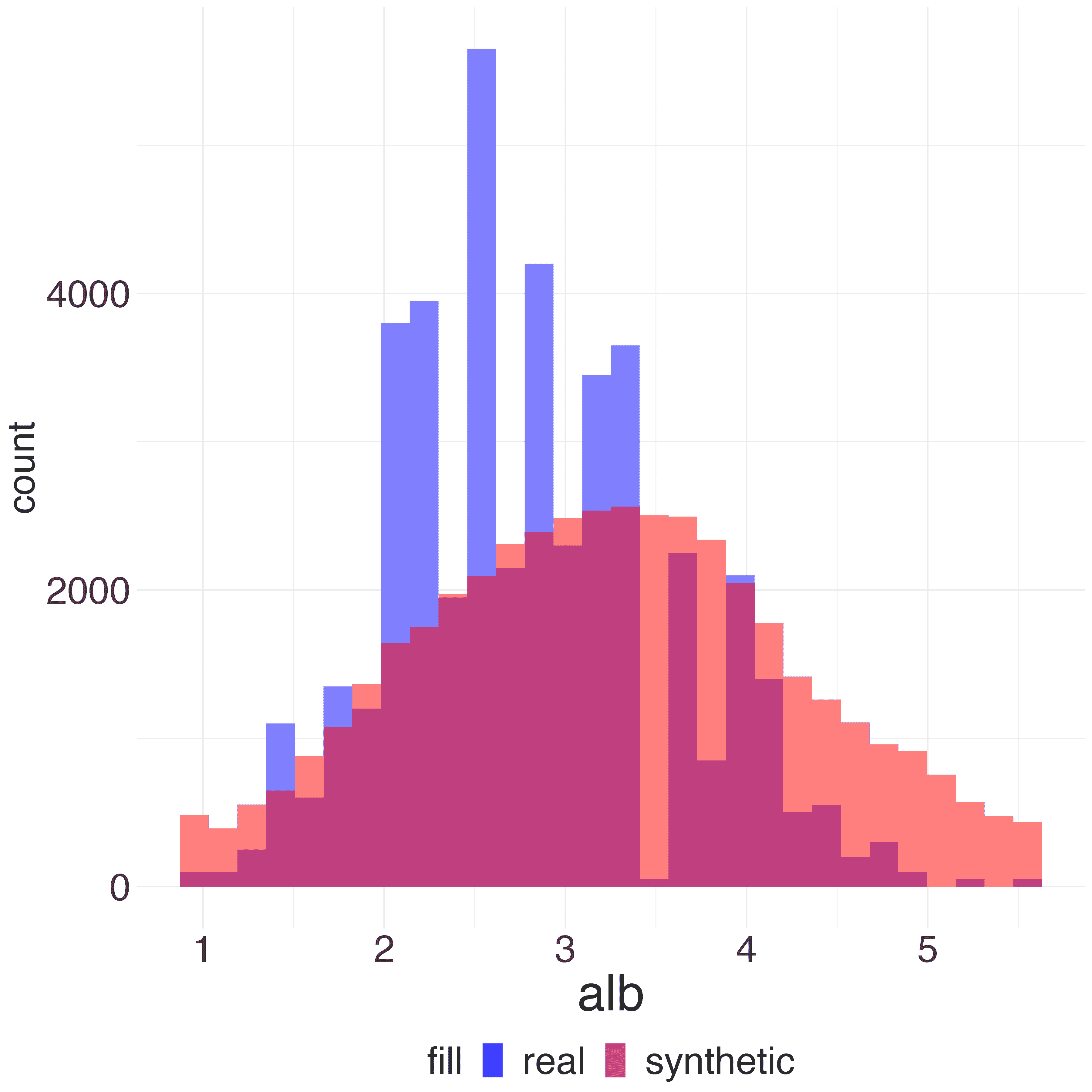} &
            \includegraphics[width=0.14\textwidth]{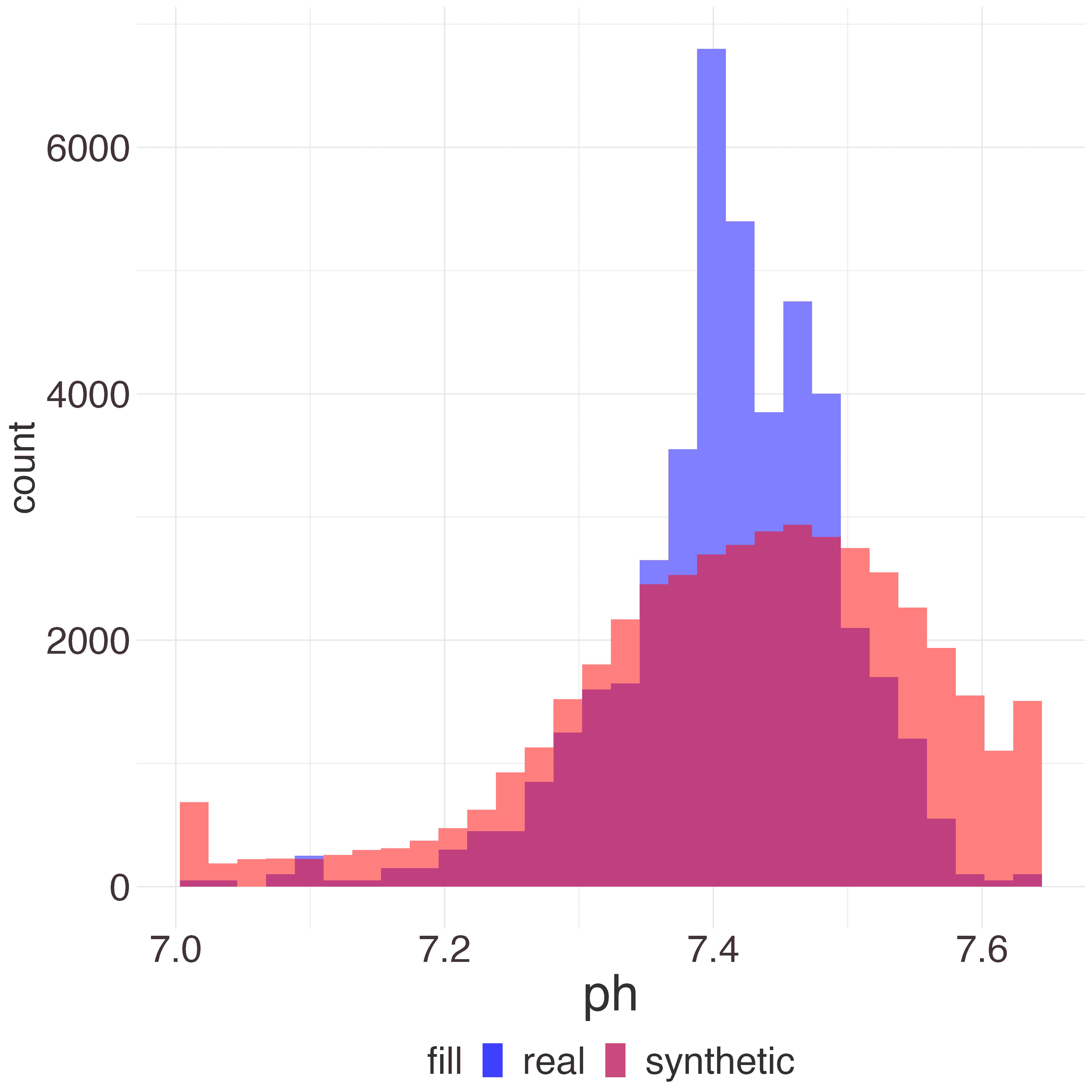}  \\
            (a) age & (b) aps & (c) surv2m & (d) resp & (e) alb & (f) ph \\
            \multicolumn{6}{c}{(4) CTGAN.} \\
            &&&&& \\
            \includegraphics[width=0.14\textwidth]{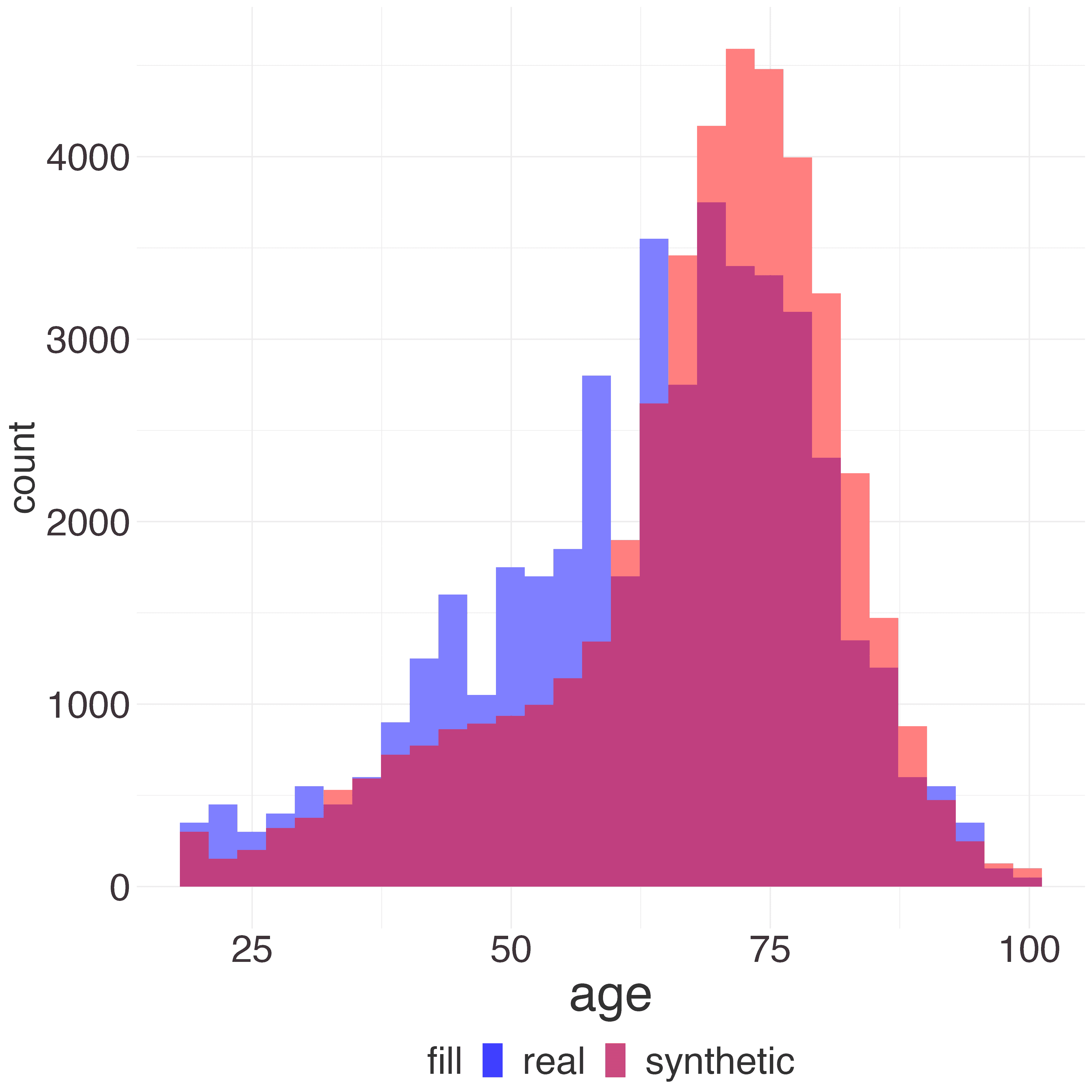} &  
            \includegraphics[width=0.14\textwidth]{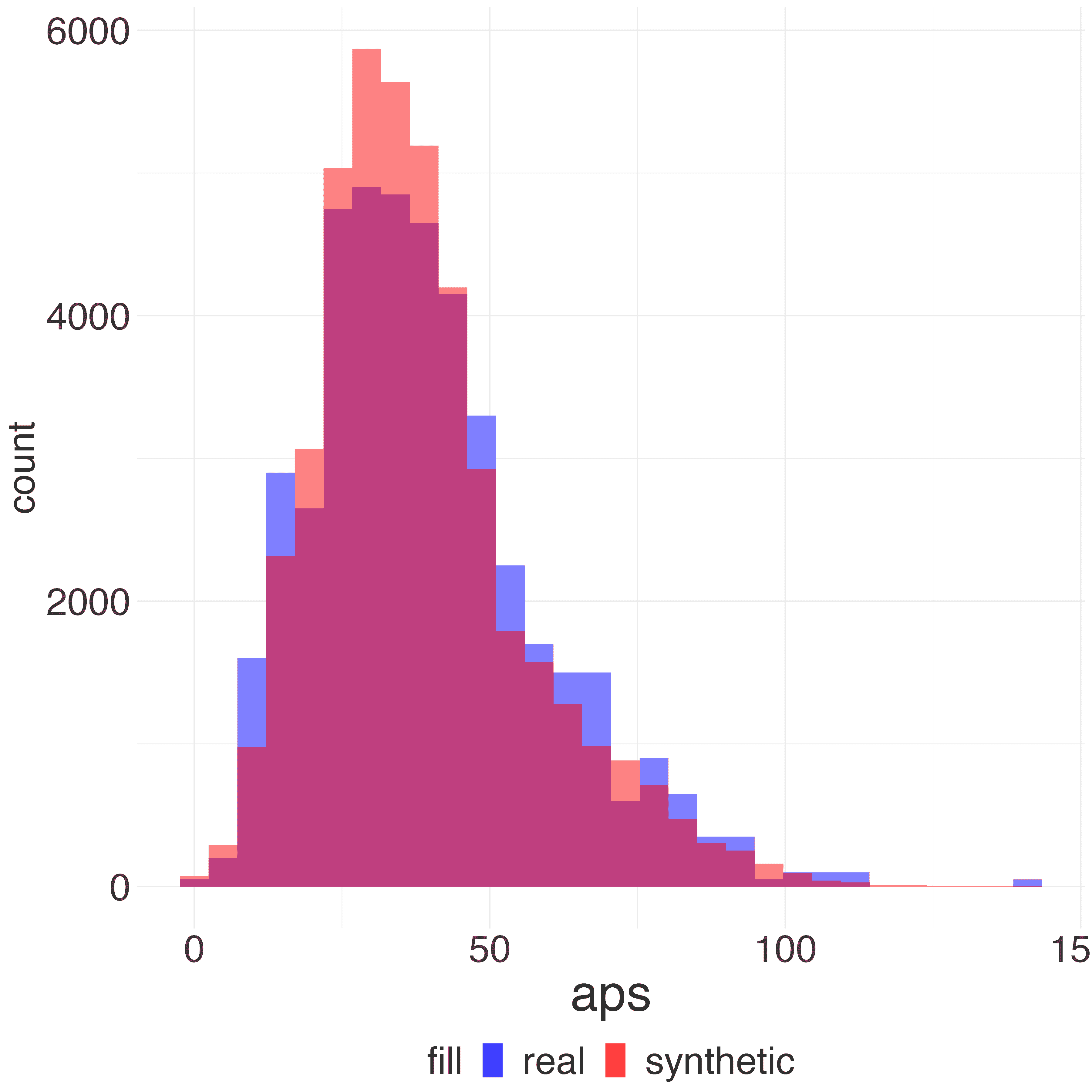} &
            \includegraphics[width=0.14\textwidth]{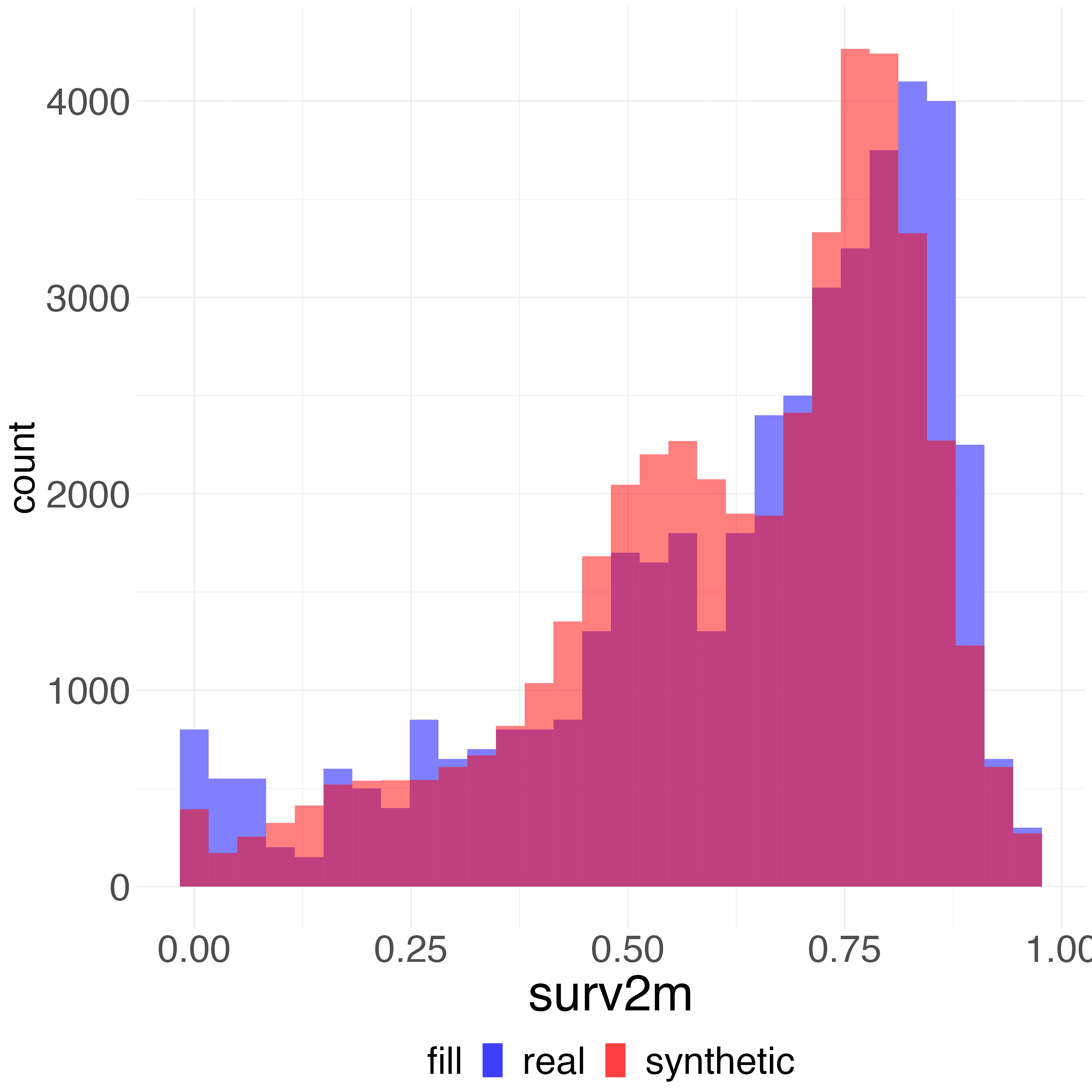} &
            \includegraphics[width=0.14\textwidth]{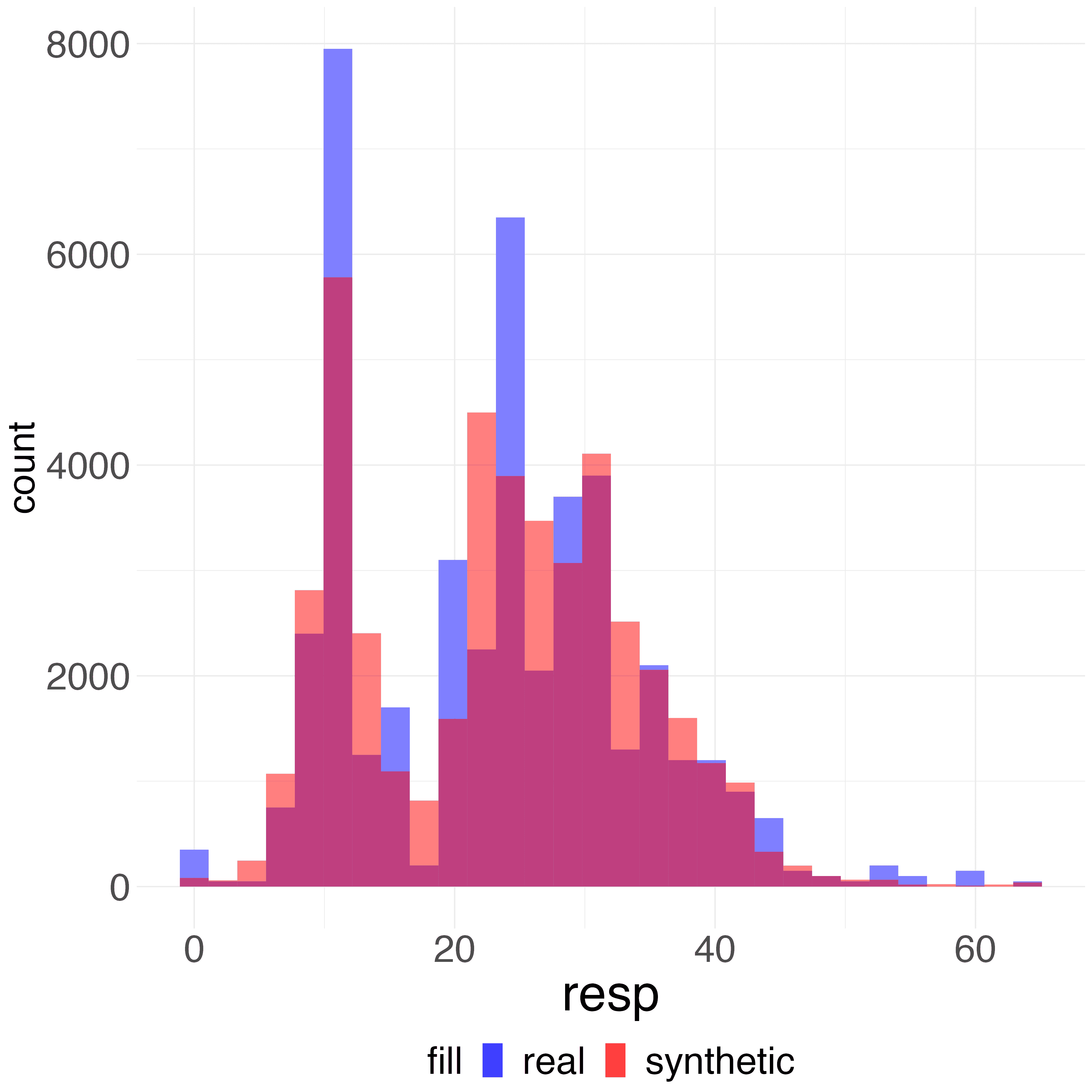} &
            \includegraphics[width=0.14\textwidth]{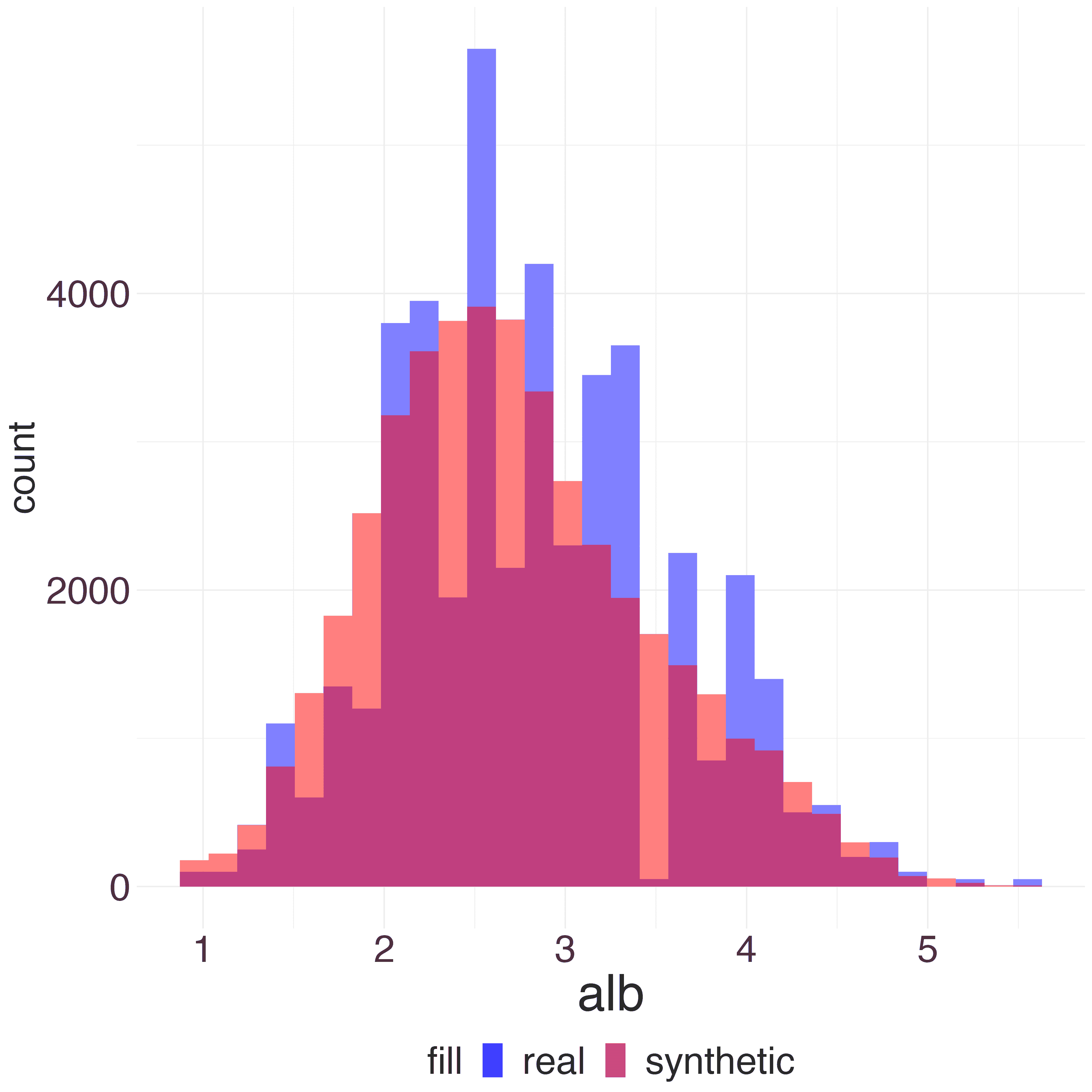} &
            \includegraphics[width=0.14\textwidth]{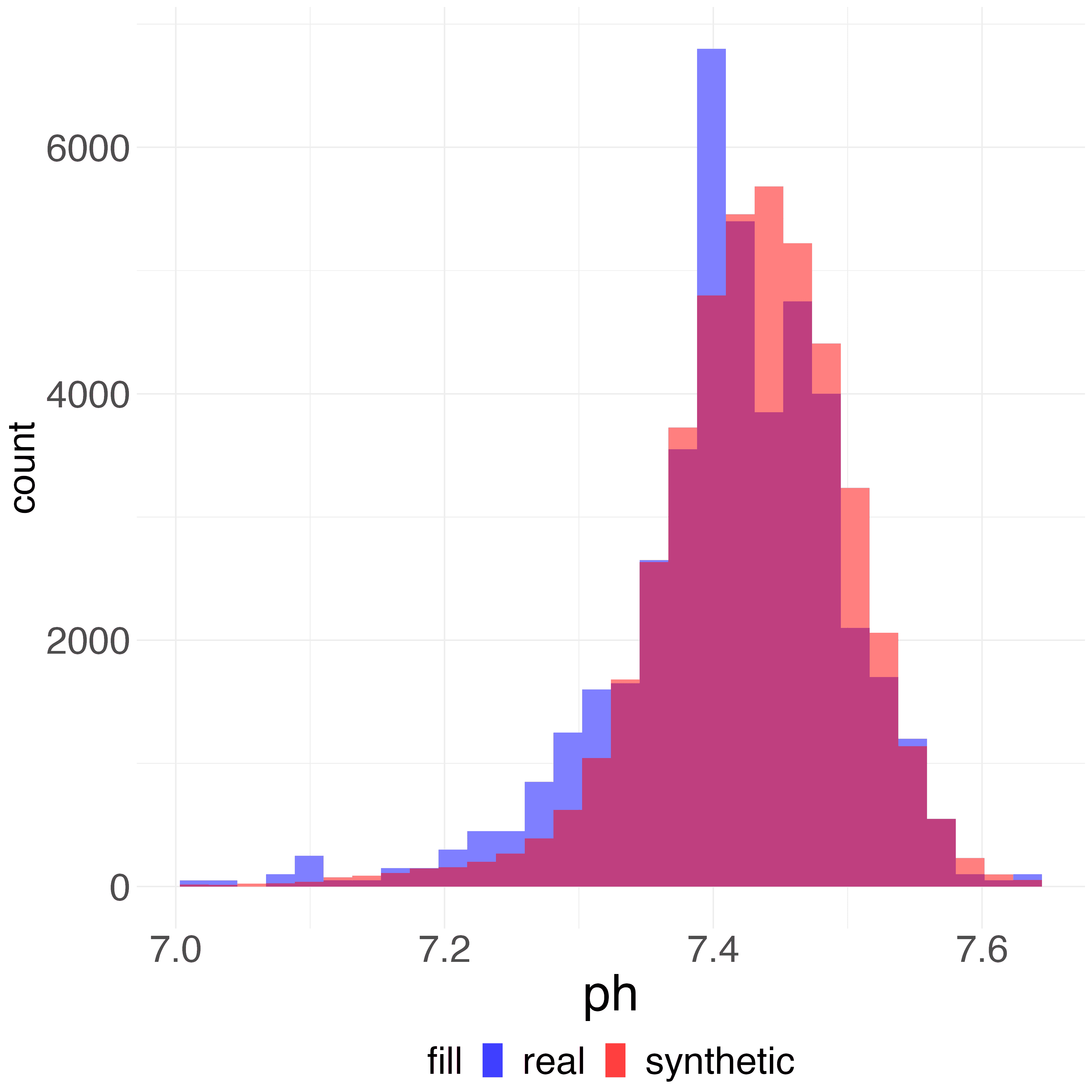}  \\
            (a) age & (b) aps & (c) surv2m & (d) resp & (e) alb & (f) ph \\
            \multicolumn{6}{c}{(5) TVAE.} \\
            &&&&& \\
    \end{tabular}
    \caption{SUPPORT2 data: Overlapping empirical marginal histograms of covariates \textit{age, aps, surv2m, resp, alb} and \textit{ph} estimate on the real data (blue) and synthetic data (red) generated by a CTGAN and TVAE.}
    \label{fig:margHist_support2small_CTGAN_TVAE}
\end{figure}

\begin{figure}[ht]
    \centering
    \begin{tabular}{cccccc}
            \includegraphics[width=0.14\textwidth]{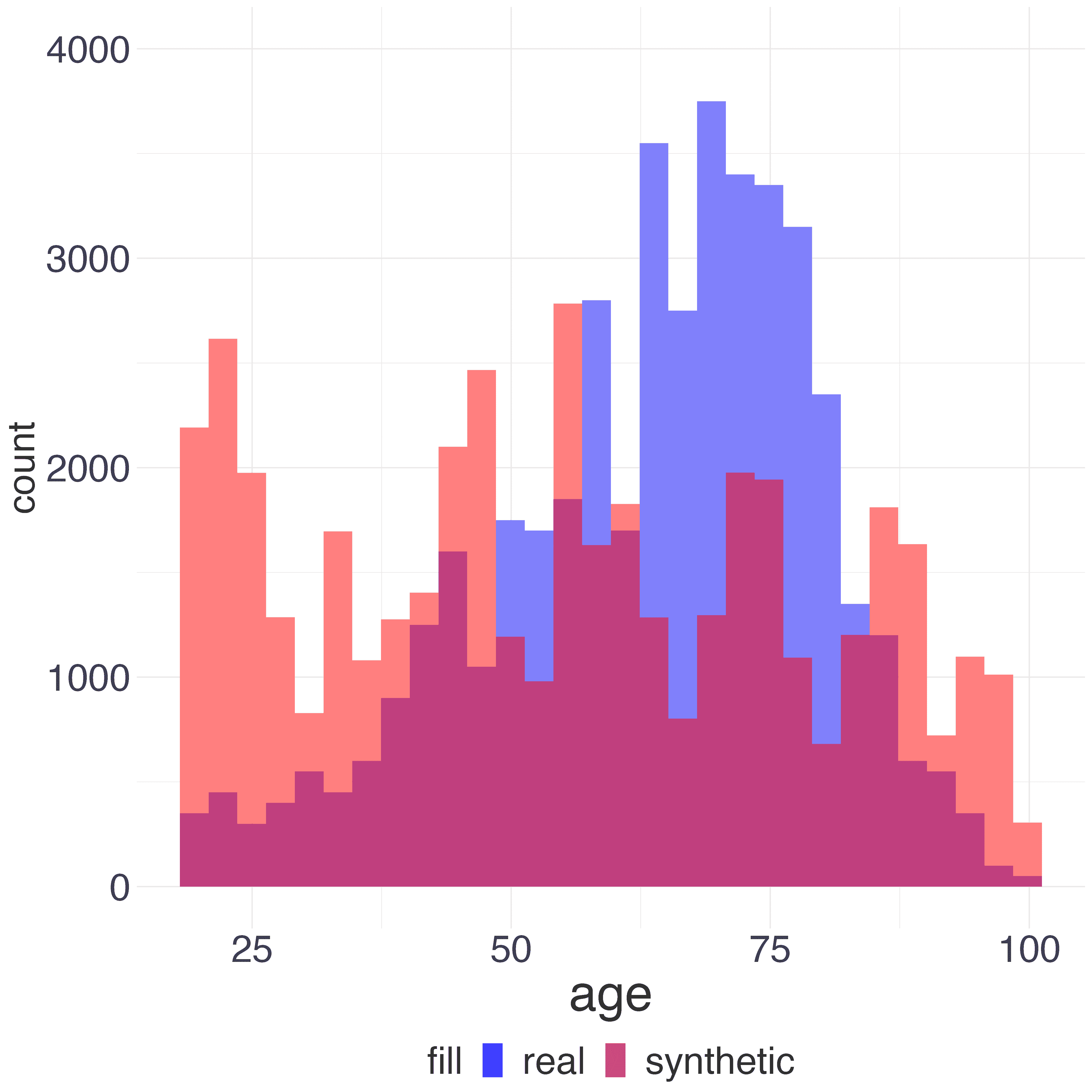} &  
            \includegraphics[width=0.14\textwidth]{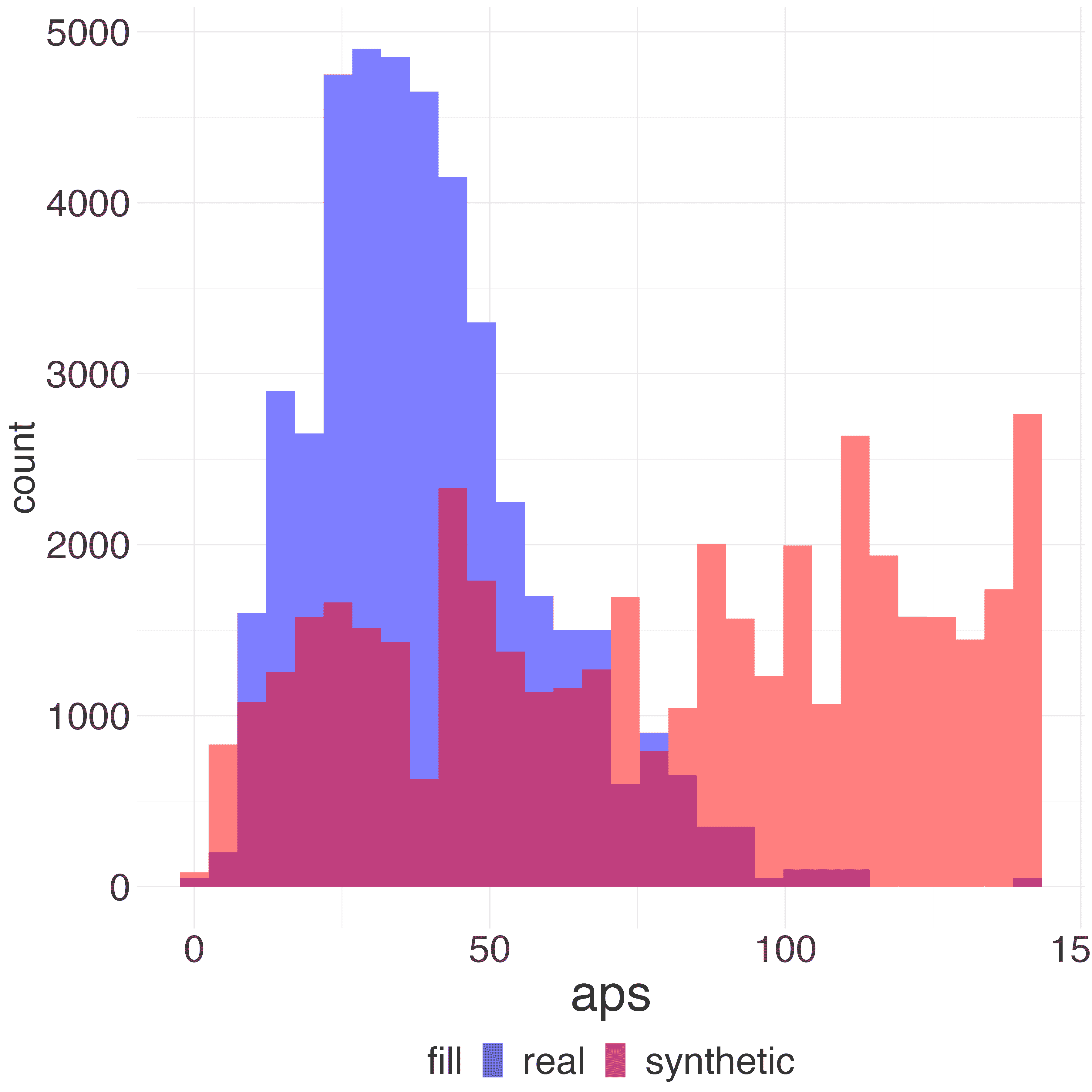} &
            \includegraphics[width=0.14\textwidth]{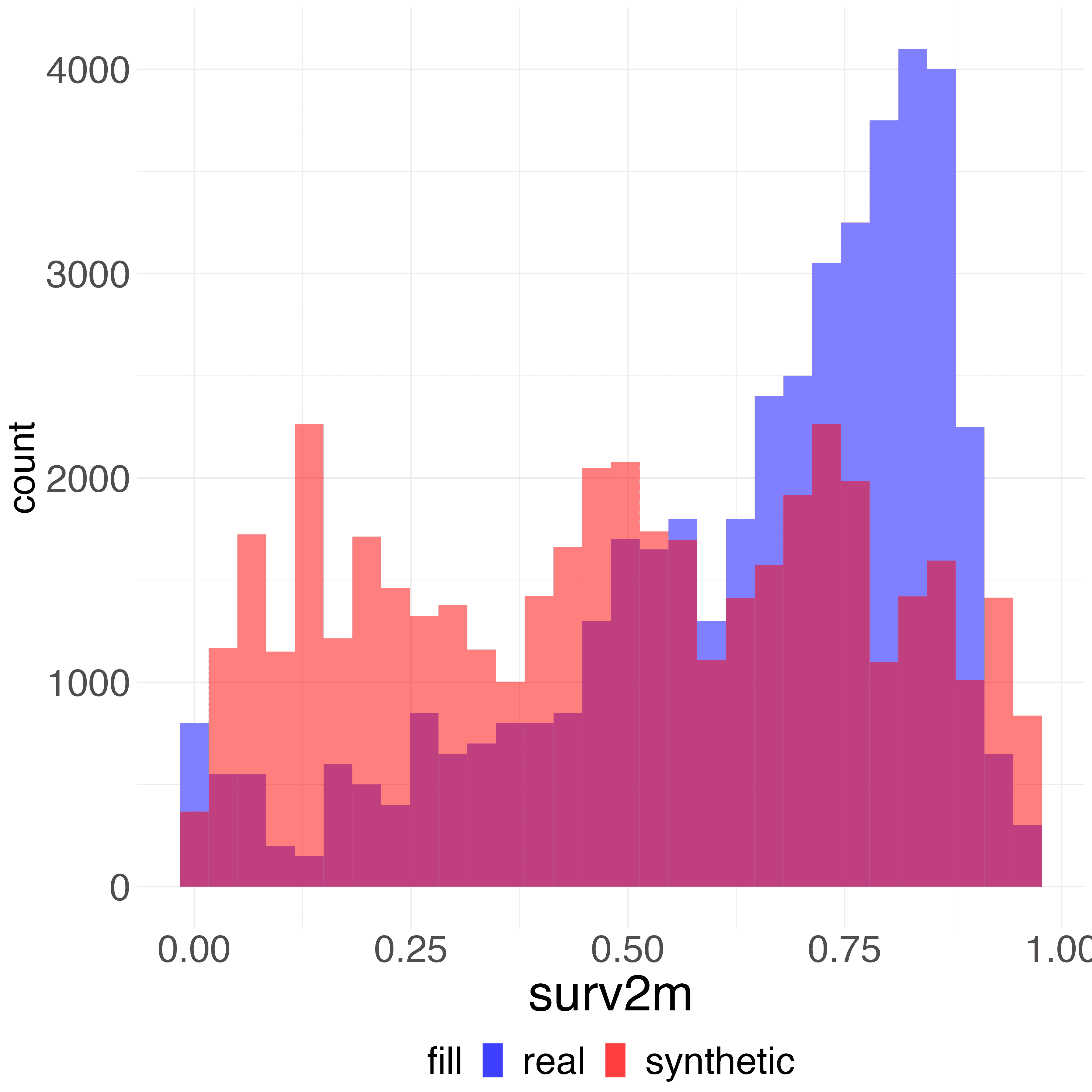} &
            \includegraphics[width=0.14\textwidth]{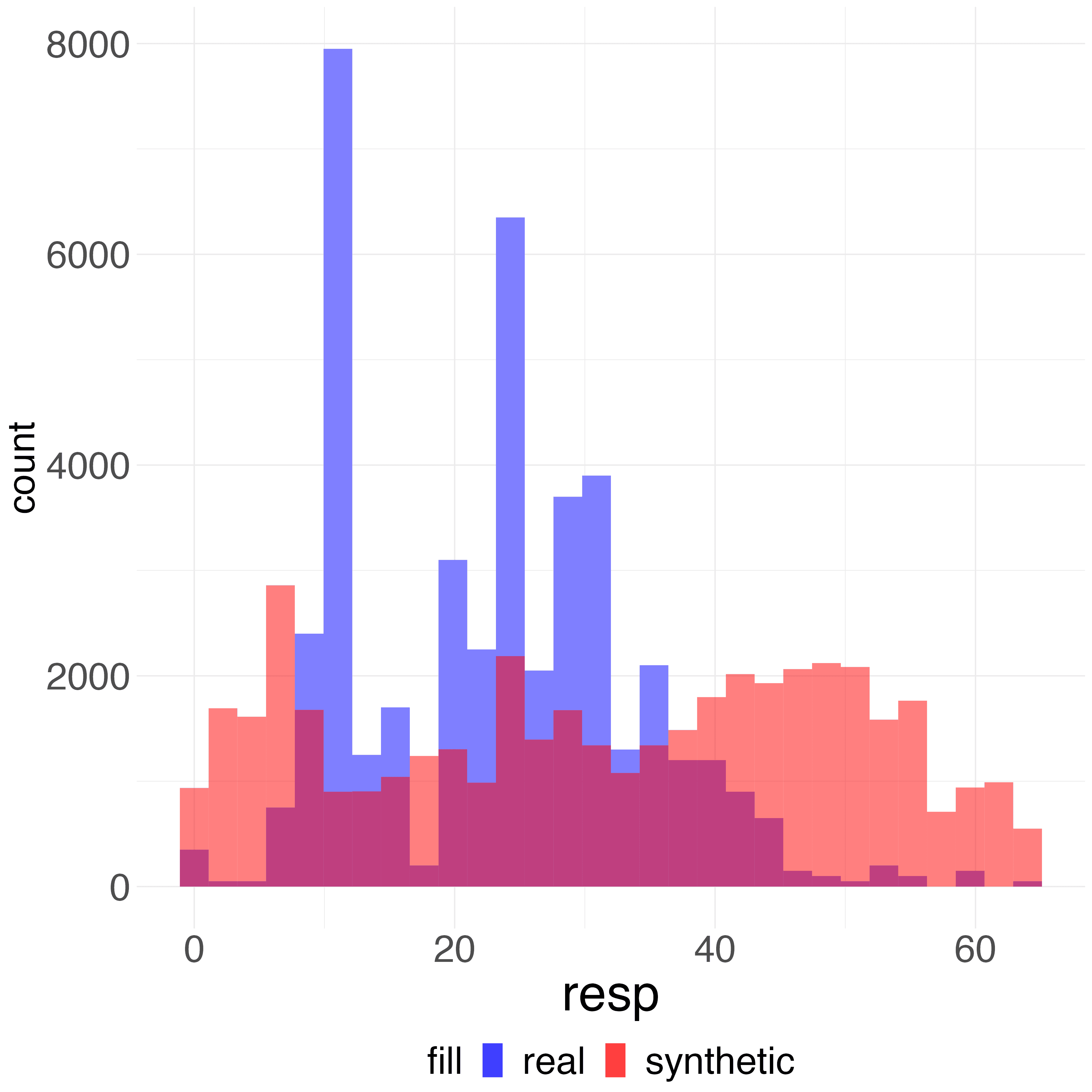} &
            \includegraphics[width=0.14\textwidth]{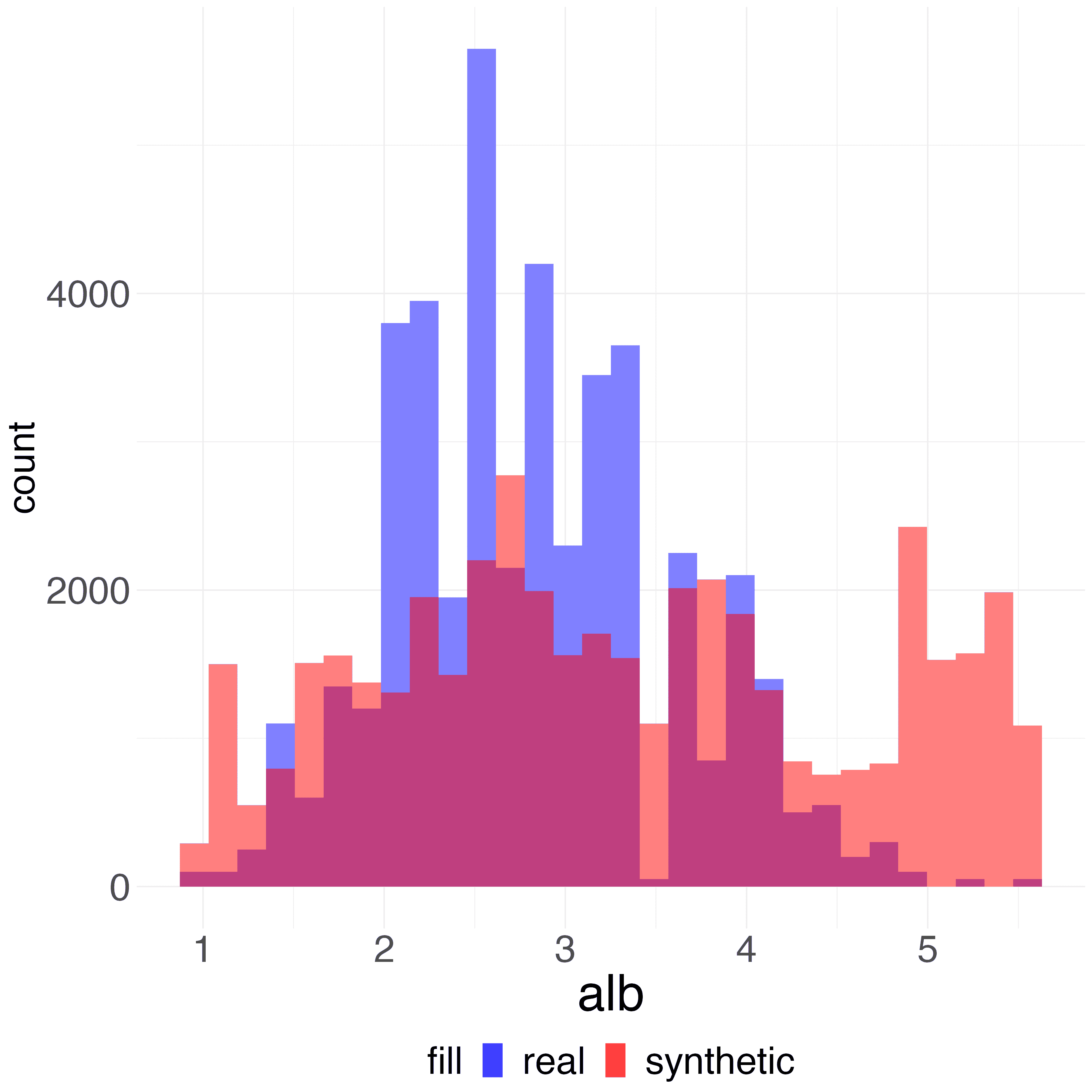} &
            \includegraphics[width=0.14\textwidth]{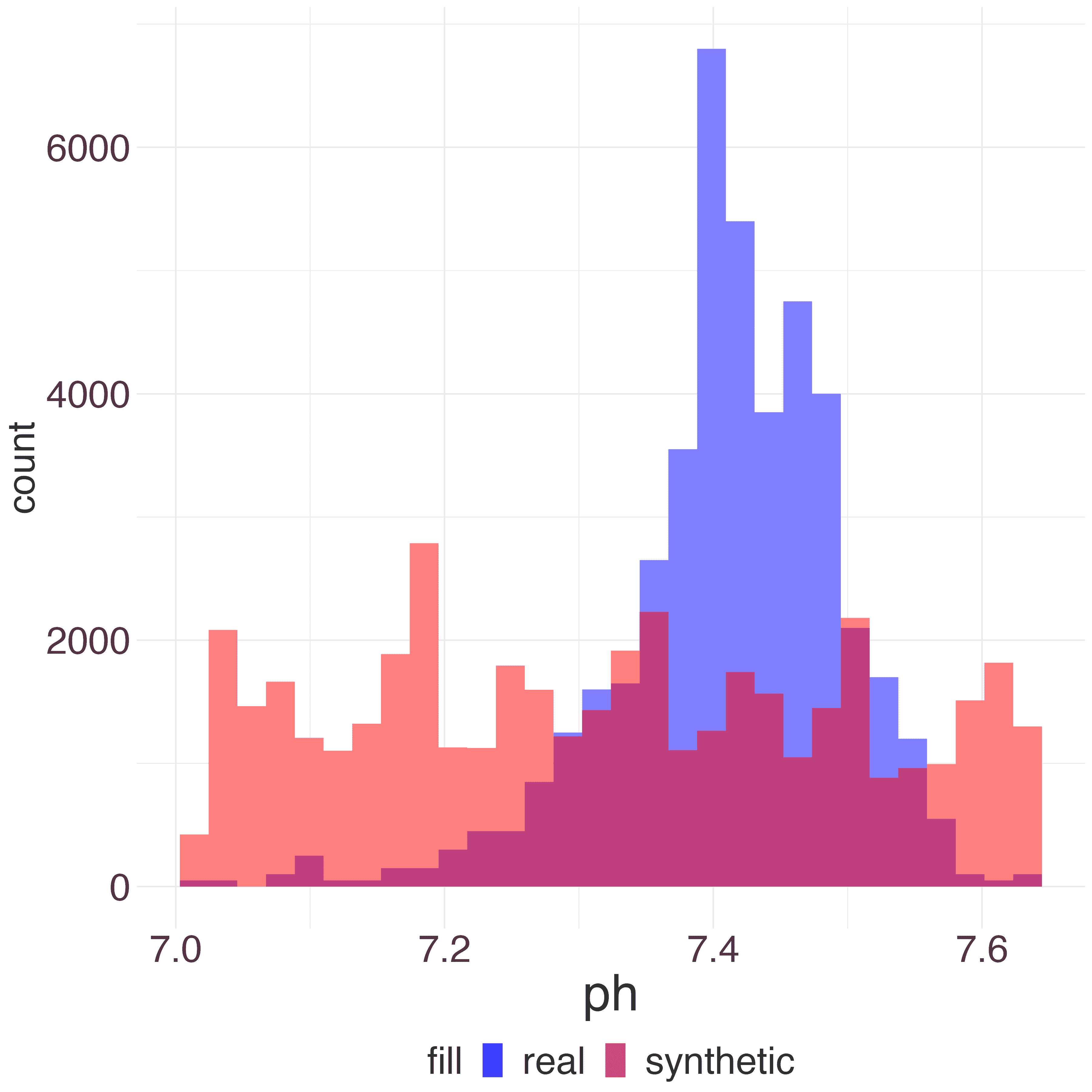}  \\
            (a) age & (b) aps & (c) surv2m & (d) resp & (e) alb & (f) ph \\
            \multicolumn{6}{c}{(6) PrivBayes, $\epsilon = 0.1$.} \\
            &&&&& \\
            \includegraphics[width=0.14\textwidth]{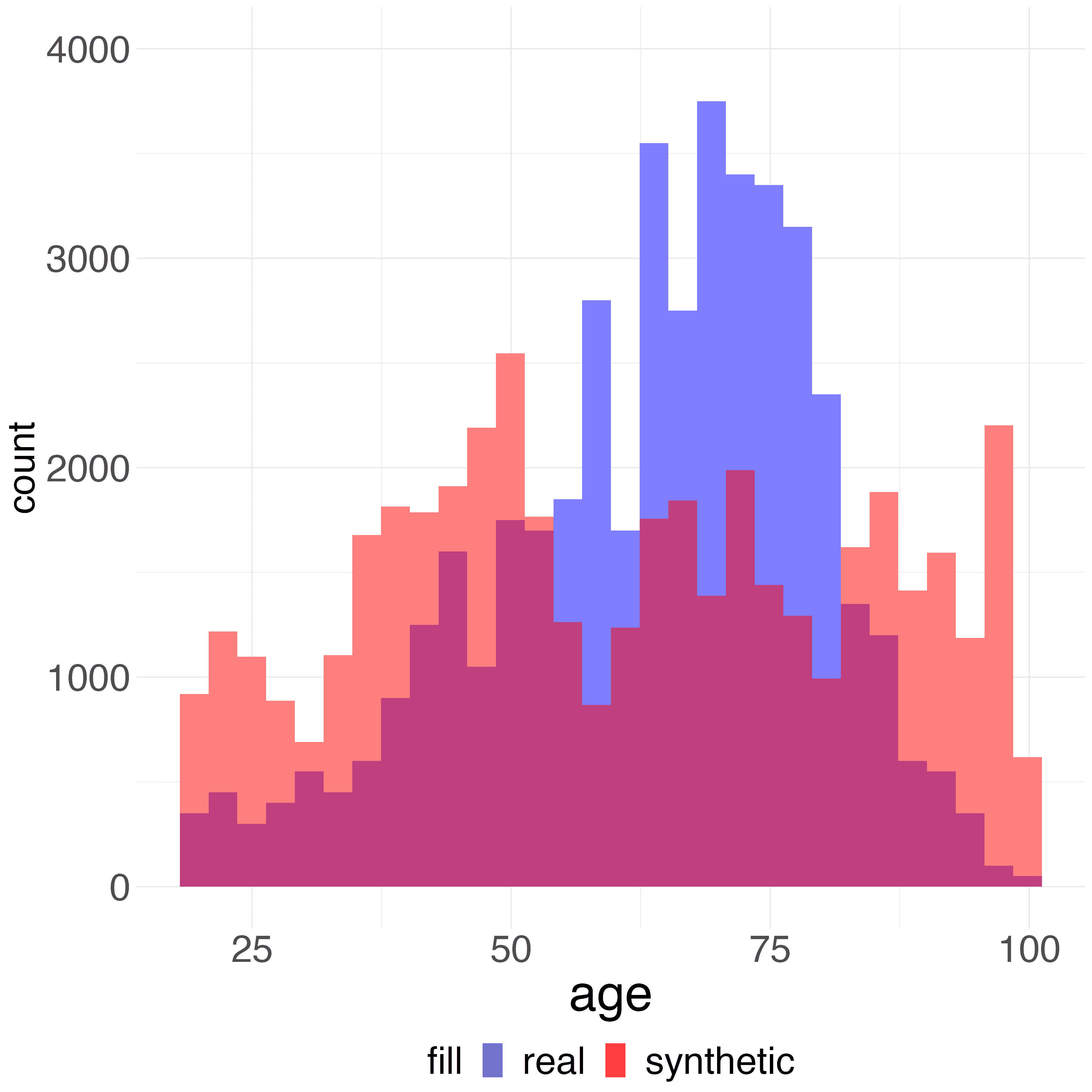} &  
            \includegraphics[width=0.14\textwidth]{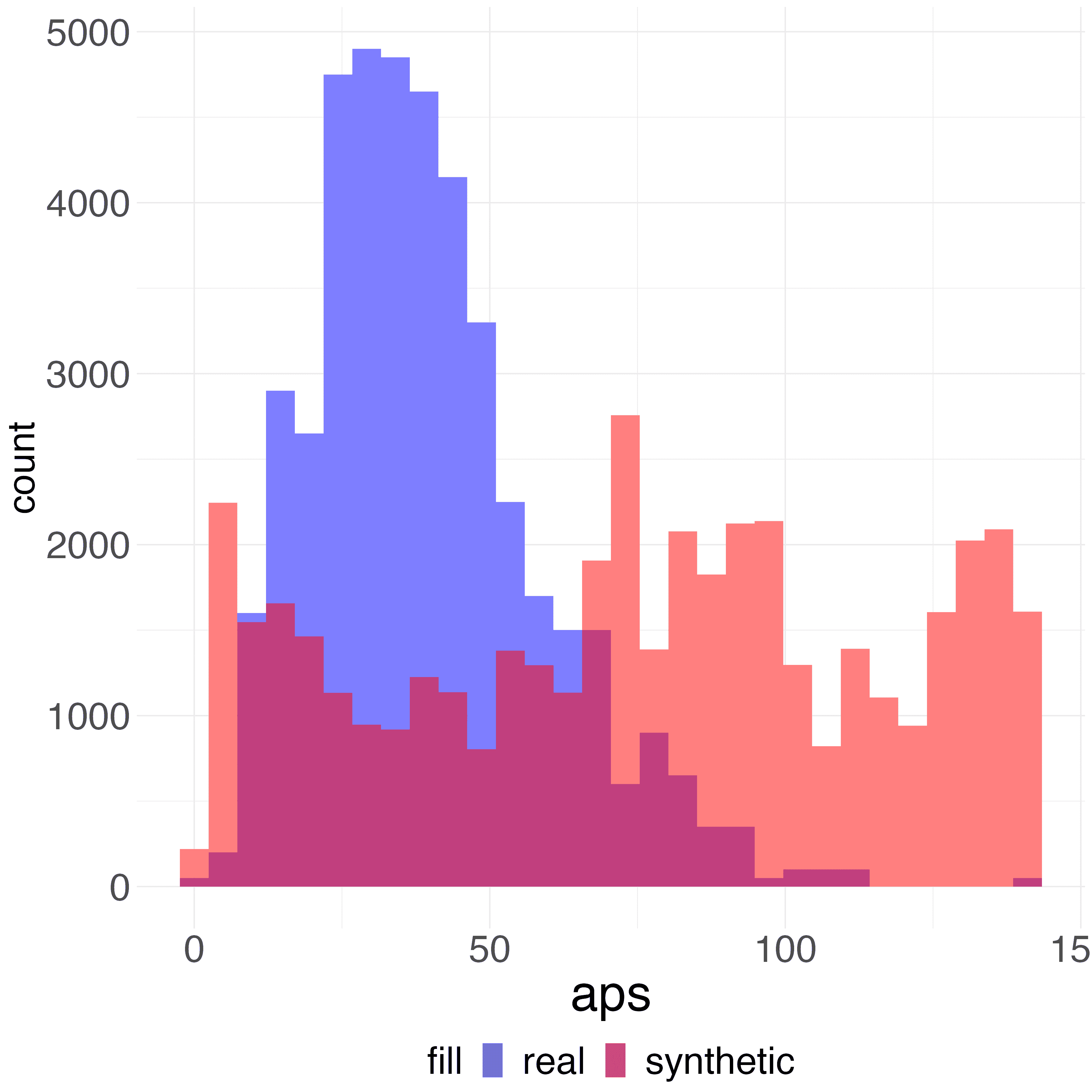} &
            \includegraphics[width=0.14\textwidth]{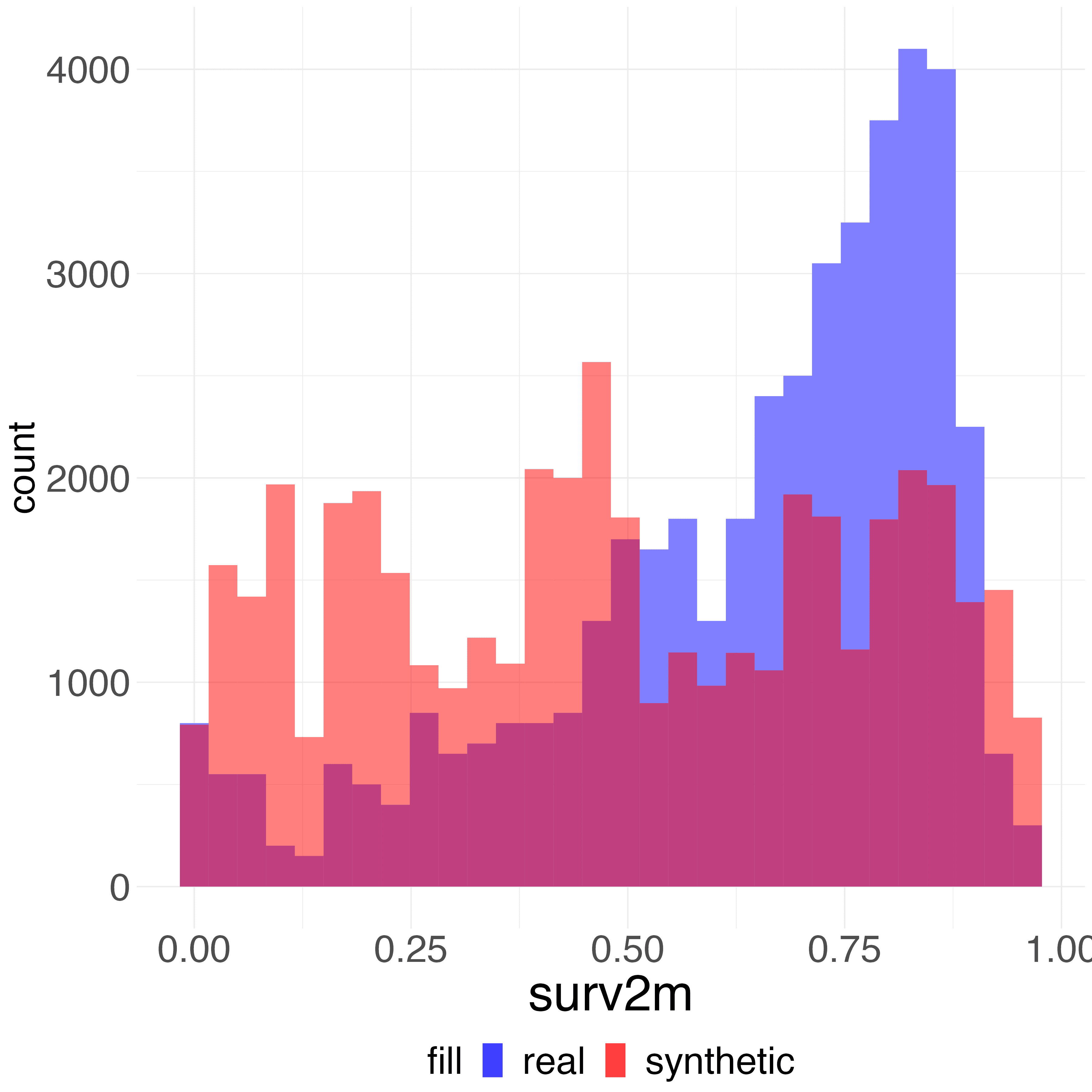} &
            \includegraphics[width=0.14\textwidth]{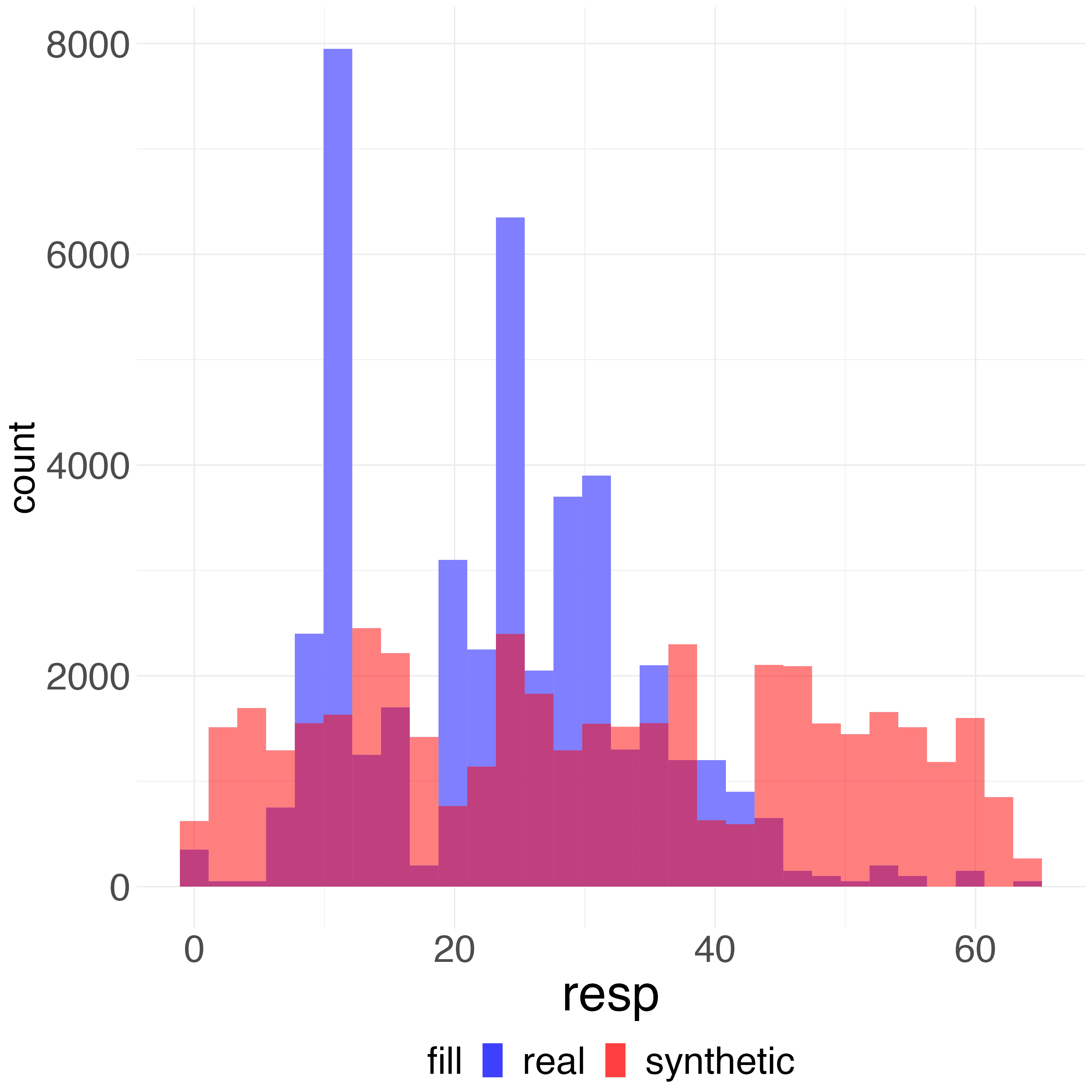} &
            \includegraphics[width=0.14\textwidth]{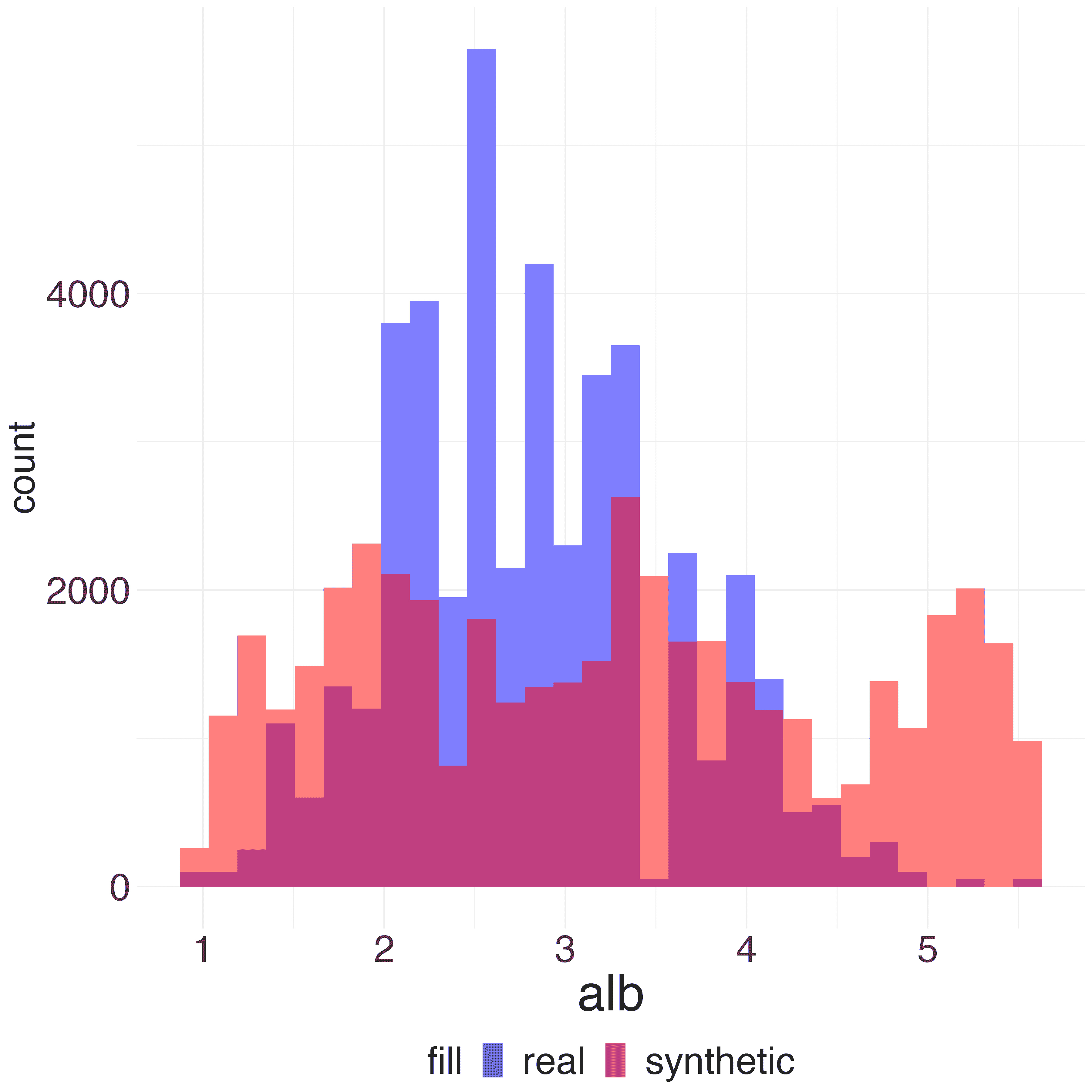} &
            \includegraphics[width=0.14\textwidth]{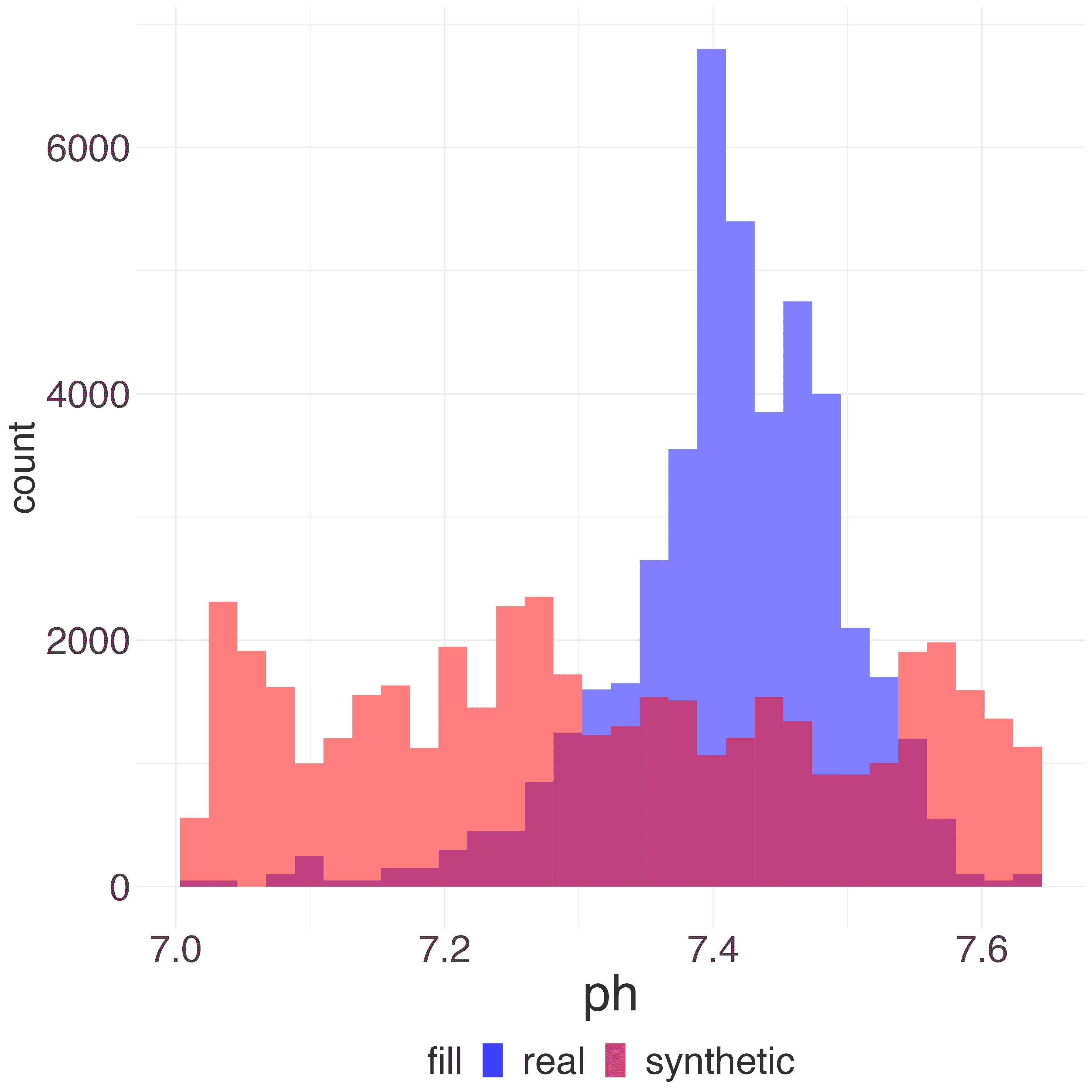}  \\
            (a) age & (b) aps & (c) surv2m & (d) resp & (e) alb & (f) ph \\
            \multicolumn{6}{c}{(7) PrivBayes, $\epsilon = 1$.} \\
            &&&&& \\
            \includegraphics[width=0.14\textwidth]{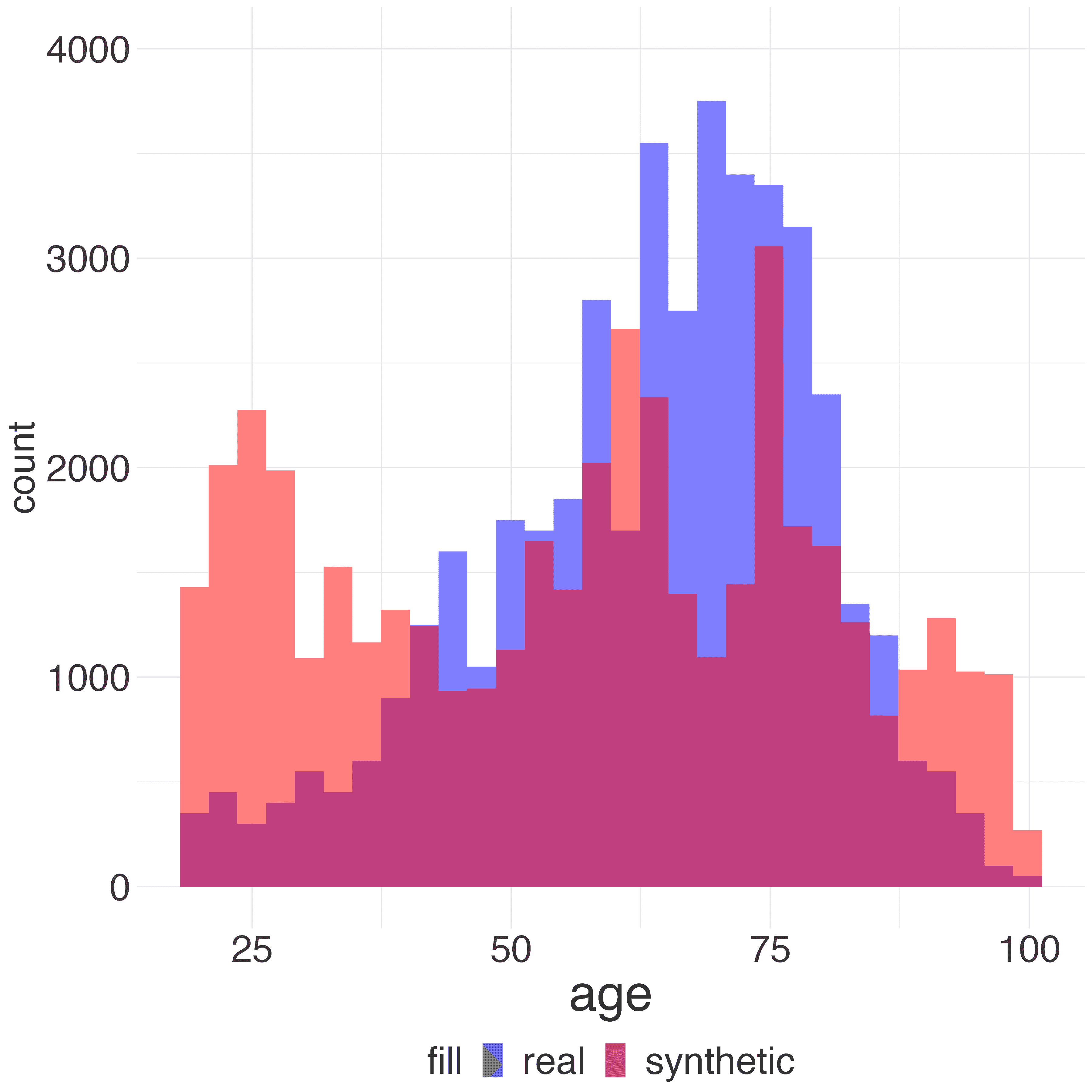} &  
            \includegraphics[width=0.14\textwidth]{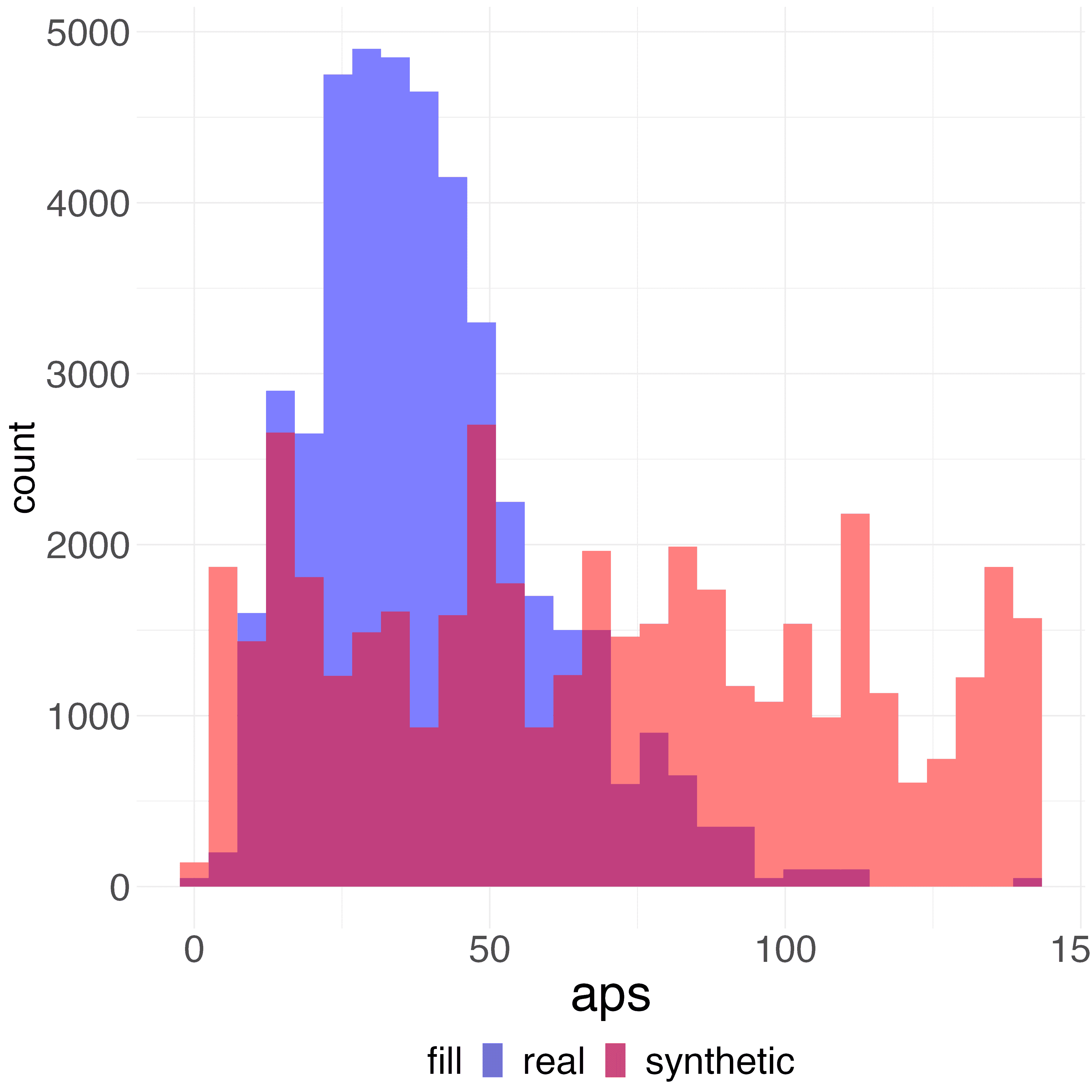} &
            \includegraphics[width=0.14\textwidth]{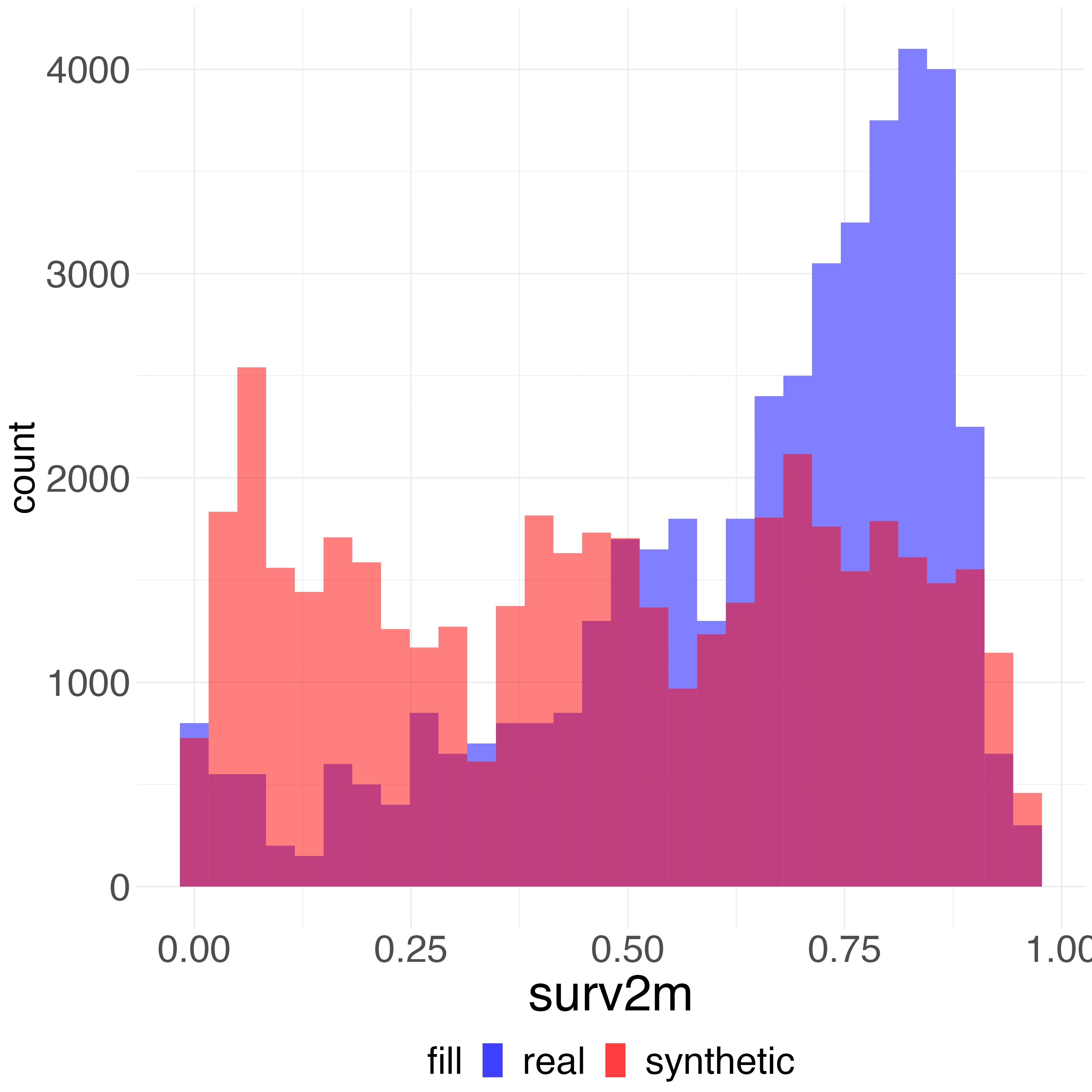} &
            \includegraphics[width=0.14\textwidth]{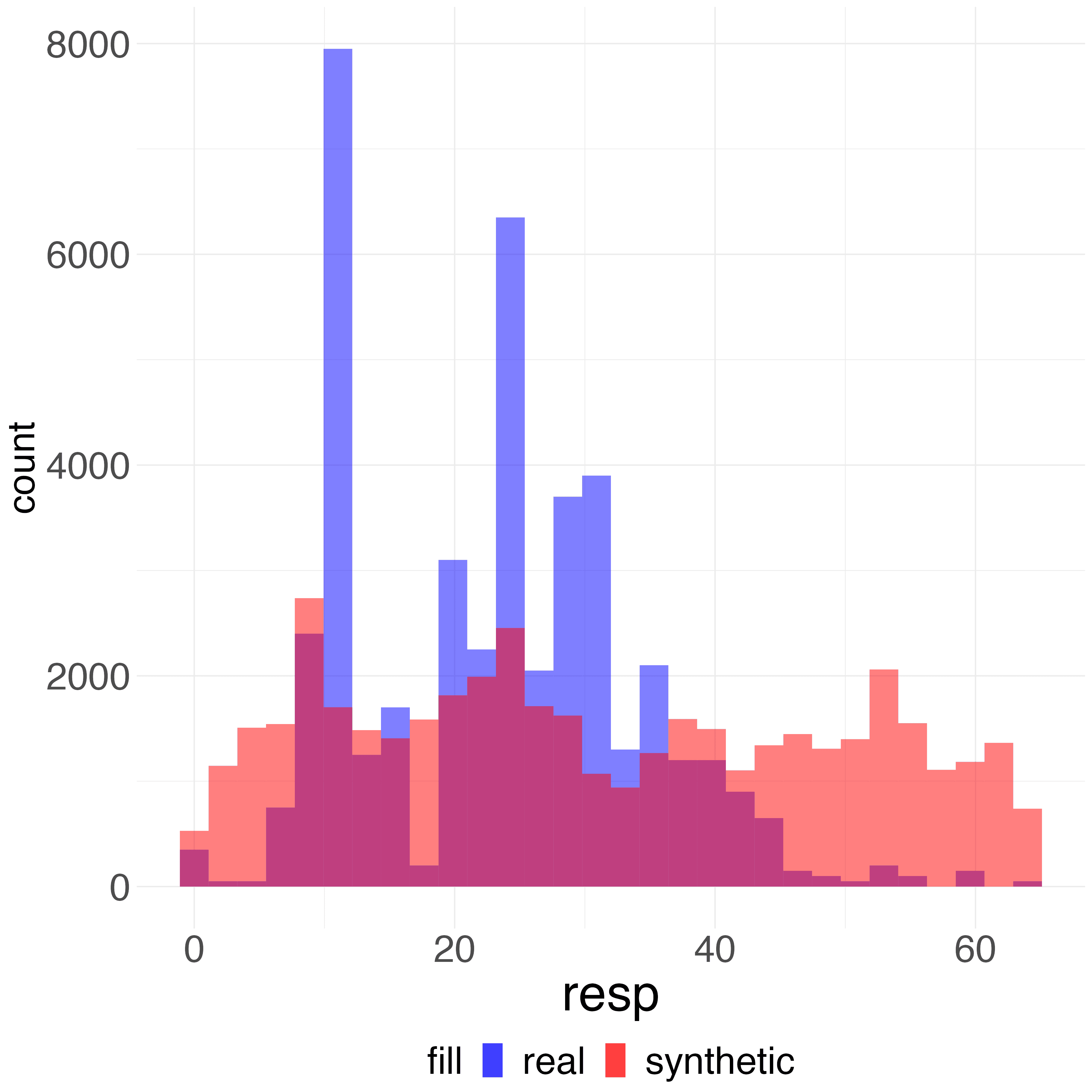} &
            \includegraphics[width=0.14\textwidth]{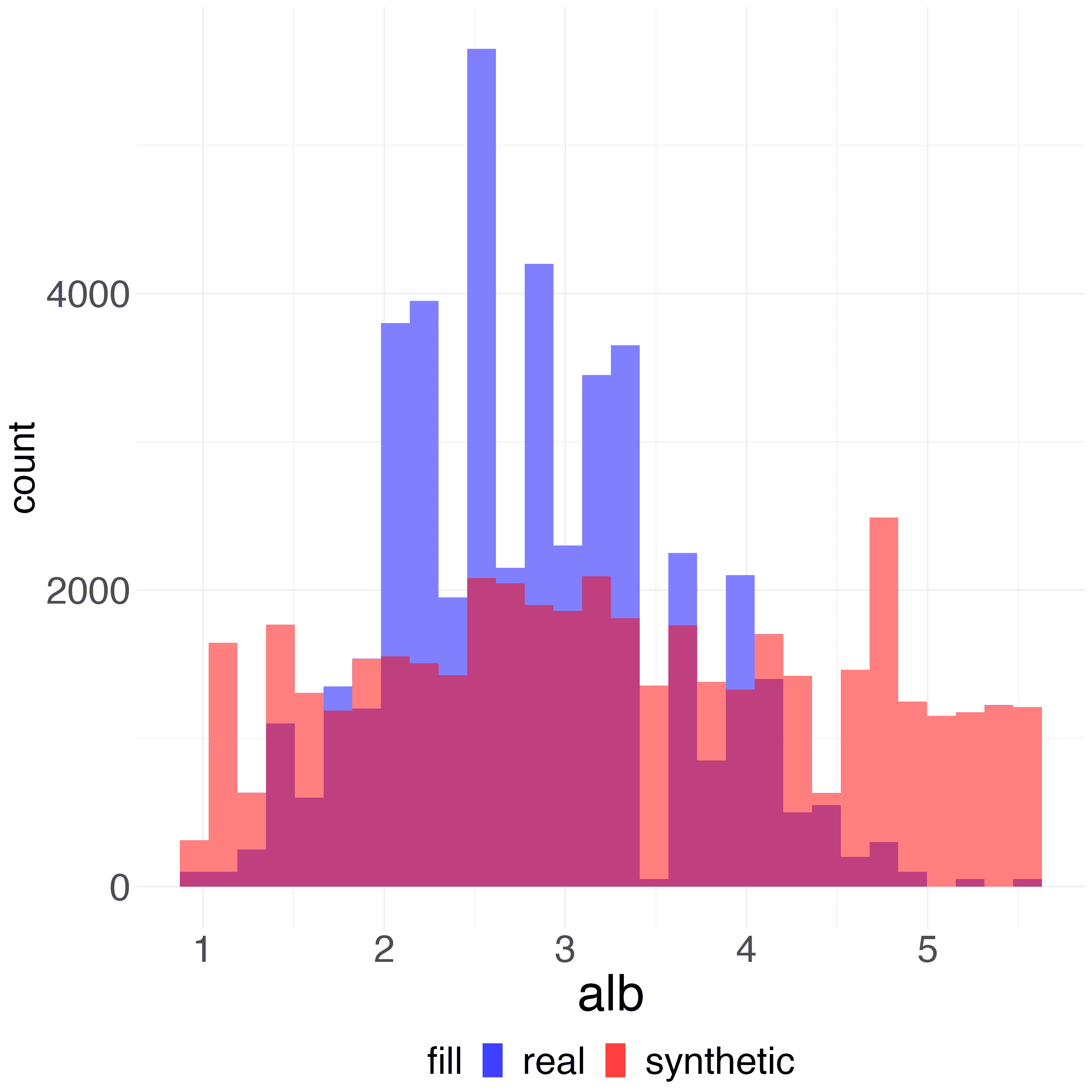} &
            \includegraphics[width=0.14\textwidth]{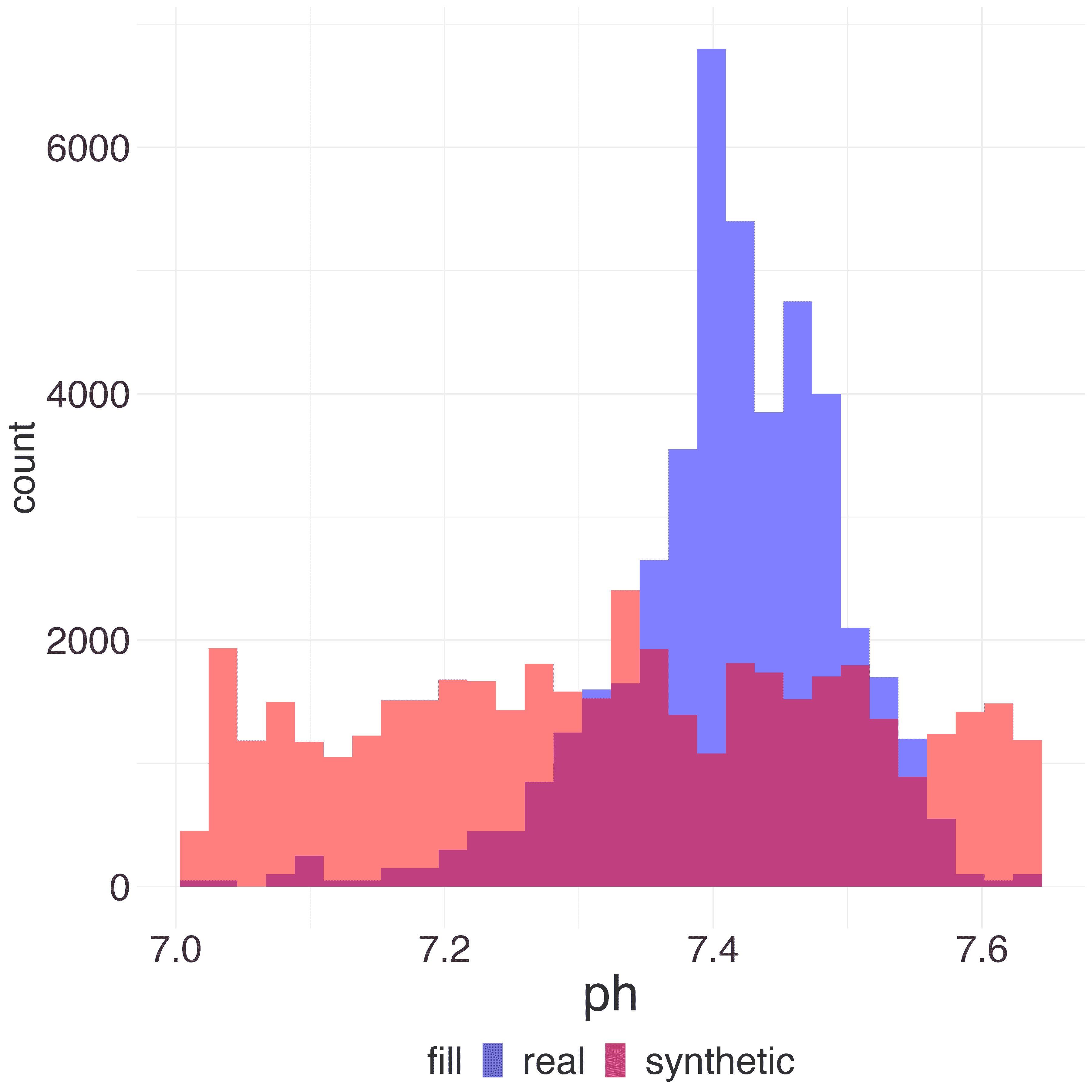}  \\
            (a) age & (b) aps & (c) surv2m & (d) resp & (e) alb & (f) ph \\
            \multicolumn{6}{c}{(8) PrivBayes, $\epsilon = 5$.} \\
            &&&&& \\
            \includegraphics[width=0.14\textwidth]{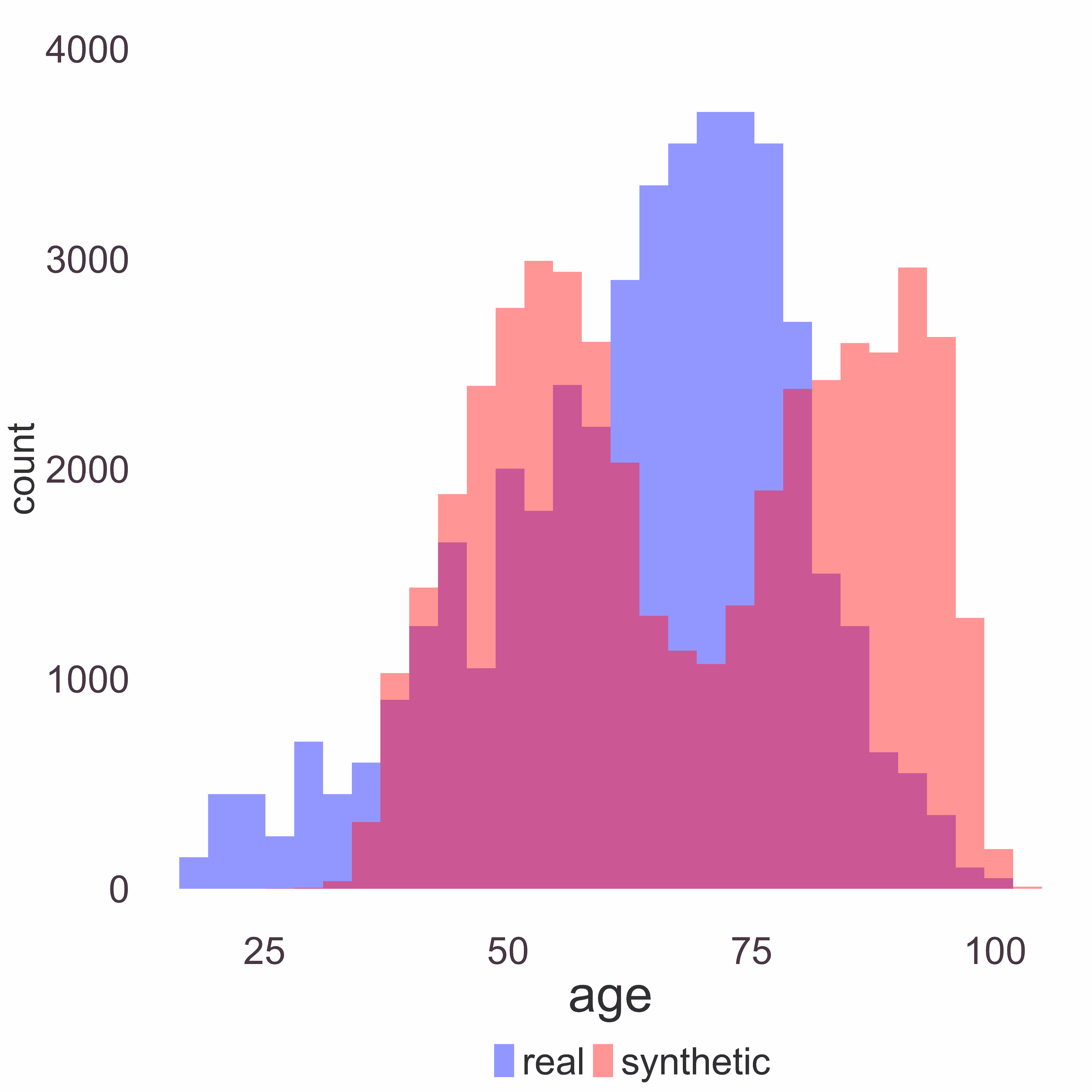} &  
            \includegraphics[width=0.14\textwidth]{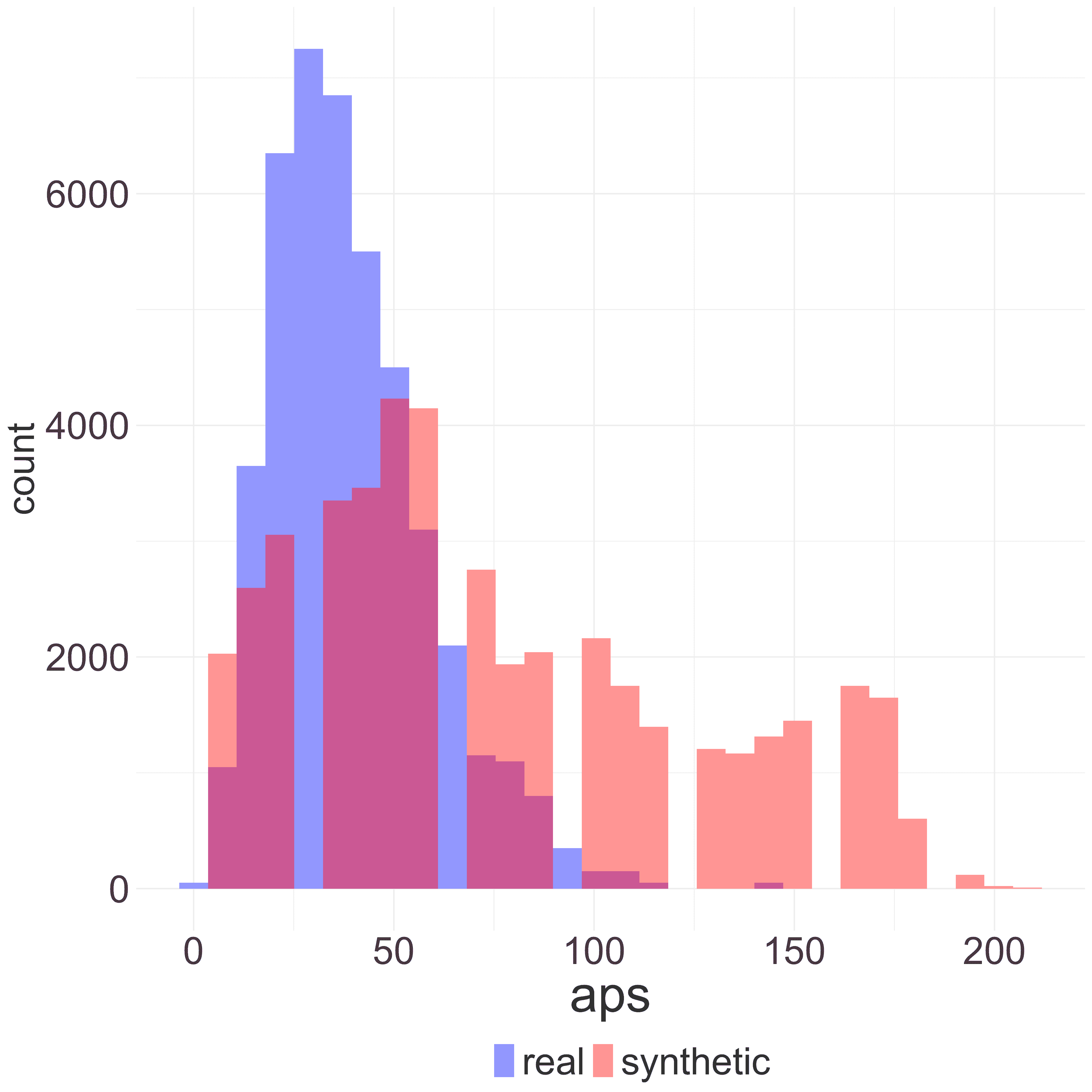} &
            \includegraphics[width=0.14\textwidth]{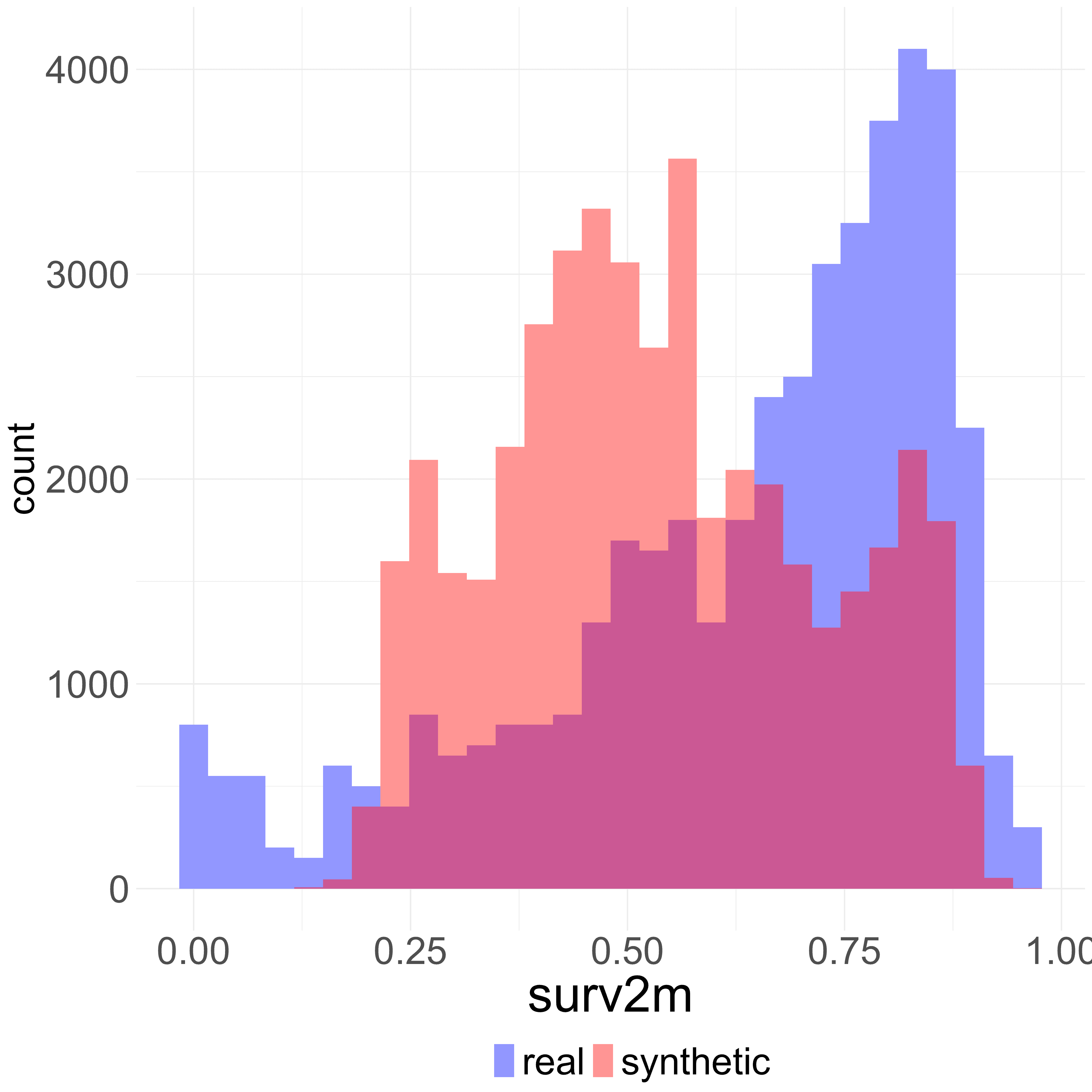} &
            \includegraphics[width=0.14\textwidth]{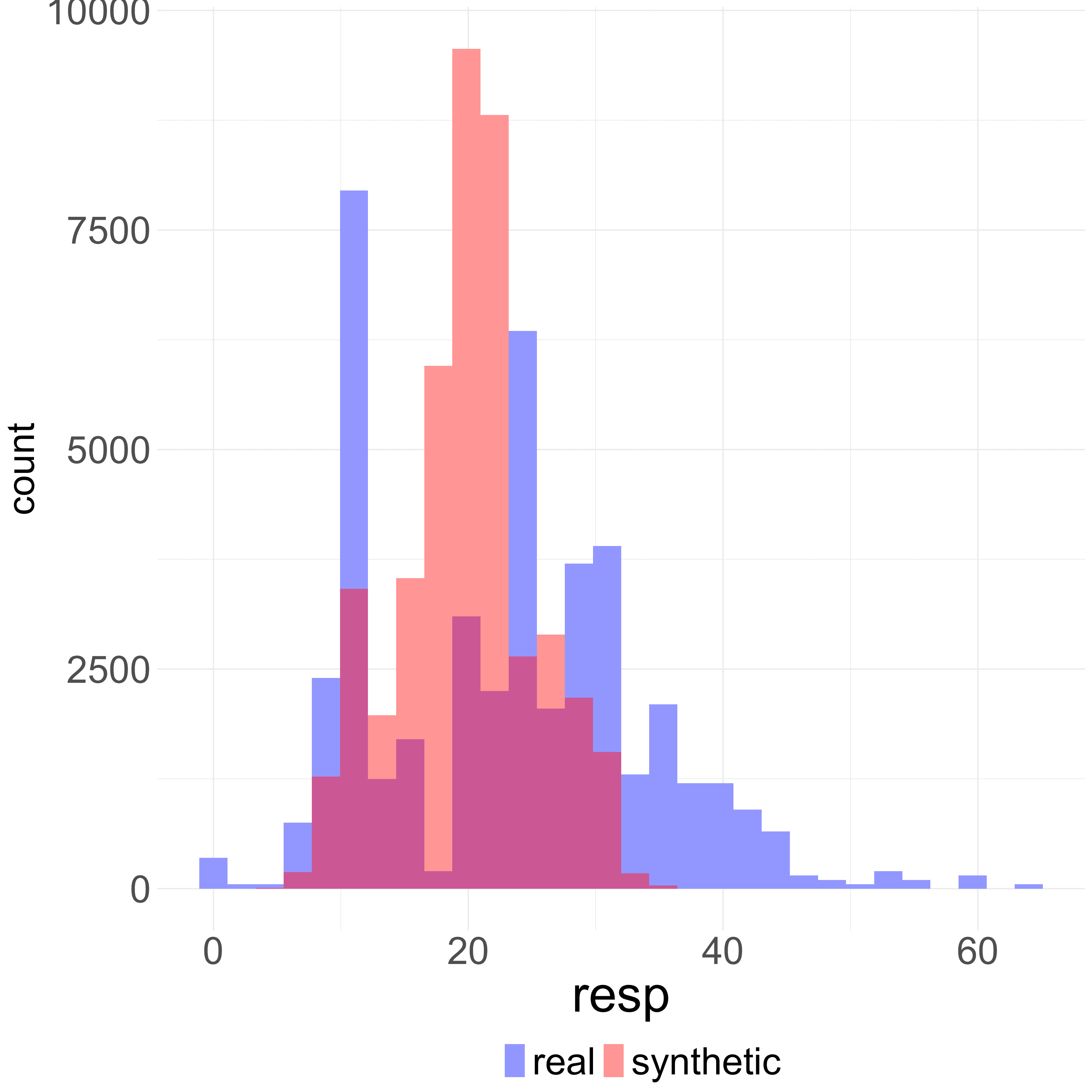} &
            \includegraphics[width=0.14\textwidth]{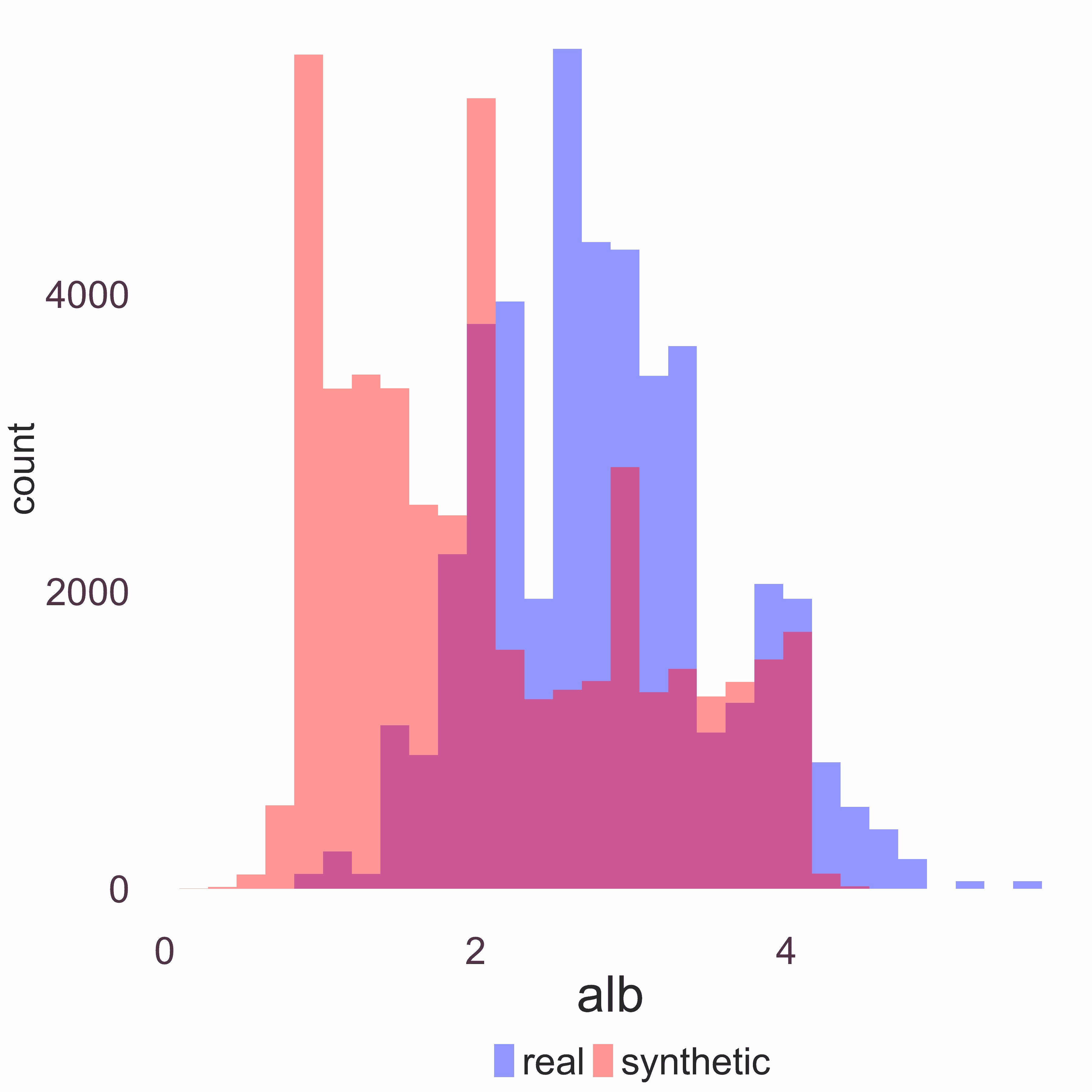} &
            \includegraphics[width=0.14\textwidth]{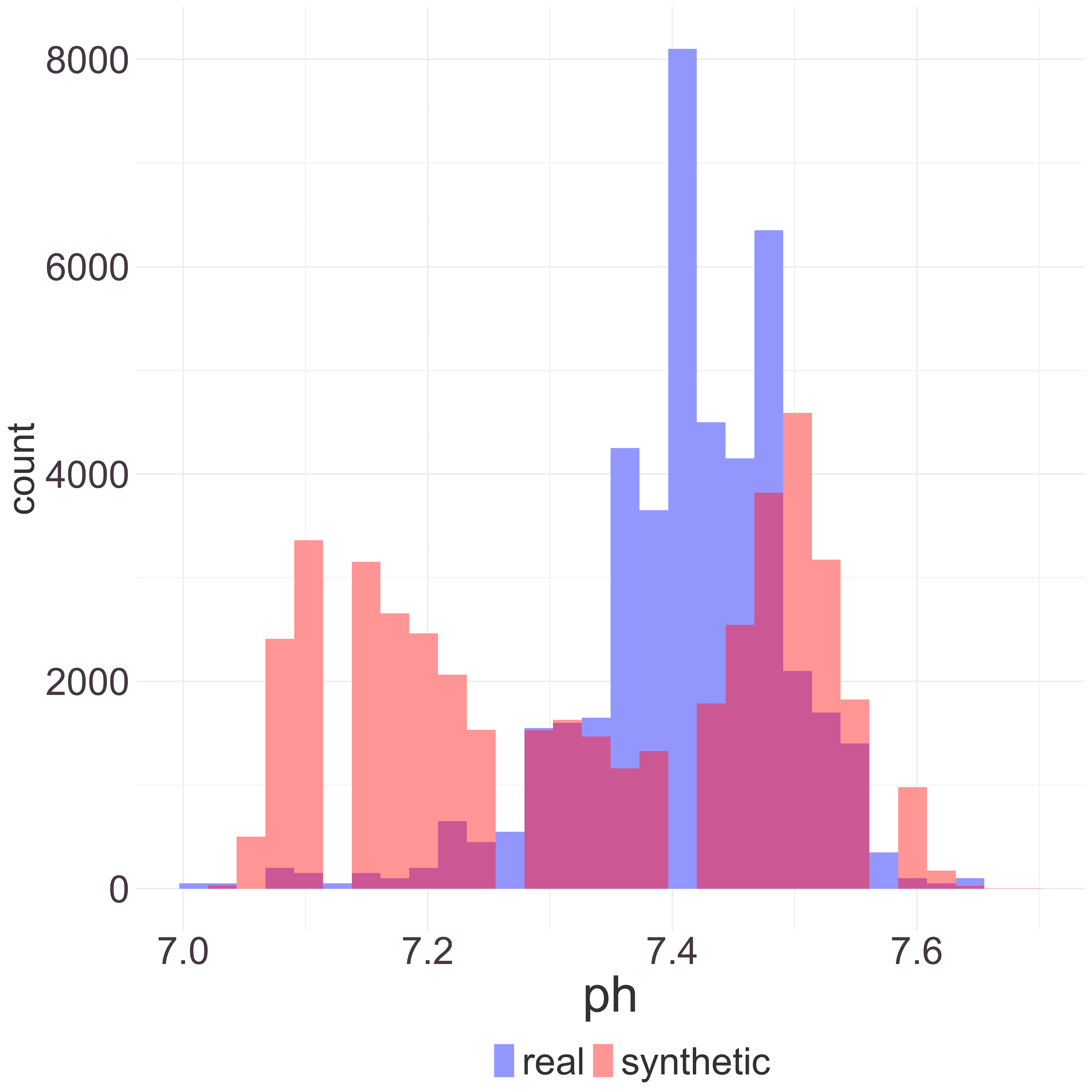}  \\
            (a) age & (b) aps & (c) surv2m & (d) resp & (e) alb & (f) ph \\
            \multicolumn{6}{c}{(9) PrivPG, $\epsilon=2.5$, $\delta=10^{-5}$.} \\
            &&&&& \\
    \end{tabular}
    \caption{SUPPORT2 data: Overlapping empirical marginal histograms of covariates \textit{age, aps, surv2m, resp, alb} and \textit{ph} estimate on the real data (blue) and synthetic data (red) generated by PrivBayes with $\epsilon \in \{0.1, 1, 5\}$ and PrivPGD with $\epsilon=2.5$ and $\delta=10^{-5}$.}
    \label{fig:margHist_support2small_PrivBayes_PrivPGD}
\end{figure}

\end{document}